%% file: ve.tex
\documentclass{article}

\PassOptionsToPackage{numbers, compress}{natbib}
\usepackage[final]{neurips_2020}

\usepackage[utf8]{inputenc} 
\usepackage[T1]{fontenc}    
\usepackage{hyperref}       
\usepackage{url}            
\usepackage{booktabs}       
\usepackage{amsfonts}       
\usepackage{nicefrac}       
\usepackage{microtype}      
\usepackage{amsmath}
\usepackage{amsthm} 
\usepackage{amssymb}
\usepackage{bbm}
\usepackage{thmtools, thm-restate}
\usepackage{graphicx}
\usepackage{subfigure}
\usepackage{wrapfig}
\usepackage{caption}
\usepackage{mathbbol} 
\usepackage{color}
\usepackage{float}
\usepackage[shortlabels]{enumitem}
\usepackage{tabularx}
\usepackage{etoolbox}
\usepackage[normalem]{ulem}
\usepackage{comment}
\usepackage{multirow}
\usepackage[dvipsnames]{xcolor}

\usepackage[normalem]{ulem}

\newtheorem{lemma}{Lemma}

\DeclareSymbolFontAlphabet{\mathbb}{AMSb}
\DeclareSymbolFontAlphabet{\mathbbl}{bbold}

\title{The Value Equivalence Principle \\ for Model-Based Reinforcement Learning}

\author{%
Christopher Grimm \\
Computer Science \& Engineering \\
University of Michigan \\ 
\texttt{crgrimm@umich.edu}
\And 
Andr\'{e} Barreto, Satinder Singh, David Silver \\
DeepMind \\ 
\texttt{\{andrebarreto,baveja,davidsilver\}@google.com}
}

\begin{document}

\input{defs}

\newtheorem{corollary}{Corollary}
\newtheorem{definition}{Definition}
\newtheorem{proposition}{Proposition}
\newtheorem{remark}{Remark}
\newtheorem{property}{Property}

\maketitle

\begin{abstract}
Learning models of the environment from data is often viewed as an essential component to building intelligent reinforcement learning (RL) agents. The common practice is to separate the learning of the model from its use, by constructing a model of the environment’s dynamics that correctly predicts the observed state transitions. In this paper we argue that the limited representational resources of model-based RL agents are better used to build models that are directly useful for value-based planning. As our main contribution, we introduce the principle of value equivalence: two models are value equivalent with respect to a set of functions and policies if they yield the same Bellman updates. We propose a formulation of the model learning problem based on the value equivalence principle and analyze how the set of feasible solutions is impacted by the choice of policies and functions. Specifically, we show that, as we augment the set of policies and functions considered, the class of value equivalent models shrinks, until eventually collapsing to a single point corresponding to a model that perfectly describes the environment.  
In many problems, directly modelling state-to-state transitions may be both difficult and unnecessary. By leveraging the value-equivalence principle one may find simpler models without compromising performance, saving computation and memory.
We illustrate the benefits of value-equivalent model learning with experiments comparing it against more traditional counterparts like maximum likelihood estimation. More generally, we argue that the principle of value equivalence underlies a number of recent empirical successes in RL, such as Value Iteration Networks, the Predictron, Value Prediction Networks, TreeQN, and MuZero, and provides a first theoretical underpinning of those results.
\end{abstract}

\section{Introduction}
\label{sec:introduction}

\input{introduction.tex}

\section{Background} 
\label{sec:background} 

\input{background}



\section{Value equivalence}   

\input{value_equivalence}

\section{Model learning based on the value-equivalence principle}
\label{sec:ve_model_learning}

\input{ve_model_learning}

\section{Experiments}
\label{sec:experiments}
\vspace{-3mm}
\input{experiments}

\vspace{-3mm}
\section{The value-equivalence principle in practice}
\label{sec:ve_practice}
\vspace{-3mm}

\input{ve_in_practice}

\vspace{-3mm} 
\section{Related work}
\label{sec:relation_prior}
\vspace{-3mm} 

\input{related_work}

\vspace{-2mm} 
\section{Conclusion}
\vspace{-2mm}
\input{conclusion}

\newpage 

\section*{Broader impact}

\input{broader_impact}

\section*{Acknowledgements}

We would like to thank Gregory Farquhar and Eszter Vertes for the great discussions regarding the value equivalence principle. We also thank the anonymous reviewers for their comments and suggestions to improve the paper.

\bibliographystyle{plainnat}
\bibliography{rl}

\newpage 
\appendix 

\input{appendix}



\end{document}

%% file: defs.tex
\newcommand{\cp}{\citep}
\newcommand{\ca}{\citeauthor}
\newcommand{\cy}{\citeyear}
\newcommand{\ct}{\citet}
\newcommand{\ctp}[1]{\ca{#1}'s~\cp{#1}}
\newcommand{\cwp}[1]{\ca{#1}, \cp{#1}}
\newcommand{\ctt}[1]{\ca{#1}~\cp{#1}} 

\newcommand{\R}{\ensuremath{\mathbb{R}}}
\newcommand{\N}{\ensuremath{\mathbb{N}}}
\newcommand{\E}{\ensuremath{\mathbb{E}}}
\newcommand{\bs}[1]{\boldsymbol{#1}}

\newcommand{\set}[1]{\ensuremath{\mathcal{#1}}}
\renewcommand{\S}{\set{S}}
\newcommand{\A}{\set{A}}
\newcommand{\V}{\set{V}}
\newcommand{\T}{\set{T}}
\newcommand{\Tpi}{\ensuremath{\T_{\pi}}}
\newcommand{\D}{\set{D}}
\newcommand{\M}{\set{M}}
\newcommand{\RR}{\set{R}}
\newcommand{\PP}{\set{P}}
\renewcommand{\H}{\set{H}}
\newcommand{\X}{\set{X}}
\newcommand{\G}{\set{G}}
\newcommand{\C}{\set{C}}
\newcommand{\AV}{\ensuremath{\tilde{\V}}}
\newcommand{\APi}{\ensuremath{\tilde{\Pi}}}
\newcommand{\PS}{\ensuremath{\mathbbl{\Pi}}}
\newcommand{\FS}{\ensuremath{\mathbbl{{V}}}}
\newcommand{\MS}{\ensuremath{\mathbbl{{M}}}}

\newcommand{\ope}[1]{\ensuremath{\mathrm{#1}}}
\newcommand{\argmin}{\ope{argmin}}
\newcommand{\argmax}{\ope{argmax}}
\renewcommand{\min}{\ope{min}}
\renewcommand{\span}{\ope{span}}
\newcommand{\rank}{\ope{rank}}
\newcommand{\hdim}{\mathcal{H}\text{-}\ope{dim}} 
\newcommand{\cspan}{\ensuremath{cp}\text{-}\ope{span}}
\newcommand{\pspan}{\ensuremath{p}\text{-}\ope{span}}
\newcommand{\cmax}{\ope{cmax}}
\newcommand{\vpath}{\ope{vpath}}
\newcommand{\ind}{\ensuremath{\mathbbm{1}}}

\newcommand{\mat}[1]{\ensuremath{\boldsymbol{{#1}}}}
\newcommand{\vphi}{\mat{\phi}}
\newcommand{\w}{\mat{w}}
\renewcommand{\r}{\mat{r}}
\renewcommand{\P}{\mat{P}}
\newcommand{\ppi}{\mat{\pi}}
\renewcommand{\v}{\mat{v}}
\newcommand{\vpi}{\ensuremath{\v_{\pi}}}
\newcommand{\VV}{\mat{V}}
\renewcommand{\b}{\mat{b}}
\newcommand{\DD}{\mat{D}}
\newcommand{\KK}{\mat{K}}
\renewcommand{\AA}{\mat{A}}
\newcommand{\vtheta}{\mat{\theta}}
\newcommand{\I}{\mat{I}}
\newcommand{\Q}{\mat{Q}}
\newcommand{\Lm}{\mat{\Lambda}}
\newcommand{\F}{\mat{F}}

\newcommand{\defi}{\ensuremath{\equiv}}

\newcommand{\apx}[1]{\ensuremath{\tilde{#1}}}
\newcommand{\pt}{\apx{p}}
\newcommand{\rt}{\apx{r}}
\newcommand{\ft}{\apx{f}}
\newcommand{\mt}{\apx{m}}
\newcommand{\vt}{\apx{v}}
\newcommand{\Tt}{\apx{\T}}
\newcommand{\Tpit}{\ensuremath{\apx{\T}_{\pi}}}
\newcommand{\Pt}{\apx{\P}}
\newcommand{\rrt}{\apx{\r}}
\newcommand{\vvt}{\apx{\mat{v}}}
\newcommand{\Ph}{\ensuremath{\hat{\P}}}
\newcommand{\vh}{\ensuremath{\hat{\v}}}
\newcommand{\AAt}{\apx{\AA}}
\newcommand{\bt}{\apx{\b}}
\newcommand{\Qt}{\apx{\Q}}
\newcommand{\Lmt}{\apx{\Lm}}
\newcommand{\pit}{\apx{\pi}}

\renewcommand{\t}{\ensuremath{\top}}

\newcommand{\citeproxy}[1]{\textbf{[?]}}

\newcommand{\loss}{\ensuremath{\ell}}

\newcommand{\dave}[1]{{\color{BrickRed}[DS: #1]}}
\newcommand{\chris}[1]{{\color{ForestGreen}[CG: #1]}}

\newcommand{\idx}{\ensuremath{\delta}}

%% file: introduction.tex
Reinforcement learning (RL) provides a conceptual framework to tackle a fundamental challenge in artificial intelligence: how to design agents that learn while interacting with the environment~\cp{sutton2018reinforcement}. It has been argued that truly general agents should be able to learn a model of the environment that allows for fast re-planning and counterfactual reasoning~\cp{russel2003artificial}. Although this is not a particularly contentious statement, the question of \emph{how} to learn such a model is far from being resolved. 
The common practice in model-based RL is to conceptually separate the learning of the model from its use.
In this paper we argue that the limited representational capacity of model-based RL agents is better allocated if the future \emph{use} of the model
({\sl e.g.}, value-based planning)
is also taken into account during its construction~\cp{joseph2013reinforcement,farahmand2017value,farahmand2018iterative}.

Our primary contribution is to formalize and analyze a clear principle that underlies this new approach to model-based RL.
Specifically, we show that, when the model is to be used for value-based planning, requirements on the model
can be naturally captured by an equivalence relation 
induced by a set of policies and functions. This leads to the \emph{principle of value equivalence}: two models are value equivalent with respect to a set of
functions and a set of policies if they yield the same updates under corresponding Bellman operators.
The policies and functions then become the mechanism through which one incorporates information about the intended use of the model during its construction. We propose a formulation of the model learning problem based on the value equivalence principle and analyze how the set of feasible solutions is impacted by the choice of policies and functions. Specifically, we show that, as we augment the set of policies and functions considered, the class of value equivalent models shrinks, until eventually collapsing to a single point corresponding to a model that perfectly describes the environment. 

We also discuss cases in which one can meaningfully restrict the class of policies and functions used to tailor the model. 
One common case is when the construction of an optimal policy through value-based planning only requires that a model predicts a subset of value functions. We show that in this case the resulting value equivalent models can perform well under much more restrictive conditions than their traditional counterparts.
Another common case is when the agent has limited representational capacity. We show that in this scenario it suffices for a model to be value equivalent with respect to appropriately-defined bases of the spaces of representable policies and functions. 
This allows models to be found with less memory or computation than conventional model-based approaches that aim at predicting all state transitions, such as maximum likelihood estimation. We illustrate the benefits of value-equivalent model learning in experiments that compare it against more conventional counterparts. More generally, we argue that the principle of value equivalence underlies a number of recent empirical successes in RL and provides a first theoretical underpinning of those results~\cp{tamar2016value,silver2017predictron,oh2017value,farquhar2018treeqn,schrittwieser2019mastering}.

%% file: background.tex
As usual, we will model the agent's interaction with the environment using a \emph{Markov Decision Process} (MDP) $\mathcal{M} \defi \langle \S, \A, r, p, \gamma \rangle$ where $\S$ is the state space, $\A$ is the action space, $r(s,a,s')$ is the reward associated with a transition to state $s'$ following the execution of action $a$ in state $s$, $p(s' | s, a)$ is the transition kernel and $\gamma \in [0,1)$ is a discount factor~\cp{puterman94markov}. For convenience we also define $r(s,a) = \E_{S' \sim p(\cdot|s,a)}[r(s,a,S')]$.
 
 A \emph{policy} is a mapping $\pi: \S \mapsto \PP(\A)$, where $\PP(\A)$ is the space of probability distributions over \A. We define $\PS \defi \{\pi \,|\,  \pi: \S \mapsto \PP(\A)\}$ as the set of all possible policies. The agent's goal is to find a policy $\pi \in \PS$ that maximizes the \emph{value} of every state, defined as
\begin{equation} 
\label{eq:v}
v_{\pi}(s) \defi \E_{\pi} \left[ \sum_{i=0}^{\infty} \gamma^{i} r(S_{t+i}, A_{t+i}) \,|\, S_{t} = s \right],
\end{equation}
where $S_t$ and $A_t$ are random variables indicating the state occupied and the action selected by the agent at time step $t$ and $\E_{\pi}[\cdot]$ denotes expectation over the trajectories induced by $\pi$. 

Many methods are available to carry out the search for a good policy~\cp{sutton2018reinforcement,szepesvari2010algorithms}. Typically, a crucial step in these methods is the computation of the value function of candidate policies---a process usually referred to as \emph{policy evaluation}. One way to evaluate a policy $\pi$ is through its \emph{Bellman operator}:
\begin{equation}
\label{eq:bo}
\Tpi[v](s) \defi \E_{A \sim \pi(\cdot | s), S' \sim  p(\cdot | s, A)} \left[r(s,A)  + \gamma v(S')\right],
\end{equation}
where $v$ is any function in the space $\FS \defi \{f \,|\, f: \S \mapsto \R\}$. It is known that $\lim_{n \rightarrow \infty} (\Tpi)^n v = v_{\pi}$, that is, starting from any $v \in \FS$, the repeated application of \Tpi\ will eventually converge to $v_{\pi}$~\cp{puterman94markov}.

In RL it is generally assumed that the agent does not know $p$ and $r$, and thus cannot directly 
compute~(\ref{eq:bo}). In \emph{model-free} RL this is resolved by replacing $v_{\pi}$ with an action-value function and estimating the expectation on the right-hand-side of~(\ref{eq:bo}) through sampling~\cp{sutton88learning}. In \emph{model-based RL}, the focus of this paper, the agent learns approximations $\rt \approx r$ and $\pt \approx p$ and use them to compute~(\ref{eq:bo}) with $p$ and $r$ replaced by \pt\ and \rt~\cp{sutton2018reinforcement}.

%% file: value_equivalence.tex
Given a state space \S\ and an action space \A, we call the tuple $m \defi (r,p)$ a \emph{model}. Note that a model plus a discount factor $\gamma$ induces a Bellman operator~(\ref{eq:bo}) for every policy $\pi \in \PS$. In this paper we are interested in computing an approximate model $\mt = (\rt, \pt)$ such that the induced operators $\Tpit$, defined analogously to~(\ref{eq:bo}), are good approximations of the true $\Tpi$. 
Our main argument is that models should only be distinguished with respect to the policies and functions they will actually be applied to. This leads to the following definition:

\begin{definition}[Value equivalence]
\label{def:ve}
 Let $\Pi \subseteq \PS$ be a set of policies and let $\V \subseteq \FS$ be a set of functions. We say that models $m$ and $\mt$ are \emph{value equivalent} with respect to $\Pi$ and $\V$ if and only if
 \begin{equation*}
\Tpi v = \Tpit v \text{ for all } \pi \in \Pi \text{ and all } v \in \V,
\end{equation*}
where \Tpi\ and \Tpit\ are the Bellman operators induced by $m$ and \mt, respectively.
\end{definition}

Two models are value equivalent with respect to $\Pi$ and \V\ if the effect of the Bellman operator induced by any policy $\pi \in \Pi$ on any function $v \in \V$ is the same for both models. Thus, if we are only interested in $\Pi$ and \V, value-equivalent models are functionally identical. This can be seen as an equivalence relation that partitions the space of models conditioned on  $\Pi$ and \V:

\begin{definition}[Space of value-equivalent models]
\label{def:vem}
Let $\Pi$ and \V\ be defined as above and let \M\ be a space of models. Given a model $m$, the \emph{space of value-equivalent models} $\M_m(\Pi, \V) \subseteq \M$ is the set of all models $\mt \in \M$ that are value equivalent to $m$ with respect to $\Pi$ and \V.
\end{definition}

Let \MS\ be a space of models containing at least one model $m^*$ which perfectly describes the interaction of the agent with the environment. More formally, $m^*$ induces the true Bellman operators \Tpi\ defined in~(\ref{eq:bo}).
Given a space of models $\M \subseteq \MS$, often one is interested in models $m \in \M$ that are value equivalent to $m^*$. We will thus simplify the notation by defining $\M(\Pi, \V) \defi \M_{m^*}(\Pi, \V)$.

\subsection{The topology of the space of value-equivalent models}

The space $\M(\Pi,\V)$ contains all the models in \M\ that are value equivalent to the true model $m^*$ with respect to $\Pi$ and \V. Since any two models $m, m' \in \M(\Pi,\V)$ are equally suitable for value-based planning using $\Pi$ and $\V$, we are free to use other criteria to choose between them. For example, if $m$ is much simpler to represent or learn than $m'$, it can be preferred without compromises. 

Clearly, the principle of value equivalence can be useful if leveraged in the appropriate way. In order for that to happen, it is important to understand the space of value-equivalent models $\M(\Pi,\V)$. We now provide intuition for this space by analyzing some of its core properties. We refer the reader to Figure~\ref{fig:ve_pert_diagram} for an illustration of the concepts to be discussed in this section. We start with a property that follows directly from Definitions~\ref{def:ve} and~\ref{def:vem}:

\begin{wrapfigure}{r}{0.45\textwidth}
\begin{center}
\includegraphics[scale=0.5]{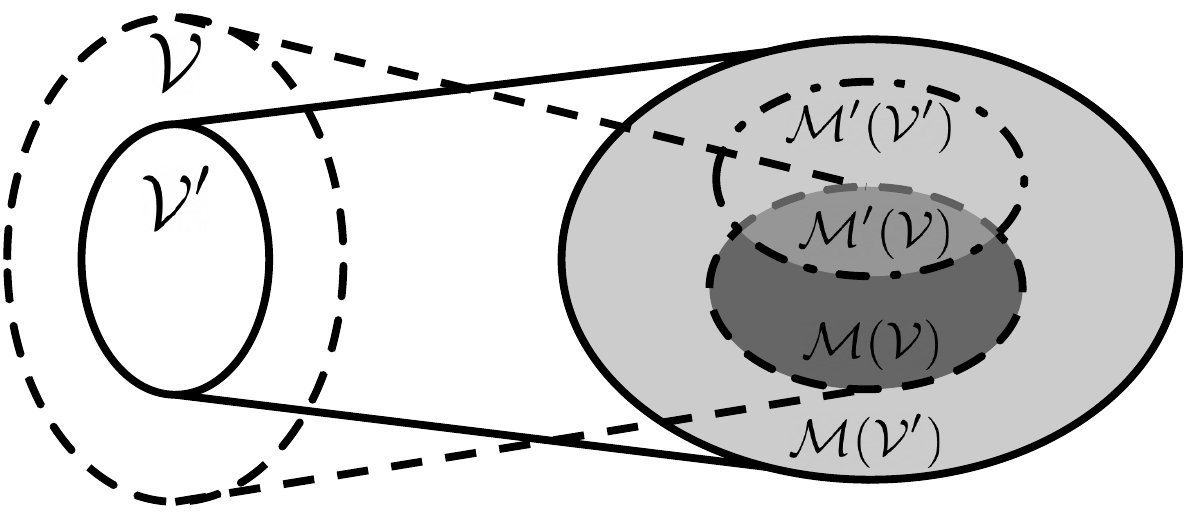}
\end{center}
\caption{Understanding the space of value-equivalent models for a fixed $\Pi$, $\M' \subset \M$ and $\V' \subset \V$. Denote $\M(\V) \defi \M(\V, \Pi)$. 
\textbf{Property~\ref{thm:pertm}}: $\M'(\V) \subset \M(\V)$ and $\M'(\V') \subset \M(\V')$. \textbf{Property~\ref{thm:pertpiv}}: $\M(\V) \subset \M(\V')$ and $\M'(\V) \subset \M'(\V')$. \textbf{Property~\ref{thm:cover}}: if $m^* \in \M$, then $m^* \in \M(\V)$.}
\label{fig:ve_pert_diagram}
\vspace{-10mm}
\end{wrapfigure}

\begin{restatable}{property}{proppertm}
\label{thm:pertm}
Given $\M' \subseteq \M$, we have that $\M'(\Pi, \V) \subseteq \M(\Pi, \V)$.
\end{restatable}


The proofs of all theoretical results are in Appendix~\ref{sec:proofs}. Property~\ref{thm:pertm} states that, given a set of policies $\Pi$ and a set of functions \V, reducing the size of the space of models \M\ also reduces the space of value-equivalent models $\M(\Pi, \V)$. One immediate consequence of this property is that, if we consider the space of all policies \PS\ and the space of all functions \FS, we have one of two possibilities: either we end up with a perfect model or we end up with no model at all. Or, more formally:

\begin{restatable}{property}{proppindown}
\label{thm:pindown}
$\M(\PS, \FS)$ either contains $m^*$ or is the empty set.
\end{restatable}

Property~\ref{thm:pertm} describes what happens to $\M(\Pi,\V)$ when we vary \M\ with fixed $\Pi$ and \V. It is also interesting to ask what happens when we fix the former and vary the latter. This leads to the next property:

\begin{restatable}{property}{proppertpiv}
\label{thm:pertpiv} 
Given $\Pi' \subseteq \Pi$ and $\V' \subseteq \V$, we have that $\M(\Pi, \V) \subseteq \M(\Pi', \V')$.
\end{restatable}

According to Property~\ref{thm:pertpiv}, as we increase the size of $\Pi$ or $\V$ the size of $\M(\Pi, \V)$ \emph{decreases}. Although this makes intuitive sense, it is reassuring to know that value equivalence is a sound principle for model selection, since by adding more policies to $\Pi$ or more values to \V\ we can progressively restrict the set of feasible solutions. Thus, if \M\ contains the true model, we eventually pin it down. Indeed, in this case the true model belongs to \emph{all} spaces of value equivalent models, as formalized below: 

\begin{restatable}{property}{propcover}
\label{thm:cover}
If $m^* \in \M$, then $m^* \in \M(\Pi, \V)$ for all $\Pi$ and all \V.
\end{restatable}
   
\subsection{A basis for the space of value-equivalent models} 
\label{eq:ve_basis}

As discussed, it is possible to use the sets $\Pi$ and \V\ to control the size of $\M(\Pi, \V)$. But what exactly is the effect of $\Pi$ and \V\ on $\M(\Pi, \V)$? How much does $\M(\Pi, \V)$ decrease in size when we, say, add one function to \V? In this section we address this and similar questions.

We start by showing that, whenever a model is value equivalent to $m^*$ with respect to discrete $\Pi$ and \V, it is automatically value equivalent to $m^*$ with respect to much larger sets. In order to state this fact more concretely we will need two definitions. Given a discrete set \H, we define $\span(\H)$ as the set formed by all linear combinations of the elements in \H. Similarly, given a discrete set \H\ in which each element is a function defined over a domain \X, we define the \emph{pointwise span} of \H\ as
\begin{equation}
\label{eq:cspan}
\pspan(\H) \defi \left\{ h: h(x) = \sum_i \alpha_{xi} h_i(x) \right\}, \text{ with } \alpha_{xi} \in \mathbb{R} \text{ for all } x \in \X, i \in \{1, \ldots, |\H|\}
\end{equation}
where $h_i \in \H$. 
Pointwise span can alternatively be characterized by considering each element in the domain separately: $g \in \pspan(\H) \iff g(x) \in \span \{ h(x) : h \in \H \}$ for all $x \in \X$.
Equipped with these concepts we present the following result:

\begin{restatable}{proposition}{proplinearspan}
\label{thm:linear_span}
For discrete $\Pi$ and $\V$, we have that $\M(\Pi, \V) = \M(\pspan(\Pi) \cap \PS, \span(\V))$.
\end{restatable}

Proposition~\ref{thm:linear_span} provides one possible answer to the question posed at the beginning of this section: the  contraction of $\M(\Pi, \V)$ resulting from the addition of one policy to $\Pi$ or one function to \V\ depends on their effect on $\pspan(\Pi)$ and $\span(\V)$. For instance, if a function $v$ can be obtained as a linear combination of the functions in \V, adding it to this set will have no effect on the space of equivalent models $\M(\Pi, \V)$. More generally, Proposition~\ref{thm:linear_span} suggests a strategy to \emph{build} the set \V: one should find a set of functions that form a basis for the space of interest. When \S\ is finite, for example, having \V\ be a basis for $\R^{|\S|}$ means that the value equivalence principle will apply to every function $\v \in \R^{|S|}$. The same reasoning applies to $\Pi$. In fact, because $\pspan(\Pi)$ grows independently pointwise, it is relatively simple to build a set $\Pi$ that covers the space of policies one is interested in. In particular, when \A\ is finite, it is easy to define a set $\Pi$ for which $\pspan(\Pi) \supseteq \PS$: it suffices to have for every state-action pair $(s,a) \in \S \times \A$ at least one policy $\pi \in \Pi$ such that $\pi(a|s) = 1$. This means that we can apply the value equivalence principle to the entire set \PS\ using $|A|$ policies only.

Combining Proposition~\ref{thm:linear_span} and Property~\ref{thm:pindown} we see that by defining $\Pi$ and \V\ appropriately we can focus on the subset of \M\ whose models perfectly describe the environment:

\begin{restatable}{remark}{corpindownspan}
\label{thm:pin_down_span}
If $\PS \subseteq \pspan(\Pi)$ and $\FS = \span(\V)$, then $\M(\Pi, \V) = m^*$ or $\M(\Pi, \V) = \emptyset$.
\end{restatable}

We have shown how $\Pi$ and \V\ have an impact on the number of value equivalent models in $\M(\Pi, \V)$; to make the discussion more concrete, we now focus on a specific model space \M\ and analyze the rate at which this space shrinks as we add more elements to $\Pi$ and \V. 
Before proceeding we define a set of functions $\mathcal{H}$ as \textit{pointwise linearly independent} if $h \notin \pspan(\mathcal{H} \setminus \{ h \})$ for all $h \in \mathcal{H}$.

Suppose both $\S$ and $\A$ are finite.
In this case a model can be defined as $m = (\r, \P)$, where $\r \in \R^{|\S||\A|}$ and $\P \in \R^{|\S| \times |\S| \times |\A|}$.
A policy can then be thought of as a vector $\ppi \in \R^{|\S||\A|}$. 
We denote the set of all transition matrices induced by transition kernels as $\mathbb{P}$. 
To simplify the analysis we will consider that $\r$ is known and we are interested in finding a model $\Pt \in \mathbb{P}$. 
In this setting, we write $\mathbb{P}(\Pi, \V)$ to denote the set of transition matrices that are value equivalent to the true transition matrix $\boldsymbol{P}^*$.
We define the dimension of a set $\mathcal{X}$ as the lowest possible Hamel dimension of a vector-space enclosing some translated version of it:
$\dim[\mathcal{X}]~=~\min_{\mathcal{W}, c \in W(\mathcal{X})}~\hdim[\mathcal{W}]$
where $W(\mathcal{X}) = \{ (\mathcal{W}, c) : \mathcal{X} + c \subseteq \mathcal{W} \}$,
$\mathcal{W}$ is a vector-space, $c$ is an offset and $\hdim[\cdot]$ denotes the Hamel dimension. 
Recall that the Hamel dimension of a vector-space is the size of the smallest set of mutually linearly independent vectors that spans the space (this corresponds to the usual notion of dimension, that is, the minimal number of coordinates required to uniquely specify each point). So, under no restrictions imposed by $\Pi$ and \V, we have that $\dim[\mathbb{P}] = (|\S| - 1)|\S||\A|$. We now show how fast the size of $\mathbb{P}(\Pi, \V)$ decreases as we extend the ranges of $\Pi$ and \V:

\begin{restatable}{proposition}{proplineargrowth}
\label{thm:linear_growth}
 Let $\Pi$ be a set of $m$ pointwise linearly independent policies $\ppi_i \in \R^{|\S||\A|}$ and let \V\ be a set of $k$ linearly independent vectors $\v_i \in \R^{|\S|}$. Then, 
 \begin{equation*}
\dim\left[\mathbb{P}(\Pi, \V)\right] \leq |\S|\left(|\S||\A| - m k\right).
\end{equation*}
\end{restatable}

Interestingly, Proposition~\ref{thm:linear_growth} shows that the elements of $\Pi$ and \V\ interact in a multiplicative way: when there are $m$  pointwise linearly independent policies, enlarging \V\ with a single function $v$ that is linearly independent of its counterparts will decrease the bound on the dimension of $\mathbb{P}(\Pi, \V)$ by a factor of $m$. This makes intuitive sense if we note that by definition $m \le |\A|$: for an expressive enough $\Pi$, each $v \in \V$ will provide information about the effect of all actions in $a \in \A$. Conversely, because $\span(\V) = k \le |\S|$, we can only go so far in pinning down the model when $m < |\A|$---which also makes sense, since in this case we cannot possibly know about the effect of all actions, no matter how big \V\ is. Note that when $m = |\A|$ and $k = |\S|$ the space $\mathbb{P}(\Pi, \V)$ reduces to $\{\P^*\}$.

%% file: ve_model_learning.tex
We now discuss how the principle of value equivalence can be incorporated into model-based RL. Often in model-based RL one learns a model $\mt = (\rt, \pt)$ without taking the space $\M(\Pi, \V)$ into account. The usual practice is to cast the approximations $\rt \approx r$ and $\pt \approx p$ as optimization problems over a model-space $\mathcal{M}$ that do not involve the sets $\Pi$ and \V. Given a space \RR\ of possible approximations \rt, we can formulate the approximation of the rewards as $\argmin_{\rt \in \RR} \loss_r(r,\rt)$, where $\loss_r$ is a loss function that measures the dissimilarity between $r$ and \rt. The approximation of the transition dynamics can be formalized in an analogous way: $\argmin_{\pt \in \PP} \loss_p(p,\pt)$, where \PP\ is the space of possible approximations \pt.

A common choice for $\loss_r$ is
\begin{equation}
\label{eq:cml_r}
\loss_{r, \D}(r, \rt) \defi \E_{(S, A) \sim \D}\left[ (r(S, A) - \rt(S, A))^2 \right],
\end{equation}
where \D\ is a distribution over $\S \times \A$. The loss $\loss_p$ is usually defined based on the principle of \emph{maximum likelihood estimation} (MLE):
\begin{equation}
\label{eq:cml_p} 
\loss_{p, \D}(p, \pt) \defi \E_{(S,A) \sim \D}\left[ \mathrm{D}_{\mathrm{KL}}(p(\cdot | S, A) \, || \, \pt(\cdot | S, A)) \right],
\end{equation}
where $\mathrm{D}_{\mathrm{KL}}$ is the Kullback-Leibler (KL) divergence. 
Since we normally do not have access to $r$ and $p$, the losses~(\ref{eq:cml_r}) and~(\ref{eq:cml_p}) are usually minimized using transitions sampled from the environment~\cp{sutton2008dyna}.
There exist several other criteria to approximate $p$ based on state transitions, such as maximum {\sl a posteriori} estimation, maximum
entropy estimation, and Bayesian posterior inference~\cp{farahmand2018iterative}. Although we focus on MLE for simplicity, our arguments should extend to these other criteria as well.

Both~(\ref{eq:cml_r}) and~(\ref{eq:cml_p}) have desirable properties that justify their widespread adoption~\cp{milar2003maximum}. However, we argue that ignoring the future use of \rt\ and \pt\ may not always be the best choice~\cp{joseph2013reinforcement,farahmand2017value}. To illustrate this point, we now show that, by doing so, one might end up with an approximate model when an exact one were possible. 
Let $\PP(\Pi, \V)$ be the set of value equivalent transition kernels in $\PP$. Then,
%
\begin{restatable}{proposition}{propmle}
\label{thm:mle}
The maximum-likelihood estimate of $p^*$ in \PP\ may not belong to a $\PP(\Pi, \V) \ne \emptyset$.
\end{restatable}
Proposition~\ref{thm:mle} states that, even when there exist models in \PP\ that are value equivalent to $p^*$ with respect to $ \Pi$ and \V, the minimizer of~(\ref{eq:cml_p}) may not be in $\mathcal{P}(\Pi, \mathcal{V})$. In other words, even when it is possible to perfectly handle the policies in $\Pi$ and the values in \V, the model that achieves the smallest MLE loss will do so only approximately. This is unsurprising since the loss~(\ref{eq:cml_p}) is agnostic of $\Pi$ and \V, providing instead a model that represents a compromise across all policies \PS\ and all functions \FS.

We now define a value-equivalence loss that explicitly takes into account the sets $\Pi$ and \V:
\begin{equation}
\label{eq:ve_loss}
\loss_{\Pi, \V}(m^*,\mt) \defi \sum_{\pi \in \Pi} \sum_{v \in \V} \| \Tpi v - \Tpit v||,
\end{equation}
where $\Tpit $ are Bellman operators induced by \mt\ and $|| \cdot ||$ is a norm. Given~(\ref{eq:ve_loss}), the problem of learning a model based on the value equivalence principle can be formulated as $\argmin_{\mt \in \M} \loss_{\Pi, \V}(m^*,\mt)$.

As noted above, we usually do not have access to \Tpi, and thus the loss~(\ref{eq:ve_loss}) will normally be minimized based on sample transitions. Let $\S_{\pi} \defi \{(s^{\pi}_{i}, a^{\pi}_{i}, r^{\pi}_{i}, \hat{s}^{\pi}_{i}) | i = 1, 2, ..., n^{\pi}\}$ be $n^{\pi}$ sample transitions associated with policy $\pi \in \Pi$. 
We assume that the initial states $s^{\pi}_i$ were sampled according to some distribution $\D'$ over \S\ and the actions were sampled according to the policy $\pi$, $a^{\pi}_i \sim \pi(\cdot | s^{\pi}_i)$ (note that $\D'$ can be the distribution resulting from a direct interaction of the agent with the environment). When $\|\cdot\|$ appearing in~(\ref{eq:ve_loss}) is a $p$-norm , we can write its empirical version as
\begin{equation}
\vspace{-1mm}
\label{eq:ve_loss_emp}
\loss_{\Pi, \V, \D'}(m^*,\mt) \defi 
\sum_{\pi \in \Pi} \sum_{v \in \V}  \sum_{s \in {\S}_{\pi}^{'}}
\left[\dfrac{\sum_{i=1}^{n^\pi} \ind\{s^\pi_i = s\} (r^{\pi}_i + \gamma v(\hat{s}^{\pi}_i))} {\sum_{i=1}^{n^\pi} \ind\{s^\pi_i = s\}} - \Tpit v[s] \right]^p,
\end{equation}
where ${\S}_{\pi}^{'}$ is a set containing only the initial states $s^{\pi}_i \in \S_{\pi}$ and $\ind\{\cdot\}$ is the indicator function.
We argue that, when we know policies $\Pi$ and functions \V\ that are sufficient for planning, the appropriate goal for model-learning is to minimize the value-equivalence loss (\ref{eq:ve_loss}). As shown in Proposition~\ref{thm:mle}, the model \mt\ that minimizes~(\ref{eq:cml_r}) and~(\ref{eq:cml_p}) may not achieve zero loss on~(\ref{eq:ve_loss}) even when such a model exists in \M. In general, though, we should not expect there to be a model $\mt \in \M$ that leads to zero value-equivalence loss. 
Even then, value equivalence may lead to a better model than conventional counterparts (see Figure~\ref{fig:ve_app_diagram} for intuition and Appendix~\ref{sec:subspace_example} for a concrete example).

\begin{wrapfigure}{r}{0.45\textwidth}
\vspace{-13mm}
\begin{center}
\includegraphics[width=0.45\textwidth]{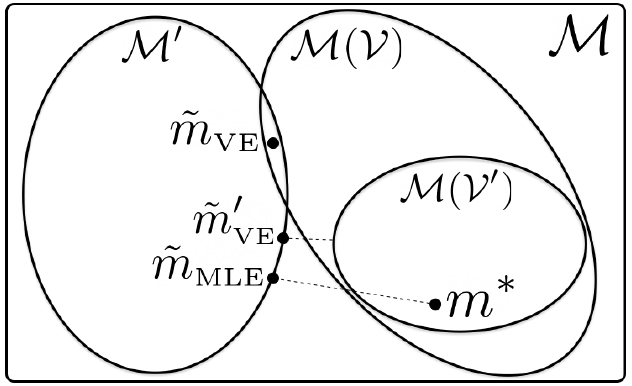}
\end{center}
\caption{When the hypothesis space $\M'$ and the space of value-equivalent models $\M(\V)$ intersect, the resulting model $\mt_{{\scriptscriptstyle \mathrm{VE}}}$ has zero loss~(\ref{eq:ve_loss}), while the corresponding MLE model $\mt_{{\scriptscriptstyle \mathrm{MLE}}}$ may not (Proposition~\ref{thm:mle}). But even when $\M'$ and $\M(\V')$ do not intersect, the resulting $\mt'_{{\scriptscriptstyle \mathrm{VE}}}$ can outperform $\mt_{{\scriptscriptstyle \mathrm{MLE}}}$ with the appropriate choices of $\Pi$ and $\V$, as we illustrate in the experiments of Section~\ref{sec:experiments}.}
\label{fig:ve_app_diagram}
\vspace{-7mm}
\end{wrapfigure}

\vspace{-2mm}

\subsection{Restricting the sets of policies and functions}\label{sec:restrict}

The main argument of this paper is that, rather than learning a model that suits all policies \PS\ and all functions \FS, we should instead focus on the sets of policies $\Pi$ and functions \V\ that are necessary for planning.
But how can we know these sets {\sl a priori}? We now show that it is possible to exploit structure on both the \emph{problem} and the \emph{solution} sides. 


First, we consider structure in the \emph{problem}. Suppose we had access to the true model $m^*$. Then, given an initial function $v$, a value-based planning algorithm that makes use of $m^*$ will generate a sequence of functions $\vec{\V}_v \defi \{ v_1, v_2, ... \}$~\cp{dabney2020valueimprovement}. Clearly, if we replace $m^*$ with any model in $\M(\PS, \vec{\V}_v)$, the behavior of the algorithm starting from $v$ remains unaltered. This allows us to state the following:

\begin{restatable}{proposition}{propvip}
Suppose $v \in \V' \implies \Tpi v \in \V'$ for all $\pi \in \PS$. Let $\pspan(\Pi) \supseteq \PS$ and $\span(\V) = \V'$. Then, starting from any $v' \in \V'$, any $\mt \in \M(\Pi, \V)$ yields the same solution as $m^*$.
\label{thm:vip}
\vspace{-1mm}
\end{restatable}
Because \Tpi\ are contraction mappings, it is always possible to define a $\V' \subset \FS$ such that the condition of Proposition~\ref{thm:vip} holds: we only need to make $\V'$ sufficiently large to encompass $v$ and the operators' fixed points. 
But in some cases there exist more structured $\V'$: in Appendix~\ref{sec:proofs} we give an example of a finite state-space MDP in which a sequence $\v_1, \v_2 = \Tpi \v_1, \v_3 = \T_{\pi'} \v_2, ...$ that reaches a specific $k$-dimensional subspace of $\R^{|S|}$ stays there forever. The value equivalence principle provides a mechanism to exploit this type of structure, while conventional model-learning approaches, like MLE, are oblivious to this fact. Although in general we do not have access to $\V'$, in some cases this set will be revealed through the very process of enforcing value equivalence. For example, if \mt\ is being learned online based on a sequence $v_1, v_2 = \Tpi v_1, v_3 = \T_{\pi'} v_2, ...$, as long as the sequence reaches a $v_i \in \V'$ we should expect \mt\ to eventually specialize to $\V'$~\cp{farahmand2018iterative,schrittwieser2019mastering}.

Another possibility is to exploit geometric properties of the \emph{value functions} $\vec{\V}_v$. It is known that the set of all value functions of a given MDP forms a polytope $\ddot{\V} \subset  \FS$~\cp{dadashi2019value}. Even though the sequence $\vec{\V}_v$ an algorithm generates may not be strictly inside the polytope $\ddot{\V}$, this set can still serve as a reference in the definition of $\V$. For example, based on Proposition~\ref{thm:linear_span}, we may want to define a \V\ that spans as much of the polytope $\ddot{\V}$ as possible~\cp{bellemare2019geometric}. This suggests that the functions in \V\ should be actual value functions $v_{\pi}$ associated with policies $\pi \in \PS$. In Section~\ref{sec:experiments} we show experiments that explore this idea. 
    
We now consider structure in the \emph{solution}. Most large-scale applications of model-based RL use function approximation. 
Suppose the agent can only represent policies $\pi \in \APi$ and value functions $v \in \AV$. Then, a value equivalent model $\mt \in \M(\APi, \AV)$ is as good as any model. To build intuition, suppose the agent uses state aggregation to approximate the value function.
In this case two models with the same transition probabilities between clusters of states are indistinguishable from the agent's point of view. It thus makes sense to build \V\ using piecewise-constant functions that belong to the space of function representable by the agent, $v \in \AV$. The following remark generalises this intuition:
     
\begin{remark}
\label{thm:linear_app}
Suppose the agent represents the value function using a linear function approximation: $\AV = \{ \vt \, | \, \vt(s) = \sum_{i =1}^d \phi_i(s) w_i \}$, where $\phi_i: \S \mapsto \R$ are fixed features and $\w \in \R^d$ are learnable parameters. In addition, suppose the agent can only represent policies $\pi \in \APi$. Then, Proposition~\ref{thm:linear_span} implies that if we use the features themselves as the functions adopted with value equivalence, $\V = \{ \phi_i \}_{i=1}^d$, we have that $\M(\APi, \{\phi_i\}_{i=1}^d) = \M(\APi, \AV)$. 
In other words, models that are value equivalent with respect to the features are indiscernible to the agent.
\end{remark}

According to the remark above, when using linear function approximation, a model that is value equivalent with respect to the approximator's features will perform no worse than any other model. This prescribes a concrete way to leverage the value equivalence principle in practice, since the set of functions \V\ is automatically defined by the choice of function approximator. Note that, although the remark is specific to linear value function approximation, it applies equally to linear and non-linear models (this is in contrast with previous work showing the equivalence between model-free RL using linear function approximation and model-based RL with a linear model for expected features~\cp{parr2008analysis,sutton2008dyna}). The principle of finding a basis for \AV\ also extends to non-linear value function approximation, though in this case it is less clear how to define a set \V\ that spans \AV. One strategy is to \emph{sample} the functions to be included in \V\ from the set \AV\ of (non-linear) functions the agent can represent. Despite its simplicity, this strategy can lead to good performance in practice, as we show next.

%% file: experiments.tex
We now present experiments illustrating the usefulness of the value equivalence principle in practice. Specifically, we compare models computed based on value equivalence (VE) with models resulting from maximum likelihood estimation (MLE). All our experiments followed the same protocol: $(i)$ we collected sample transitions from the environment using a policy that picks actions uniformly at random, $(ii)$ we used this data to learn an approximation \rt\ using~(\ref{eq:cml_r}) as well as approximations \pt\ using either MLE (\ref{eq:cml_p}) or VE~(\ref{eq:ve_loss_emp}), $(iii)$ we learned a policy \pit\ based on $\mt = (\rt,\pt)$, and $(iv)$ we evaluated \pit\ on the actual environment. The specific way each step was carried out varied according to the characteristics of the environment and function approximation used; see App.~\ref{sec:details_exp} for details.

One of the central arguments of this paper is that the value equivalence principle can yield a better allocation of the limited resources of model-based agents. In order to verify this claim, we varied the representational capacity of the agent's models \mt\ and assessed how well MLE and VE performed under different constraints. As discussed, VE requires the definition of two sets: $\Pi$ and \V. It is usually easy to define a set of policies $\Pi$ such that $\pspan(\Pi) \supseteq \PS$; since all the environments used in our experiments have a finite action space \A, we accomplished that by defining $\Pi = \{ \pi^a \}_{a \in \A}$ where $\pi^a(a | s) = 1$ for all $s \in \S$. We will thus restrict our attention to the impact of the set of functions \V. 

As discussed, one possible strategy to define \V\ is to use actual value functions in an attempt to span as much as possible of the value polytope $\ddot{\V}$~\cp{bellemare2019geometric}. Figure~\ref{fig:res_value} shows results of VE when using this strategy. Specifically, we compare VE's performance with MLE's on two well known domains: ``four rooms''~\cp{sutton1999between} and ``catch''~\cp{mnih2014recurrent}. For each domain, we show two types of results: we either fix the capacity of the model \pt\ and vary the size of \V\ or vice-versa (in the Appendix we show results with all possible combinations of model capacities and sizes of \V). Note how the models produced by VE outperform MLE's counterparts across all scenarios, and especially so under stricter restrictions on the model. This corroborates our hypothesis that VE yields models that are tailored to future use.

\begin{figure}
\centering
\subfigure[Four rooms (fixed \V) \label{fig:four_rooms_value}]{
\includegraphics[scale=0.1825]{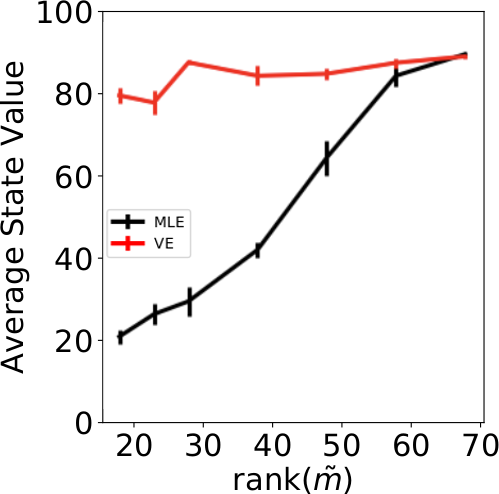}
}
\subfigure[Four rooms (fixed \mt) \label{fig:four_rooms_value_transpose}]{
\includegraphics[scale=0.25]{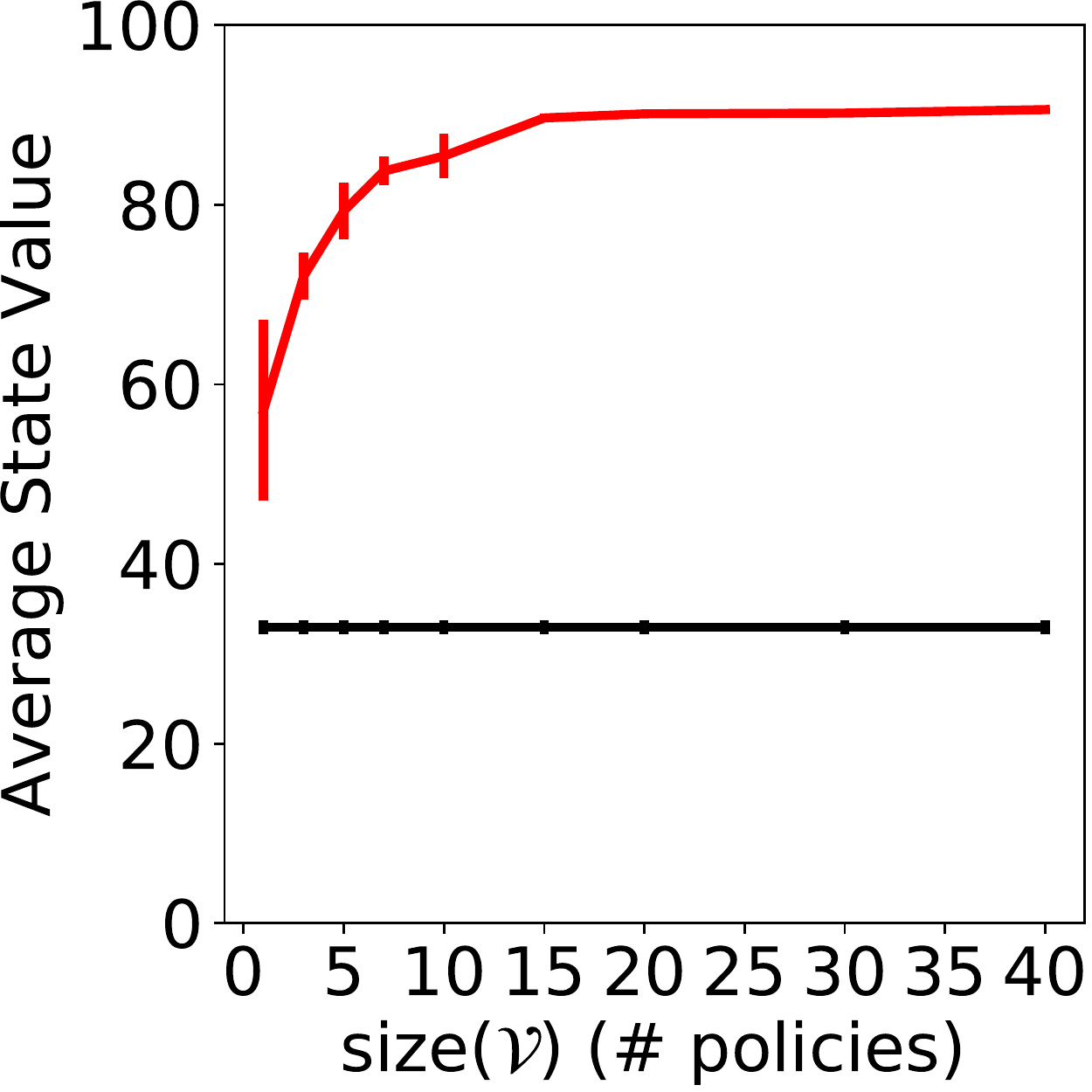}
}
\subfigure[Catch (fixed \V) \label{fig:catch_value}]{
\includegraphics[scale=0.25]{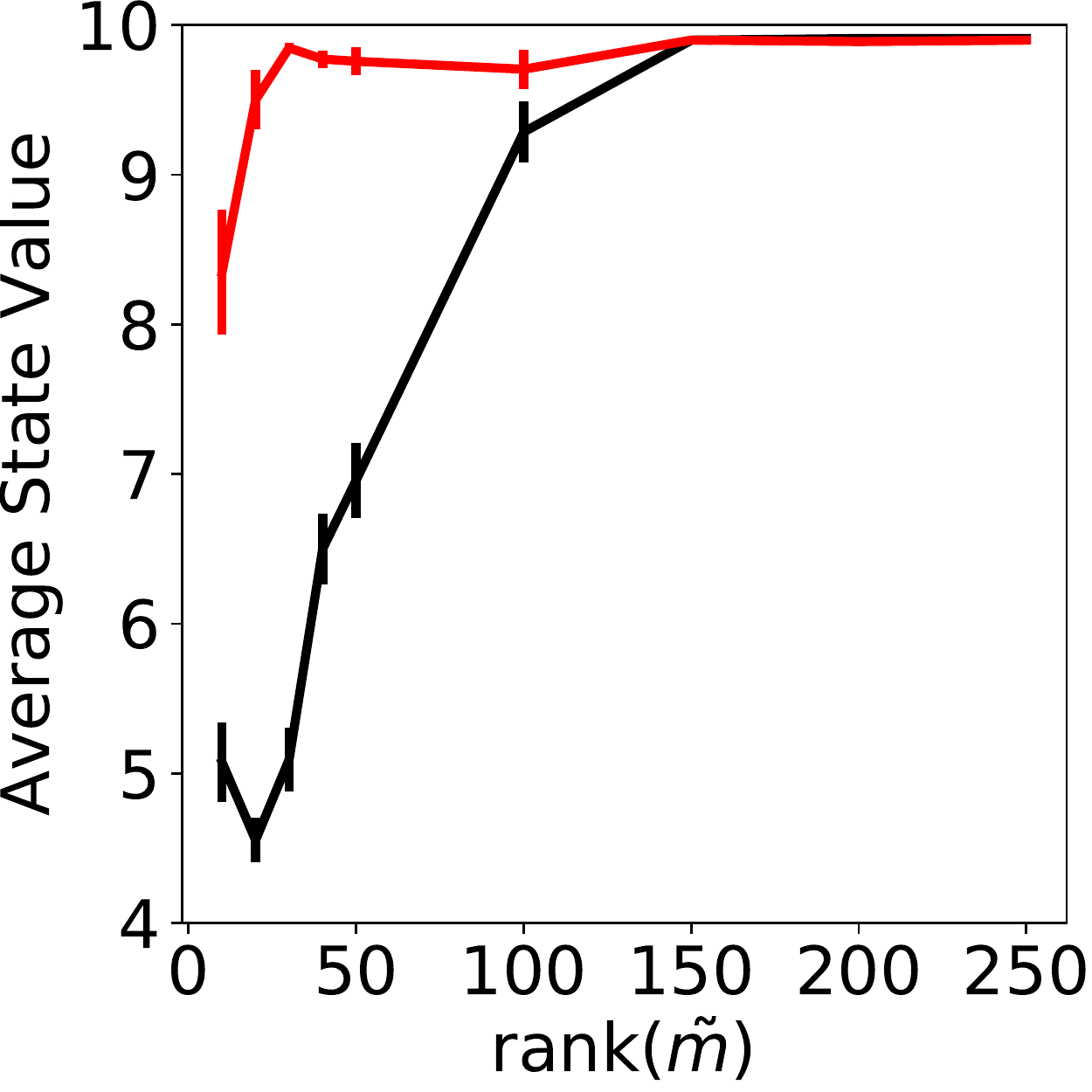}
 }
\subfigure[Catch (fixed \mt) \label{fig:catch_value_transpose}]{
\includegraphics[scale=0.25]{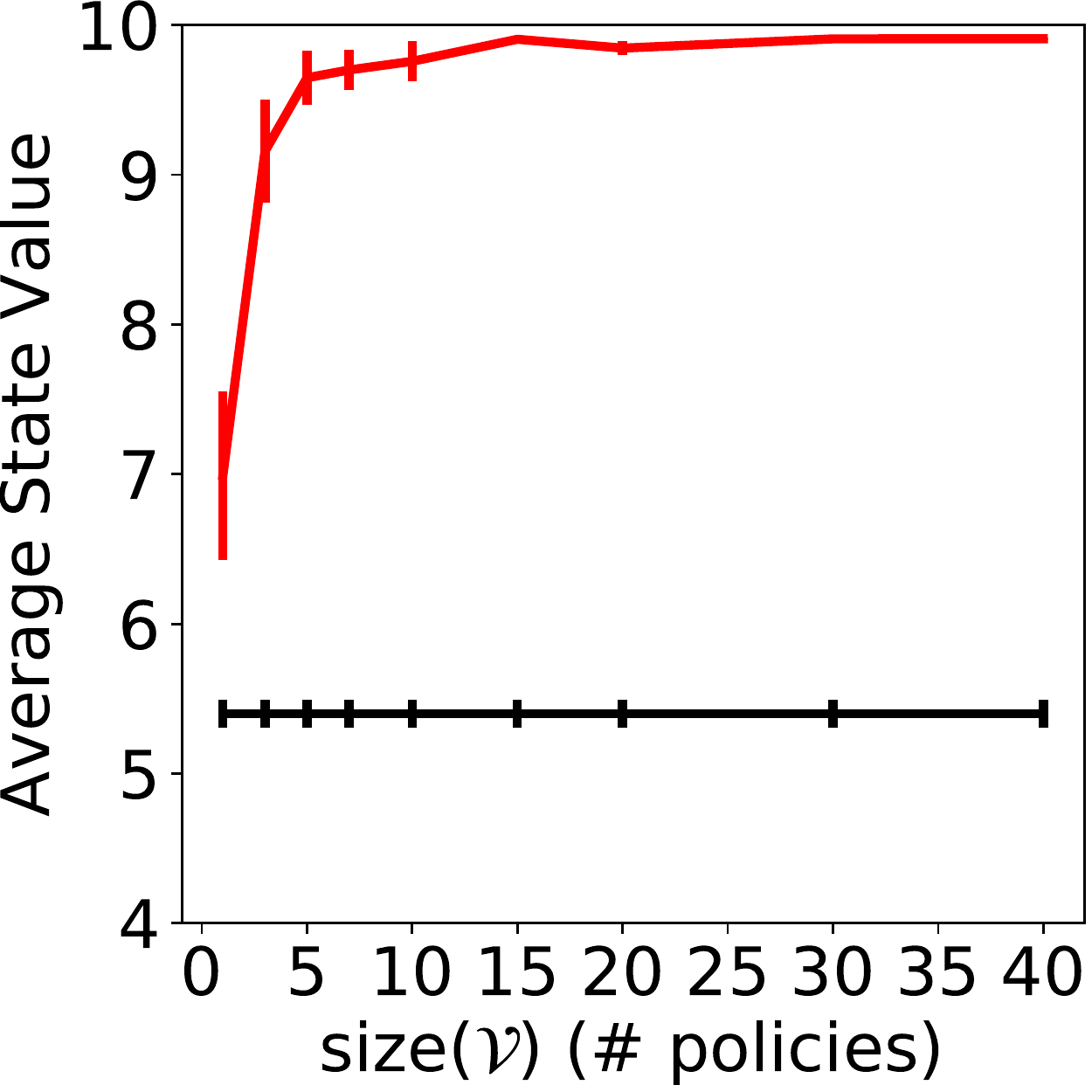}
}
\caption{ Results with \V\ composed of true value functions of randomly-generated policies. The models \pt\ are rank-constrained transition matrices $\Pt = \DD \KK$, with $\DD \in \R^{|\S| \times k}$, $\KK \in \R^{k \times |\S|}$, and $k < |\S|$. Error bars are one standard deviation over $30$ runs. \label{fig:res_value}}
\vspace{-5mm}
\end{figure}

Another strategy to define \V\ is to use functions from \AV, the space of functions representable by the agent, in order to capture as much as possible of this space. In Figure~\ref{fig:res_basis} we compare VE using this strategy with MLE. Here we use as domains catch and ``cart-pole''~\citep{barto1983neuronlike} (but see Appendix for the same type of result on the four-rooms environment). As before, VE largely outperforms MLE, in some cases  with a significant improvement in performance. We call attention to the fact that in cart-pole we used neural networks to represent both the transition models \pt\ and the value functions \vt, which indicates that VE can be naturally applied with nonlinear function approximation. 

It is important to note the broader significance of our experiments. While our theoretical analysis of value equivalence focused on the case where $\M$ contained a value equivalent model, this is not guaranteed in practice.
Our experiments illustrate that, in spite of lacking such a guarantee, we see a considerable gap in performance between VE and MLE, indicating that VE models still offer a strong benefit. 
Our goal here was to provide insight into the value equivalence principle; in the next section we point to prior work to demonstrate the utility of value equivalence in large-scale settings.

\begin{figure}
\centering
\subfigure[Catch (fixed \V) \label{fig:catch_basis}]{
\includegraphics[scale=0.25]{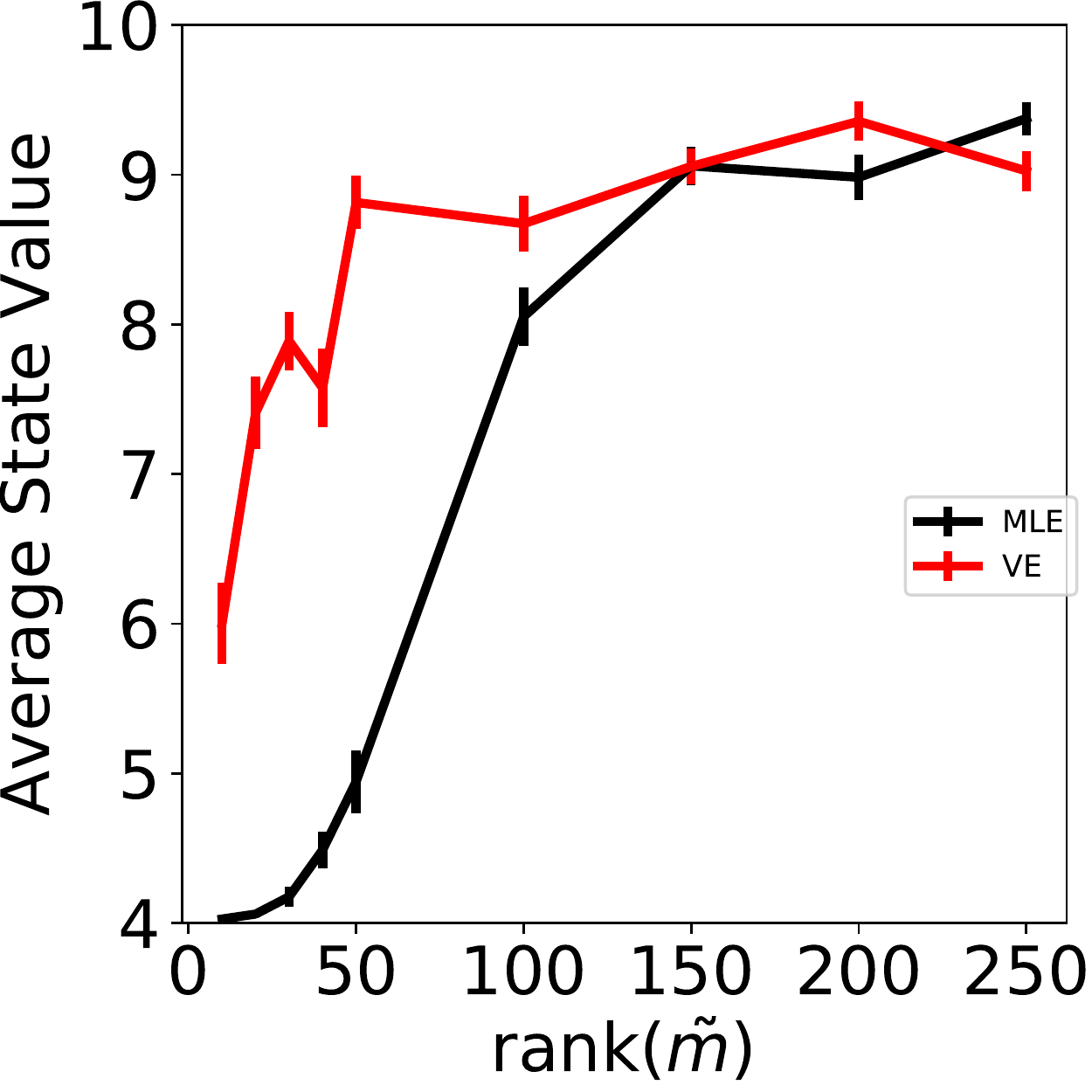}
 }
\subfigure[Catch (fixed \mt) \label{fig:catch_basis_transpose}]{
\includegraphics[scale=0.25]{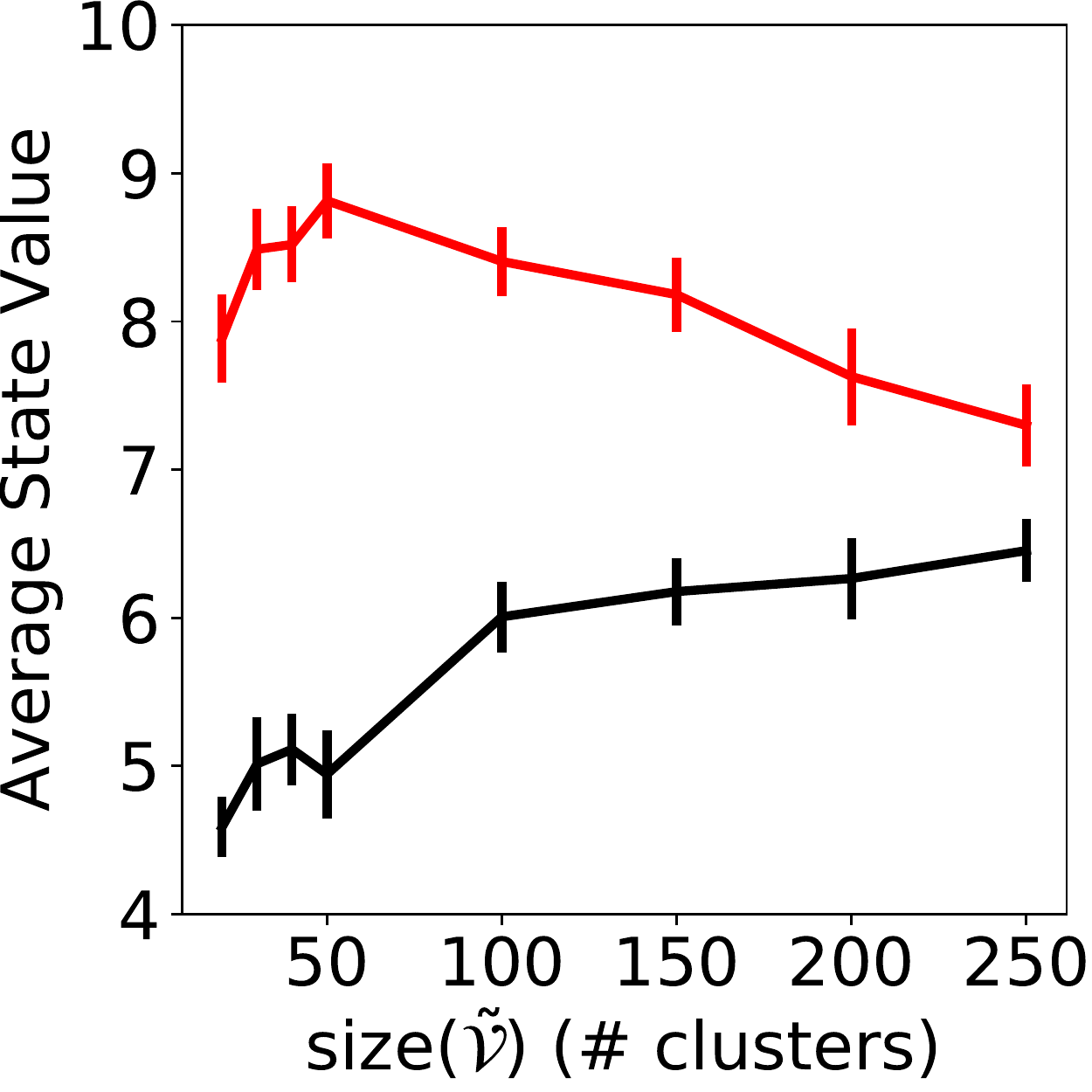}
}
\subfigure[Cart-pole (fixed \V) \label{fig:cartpole_basis}]{
\includegraphics[scale=0.25]{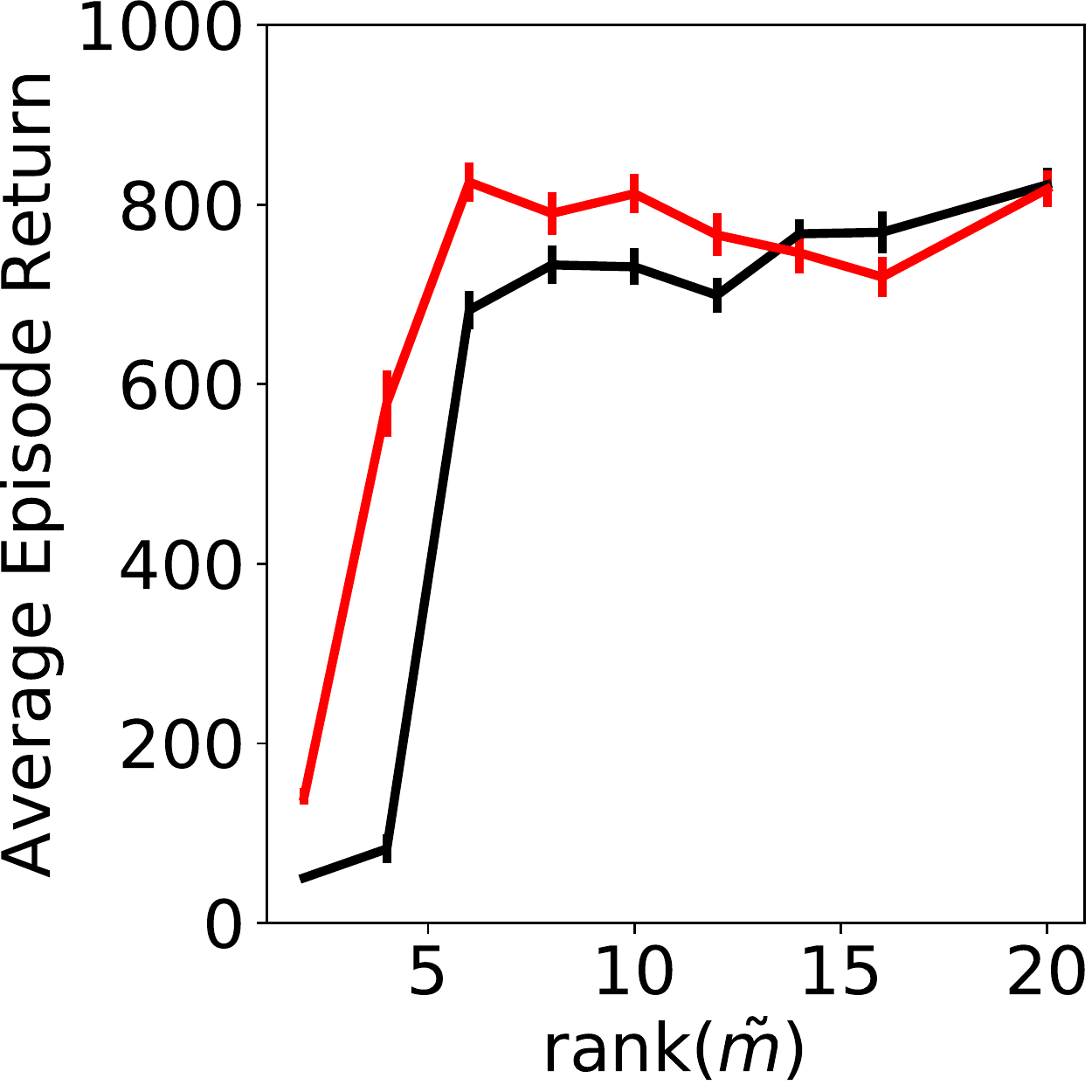}
}
\subfigure[Cart-pole (fixed \mt) \label{fig:cartpole_basis_transpose}]{
\includegraphics[scale=0.25]{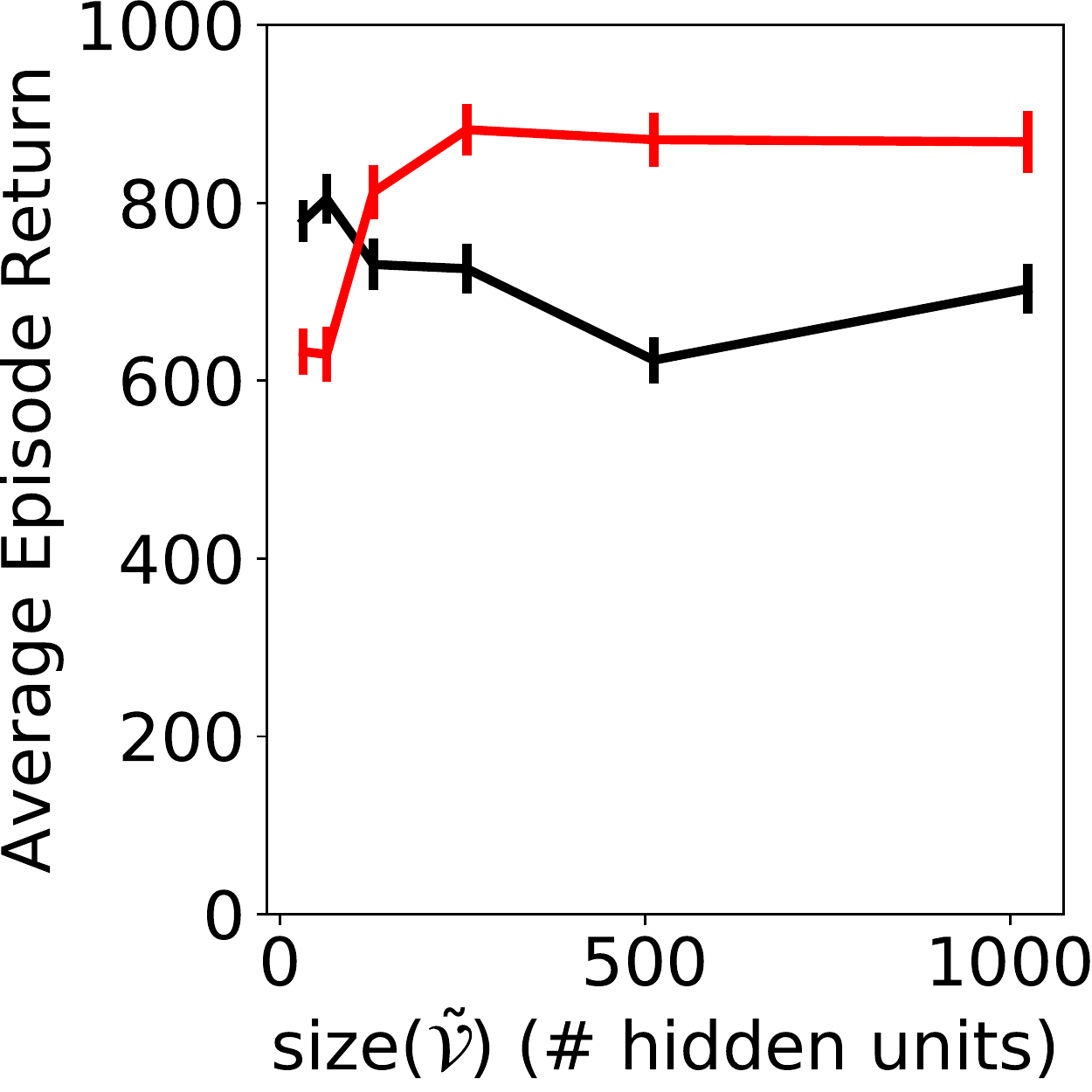}
}
\caption{ Results with \V\ composed of functions sampled from the agent's representational space \AV. \textbf{(a--b)} Functions in \V\ are the features of the linear function approximation (state aggregation), as per Remark~\ref{thm:linear_app}. Models \pt\ are rank-constrained transition matrices ({\sl cf.} Figure~\ref{fig:res_value}). \textbf{(c--d)}  Functions in \V\ are randomly-generated neural networks. Models \pt\ are neural networks with rank-constrained linear transformations between layers (Appendix~\ref{sec:details_exp}). Error bars are one standard deviation over $30$ runs. \label{fig:res_basis}
\vspace{-5mm}
}
\end{figure}

%% file: ve_in_practice.tex
Recently, there have been several successful empirical works that can potentially be understood as applications of the value-equivalence principle, like \ctp{silver2017predictron} \emph{Predictron}, \ctp{oh2017value} \emph{Value Prediction Networks}, \ctp{farquhar2018treeqn} \emph{TreeQN}, and \ctp{schrittwieser2019mastering} \emph{MuZero}. Specifically, the model-learning aspect of these prior methods can be understood, with some abuse of notation, as a value equivalence principle of the form $\T v = \Tt v$, where $\mathcal{T}$ is a Bellman operator applied with the true model $m^*$ and $\tilde{\mathcal{T}}$ is a Bellman operator applied with an approximate model $\tilde{m}$.  

There are many possible forms for the operators \T\ and \Tt. First, value equivalence can be applied to an uncontrolled Markov reward process; the resulting operator $\Tpi$ is analogous to having a single policy in $\Pi$. Second, it can be applied over $n$ steps, using a Bellman operator $\Tpi^n$ that rolls the model forward $n$ steps: $\Tpi^n [v](s) = \E_{\pi}[R_{t+1} + ... + \gamma^{n-1} R_{t+n} + \gamma^n v_{\pi}(S_{t+n}) | S_t = s]$, or a $\lambda$-weighted average $\Tpi^\lambda$~\cp{bertsekas96neuro-dynamic}. Third, a special case of the $n$-step operator $\T_{a_1...a_n}$ can be applied to an open-loop action sequence $\{a_1, ..., a_n\}$. Fourth, it can be applied to the Bellman optimality operator, $\T_{G_v}$, where $G_v$ is the ``greedy'' policy induced by $v$ defined as $G_v(a|s) = \ind\{a = \argmax_{a'} \mathbb{E}[R + \gamma v(S’) | s, a']\}$. This idea can also be extended to an $n$-step greedy search operator, $\T_{G_v^n} [v](s) = \max_{a_1, ..., a_n} \mathbb{E}[R_{t+1} + ... + \gamma^{n-1} R_{t+n} + \gamma^n v(S_{t+n}) | S_t = s, A_{t} = a_1, ..., A_{t+n} = a_n]$. Finally, instead of applying value equivalence over a fixed set of value functions \V, we can have a set $\V_t$ that varies over time---for example, $\V_t$ can be a singleton with an estimate of the value function of the current greedy policy.

The two operators \T\ and \Tt\ can also differ. For example, on the environment side we can use the optimal value function, which can be interpreted as $\T^\infty v = v^*$ \cite{tamar2016value,silver2017predictron}, while the approximate operator can be $\Tpit^\lambda$ \cite{silver2017predictron} or $\Tt_{G_{v_t}}^n$ \cite{tamar2016value}. We can also use approximate values $\Tt v \approx \T v'$ where $v' \approx v$, for example by applying $n$-step operators to approximate value functions, $\Tt^n v \approx \T^n v' = \T^n \T^k v = \T^{n+k} v$ \cite{oh2017value,schrittwieser2019mastering} or $\Tt^n v \approx \T^n v' = \T^n \Tt^k v$ \cite{farquhar2018treeqn}, or even to approximate policies, $\Tt^n v_a \approx \T^n v'_a$ where $v_a = \pi(a| s) \approx \pi'(a | s) = v'_a$ for all $a \in \A$  \cite{schrittwieser2019mastering}.
The table below characterises the type of value equivalence principle used in prior work. We conjecture that this captures the essential idea underlying each method for model-learning, acknowledging that we ignore many important details.

\begin{tabular}{llll}
\label{tab:related}
Algorithm & Operator $\tilde{\mathcal{T}}$ & Policies $\Pi$ & Functions $\mathcal{V}$ \\
\hline 
Predictron \cite{silver2017predictron} & $\tilde{\mathcal{T}}^\lambda v_t$ & None & Value functions for pseudo-rewards \\
VIN \cite{tamar2016value} & $\tilde{\mathcal{T}}_{G_{v_t}}^n v_t$ & $G_{v_t}$ & Value function \\
TreeQN \cite{farquhar2018treeqn} & $\tilde{\mathcal{T}}_{G_{v_t}^n}^n v_t$ & $G_{v_t}^n$ & Value function \\
VPN \cite{oh2017value} & $\tilde{\mathcal{T}}_{a_1..a_n}^n v_t$ & $\{a_1, ..., a_n\} \sim \pi_t$ & Value function \\
MuZero \cite{schrittwieser2019mastering} & $\tilde{\mathcal{T}}_{a_1...a_n}^n v_t$ & $\{a_1, ..., a_n\} \sim \pi_t$ & Distributional value bins, policy components \\
\end{tabular}

All of these methods, with the exception of VIN, sample the Bellman operator, rather than computing full expectations ({\sl c.f.}~(\ref{eq:ve_loss_emp})). In addition, all of the above methods jointly learn the state representation alongside a value-equivalent model based upon that representation. 
Only MuZero includes both many policies and many functions, which may be sufficient to approximately span the policy and function space required to plan in complex environments; this perhaps explains its stronger performance.

%% file: related_work.tex

\ca{farahmand2013value}'s \cp{farahmand2013value,farahmand2017value} \emph{value-aware model learning} (VAML) is based on a premise similar to ours. They study a robust variant of~(\ref{eq:ve_loss}) that considers the worst-case choice of $v \in \V$ and provide the gradient when the value-function approximation \vt\ is linear and the model \pt\ belongs to the exponential family. Later, \ct{farahmand2018iterative} also considered the case where the model is learned iteratively. Both versions of VAML come with finite sample-error upper bound guarantees~\cp{farahmand2013value,farahmand2017value,farahmand2018iterative}. More recently,  \ct{asadi2018equivalence} showed that minimizing the VAML objective is equivalent to minimizing the Wasserstein metric. \ct{abachi2020policy} applied the VAML principle to policy gradient methods. The theory of VAML is complementary to ours: we characterise the space of value-equivalent models, while VAML focuses on the solution and analysis of the induced optimization problem.

 \ct{joseph2013reinforcement} note that minimizing prediction error is not the same as maximizing the performance of the resulting policy, and propose an algorithm that optimizes the parameters of the model rather than the policy's. \ct{ayoub2020model} proposes an algorithm that keeps a set of models that are consistent with the most recent value function
estimate.
They derive regret bounds for the algorithm which suggest that value-targeted regression estimation is both sufficient and efficient for model-based RL.

More broadly, other notions of equivalence between MDPs have been proposed in the literature~\cp{dean1997model,poupart2002value,givan2003equivalence,ravindran2004approximate,ferns2004metrics,li2006towards,taylor2009bounding,poupart2013value,castro2020scalable,van2020plannable}. 
Any notion of equivalence over states can be recast as a form of state aggregation; in this case the functions mapping states to clusters can (and probably should) be used to enforce value equivalence (Remark~\ref{thm:linear_app}). But the principle of value equivalence is more general: it can be applied with function approximations other than state aggregation and can be used to exploit structure in the problem even when there is no clear notion of state abstraction (Appendix~\ref{sec:subspace_example}).

In this paper we have assumed that the agent has access to a well-defined notion of state $s \in \S$. More generally, the agent only receives observations from the environment and must construct its own state function---that is, a mapping from histories of observations to features representing states. This is an instantiation of the problem known as \emph{representation learning}~\cp{watter2015embed,igl18a,corneil18a,franccois2019combined,zhang2019solar,zhang2019learning,gelada19a,biza2020learning}.  An intriguing question which arises in this context is whether a model learned through value equivalence induces a space of ``compatible'' state representations, which would suggest that the loss~(\ref{eq:ve_loss}) could also be used for representation learning. This may be an interesting direction for future investigations.

%% file: conclusion.tex
We introduced the principle of value equivalence: two models are value equivalent with respect to a set of functions and a set of policies if they yield the same updates of the former on the latter. Value equivalence formalizes the notion that models should be tailored to their future use and provides a mechanism to incorporate such knowledge into the model learning process. It also unifies some important recent work in the literature, shedding light on their empirical success. Besides helping to explain some past initiatives, we believe the concept of value equivalence may also give rise to theoretical and algorithmic innovations that leverage the insights presented.


%% file: broader_impact.tex
The bulk of the research presented in this paper consists of foundational theoretical results about the learning of models for model-based reinforcement learning agents.
While applications of these agents can have social impacts depending upon their use, our results merely serve to illuminate desirable properties of models and facilitate the subsequent training of agents using them.
In short, this work is largely theoretical and does not present any foreseeable societal impact, except in the general concerns over progress in artificial intelligence.

%% file: appendix.tex
\vspace{7mm} 
\begin{center}
\vspace{7mm} 
\noindent\makebox[\textwidth]{\rule{\textwidth}{2.0pt}} \\
\vspace{3mm} 
{\bf {\LARGE  The Value-Equivalence Principle \vspace{2mm} \\ for Model-Based Reinforcement Learning} \\
\vspace{2mm} {\Large Supplementary Material} }
\vspace{5mm}
\noindent\makebox[\textwidth]{\rule{\textwidth}{1.0pt}}
\begin{minipage}[t]{0.4\textwidth}
\centering 
\textbf{Christopher Grimm} \\
Computer Science \& Engineering \\
University of Michigan \\
\texttt{crgrimm@umich.edu}
\end{minipage}%
\begin{minipage}[t]{0.6\textwidth}
\centering 
\textbf{Andr\'{e} Barreto, Satinder Singh, David Silver} \\
DeepMind \\
\texttt{\{andrebarreto,baveja,davidsilver\}@google.com}
\end{minipage}
\newline 
\end{center}

\maketitle

In this supplement we give details of our theoretical results and experiments that had to be left out of the main paper due to space constraints. We prove our theoretical results and provide a detailed description of our experimental procedure. Importantly, we present an illustrative example showing how value equivalence (VE) may lead to a better solution for a Markov decision process (MDP) than maximum-likelihood estimate (MLE). \emph{We show this to be true both in the exact case, when there exist a value-equivalent model in the model class considered, and in the approximate case, when such a model does not exist in the model class.}
Our appendix is organized as follows:
\begin{itemize}
\item Section~\ref{sec:proofs2} contains derivations of the properties and propositions presented in the main text.
\item Section~\ref{sec:subspace_example} contains a sequence of examples using a toy MDP that illustrate points made in the discussion surrounding Propositions~\ref{thm:mle}~and~\ref{thm:vip}. 
Moreover, we include an additional result which illustrates a situation in which approximate VE models can outperform the MLE model.
\item Section~\ref{sec:details_exp} provides a detailed outline of the pipeline used across our experiments in the main text. We also report several additional results that had to be left out of the main paper due to space constraints. 
\end{itemize}

The numbering of equations, figures and citations resume from what is used in the main paper.

\section{Appendix}

\subsection{Proofs of theoretical results and illustrative examples}
\label{sec:proofs}

\subsubsection{Proofs}
\label{sec:proofs2}

\proppertm*
\begin{proof} 
 This result directly follows from Definitions~\ref{def:ve} and~\ref{def:vem}.
\end{proof}

\proppindown*
\begin{proof}
$\M(\PS, \FS) \subseteq \MS(\PS, \FS) = \{ m^*\}$ (Property~\ref{thm:pertm}). 
\end{proof}

\proppertpiv*
\begin{proof}
 We will show the result by contradiction. Suppose there is a model $\mt \in \M(\Pi, \V)$ such that $\mt \notin \M(\Pi', \V')$. This means that there exists a $\pi \in \Pi'$ and a $v \in \V'$ for which $\Tpit v \ne \Tpi v$. But since $\Pi' \subseteq \Pi$ and $\V' \subseteq \V$, it must be the case that $\pi \in \Pi$ and $v \in \V$, which contradicts the claim that $\mt \in \M(\Pi, \V)$. 
\end{proof}

\propcover*
\begin{proof}
$m^* \in \M(\PS, \FS) \subseteq \M(\Pi, \V)$ (Property~\ref{thm:pertpiv}).
\end{proof}


\proplinearspan*
\begin{proof}
Let $\pi \in \pspan(\Pi) \cap \PS$. Based on~(\ref{eq:cspan}), we know that there exists an $\alpha_s \in \R^{|\Pi|}$ such that $\pi(\cdot | s) = \sum_i \alpha_{si} \pi_i(\cdot|s)$, where $\pi_i \in \Pi$. Thus, for $\mt \in \M(\Pi, \V)$, we can write
\begin{equation*}
\begin{array}{rl}
\Tpit[v](s) 
& = \E_{A \sim \pi(\cdot | s), S' \sim  \tilde{p}(\cdot | s, A)} \left[\tilde{r}(s,A)  + \gamma v(S')\right] \\
& = \int \pi(a| s) \E_{S' \sim  \tilde{p}(\cdot | s, a)} \left[\tilde{r}(s,a)  + \gamma v(S')\right] da \\
& = \int \sum_i \alpha_{si} \pi_i(a|s) \E_{S' \sim  \tilde{p}(\cdot | s, a)} \left[\tilde{r}(s,a)  + \gamma v(S')\right] da \\
& = \sum_i \alpha_{si} \int  \pi_i(a|s) \E_{S' \sim  \tilde{p}(\cdot | s, a)} \left[\tilde{r}(s,a)  + \gamma v(S')\right] da \\
& = \sum_i \alpha_{si} \E_{A \sim \pi_i(\cdot | s), S' \sim  \tilde{p}(\cdot | s, a)} \left[\tilde{r}(s,a)  + \gamma v(S')\right] \\
& = \sum_i \alpha_{si} \T^{\pi_i}[v](s). \\
\end{array}
\end{equation*}
Let  $v \in \span(\V)$. We know there is a $\beta \in \R^{|\V|}$ such that $v = \sum_i \beta_i \v_i$, with $v_i \in \V$.
\begin{equation*}
\begin{array}{rl}
\Tpit[v](s) 
& = \E_{A \sim \pi(\cdot | s), S' \sim  \tilde{p}(\cdot | s, A)} \left[\tilde{r}(s,A)  + \gamma \sum_i \beta_i v_i(S')\right] \\
& = \sum_i \beta_i \E_{A \sim \pi(\cdot | s), S' \sim  \tilde{p}(\cdot | s, A)} \left[\tilde{r}(s,A)  + \gamma v_i(S')\right] \\
& = \sum_i \beta_i \Tpit[v_i](s).  \\
\end{array}
\end{equation*}
\end{proof}

In order to prove Proposition~\ref{thm:pindown} we will need four lemmas which we state and prove below.

\begin{lemma}
\label{lemma:prod_rank}
For arbitrary matrices $\bs{A} \in \mathbb{R}^{k \times n}, \bs{C} \in \mathbb{R}^{m \times \ell}$, we can construct a vector-space
$\mathcal{B} = \{ \bs{B} \in \mathbb{R}^{n \times m} : \bs{A}\bs{B}\bs{C} = \boldsymbol{0} \}$
where $\boldsymbol{0}$ denotes a $k \times \ell$ matrix of zeros. 
It follows that 
\begin{equation}
\hdim[\mathcal{B}] = nm - \rank(\bs{A})\cdot \rank(\bs{C}). 
\end{equation}
\end{lemma}
\begin{proof}
We begin by converting the condition $\bs{A}\bs{B}\bs{C} = \boldsymbol{0}$ into a  matrix-vector product.
Let $\bs{a}^i$ and $\bs{c}^j$ denote the i'th row of $\bs{A}$ and j'th column of $\bs{C}$ respectively. 
Observe that $(\bs{A}\bs{B}\bs{C})_{ij} = \bs{a}^i \bs{B} \bs{c}^j = \sum_{x,y} \bs{a}^i_x \bs{c}^j_y \bs{B}_{xy}$, which implies that 
\begin{equation}
\bs{A}\bs{B}\bs{C} = \boldsymbol{0} \iff \sum_{x, y} \bs{a}^i_x \bs{c}^j_y \bs{B}_{xy} = 0 \ \ \forall i \in [k], j \in [\ell]
\end{equation}
where $[k]$ denotes $\{1, \ldots, k\}$.

For each $(i, j)$ pair, the above expression is suggestive of a dot-product between two $n \times m$ vectors: a combination of $\bs{a}^i$ and $\bs{c}^j$, and a ``flattened'' version of $\bs{B}$.
Define the former combination of vectors as $\bs{d}^{ij} = [\bs{a}^i_1\bs{c}^j_1, \bs{a}^i_1 \bs{c}^j_2, \cdots,  \bs{a}^i_n \bs{c}^j_m]^\top \in \R^{nm \times 1}$, and stack them as rows as: $\bs{D} = [ \bs{d}^{11},  \bs{d}^{12}, \cdots, \bs{d}^{nm} ]^\top \in \mathbb{R}^{k\ell \times nm}$.
To flatten $\bs{B}$, simply define $
\bs{b} = [\bs{B}_{11}, \bs{B}_{12}, \cdots,  \bs{B}_{nm}]^\top \in \R^{nm \times 1}$.

We now have that $\bs{A}\bs{B}\bs{C} = \boldsymbol{0} \iff  \bs{D}\bs{b} = \boldsymbol{0}$. Moreover, unravelling the matrices in $\mathcal{B}$ does not change the dimension of the space, thus:
\begin{equation}
\hdim[\mathcal{B}] = \hdim[\{ \bs{b} \in \mathcal{R}^{nm \times 1} : \bs{D}\bs{b} = \boldsymbol{0} \}] = nm - \rank(\bs{D})
\end{equation} 
where the last equality comes from a application of the rank-nullity theorem.

Finally notice that the construction of $\bs{d}^{ij}$ can be thought of as vertically stacking $n$ copies of $\bs{c}^j$ each scaled by a different entry in $\bs{a}^i$. We can also find scaled copies of $\bs{a}^i$ by $\bs{c}^j_k$ in $\bs{d}^{ij}$ by selecting indices from the combined vector at regular intervals of $m$:
$\bs{d}^{ij}_{k + (\ell-1)m} = \bs{c}^j_k \cdot \bs{a}^i_\ell$ for $\ell \in \{1, \ldots n\}$.

This means that scaled copies of both $\bs{a}^i $ and $\bs{c}^j$ can be found by selecting specific groups of  indices in $\bs{d}^{ij}$.
It follows that if $\bs{a}^1, \ldots, \bs{a}^n$ are linearly independent then so are $\bs{d}^{1j}, \ldots, \bs{d}^{nj}$ for any $j$.
And similarly, if $\bs{c}^1, \ldots, \bs{c}^m$ are linearly independent then so are $\bs{d}^{i1}, \ldots, \bs{d}^{im}$ for any $i$. 
Hence if $\bs{a}^1, \ldots \bs{a}^n$ and $\bs{c}^1, \ldots, \bs{c}^m$ are both linearly independent sets, then so is $\bs{d}^{11}, \bs{d}^{12}, \ldots, \bs{d}^{nm}$. 
Since these $\bs{a}^i$ and $\bs{c}^j$ vectors form the rows and columns of rank $n$ and $m$ matrices: $\bs{A}$ and $\bs{C}$, their corresponding sets of row and column vectors are linearly independent.
Thus we have that $\rank(\bs{D}) = \rank(\bs{A}) \cdot \rank(\bs{C})$, completing the proof.
\end{proof}

\begin{lemma}
\label{lemma:dim_translational_invariance}
For any $\bs{c}$ and $\mathcal{Y} + \bs{c} = \{y + \bs{c} : y \in \mathcal{Y} \}$ it follows that $\dim[\mathcal{Y} + \bs{c}] = \dim[\mathcal{Y}]$.
\end{lemma}
\begin{proof}
\begin{equation*}
\dim[\mathcal{Y}+\bs{c}] = \underset{(\mathcal{V}, \bs{c}') : \mathcal{Y} + (\bs{c} + \bs{c}') \subseteq \mathcal{W}}{\min} \hdim[\mathcal{W}] = \underset{(\mathcal{W}, \bs{c}') : \mathcal{Y} + \bs{c}' \subseteq \mathcal{W}}{\min} \hdim[\mathcal{W}] = \dim[\mathcal{Y}]
\end{equation*}
\end{proof}

\begin{lemma}
\label{lemma:vs_equality}
If $\mathcal{Y}$ is a vector-space then $\hdim[\mathcal{Y}] = \dim[\mathcal{Y}]$.
\end{lemma} 

\begin{proof}
Recall the definition of $\dim[\mathcal{Y}]$:
\begin{equation*}
\dim[\mathcal{Y}] = \underset{(\mathcal{W}, \bs{c}) : \mathcal{Y} + \bs{c} \subseteq \mathcal{W}}{\min} \hdim[\mathcal{W}]
\end{equation*}
where $\mathcal{W}$ is a vector-space.
By choosing $\mathcal{W} = \mathcal{Y}$ and $\bs{c} = \bs{0}$ we see that $\dim[\mathcal{Y}] \leq \hdim[\mathcal{Y}]$. 

Suppose then that $\dim[\mathcal{Y}] < \hdim[\mathcal{Y}]$. 
This implies that there is a vector space $\mathcal{W}$ and offset $\bs{c}$ with $d = \hdim[\mathcal{W}] < \hdim[\mathcal{Y}]$ and $\mathcal{Y} + \bs{c} \subseteq \mathcal{W}$. 
This means that for every $\bs{y} \in \mathcal{Y}$: $\bs{y} + \bs{c} = \sum_{i=1}^d \alpha_i^{\bs{y}} \bs{w}_i$ for some $\alpha_{1:d}^y$ where $\bs{w}_{1:d}$ are a basis of $\mathcal{W}$. Since $\mathcal{Y}$ is a vector space it must contain the $\boldsymbol{0}$ vector, hence $\bs{c} = \sum_{i=1}^d \alpha_i^{\boldsymbol{0}} \bs{w}_i$. 
Accordingly any $\bs{y} \in \mathcal{Y}$ can be written as $\bs{y} = \sum_{i=1}^d (\alpha_i^{\bs{y}} - \alpha_i^{\boldsymbol{0}}) \bs{w}_i$. However, this is a contradiction since $\hdim[\mathcal{W}] < \hdim[\mathcal{Y}]$.
Hence $\dim[\mathcal{Y}] = \hdim[\mathcal{Y}]$.
\end{proof}

\begin{lemma}
\label{lemma:dim_subset}
If $\mathcal{X} \subseteq \mathcal{Y}$ then $\dim[\mathcal{X}] \leq \dim[\mathcal{Y}]$. 
\end{lemma}
\begin{proof}
If $\mathcal{X} \subseteq \mathcal{Y}$ then for any $\bs{c}$,  $\mathcal{X} + \bs{c} \subseteq \mathcal{Y} + \bs{c}$.
Because of the above, for any vector-space $\mathcal{W}$: $\mathcal{W} \supseteq \mathcal{Y} + \bs{c} \implies \mathcal{W} \supseteq \mathcal{X} + \bs{c}$, hence: 
$\{(\mathcal{W}, c) : \mathcal{X} + \bs{c} \subset \mathcal{W} \} \supseteq \{ (\mathcal{W}, \bs{c}) : \mathcal{Y} + \bs{c} \subset \mathcal{W} \}$. 
Notice that this last set-relation corresponds the set of vector-spaces that $\dim[\cdot]$ is minimizing over for $\mathcal{X}$ and $\mathcal{Y}$ respectively. Hence $\dim[\mathcal{X}] \leq \dim[\mathcal{Y}]$.
\end{proof}

\proplineargrowth*
\begin{proof}
First note that if $\ppi_i \notin \pspan(\Pi \setminus \{ \ppi_i \})$ then $\ppi_i \notin \span(\Pi \setminus \{ \ppi_i \})$.
Hence, pointwise linear independence implies linear independence.

Since $|\S|$ and $|\A|$ are finite, we can assume that $\A = \{1, \ldots, |\A||\}$ and $\S = \{1, \ldots, |\S|\}$.
For any transition probability kernel $\tilde{p}(s' | s, a)$ we can construct matrix $\tilde{\boldsymbol{P}} \in \R^{|\S||\A| \times |\S|}$ with $\tilde{\P}_{(a-1)|\S| + s,s'} = \tilde{p}(s'|s,a)$.
Denote the constructed matrix corresponding to the true dynamics as $\boldsymbol{P}$.
For any $\boldsymbol{\pi}_i$ we can construct a matrix $\boldsymbol{\Pi}_i \in \R^{|S| \times |S||A|}$ with $(\boldsymbol{\Pi}_i)_{s, (a - 1)|\S| + s} = \pi_i(a | s)$. 
Vertically stack these $m$ \  $\boldsymbol{\Pi}_i$ matrices to construct $\boldsymbol{\Pi} \in \mathbb{R}^{m|S| \times |S||A|}$.
Additionally we construct $\boldsymbol{V} \in \R^{|S| \times k}$ with $\boldsymbol{V}_{j, \ell} = (\boldsymbol{\v_\ell})_j$.
Note that $\mathbb{P}(\Pi, \V) = \{ \tilde{\boldsymbol{P}} \in \mathbb{P} : \boldsymbol{\Pi}(\tilde{\boldsymbol{P}} - \boldsymbol{P})\boldsymbol{V} = \boldsymbol{0} \}$.
Define the sets 
$\mathcal{X} = \{ \boldsymbol{X} \in \mathbb{R}^{|S||A| \times |S|} : \boldsymbol{\P} \boldsymbol{X} \boldsymbol{V} = \boldsymbol{0} \}$
and
$\mathcal{Y} = \{ \tilde{\boldsymbol{P}} \in \mathbb{R}^{|S||A| \times |S|} : \boldsymbol{\Pi}(\tilde{\boldsymbol{P}} - \boldsymbol{P})\boldsymbol{V} = \boldsymbol{0} \}$.

Note the following three facts:
\begin{enumerate}
\item $\dim[\mathcal{X}] = \dim[\mathcal{Y}]$ since our notion of dimension is translation-invariant (Lemma \ref{lemma:dim_translational_invariance}).
\item $\dim[\mathcal{X}] = \hdim[\mathcal{X}]$ since $\mathcal{X}$ is a vector-space (Lemma \ref{lemma:vs_equality}). 
\item $\mathbb{P}(\Pi, \V) \subseteq \mathcal{Y}$ which implies that $\dim[\mathbb{P}(\Pi, \V)] \leq \dim[\mathcal{Y}]$ (Lemma \ref{lemma:dim_subset}).
\end{enumerate}

Taken together this gives us that
\begin{equation*}
\dim[\mathbb{P}(\Pi, \V)] \leq \dim[\mathcal{Y}] = \hdim[\mathcal{X}].
\end{equation*}

We can now apply Lemma \ref{lemma:prod_rank} to obtain $\dim[\mathcal{X}] = |S|^2|A| - k \cdot \rank(\boldsymbol{\Pi})$. Notice that $\rank(\boldsymbol{\Pi}) = \min \{ |S||A|, m|S| \}$. Thus $\dim[\mathbb{P}(\Pi, \mathcal{V})] \leq  |S|(|S||A| - mk)$ as needed.
\end{proof}

\propmle* 
\begin{proof}
 Suppose we are trying to estimate a transition matrix $\P \in \R^{n \times n}$ and choose to use one parameter $\theta_i \in \R$ per row. Specifically, we parametrize the distribution on the $i$-th row as 
\begin{equation*}
\pt_{ii} = \theta_i \text{ and } \pt_{ij} = (1-\theta_i)/(n-1), \text{ for } i \ne j, \text{ with } \theta_i \in [0,1],
\end{equation*}
where $p_{ij} = p(s_j | s_i)$.
We can then write the expected likelihood function for $\vtheta \in \R^n$ as 
\begin{equation*}
\begin{array}{cl}
m(\vtheta) 
& = \sum_{i} \left[ p_{ii} \ln \theta_i + \sum_{j \ne i} p_{ij} \ln (1-\theta_i) - \sum_{j \ne i} p_{ij} \ln (n-1) \right] \\ 
& = \sum_{i} \left[ p_{ii} \ln \theta_i + (1 - p_{ii}) \ln (1-\theta_i) - (1 - p_{ii} ) \ln (n-1) \right], \\
\end{array}
\end{equation*}
which leads to the likelihood equation
\begin{equation*}
\begin{array}{cl}
0 = \dfrac{\partial m(\vtheta)}{\theta_i}  = \dfrac{p_{ii}}{\theta_i} + \dfrac{1 - p_{ii}}{\theta_i - 1} = \dfrac{p_{ii}(\theta_i-1) + (1 - p_{ii}) \theta_i}{\theta_i (\theta_i - 1)}
= \dfrac{\theta _i - p_{ii}}{\theta_i (\theta_i - 1)}. \\
\end{array}
\end{equation*}
The MLE solution is thus to have $\theta_i = p_{ii}$ for $i = 1, 2,..., n$. This means that the solution provided by MLE will not be exact if and only if 
\begin{equation}
\label{eq:mle_inexact}
p_{ij} \ne p_{ik} \text{ for any } (i,j,k) \text{ such that } i \ne j \ne k.
\end{equation}
 Now, suppose we have $\V = \{ v \}$ with $v_i =1$ for some $i$ and $v_j = 0$ for $j \ne i$. In this case it is possible to get an exact value-equivalent solution---that is, $\P\v = \Pt \v$--- by making $\theta_i = p_{ii}$ and $\theta_j = 1 - (n-1) p_{ii}$ for $j \ne i$, regardless of whether~(\ref{eq:mle_inexact}) is true or not. 
\end{proof}

\propvip*
\begin{proof}
Denote the Bellman operator under a policy that always selects action $a$ as $\mathcal{T}_a$, 
the greedy Bellman operator as $\mathcal{T}v = \max_{a} \mathcal{T}_a v$ and the Bellman operator under a policy $\pi$ as $\mathcal{T}_\pi$, as before. Let $\mathcal{T}^{(n)}v$ represent $n$ successive applications of operator $\mathcal{T}$ on value $v$. 

Note that for any $v \in \FS$ we can construct a $\pi_v(s) = \text{argmax}_a (\mathcal{T}_a v)(s)$ such that $\mathcal{T}v = \max_{a}\mathcal{T}_a v =  \mathcal{T}_{\pi_v} v$. 
This implies that the greedy Bellman operator is included in the assumption of our proposition:
\begin{equation}\label{eq:greedy_op}
v \in \V' \implies \mathcal{T}v \in \V'.
\end{equation}

We now begin by showing that:
\begin{equation}
\label{eq:main_imp}
\mathcal{T}^{(n)}v = \tilde{\mathcal{T}}^{(n)} v \in \V' \implies \mathcal{T}^{(n+1)}v = \tilde{\mathcal{T}}^{(n+1)} v \in \V'
\end{equation}
for any $v \in \FS$ and any $n > 0$.
Assume that $\mathcal{T}^{(n)} v = \tilde{\mathcal{T}}^{(n)} v \in \V'$. 
Since $\mathcal{T}^{(n)} v \in \V'$ and $\V' = \span(\V)$, we can use use value equivalence to obtain:
\begin{equation*}
\mathcal{T}_a \mathcal{T}^{(n)} v = \tilde{\mathcal{T}}_a \mathcal{T}^{(n)} v.
\end{equation*}
for any $a \in \A$. Next, since $\mathcal{T}^{(n)}v = \tilde{\mathcal{T}}^{(n)}v$ we can write:
\begin{equation}
\label{eq:eq_ts}
\mathcal{T}_a \mathcal{T}^{(n)} v =  \tilde{\mathcal{T}}_a \tilde{\mathcal{T}}^{(n)} v. 
\end{equation}

Since~(\ref{eq:eq_ts}) holds for any $a \in \A$, we can write:
\begin{equation*}
\mathcal{T}^{(n+1)} v = \max_{a} \mathcal{T}_a \mathcal{T}^{(n)} v = \max_a \tilde{\mathcal{T}}_a \tilde{\mathcal{T}}^{(n)} v = \tilde{\mathcal{T}}^{(n+1)} v. 
\end{equation*}
We know from (\ref{eq:greedy_op}) that the fact that  $\mathcal{T}^{(n)}v \in \V'$ implies that $\mathcal{T}^{(n+1)}v \in \V'$.
Thus we have shown that~(\ref{eq:main_imp}) is true.

Finally, by choosing $v \in \V'$ and using analogous reasoning as as above, we can show that $\mathcal{T}_a v = \tilde{\mathcal{T}}_a v$ and $\mathcal{T}v = \max_a \mathcal{T}_a v = \max_a \tilde{\mathcal{T}}_a v = \tilde{\mathcal{T}} v$, and since $v \in \V'$, $\tilde{\mathcal{T}}v = \mathcal{T}v \in \V'$.
Thus $\mathcal{T}^{(n)}v = \mathcal{T}^{(n)}v$ for all $n \in \mathbb{N}$.
This is sufficient to conclude that
\begin{equation*}
\tilde{v}^* = \lim_{n \to \infty} \tilde{\mathcal{T}}^{(n)} v' = \lim_{n \to \infty} \mathcal{T}^{(n)} v' = v^*,
\end{equation*}
as needed.

\end{proof}

\subsubsection{Examples with a simple MDP}
\label{sec:subspace_example}
\begin{wrapfigure}{r}{0.4\textwidth}
\centering
\includegraphics[scale=0.7]{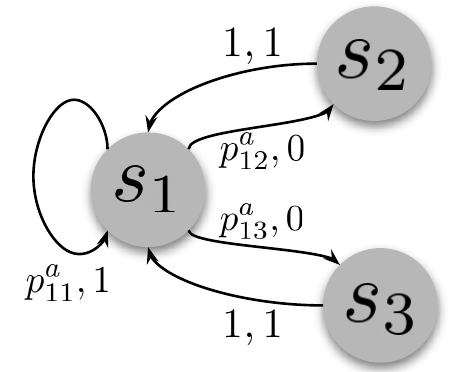}
\caption{} 
\label{fig:example_mdp}
\end{wrapfigure}
Consider the 3 state MDP with states $s_1, s_2, s_3$ and actions $\A = \{a_1, a_2\}$.
Transitioning to state $s_1$ always incurs a reward of $1$, taking any action in states $s_2$ and $s_3$ always results in transitioning to $s_1$ and taking action $a \in \A$ from $s_1$ transitions among the other states according to action-dependent distribution $(p_{11}^a, p_{12}^a, p_{13}^a)$.  
This MDP is depicted in Figure~\ref{fig:example_mdp}. 
We now use this MDP to illustrate several points made in the main text.

\paragraph{Closure under Bellman updates}
We now address the discussion surrounding Proposition~\ref{thm:vip} in the main text.
Consider a the following two-dimensional subspace of value functions $\mathcal{R} = \{ [x, y, y]^\top : x, y \in \R \}$.
We now show that, for the MDP described above, $\mathcal{R}$ exhibits closure under arbitrary Bellman updates.

For an arbitrary policy $\pi : \S \mapsto \mathcal{P}(\A)$ the Bellman update for a value function $\bs{v} \in \R^3$ is given by $\mathcal{T}^\pi \bs{v} = \bs{R}^\pi + \gamma \bs{P}^\pi \bs{v}$ where
\begin{equation*}
\bs{R}^\pi = \begin{bmatrix}
\sum_{a \in \A} \pi(a | s_1) p_{11}^a \\ 1 \\ 1
\end{bmatrix},\
\bs{P}^\pi = \begin{bmatrix}
\sum_{a \in \A} \pi(a | s_1) p_{11}^a &
\sum_{a \in \A} \pi(a | s_2) p_{12}^a & 
\sum_{a \in \A} \pi(a | s_3) p_{13}^a \\ 
1 & 0 & 0 \\ 
1 & 0 & 0
\end{bmatrix}
\end{equation*}

Suppose $\bs{v} \in \mathcal{R}$, then $\bs{v} = [a, b, b]^\top$ for some $a, b \in \R$. 
Notice that for such a value function the following holds:
\begin{equation*}
\mathcal{T}^\pi v = 
\begin{bmatrix}
\bs{R}^\pi_{1} + \gamma [a \bs{P}^\pi_{11} + b (1 - \bs{P}^\pi_{11})] \\
1 + \gamma a \\ 
1 + \gamma a
\end{bmatrix} \in \mathcal{R},
\end{equation*}
thus we have illustrated that the two-dimensional subspace $\mathcal{R}$ is closed under arbitrary Bellman updates in our 3 state MDP. This means that, once a sequence $\bs{v}_1, \bs{v}_2 = \T_{\pi}\bs{v}_1, \bs{v}_3 = \T_{\pi'} \bs{v}_2...$ reaches a $\bs{v}_i \in \mathcal{R}$, it stays in $\mathcal{R}$. We can then exploit this property finding value-equivalent models with respect to $\mathcal{R}$, as we show next.

\paragraph{A model class for which \emph{exact} VE outperforms MLE}
We now provide an example of the scenario discussed around Proposition~\ref{thm:mle} in the main text by examining the setting where a model, from a restricted class, must be learned to approximate the dynamics of our MDP. 
We restrict our model class by requiring that for each action $a \in \A$ we represent $(p_{11}^a, p_{12}^a, p_{13}^a)$ as $((1-\theta^a)/2, \theta^a, (1-\theta^a)/2)$.
Before continuing we note a few properties of value functions of our MDP.
Notice that for any $\bs{v}^\pi$ we can write:
\begin{equation*}
\begin{aligned}
&\bs{v}^\pi_1 = \sum_{a \in \A} \pi(a | s_1) [ p_{11}^a (1 + \gamma \bs{v}_1^\pi ) + (1 - p_{11}^a)(\gamma^2 \bs{v}_1^\pi)], \\
&\bs{v}^\pi_2 = 1 + \gamma \bs{v}^\pi_1, \\ 
&\bs{v}^\pi_3 = 1 + \gamma \bs{v}^\pi_1,
\end{aligned}
\end{equation*}
which illustrates that $v^\pi$ \textit{exclusively depends} on the value of $\bs{P}_{11}^\pi \defi \sum_{a \in \A} \pi(a | s_1) p_{11}^a$.

First we consider the MLE solution to this problem: it can be easily shown (see the proof of Proposition~\ref{thm:mle}) that, for the model class defined above, $\theta^a = p_{12}^a$ for all $a \in \A$ maximizes the likelihood.
However notice that this implies that our approximation of $p_{11}^a$ equals $(1 - p_{12}^a) / 2$ which is clearly not true in general.
Thus, there are settings of $(p_{11}^a, p_{12}^a, p_{13}^a)$ and policies for which the value function produced by MLE, $\tilde{\bs{v}}^\pi$, is not equivalent to the true value function $\bs{v}^\pi$. 

Next we consider learning a value-equivalent model with the same restricted model class. 
Suppose we wish our model to be value equivalent to value $\bs{v} = [1, 0, 0]^\top$ and all policies. 

Note that any VE model with respect to $\V = \{ \bs{v} \}$: $\{\tilde{\bs{P}}^a\}_{a \in \A}$, must satisfy $\tilde{\bs{P}}^a v = \bs{P}^a v$.
By requiring value equivalence with just $v$ we have:
\begin{equation*}
\tilde{\bs{P}}^a \bs{v} = 
\begin{bmatrix}
\tilde{p}^a_{11} \\ \tilde{p}^a_{21} \\ \tilde{p}^a_{31}
\end{bmatrix} = 
\begin{bmatrix}
p_{11}^a \\ 1 \\ 1
\end{bmatrix}
= 
\bs{P}^a \bs{v}
\end{equation*}
which implies that $\tilde{p}_{11}^a = p_{11}$, $\tilde{p}_{21}^a = \tilde{p}_{31}^a = 1$ and $\tilde{p}_{22}^a = \tilde{p}_{23}^a = \tilde{p}_{32}^a = \tilde{p}_{33}^a = 0$ for all $a \in \A$.

Taking these constraints together restricts the class of VE models to those of the form:
\begin{equation*}
\tilde{\bs{P}} =
\begin{bmatrix}
p_{11}^a & \tilde{p}_{12}^a & \tilde{p}_{13}^a \\
1 & 0 & 0 \\
1 & 0 & 0
\end{bmatrix}
\end{equation*}
where $\tilde{p}_{1i}^a$ are ``free variables'' for all $i = 2, 3$ and $a \in \A$. 

Notice that when $p_{11}^a \leq 0.5$ for all $a \in \A$, we can find a value equivalent model by setting: $(1 - \theta^a) / 2 = p_{11}^a$. 
This means that the values produced by these value equivalent models exactly match those of the environment: $\tilde{\bs{v}}^\pi = \bs{v}^\pi$ for all $\pi$ (and thus the solution of this model also coincides with the optimal value function, $\tilde{\bs{v}}^* = \bs{v}^*$).

\paragraph{A model class for which \emph{approximate} VE outperforms MLE}
In the previous example we showed that it is possible to have an MDP and a restricted model class such that VE models are able to perfectly estimate $\bs{v}^*$ while MLE models fail to do so.
Notice that in this example a value equivalent model \textit{actually existed}, which is not guaranteed in general.
We now show a related example where, in spite of an exactly value equivalent model not existing, an agent trained using an \textit{approximate} value equivalent model will outperform its MLE counterpart.

We use our example MDP from before, shown in Figure~\ref{fig:example_mdp}, and denote its actions $\A = \{a, b\}$ for later notational convenience.
We set our environment's transition dynamics accordingly: $p^a \defi (p_{11}^a, p_{12}^a, p_{13}^a) = (0.6, 0.4, 0.0)$ and $p^b \defi (p_{11}^b, p_{12}^b, p_{13}^b) = (0.4, 0.2, 0.4)$.
We also use the same model class as above: $(\tilde{p}_{11}^i, \tilde{p}_{12}^i, \tilde{p}_{13}^i) = (0.5(1-\theta^i), \theta^i, 0.5(1-\theta^i))$ for each $i \in \A$, being mindful of the boundary conditions $\theta^i \in [0, 1]$. 

As we saw in the previous example, the MLE estimator for this problem will produce the following approximations:
$p^a_{\text{MLE}} = (0.3, 0.4, 0.3)$, $p^b_{\text{MLE}} = (0.4, 0.2, 0.4)$.

We now consider what an approximate VE model will produce using the same value as before: $\bs{v} = [1, 0, 0]^\top$ and all policies.
Recall that we're optimizing the following loss:
\begin{equation*}
\begin{aligned}
\sum_{j \in \{a, b\}} \sum_{i=1}^3 ((\tilde{\bs{P}}^jv)_i - (\bs{P}^jv)_i))^2 &= \sum_{j \in \{ a, b\}} (\tilde{p}^j_{11} - p^j_{11})^2 + ((\tilde{p}^j_{12} + \tilde{p}^j_{13}) - (p^j_{12} + p^j_{13}))^2 \\ 
&= \sum_{j \in \{ a, b\}} (\tilde{p}^j_{11} - p^j_{11})^2 + ((1 - \tilde{p}^j_{11}) - (1 - p^j_{11}))^2 \\ 
&= \sum_{j \in \{ a, b\}} 2(\tilde{p}^j_{11} - p^j_{11})^2 \\
&= 2(\tilde{p}^a_{11} - p^a_{11})^2 + 2(\tilde{p}^b_{11} - p^b_{11})^2.
\end{aligned}
\end{equation*}

The form of this loss indicates that VE will attempt to minimize the MSE of $\tilde{p}^a_{11}$ and $\tilde{p}^b_{11}$ separately. 
Notice that for action $a$, we cannot perfectly estimate $p_{11}$ due to the boundary conditions on $\theta^a$.
However, VE will still find the closest possible $\tilde{p}_{11}$ that respects the boundary condition, giving:
$\tilde{p}^a_\text{VE} = (0.5, 0.0, 0.5)$, $\tilde{p}^b_\text{VE} = (0.4, 0.2, 0.4)$. 

We now display these models together in the following table:
\begin{center}
\begin{tabular}{|c | c | c | c || c | c | c |}
\hline 
& $\tilde{p}^a_{11}$ & $\tilde{p}^a_{12}$ & $\tilde{p}^a_{13}$ & $\tilde{p}^b_{11}$ & $\tilde{p}^b_{12}$ & $\tilde{p}^b_{13}$  \\
\hline
MDP & 0.6 & 0.4 & 0.0 & 0.4 & 0.2 & 0.4 \\
MLE & 0.3 & 0.4 & 0.3 & 0.4 & 0.2 & 0.4 \\ 
VE & 0.5 & 0.0 & 0.5 & 0.4 & 0.2 & 0.4 \\
\hline
\end{tabular}
\end{center}

Notice that when optimally planning on this MDP, an agent can obtain the most reward by transitioning from $s_1$ to $s_1$ as often as possible. 
The agent can do this taking the action among $\{a, b\}$ that is mostly likely to induce a self-transition each time it is at $s_1$. 
In the true environment and the VE model this action is $a$.
However, notice that the MLE model would instead prefer the sub-optimal action $b$, since $(\tilde{p}^b_{\text{MLE}})_{11} > (\tilde{p}^a_{\text{MLE}})_{11}$.

This is a concrete example where VE outperforms MLE even though there is no value-equivalent models in the model class considered (that is, VE can be enforced only approximately).

\subsection{Experimental details}
\label{sec:details_exp}

\begin{figure}[H]
\centering
\subfigure[Catch \label{env:catch}]{
\includegraphics[scale=0.25]{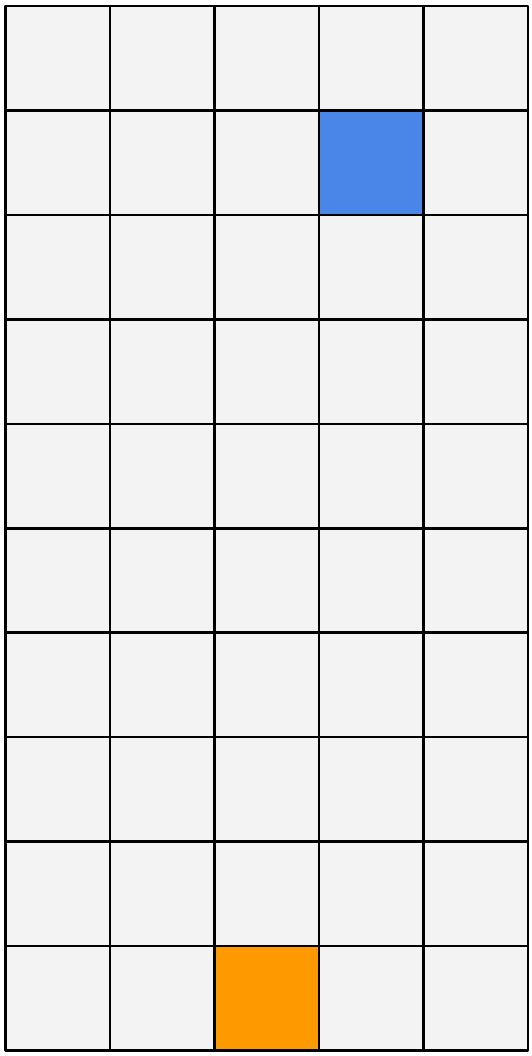}
 }
 \hspace{10mm}
\subfigure[Four Rooms \label{env:four_rooms}]{
\includegraphics[scale=0.25]{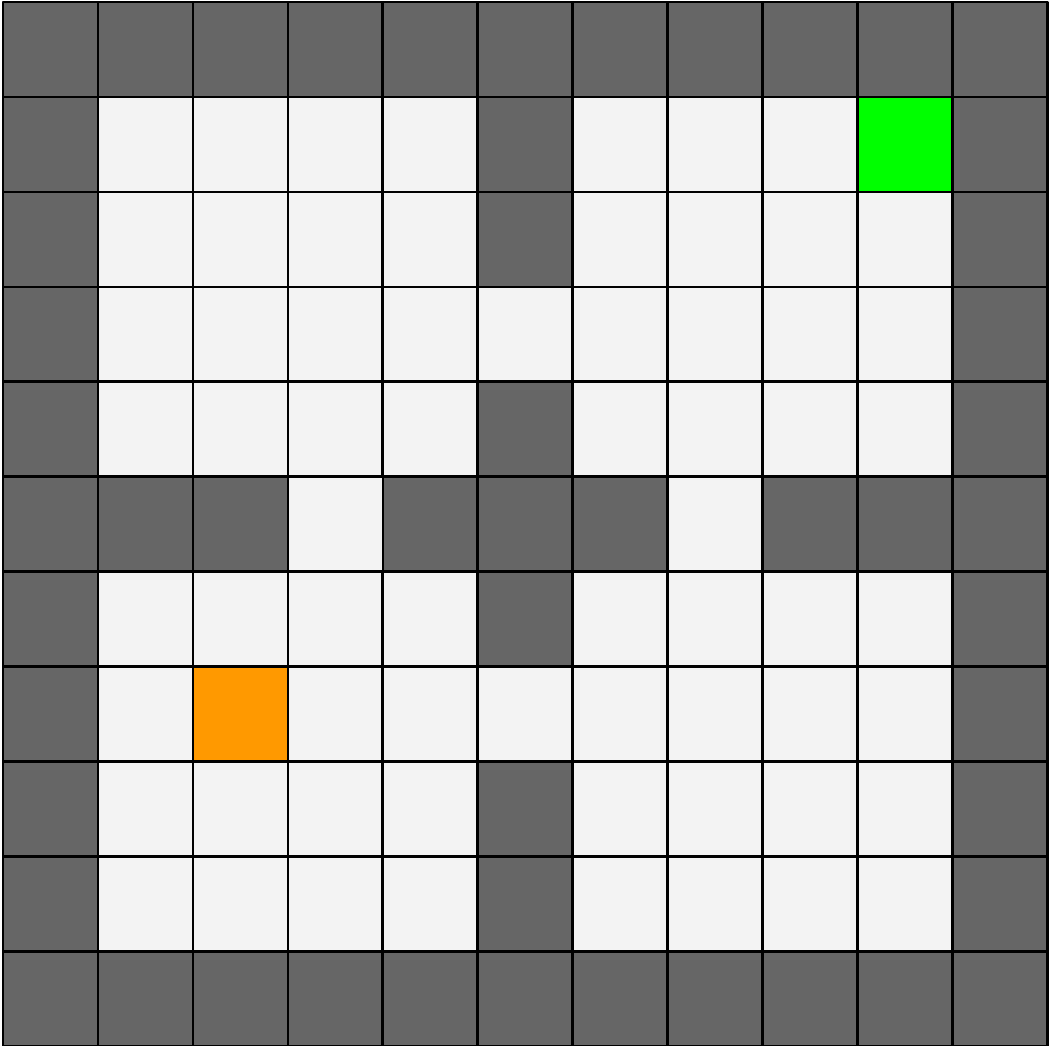}
}
\hspace{10mm}
\subfigure[Cart-pole \label{env:cartpole}]{
\includegraphics[scale=0.32]{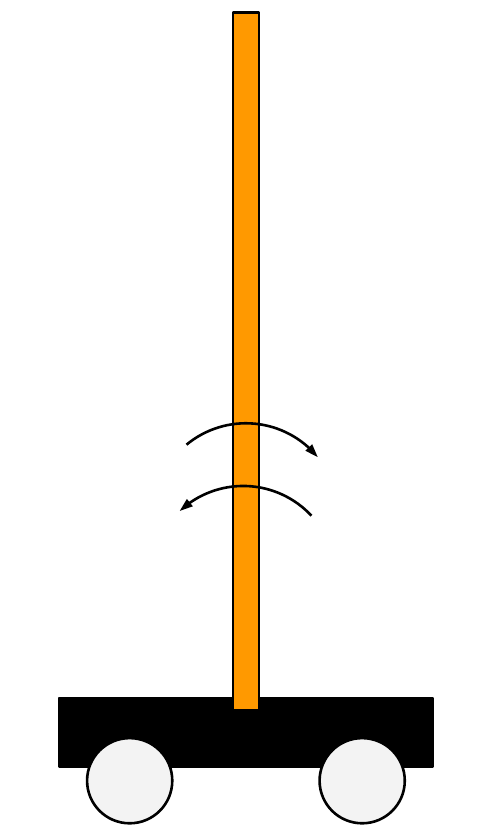}
}
\caption{
\textbf{(a) Catch:} the agent has three actions corresponding to moving a paddle (orange) left, right and staying in place. 
Upon initialization, a ball (blue) is placed at a random square at the top of the environment and at each step it descends by one unit. 
Upon reaching the bottom of the environment the ball is returned to a random square at the top.
The agent receives a reward of $1.0$ if it moves its paddle and intercepts the ball.
\textbf{(b) Four Rooms:} the agent (orange) has four actions corresponding to up, down, left and right movement.
When the agent takes an action, it moves in its intended direction with 90\% of the time and in an random other direction otherwise.
There is a rewarding square in the right top corner (green). If the agent transitions into this square it receives a reward of $1.0$. 
\textbf{(c) Cart-pole:} In Cart-pole, the agent may choose between three actions: pushing the cart to the left, right or not pushing the cart.
There is a pole balanced on top of the cart that is at risk of tipping over.
The agent is incentivized to keep the pole up-right through a reward of $\cos(\theta)$ at each step where $\theta$ is the angle of the pole ($\theta = 0$ implies the pole is perfectly up-right).
If the pole's height drops below a threshold, the episode terminates and the agent receives a reward of $0.0$.
The cart itself is resting on a table; if it falls off the table, the episode similarly terminates with a reward of $0.0$.
}
\label{fig:environments}
\end{figure}

\subsubsection{Environment description}

The environments used in our experiments are described in depth in Figure~\ref{fig:environments}. In both Catch and Four Rooms a tabular representation is employed in which each of the environment's finitely many states (250 and 68, respectively) is represented by an index. 
In Cart-pole we have a continuous state space $\S \subset \R^5$ (so $|\S| = \infty$). Each state $s \in \R^5$ consists of the cart position, cart velocity, sine / cosine of pole angle, and pole's angular velocity.

\subsubsection{Experimental pipeline}

As mentioned in the main text, a common experimental pipeline is used across all of our results, with slight variations depending upon the experiment type and environment. This pipeline is described at a high-level below:

\begin{enumerate}[(i)]
\item \label{it:data_collection} \textbf{Data collection:} Data is collected using a policy which selects actions uniformly at random. 
\item \label{it:model_training} \textbf{Model training:} The collected data is used to train a model.
\item \label{it:policy_construction}\textbf{Policy construction:} The model is used to produce a policy. 
\item \label{it:policy_evaluation}\textbf{Policy evaluation:} The policy is evaluated to assess the quality of the model.
\end{enumerate}
We now discuss steps~\ref{it:model_training}, \ref{it:policy_construction} and
\ref{it:policy_evaluation} in detail.

\paragraph{\ref{it:model_training} Model training}

All of our experiments involve restricting the capacity of the class of models that the agent can represent: $\M$.
In general we restrict the rank of the models in $\M$, but, depending upon the nature of the model, this restriction is carried in different ways.
\begin{enumerate}
\item \label{it:rest_tab} \textbf{Tabular models:} On domains with $|S| < \infty$, we employ tabular models.
In what follows, $n \times m$ matrices referred to as ``row-stochastic'' are ensured to be as such by the following parameterization:
\begin{enumerate}[label=(\alph*)]
\item A matrix $\F \in \R^{n \times m}$ is sampled with entries $\F_{ij} \sim \text{Uniform}([-1, 1])$. 
\item A new matrix $\P_F$ is produced by applying row-wise softmax operations with temperature $\tau = 1$ to \F: 
\begin{equation*}
(\P_F)_{ij} = \frac{\text{exp}(\F_{ij})}{\sum_{k} \text{exp}(\F_{ik})}.
\end{equation*}
Here, $\F$ can be thought of as the parameters of $\P_F$, which often will suppress as $\Pt$ for clarity. 
\end{enumerate}
That is, a model is represented by $|A|$\   $|S|\times|S|$ row-stochastic matrices: $\Pt^1, \ldots, \Pt^{|A|}$. 
We ensure that each of these matrices has rank $k$ by factoring it as follows: $\Pt^a = \DD^a \KK^a$ where $\DD^a \in \R^{|S| \times k}$, $\KK^a \in \R^{k \times |S|}$ and both are row-stochastic as well.
\item \textbf{Neural network models:} On domains with $|S| = \infty$ we instead use a neural network parameterized by $\theta$: $f_\theta : (\S, \A) \mapsto  (\S, \R)$. 
$f_\theta$ takes a state and action as input and outputs an approximation of the expected next state and next reward.
As an analogue to the rank restriction applied in the tabular case, we restrict the rank of weight matrices in all fully-connected layers in $f_\theta$.
Denote a fully-connected layer in $f_\theta$ as $L(x) = \sigma(W x + b)$ where $\sigma(\cdot)$ is an activation function, $W$ is a weight matrix and $b$ is a bias term. 
We restrict $f_\theta$ by replacing each $L(x)$ with $L_k(x) = \sigma((DK)x + b)$ where $D, K \in \R^{|S| \times k}, \R^{k \times |S|}$.  
\end{enumerate}

The models with the restrictions above are trained based on data collected by a policy that selects actions uniformly at random. With a small abuse of notation, denote the collected data as $\D = (s_i, a_i, r_i, s'_i)_{i=1}^N$. We will now describe how this data is used to train models in different contexts.

\begin{enumerate}
\item \textbf{Tabular models:} When training a tabular model with capacity restricted to rank $k$, we use the following expressions: 
\begin{enumerate}[label=(\alph*)]
\item \textbf{Reward}: In our experiments rewards are represented in the same way for both VE and MLE models:
\begin{equation*}
\tilde{R}_{s,a} = \frac{\sum_{i=1}^N r_i \ind\{ s_i = s, a_i = a \}}{\sum_{i=1}^N \ind\{ s_i = s, a_i = a \}}, \ \ \ \ 
\end{equation*}
where $\ind\{\cdot\}$ is the indicator function.
\item \textbf{Transition dynamics (MLE)}:
To learn the transition dynamics we first parameterize $\Pt^a = \DD^a \KK^a$ for all $a \in \A$, where $\DD^a$ and $\KK^a$ are row-stochastic matrices (see item~\ref{it:rest_tab} in the section ``Restricting Model Capacity'' above). Because we are assuming \S\ to be finite, we can identify each state $s \in \S$ by an index. Let $\idx(s) \in \{1, ..., |S|\}$  be an index that uniquely identifies state $s$. We then compute $\Pt^a = \DD^a \KK^a$ by minimizing the following loss with respect to $\DD^a$ and $\KK^a$:
\begin{equation*}
    \tilde{\loss}_{p, \D}(\P^a, \Pt^a) \defi - \sum_{i=1}^N \ind\{a_i = a\} \log \left[(\DD^a \KK^a)_{\idx(s_i) \idx(s_i')}\right],
\end{equation*}
where $(\DD^a \KK^a)_{ij}$ is the element in the $i$-th row and $j$-th column of matrix $\DD^a \KK^a$. Note that the expression above is the empirical version of expression~(\ref{eq:cml_p}) in the paper~\cp{farahmand2017value}.
\item \textbf{Transition dynamics (VE)}:
In the VE setting we have a set of value functions and policies: $\V$ and $\Pi$. We have one transition matrix $\Pt^\pi$ associated with each policy $\pi \in \Pi$. As discussed in Section~\ref{sec:experiments}, in our experiments we used $\Pi = \{ \pi^a \}_{a \in \A}$, where $\pi^a(a | s) = 1$ for all $s \in \S$. Thus, we end up with the same parameterized probability matrices as above: $\Pt^a = \DD^a \KK^a$. Let $\D_{ia} \subseteq \D$ be the sample transitions starting in state $i$ where action $a$ was taken, that is, $(s_j, a_j, r_j, s'_j) \in \D_{ia}$ if and only if $\idx(s_j) = i$ and $a_j = a$. We computed $\Pt^a = \DD^a \KK^a$ by minimizing the following loss with respect to $\DD^a$ and $\KK^a$:
\begin{equation*}
\label{eq:ve_loss_emp2}
\loss_{\pi^a, \V, \D}(\P^a,\Pt^a) \defi 
\sum_{i, a} \sum_{\v \in \V} \left( \dfrac{1}{|\D_{ia}|} \sum_{(s,a,r,s') \in \D_{ia}} v_{\idx(s')} - \sum_j (\DD^a \KK^a)_{ij} v_j  \right)^2.
\end{equation*}
Note that the expression above corresponds to equation~(\ref{eq:ve_loss_emp}) when learning transition matrices associated with policies $\{ \pi^a \}_{a \in \A}$ in an environment with finite state space \S\ (where states $s$ can be associated with an index $i$) and $p=2$.
\end{enumerate}

\item \textbf{Neural network models}: 
When training a neural network model with capacity restrictions construct a network $f_\theta : (\S, \A) \mapsto (\S, \R)$.
The network is fully connected and takes the concatenation of $\S$ with the one-hot representation of $\A$ as input.
For a given $(s, a)$ pair we denote it's output as $\tilde{s}'_{s,a}, \tilde{r}'_{s,a} = f_\theta(s,a)$.
In all cases we train the neural network model by sampling mini-batches uniformly from $\mathcal{D}$.
It is important to note that we only use these neural network models on deterministic domains (e.g., Cart-pole) meaning that the output of the model, $\tilde{s}'$ represents a single state rather than an expectation over states.
\begin{enumerate}[label=(\alph*)]
\item \textbf{Reward:} For both VE and MLE models we train our neural network models to accurately predict the reward associated with each state action transition:
\begin{equation*}
\ell_{r, \mathcal{D}}(\theta) = \sum_{i=1}^N (\tilde{r}_{s_i, a_i} - r_i)^2.
\end{equation*}
\item \textbf{Transition dynamics (MSE):}
We learn  models by encouraging $f_\theta$ to accurately predict the next state:
\begin{equation*}
\ell_{s',\mathcal{D}}(\theta) = \sum_{i=1}^N (\tilde{s}'_{s_i,a_i} - s'_i)^2.
\end{equation*}
\item \textbf{Transition dynamics (VE):}
For VE models use (\ref{eq:ve_loss_emp}), disregarding reward terms to give:
\begin{equation*}
\ell_{\V, \mathcal{D}}(\theta) = \sum_{i=1}^n \sum_{v \in \V} (v(\tilde{s}_{s_i, a_i}) - v(s'_i))^2.
\end{equation*}
\end{enumerate}
\end{enumerate}

\paragraph{\ref{it:policy_construction} Policy construction}
In each experiment we present, after a model is constructed, we subsequently use it to construct a policy.
The manner in which we do this varies based upon the type of the experiment and the nature of the environment. 
There are three mechanisms for constructing policies from models:
\begin{enumerate}
\item \textbf{Value iteration:} For experiments with $\V = \FS$ (which are performed only with tabular models), we use the learned model $\mt = (\rt,\pt)$ to perform value iteration until convergence, yielding $\tilde{v}^*$~\cp{puterman94markov}.
Here $\tilde{v}^*$ represents the optimal value function of the model \mt. We then produce a policy according to $\pi(s) = \text{argmax}_a (\tilde{r}(s,a) + \gamma \sum_{s'} \tilde{p}(s' | s, a) \tilde{v}^*(s'))$.
\item \textbf{Approximate policy iteration with least squares temporal-difference learning (LSTD):} For experiments on environments with finite $\S$ and $\V = \AV$ we used policy iteration combined with least square policy evaluation using basis $\{ \phi_i \}_{i=1}^d$.
Specifically, each iteration of policy iteration involved the following steps:
\begin{enumerate}[label=(\alph*)]
\item Collect experience tuples using the previous policy, $\pi$, leading to $\D = (s_i, a_i, r_i, s'_i)_{i=1}^n$. 
\item Replace the reward and next-states with those predicted by the model: $\tilde{r}_i, \tilde{s}'_i = f_\theta(s_i, a_i)$, leading to 
$\D' = (s_i, a_i, \tilde{r}_i, \tilde{s}'_i)_{i=1}^n$.
\item Learn $v_w(s) = \sum_{i=1}^d w_i \phi_i(s) \approx v_\pi$ using LSTD with $\D'$.
\item Construct a new policy $\pi(s) = \text{argmax}_a (\tilde{r}_{s,a} + \gamma v_w(\tilde{s}'_{s,a}))$ where $\tilde{r}_{s,a}, \tilde{s}'_{s,a}$ are sampled from the trained model conditioned on state $s$ and action $a$.
\end{enumerate}
This procedure is repeated for a fixed number of iterations. 
\item \textbf{Deep Q-networks (DQN):} For experiments with $\V = \AV$ and infinite $\S$ we use Double Q-Learning to produce policies. 
We incorporate our learned model, $f_\theta$, by replacing elements in the replay buffer $(s, a, r, s')$ with $(s, a, \tilde{r}_{s,a}, \tilde{s}'_{s,a})$ where $\tilde{r}_{s,a}, \tilde{s}'_{s,a} = f_\theta(s,a)$.
\end{enumerate}

\paragraph{\textbf{\ref{it:policy_evaluation} Policy evaluation}}
There are two methods to evaluate the policies resulting from the {policy construction} stage described above:
\begin{enumerate}
\item For policies produced using value iteration or policy iteration plus LSTD the ensuing policy, $\pi$, is exactly evaluated on the true environment, yielding $v_\pi(s)$. Then the average value of $v_\pi(s)$ over all states is reported. 
\item For policies produced using DQN, the average return over the last $100$ episodes of training is reported. 
\end{enumerate}

\subsubsection{Classes of experiments}

In addition to varying the capacity of $\M$, there are two primary classes of experiments that were run in our paper that assess different choices of $\V$. We distinguish between these two classes below:

\paragraph{$\bs{\span(\V) \approx \ddot{\V}, \AV = \FS, \Pi = \PS}$:}
In these experiments we consider that there is no limitation on the agent's ability to represent value functions, and focus on achieving value equivalence with respect to the polytope of value functions $\ddot{\V}$ induced by the environment.
We enable the agent to represent arbitrary functions in $\FS$ by restricting ourselves to tabular environments and using dynamic programming to perform exact value iteration in our Policy Construction step.
We approximate the value polytope by randomly sampling deterministic policies: $\{\pi_1, \ldots, \pi_n\}$ and evaluating them (again using dynamic programming) to produce $\{ v_{\pi_1}, \ldots, v_{\pi_n} \}$. 
We then choose $\V = \{ v_{\pi_1}, \ldots, v_{\pi_n}\}$. 
In this setting we vary the number of policies generated.

\textbf{Corresponding experiments:} the experiments in this class vary two dimensions: (1) the rank of the model and (2) the number of policies generated. 
In Figures~\ref{fig:four_rooms_value}~and~\ref{fig:four_rooms_value_transpose} we depict plots for the Four Rooms environment that fix the number of policies while varying the rank of the model and plots that fix the rank of the model while varying the number of policies, respectively. 
Figures~\ref{fig:catch_value}~and~\ref{fig:catch_value_transpose} are analogous plots for the Catch environment.

\paragraph{$\bs{\span(\V) \approx \AV}$, $\bs{\Pi = \PS}$:}
In these experiments we explore the setting described in Remark~\ref{thm:linear_app}.  
We assume that the agent has variable ability to represent value functions, $\AV$, and attempt to learn a model in $\M(\AV, \PS)$. 
From Proposition~\ref{thm:linear_span} we only need to find $\V$ such that $\span(\V) \supseteq \AV$. 
Experiments in this class can further be broken down into two settings based upon the nature of $\AV$:
\begin{enumerate}[label=(\alph*)]
\item \textbf{Linear function approximation:}
In certain experiments our agent uses a class of linear function approximators to represent value functions: $\AV= \{ \tilde{v} : \tilde{v}(s) = \sum_{i=1}^d \phi_i(s) w_i \}$ where $\phi_i(s) : \S \mapsto \R$ and $\boldsymbol{w} \in \R^d$.
In this setting achieving $\span(\V) \supseteq \AV$ can be satisfied by choosing $\V = \{\phi_i\}_{i=1}^d$. 
For experiments using linear function approximation, we select our features $\{\phi_i\}_{i=1}^d$ to correspond to state aggregations.
This entails the following procedure:
\begin{enumerate}[(i)]
\item Collect data using a policy that selects actions uniformly at random.
\item For tabular domains (e.g., Catch, Four Rooms), convert tabular state representations into coordinate-based representations.
For Catch we convert each tabular state into the positions of both the paddle and the ball: $(x_\text{paddle}, y_\text{paddle}, x_\text{ball}, y_\text{ball})$.
For Four Rooms we use the position of the agent: $(x_\text{agent}, y_\text{agent})$.
Denote the function that performs this conversion as: $f: \S \mapsto \R^n$ where $n = 2$ and $n = 4$ for Four Rooms and Catch respectively.
\item Perform k-means clustering on these converted states to produce $d$ centers $c_{1:d}$.
\item Define $\phi_i(s) = \ind\{ \text{argmin}_j\  \|f(s) - c_j \|_2 = i  \}$, which corresponds to aggregating states according to their proximity to the previously calculated centers. 
\end{enumerate}

\textbf{Corresponding experiments:} the experiments in this class vary two dimensions: (1) the rank of the model and (2) the number of basis functions in $\{ \phi_{i} \}_{i=1}^d$. 
In Figures~\ref{fig:catch_basis}~and~\ref{fig:catch_basis_transpose} we depict plots of ``slices'' of this two-dimensional set of results on the Catch domain: \ref{fig:catch_basis} depicts fixing the number of basis functions while varying model-rank and \ref{fig:catch_basis_transpose} depicts fixing the model-rank while varying the number of basis functions.

\item \textbf{Neural network function approximation:} When Neural Networks are used to approximate the agent's value functions we have $\AV = \{ \tilde{v} : \tilde{v}(s) = g_\theta(s) \}$ where $g_\theta$ represents a neural network with a particular architecture parameterized by $\theta$. 
In our experiments we choose the architecture of $g_\theta$ to be a 2 layer neural network with a tanh activation for its hidden layer. 
Unlike the linear function approximation setting, it is less obvious how to choose $\V$ such that $\span(\V) \supseteq \AV$.
One option is to use randomly initialized neural networks in $\AV$ as our basis.
To randomly initialize a given layer in some network $g_\theta$, we select weights from a truncated normal distribution where $\mu = 0$ and $\sigma = 1/\sqrt{\text{layer-input-size}}$ and initialize biases to $0$.

However, we found in practice that a large number of these randomly initialized networks were required to achieve reasonable performance.
Instead of maintaining a large set of initializations in $\V$, we allow the elements of $\V$ themselves to be stochastic.
Every time we apply an update of gradient descent we sample a new set of randomly initialized neural networks to function as $\V$.
This is equivalent to minimizing $\mathbb{E}_\V[\loss_{\Pi, \V, \D'}(m^*, \mt)]$ where $\loss_{\Pi, \V, D'}$ is defined in \ref{eq:ve_loss_emp}.
We find that having more random elements in $\V$ decreases the variance in the performance of VE models; $|\V| = 5$ in our experiments.

\textbf{Corresponding experiments:} the experiments in this class vary two dimensions: (1) the rank of the model and (2) the width of the neural networks in $\AV$. 
In Figures~\ref{fig:cartpole_basis}~and~\ref{fig:cartpole_basis_transpose} we depict plots of ``slices'' of this two-dimensional set of results on the Catch domain: \ref{fig:cartpole_basis} depicts fixing the network width while varying model-rank and \ref{fig:cartpole_basis_transpose} depicts fixing the model-rank while the network width varies. 

\end{enumerate}

\vspace{2mm}

\subsubsection{Additional results}
\label{sec:additional_results}

In the experimental section of the main text we showed that our theoretical claims about the value equivalence principle hold in practice through a series of bivariate experiments (e.g., varying model-rank and number of bases, varying model-rank and number of policies, varying model-rank and network width). 
We displayed our results as ``slices'' of these bivariate experiments, where one variable would be held fixed and the other would be allowed to vary. 
To keep the paper concise, we only displayed a subset of these slices where the ``fixed'' variable was selected as the median value over full set we experimented with.
In what follows, we present the complete set of the experimental results we acquired. 
We indicate that a plot was included in the main text by printing its caption in bold font.

\begin{figure}[H]
\centering
\subfigure[Catch (fixed \V)]{
\includegraphics[scale=0.25]{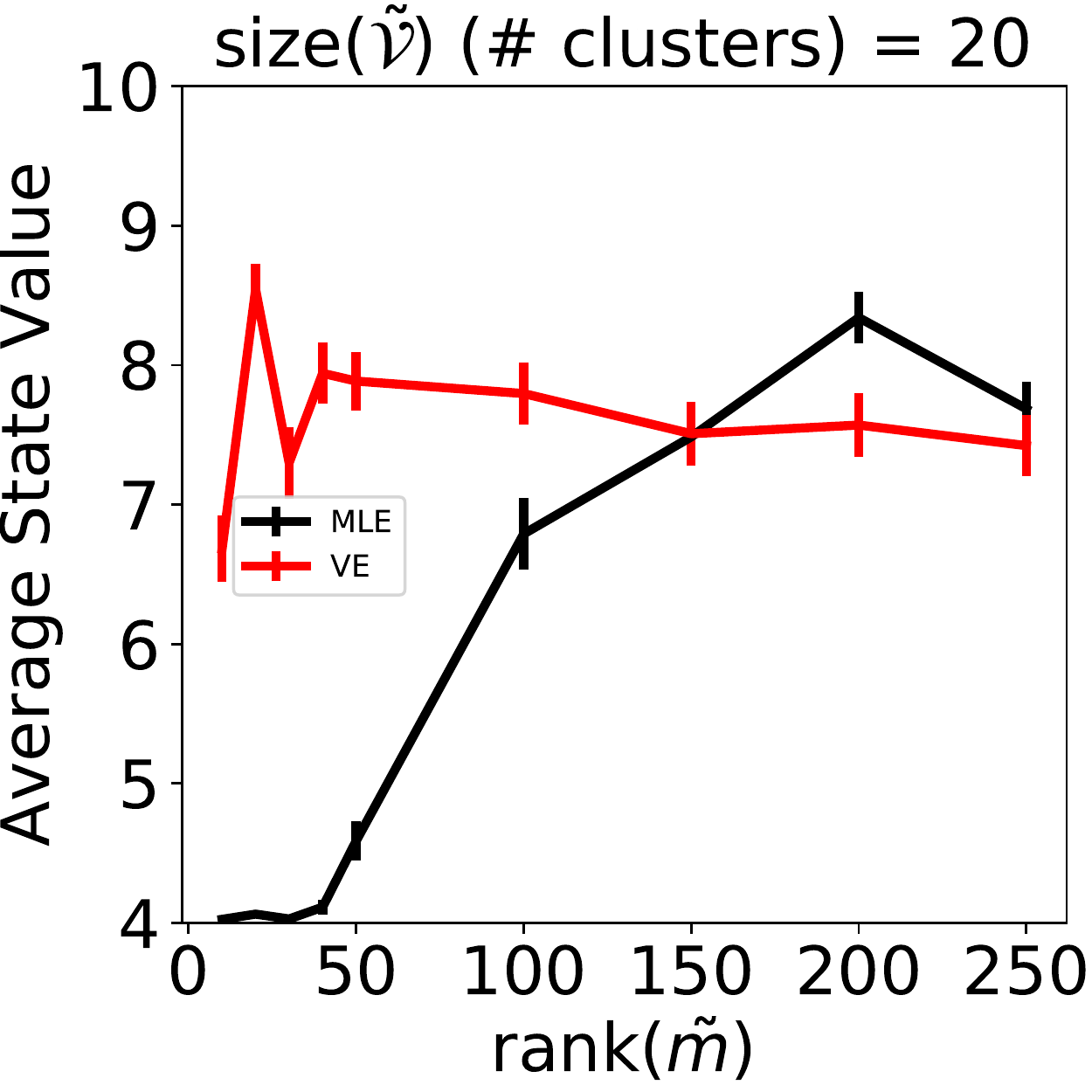}
 }
\subfigure[Catch (fixed \V)]{
\includegraphics[scale=0.25]{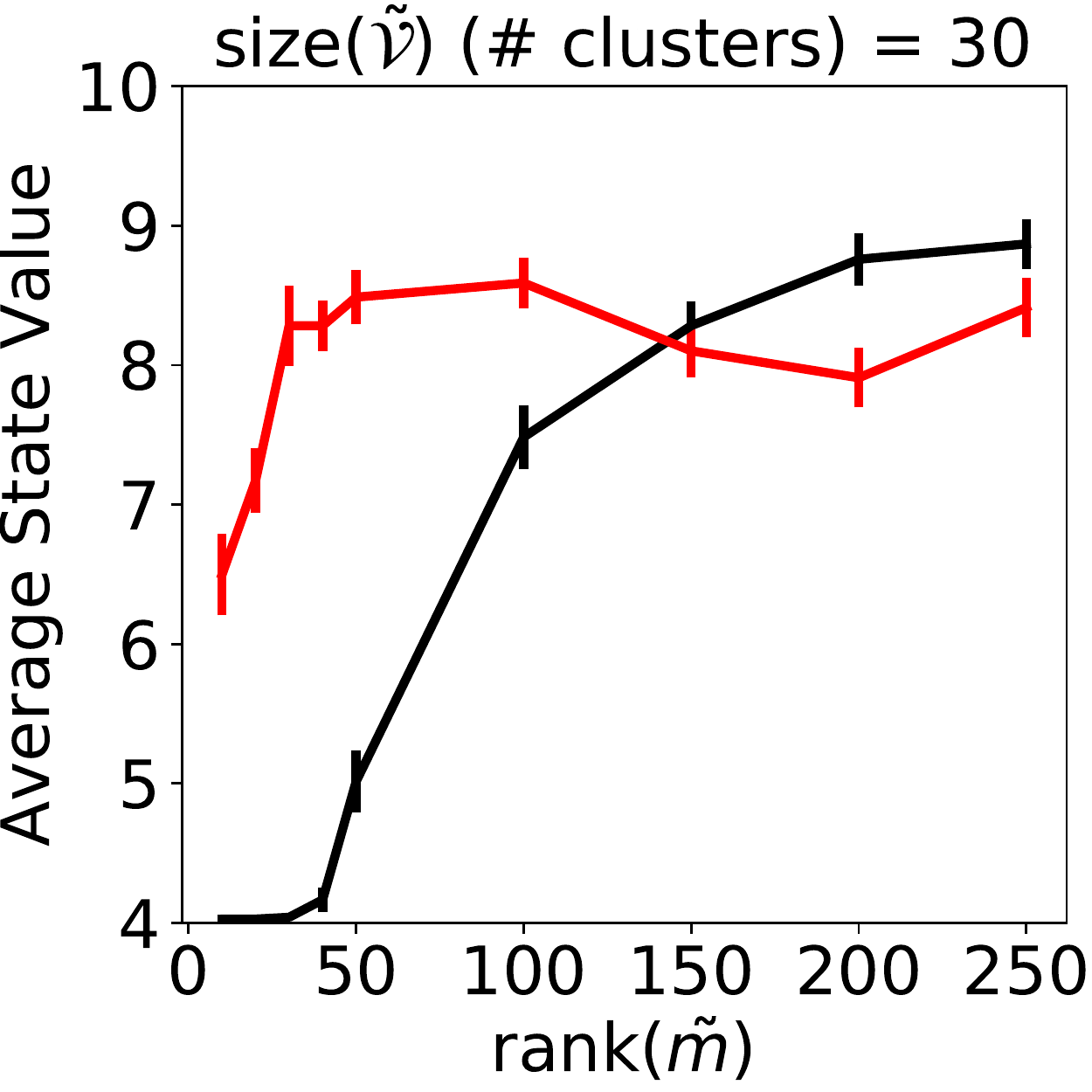}
 }
\subfigure[Catch (fixed \V)]{
\includegraphics[scale=0.25]{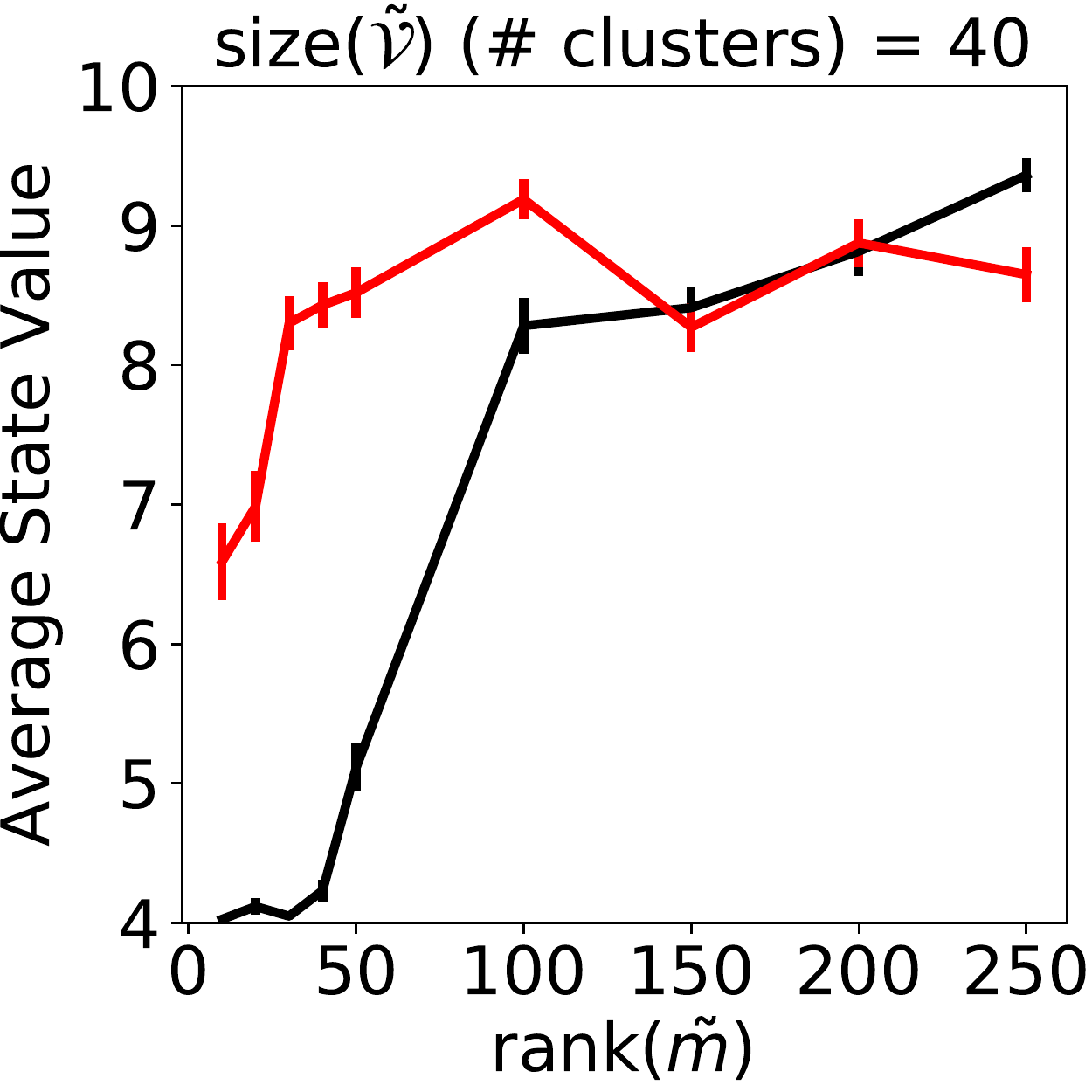}
 }
\subfigure[\textbf{Catch (fixed $\boldsymbol{\V}$)}]{
\includegraphics[scale=0.25]{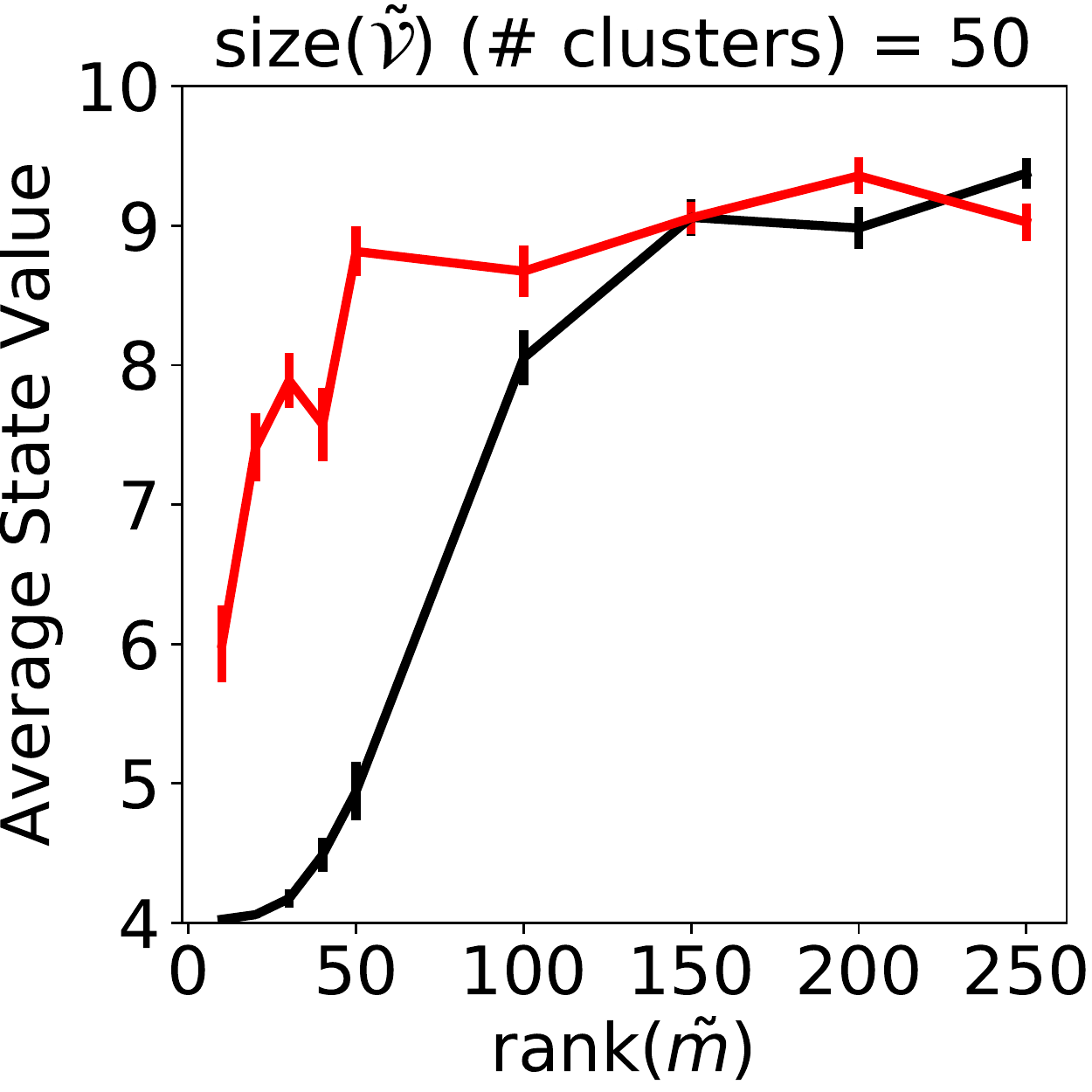}
 }
\subfigure[Catch (fixed \V)]{
\includegraphics[scale=0.25]{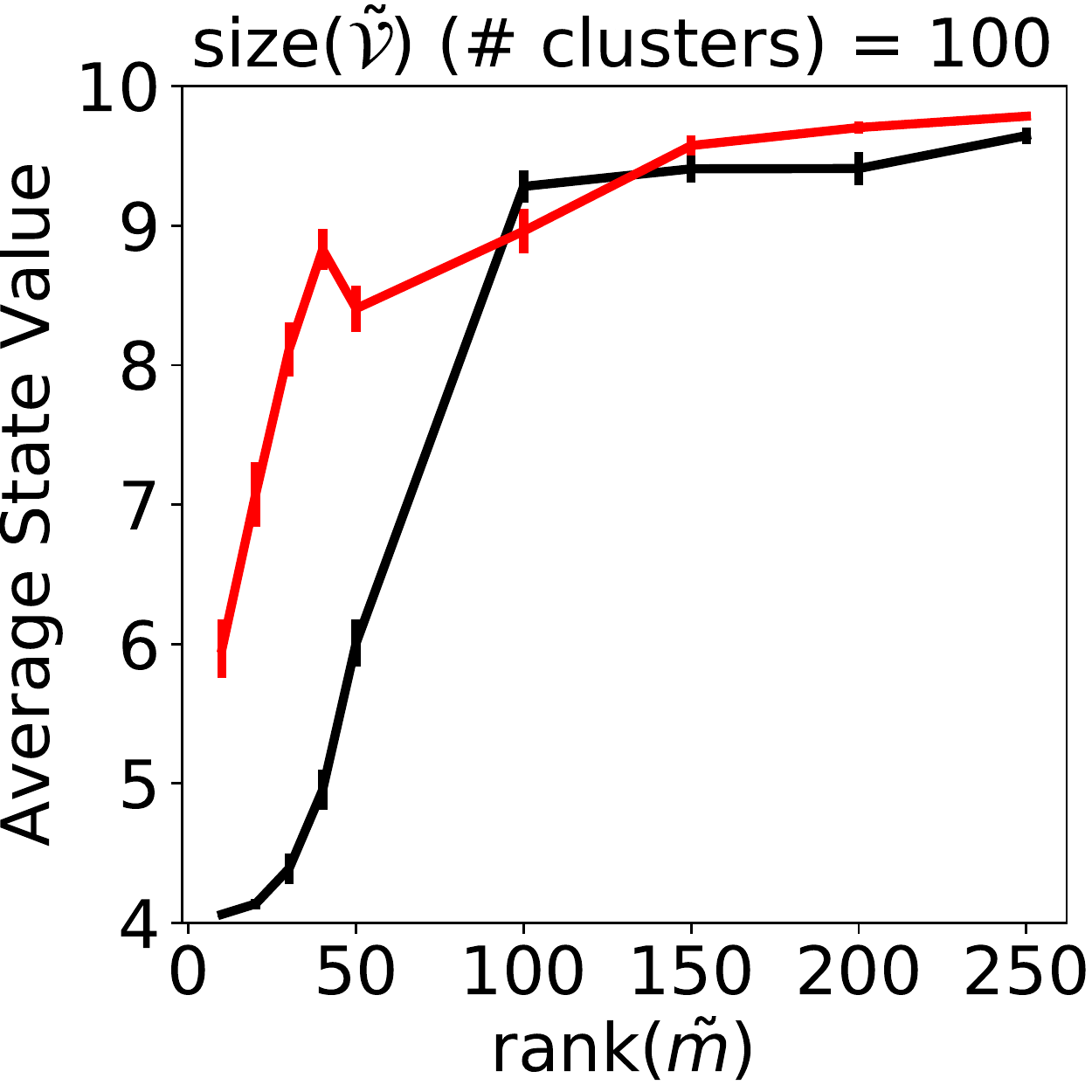}
 }
\subfigure[Catch (fixed \V)]{
\includegraphics[scale=0.25]{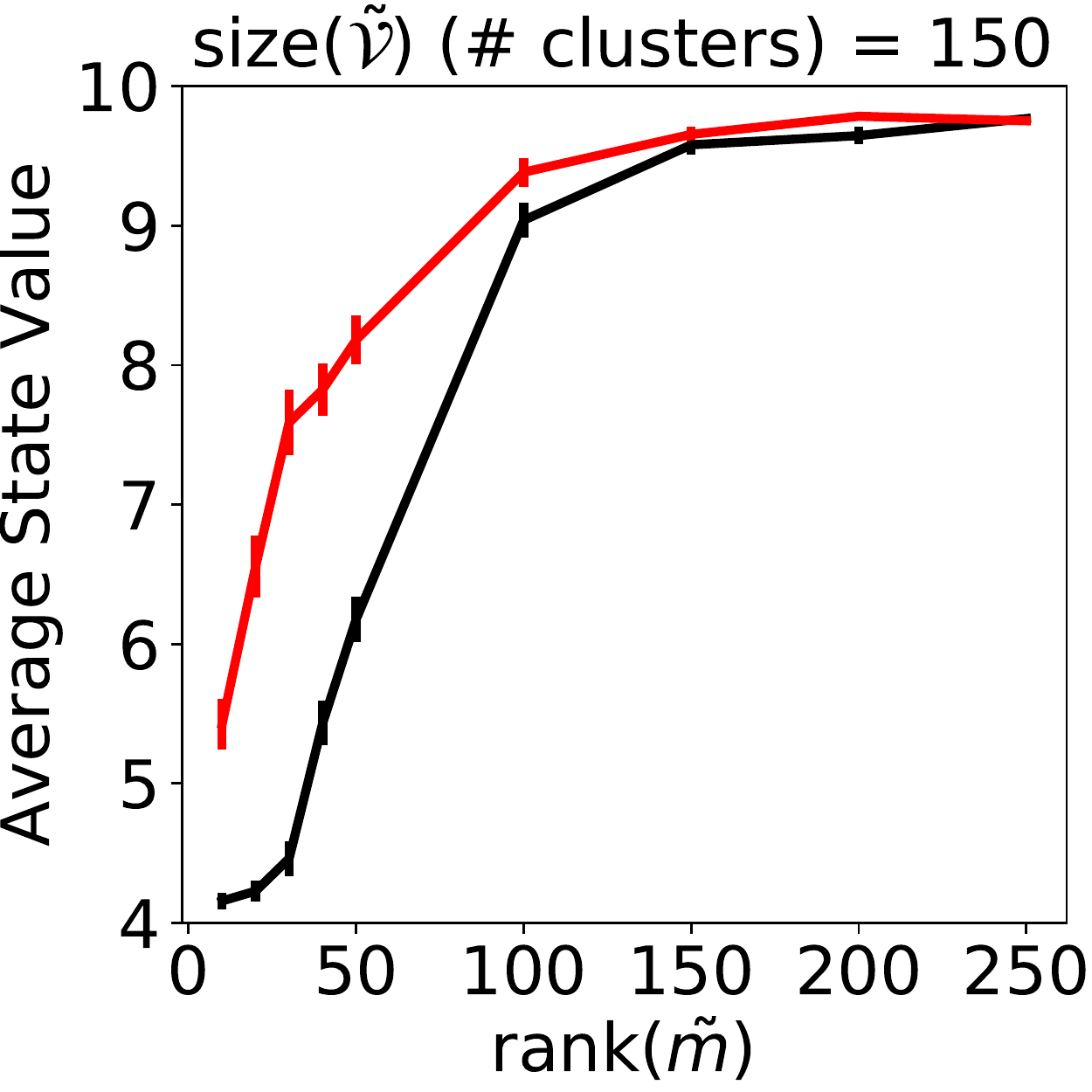}
 }
\subfigure[Catch (fixed \V)]{
\includegraphics[scale=0.25]{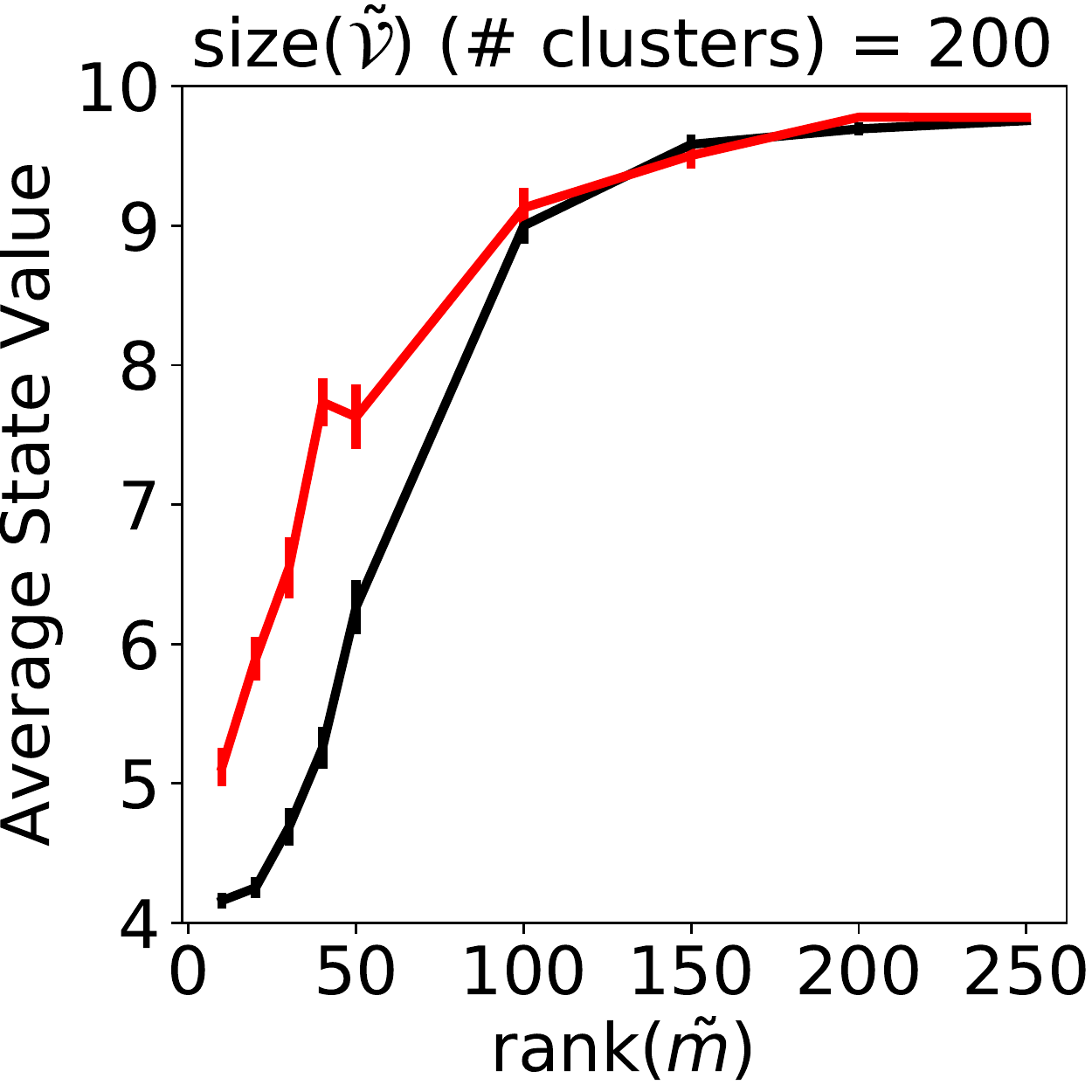}
 }
\subfigure[Catch (fixed \V)]{
\includegraphics[scale=0.25]{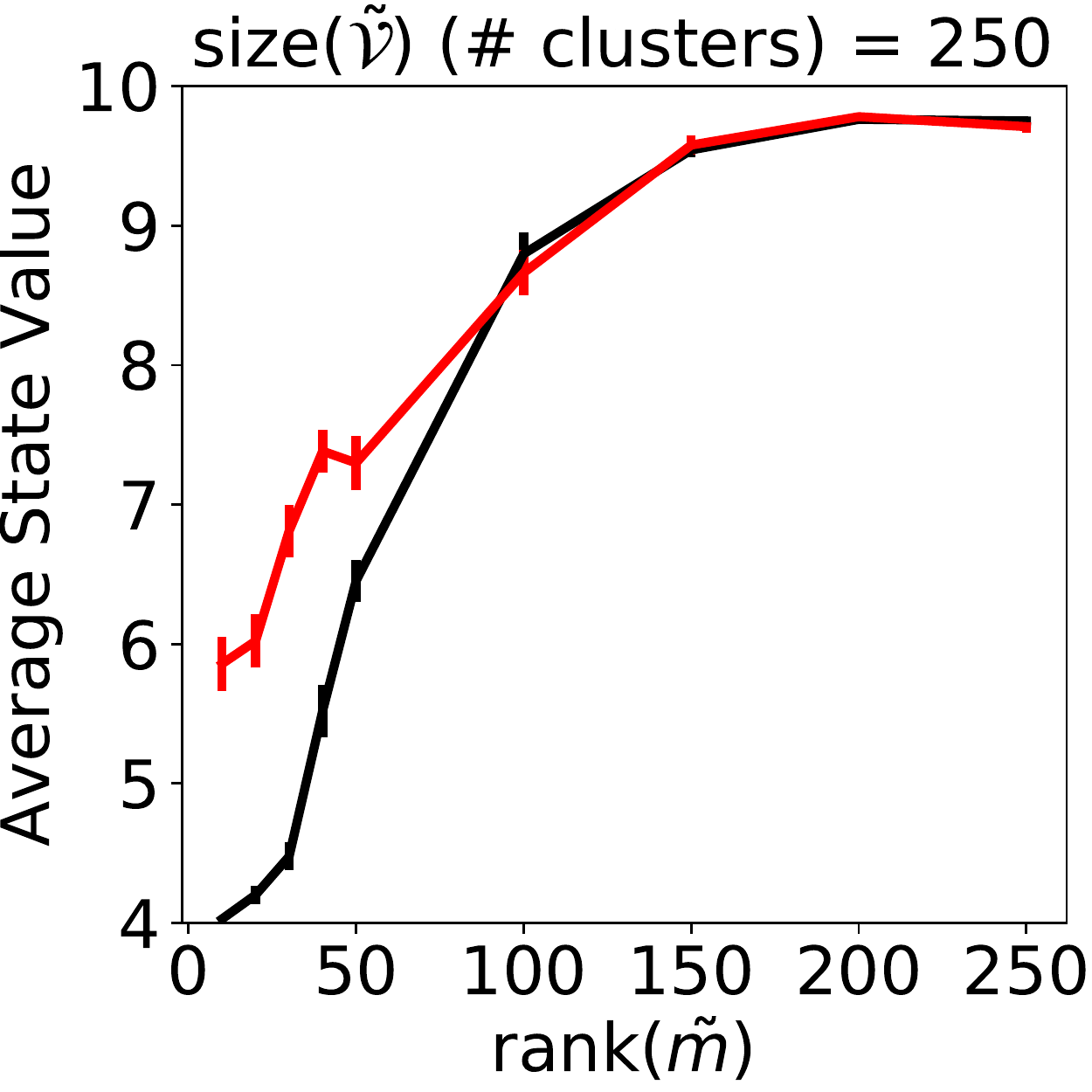}
 }
\caption{All Catch results with fixed $\V$ and $\span(\V) \approx \AV$.}
\end{figure}

\begin{figure}[H]
\centering
\subfigure[Catch (fixed $\tilde{m}$)]{
\includegraphics[scale=0.25]{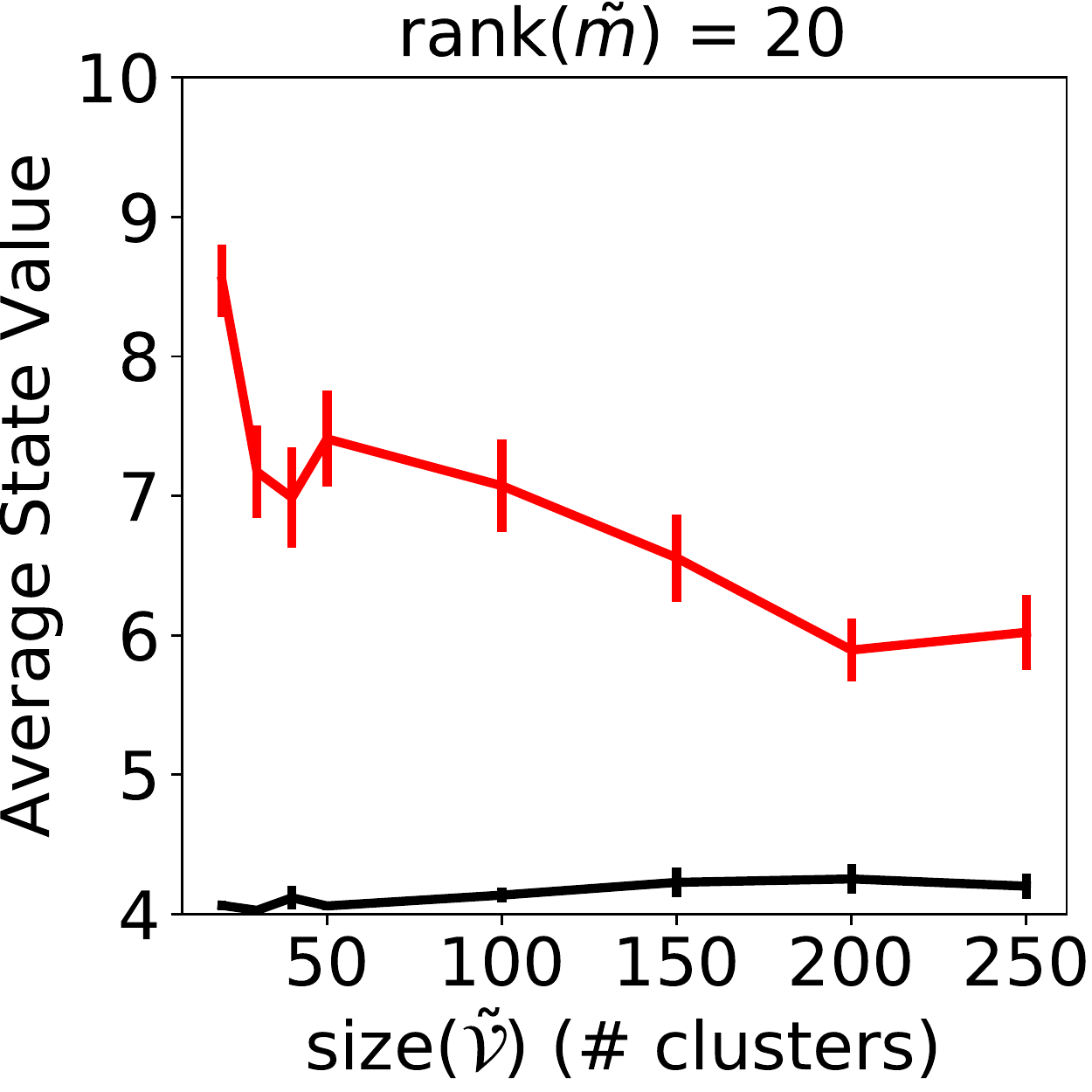}
 }
\subfigure[Catch (fixed $\tilde{m}$)]{
\includegraphics[scale=0.25]{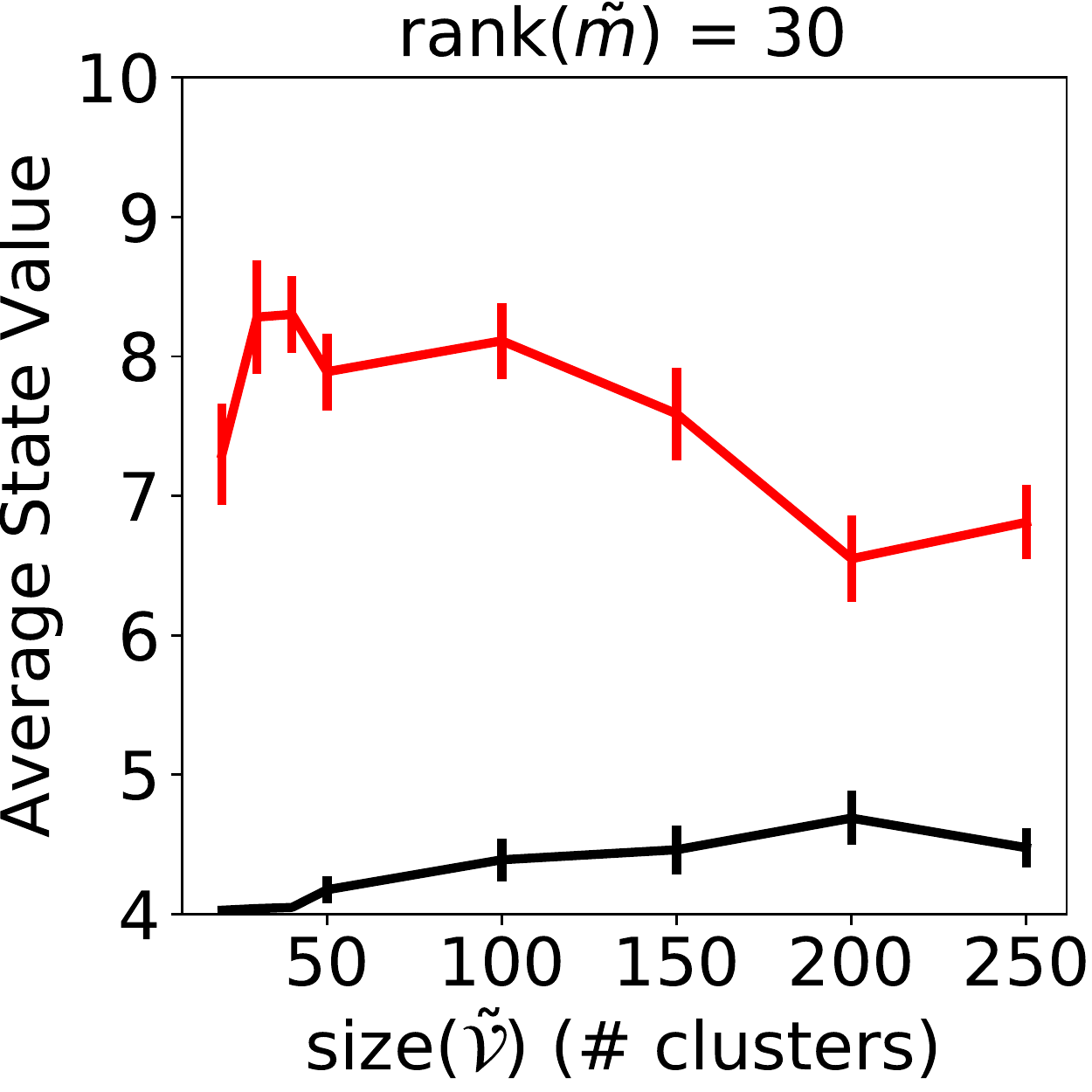}
 }
\subfigure[Catch (fixed $\tilde{m}$)]{
\includegraphics[scale=0.25]{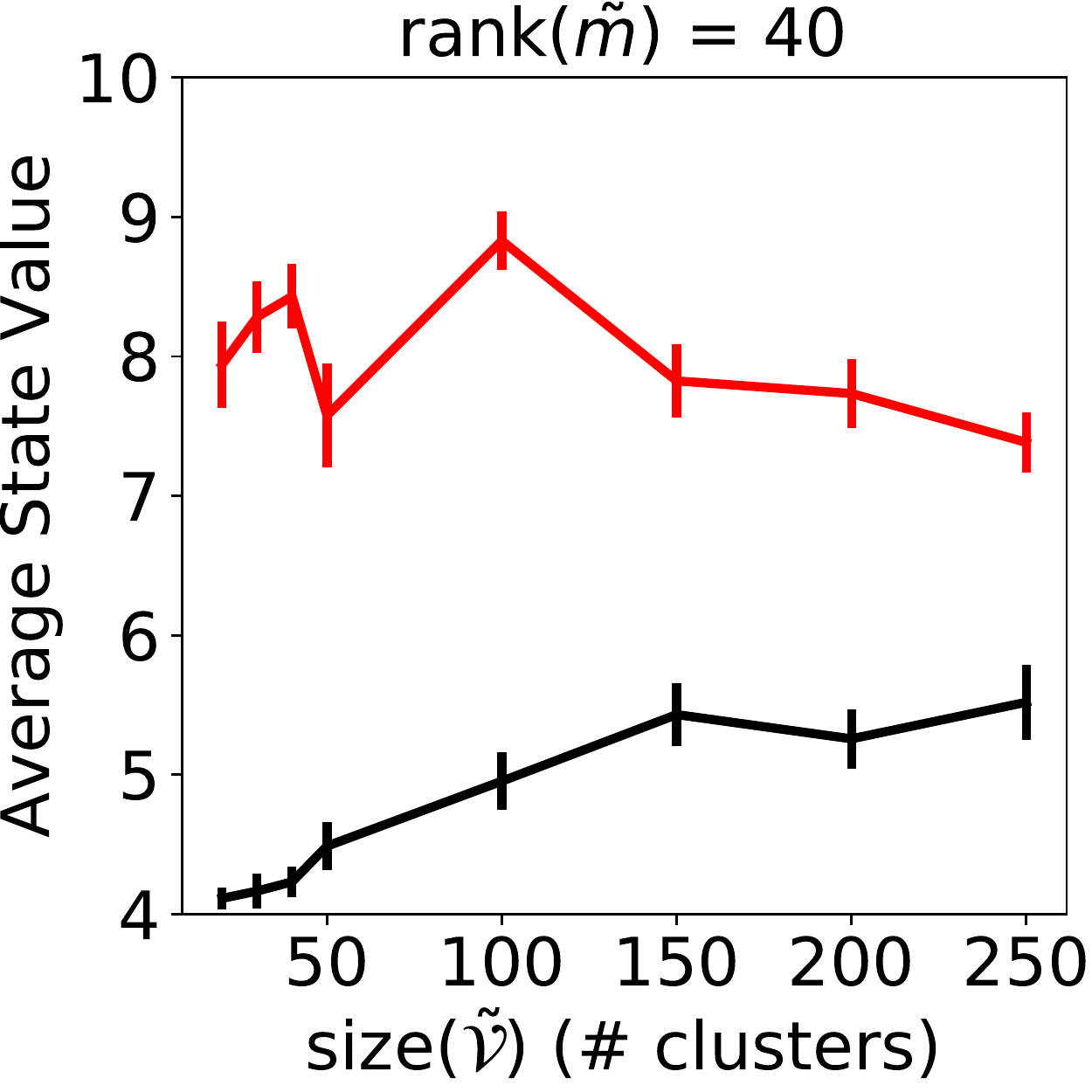}
 }
\subfigure[\textbf{Catch (fixed $\boldsymbol{\tilde{m}}$})]{
\includegraphics[scale=0.25]{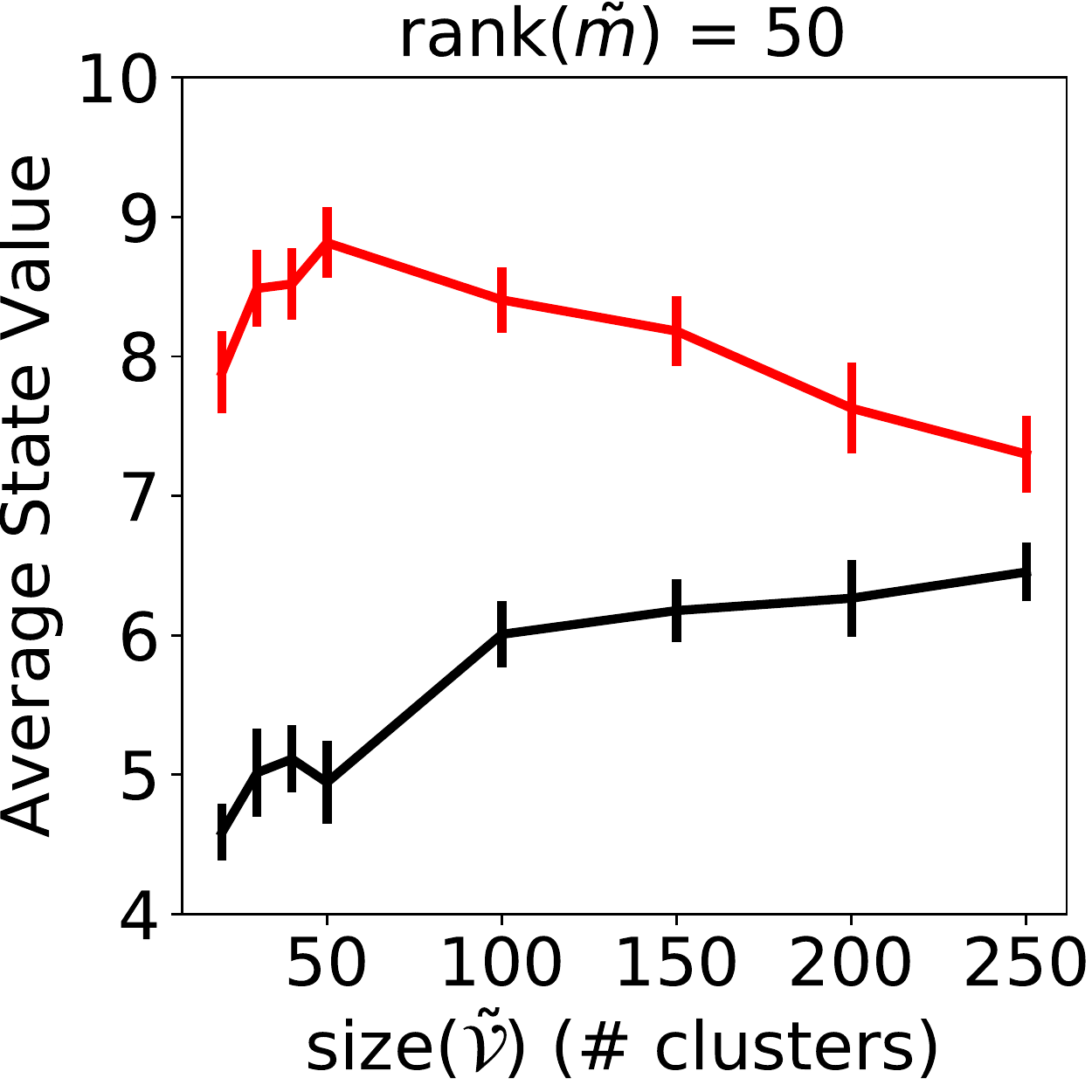}
 }
\subfigure[Catch (fixed $\tilde{m}$)]{
\includegraphics[scale=0.25]{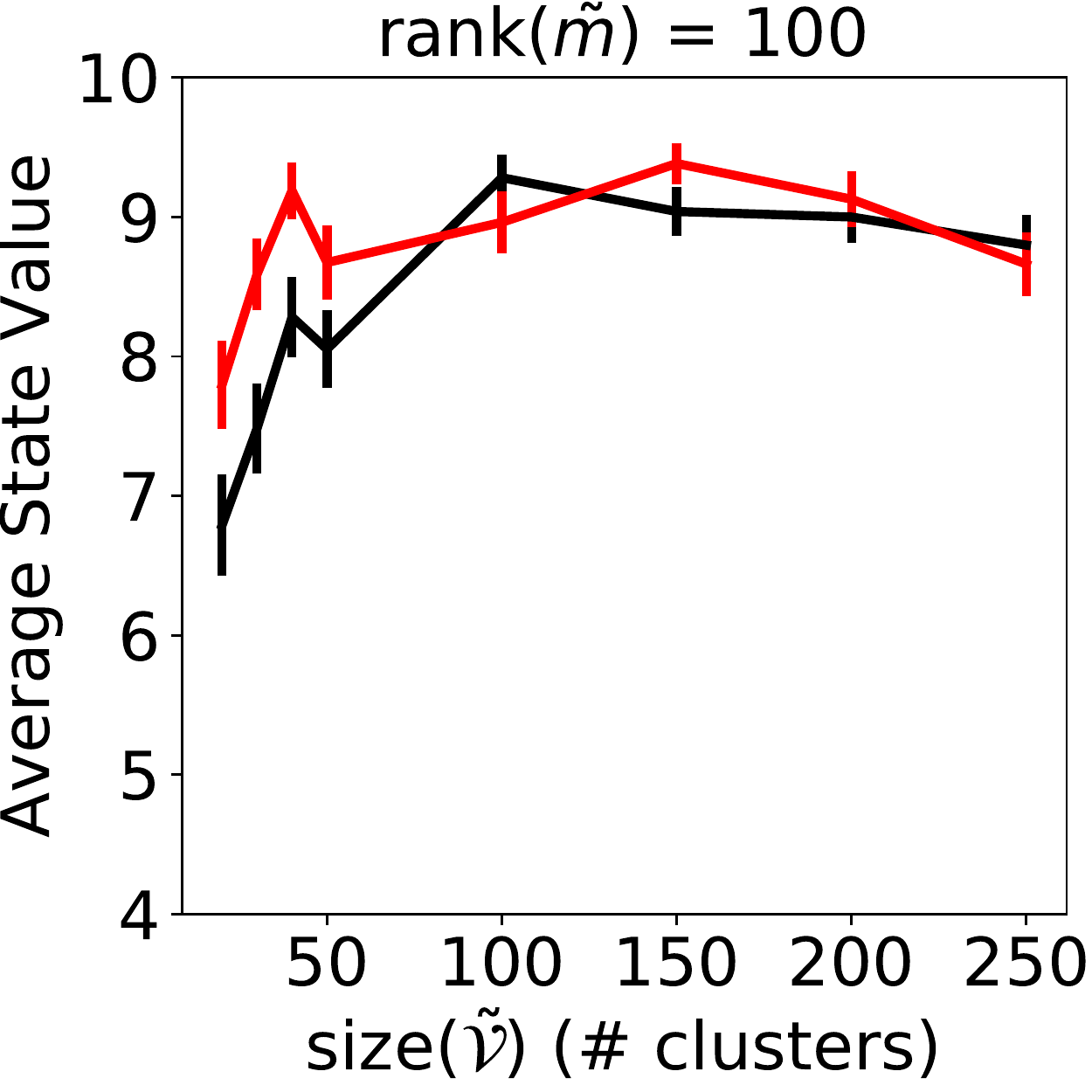}
 }
\subfigure[Catch (fixed $\tilde{m}$)]{
\includegraphics[scale=0.25]{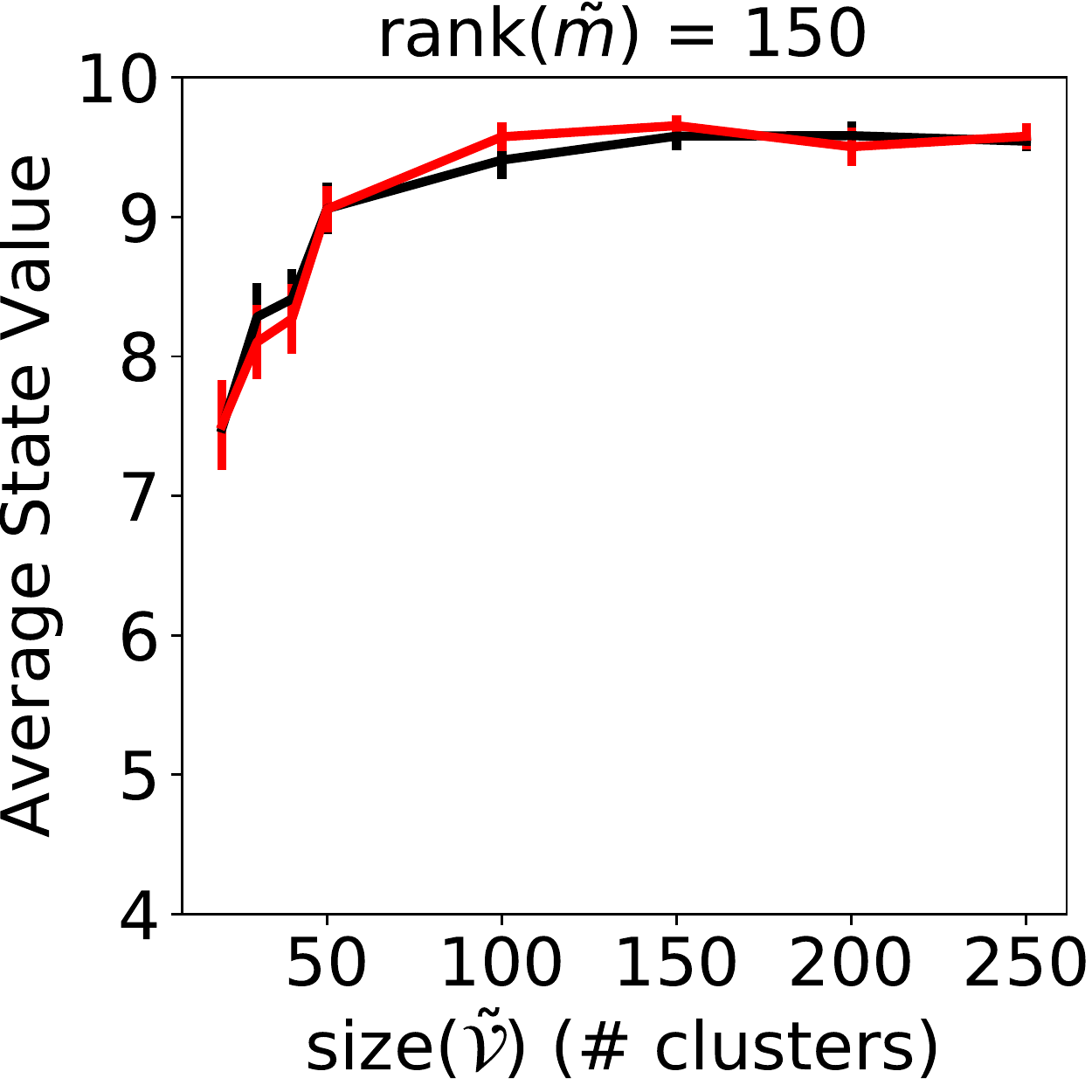}
}
\subfigure[Catch (fixed $\tilde{m}$)]{
\includegraphics[scale=0.25]{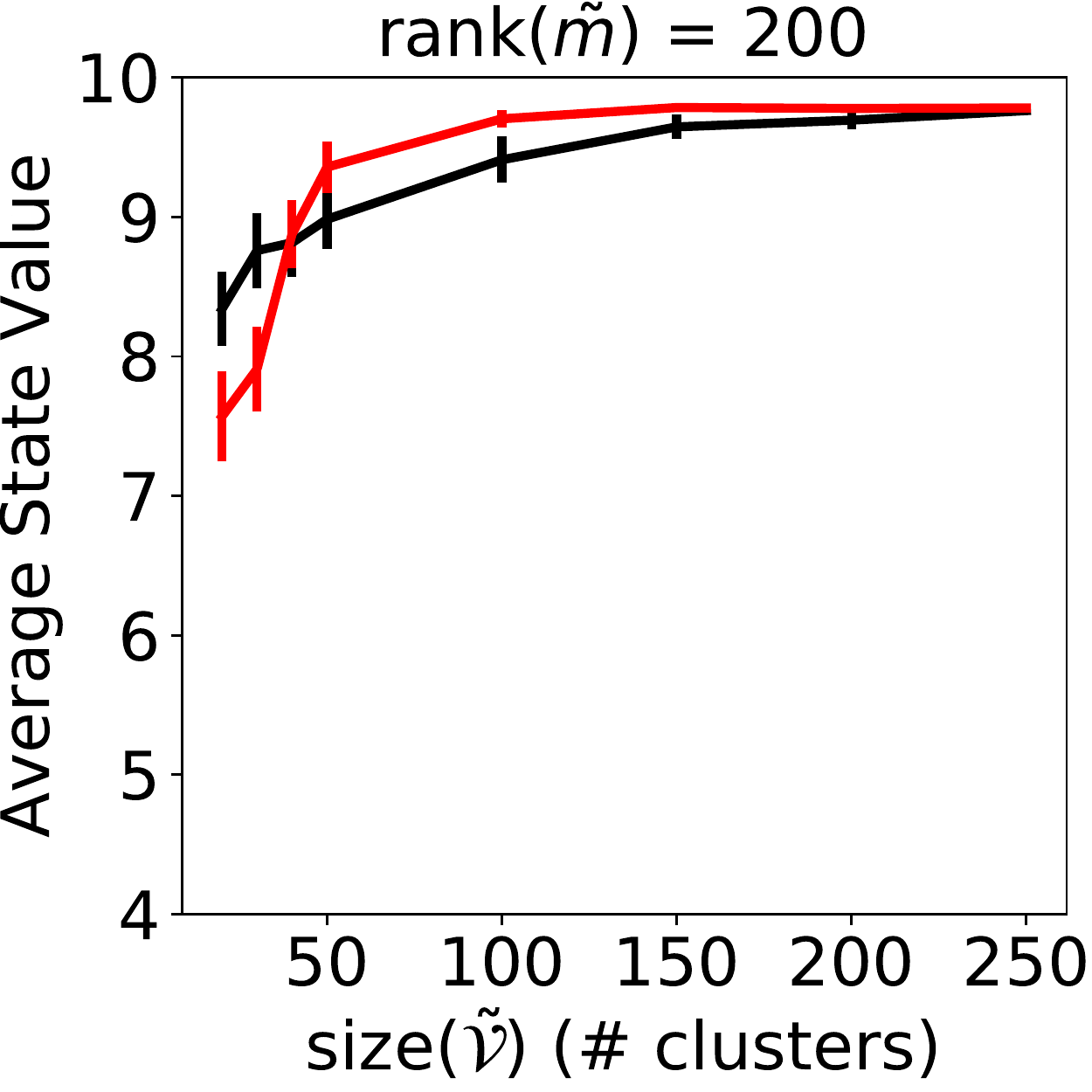}
 }
\subfigure[Catch (fixed $\tilde{m}$)]{
\includegraphics[scale=0.25]{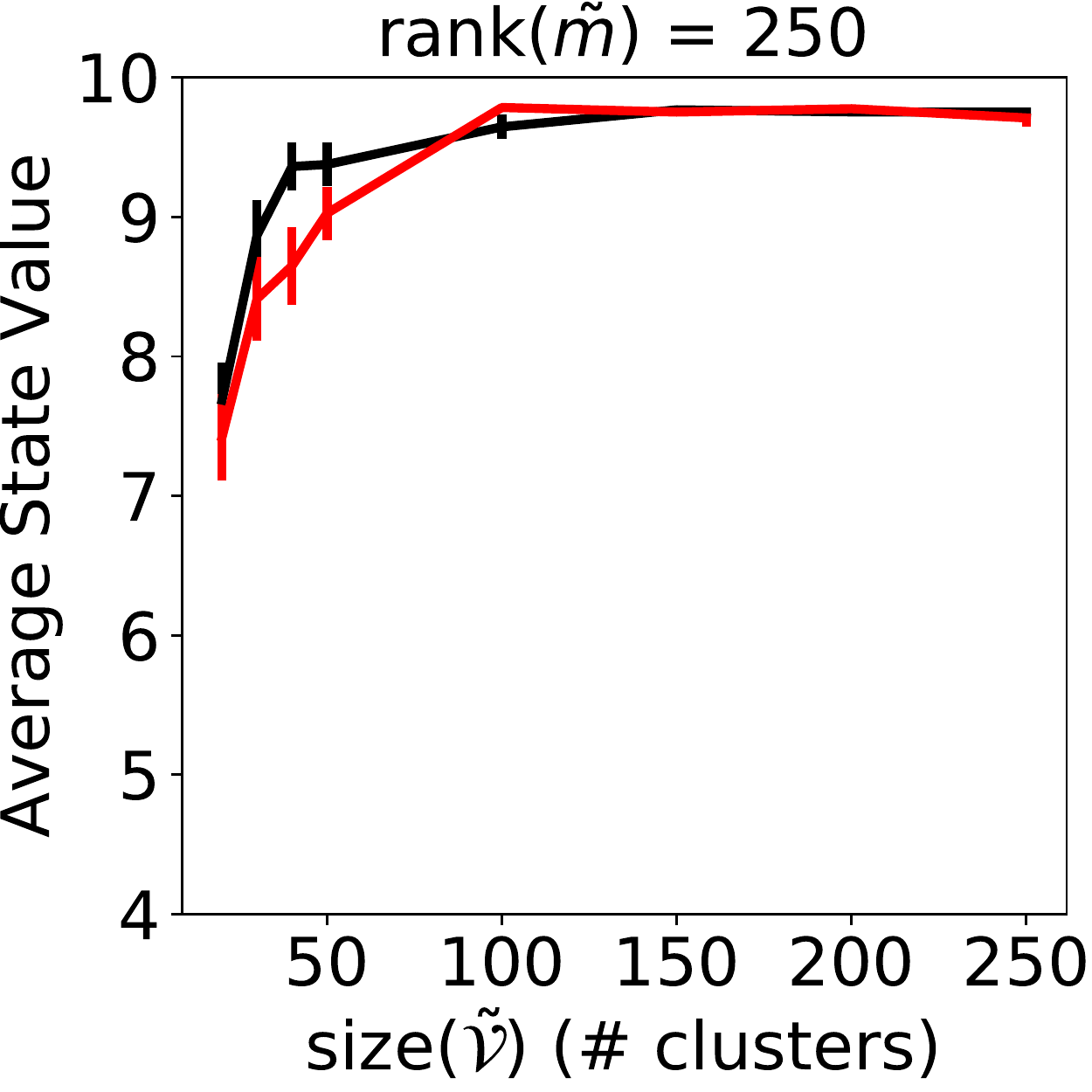}
}
\caption{All Catch results with fixed $\tilde{m}$ and $\span(\V) \approx \AV$.}
\end{figure}

\begin{figure}[H]
\centering
\subfigure[Catch (fixed $\V$)]{
\includegraphics[scale=0.25]{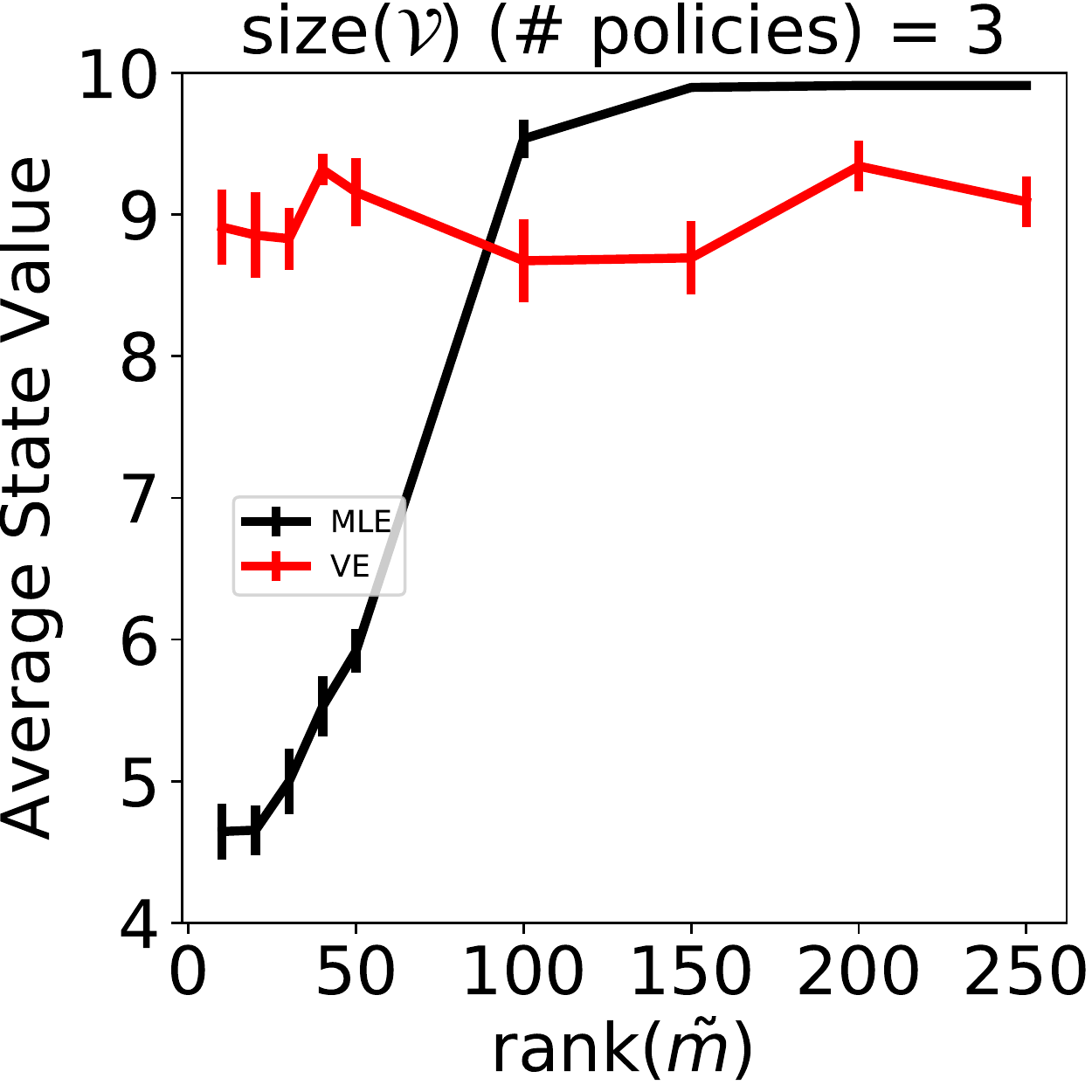}
 }
\subfigure[Catch (fixed $\V$)]{
\includegraphics[scale=0.25]{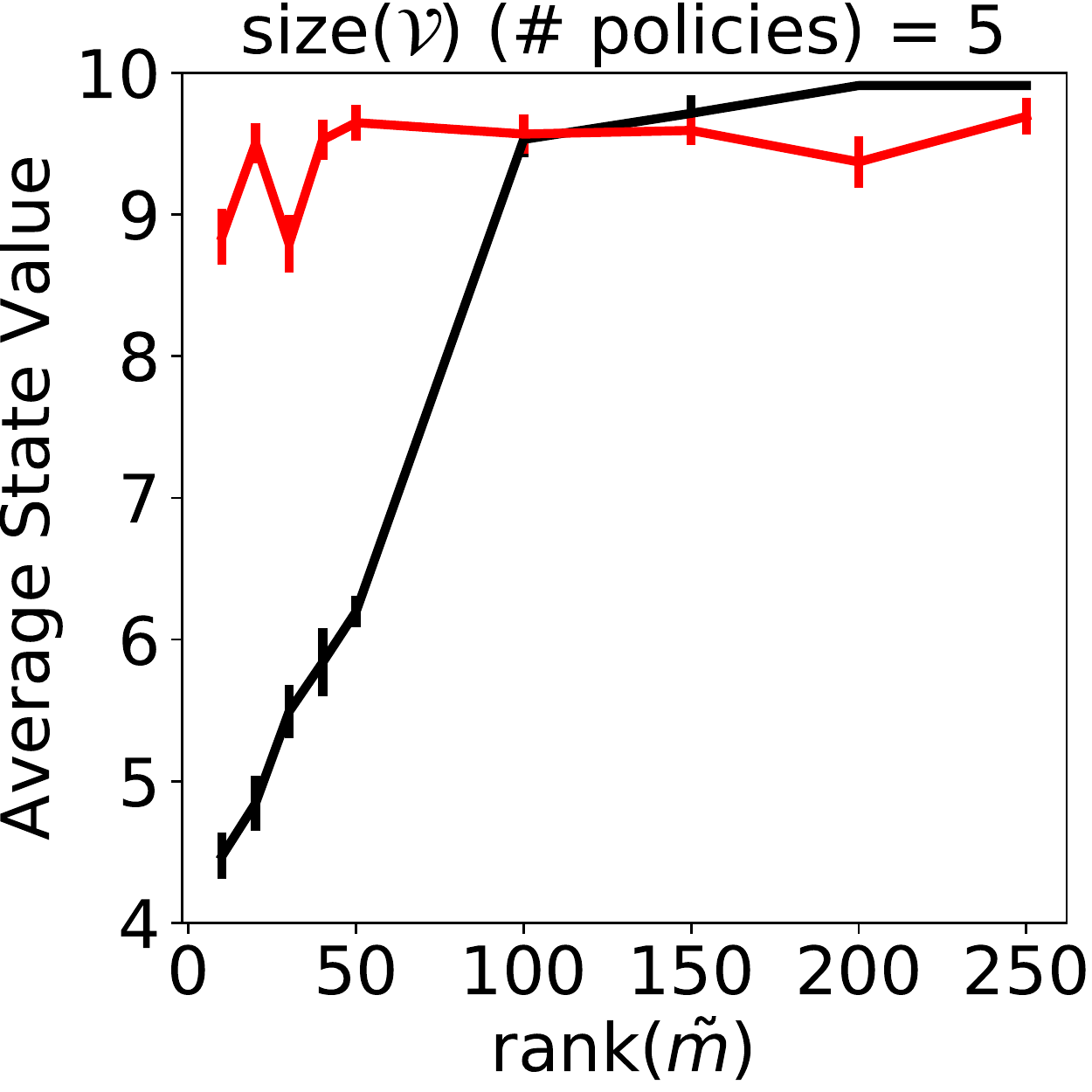}
 }
\subfigure[Catch (fixed $\V$)]{
\includegraphics[scale=0.25]{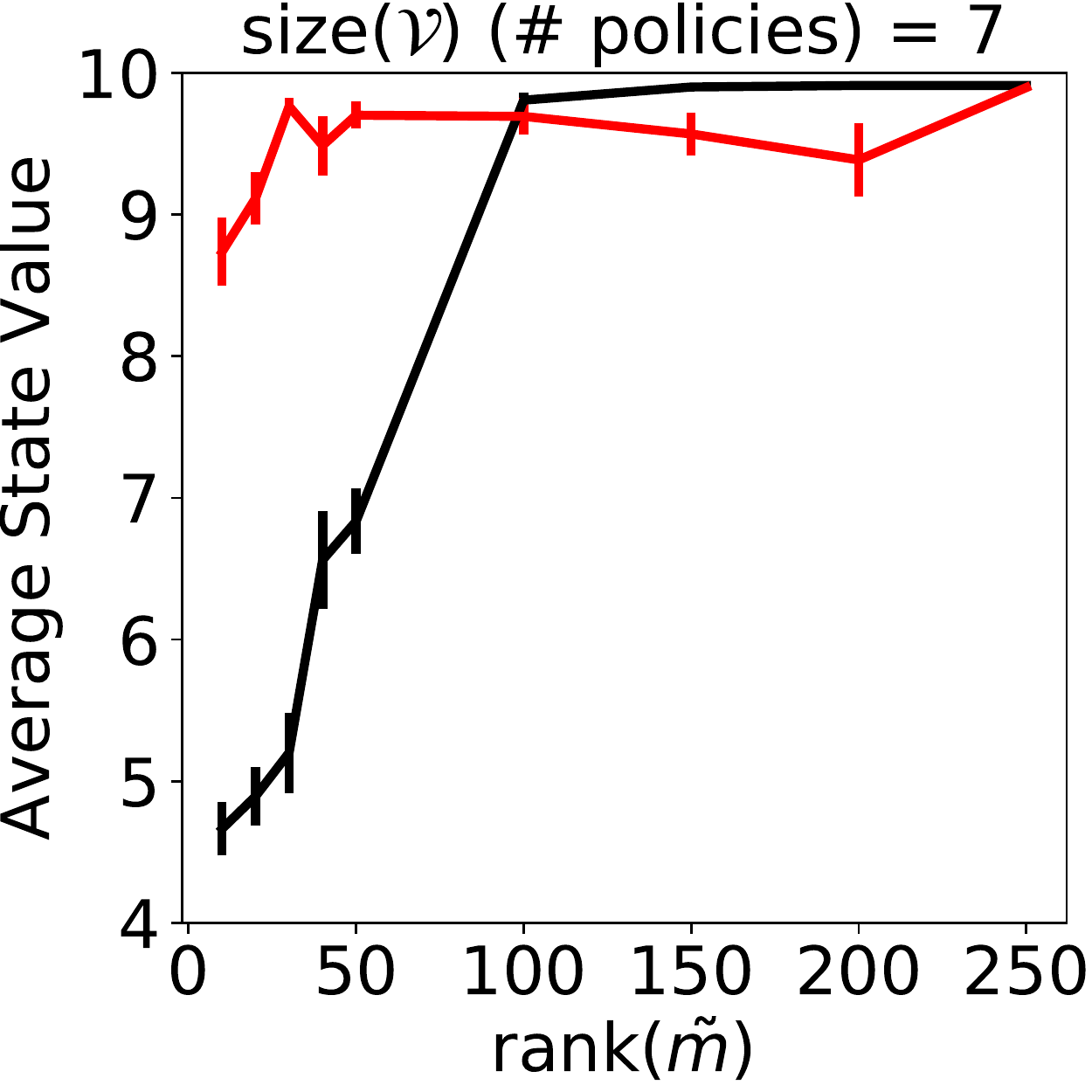}
 }
\subfigure[\textbf{Catch (fixed $\boldsymbol{\V}$)}]{
\includegraphics[scale=0.25]{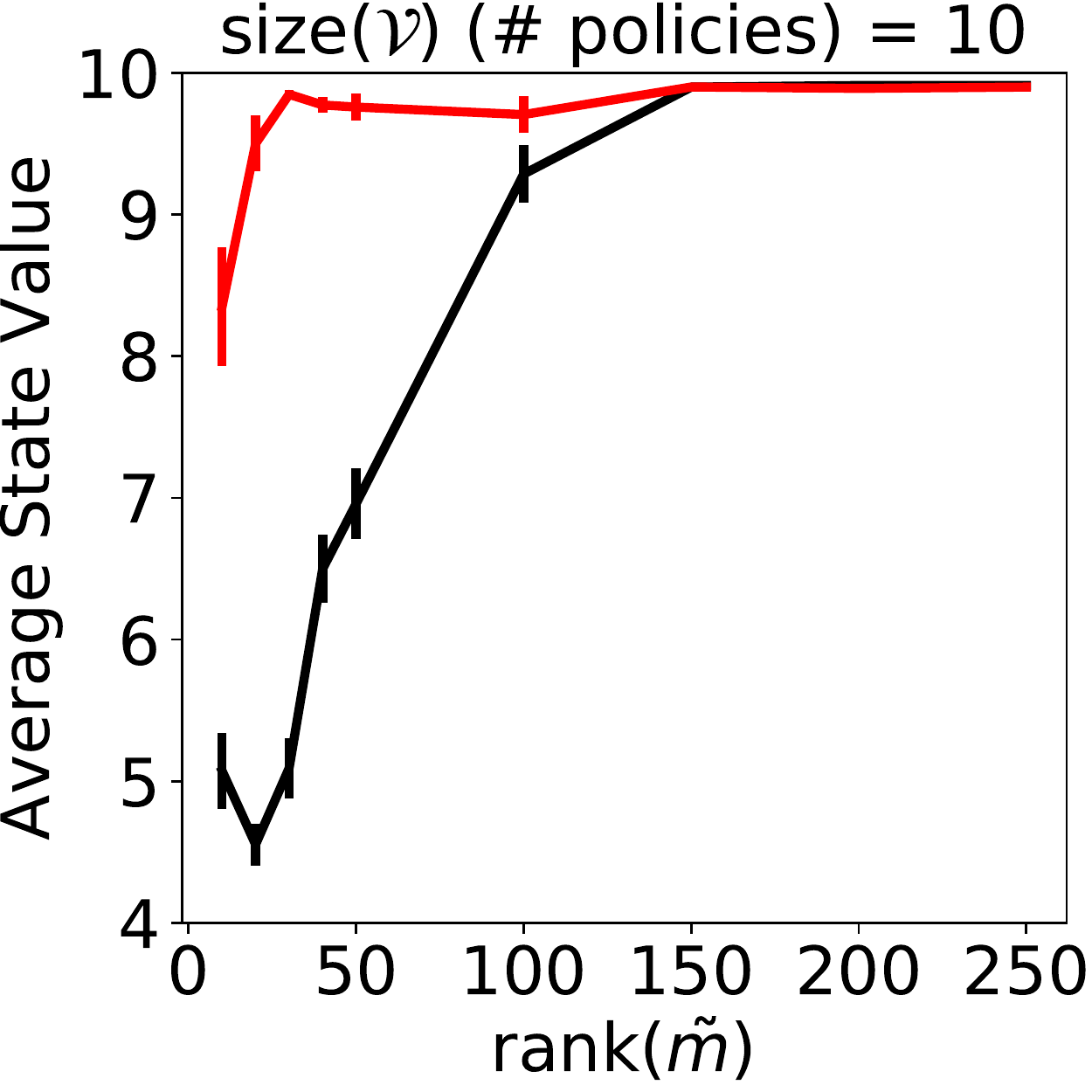}
}
\subfigure[Catch (fixed $\V$)]{
\includegraphics[scale=0.25]{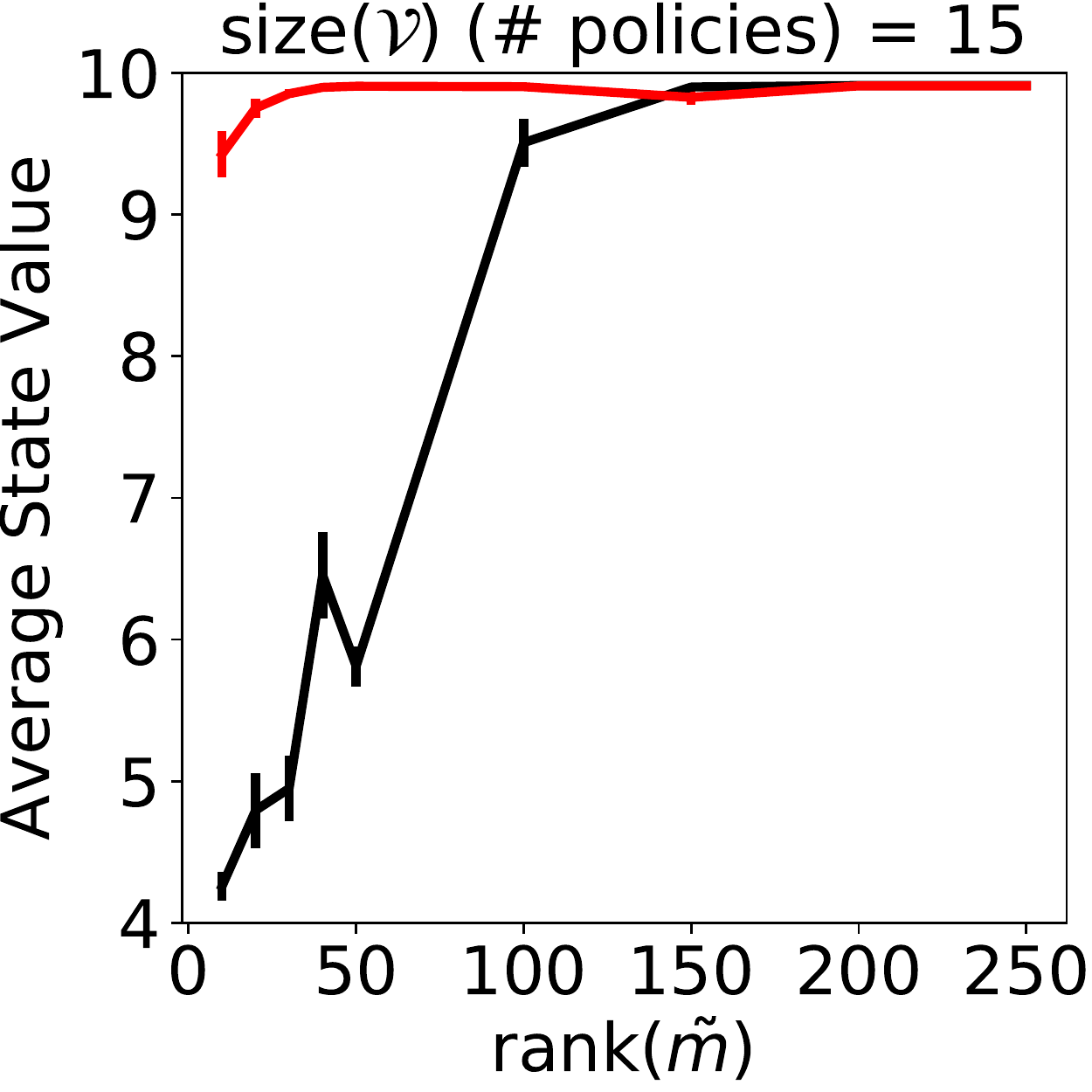}
}
\subfigure[Catch (fixed $\V$)]{
\includegraphics[scale=0.25]{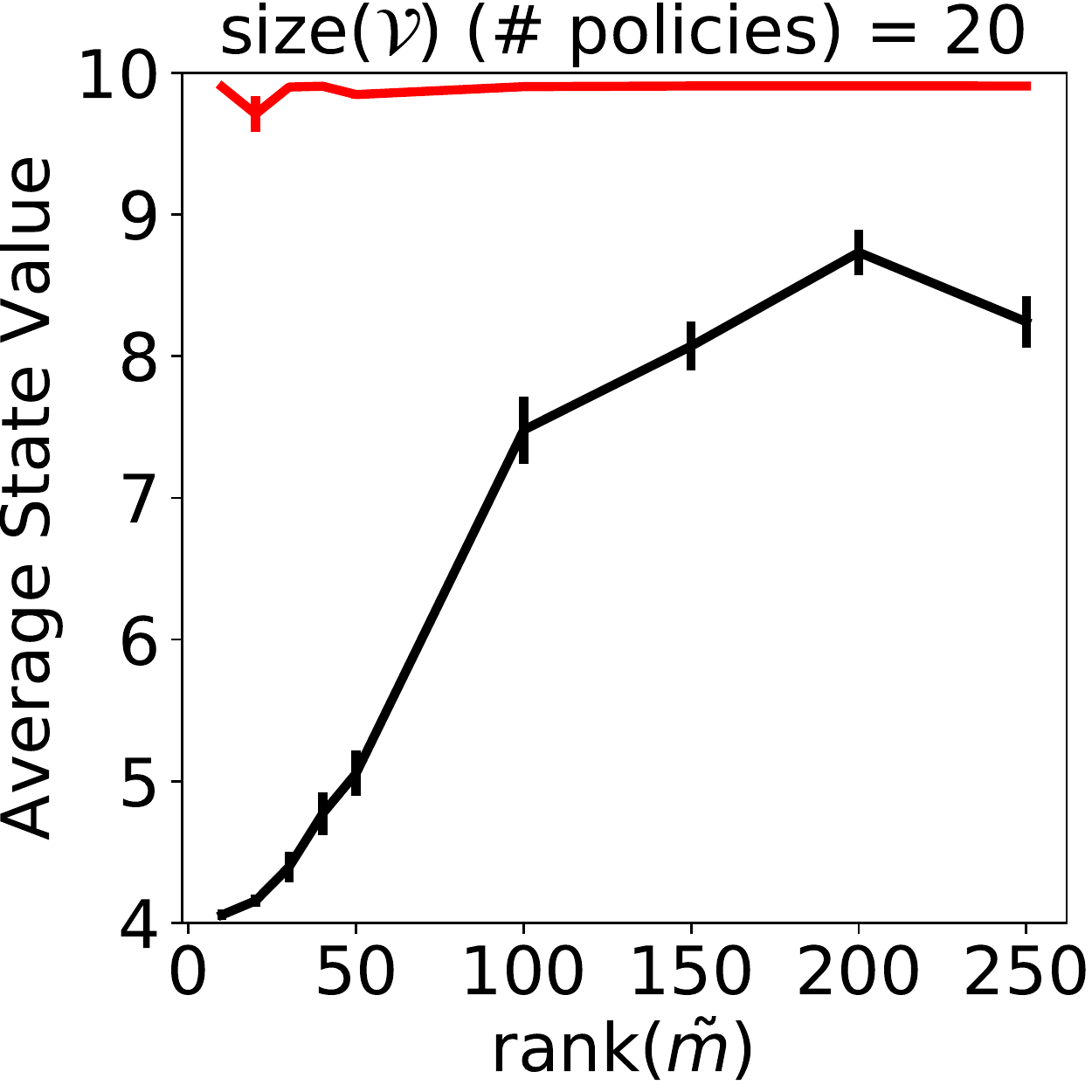}
}
\subfigure[Catch (fixed $\V$)]{
\includegraphics[scale=0.25]{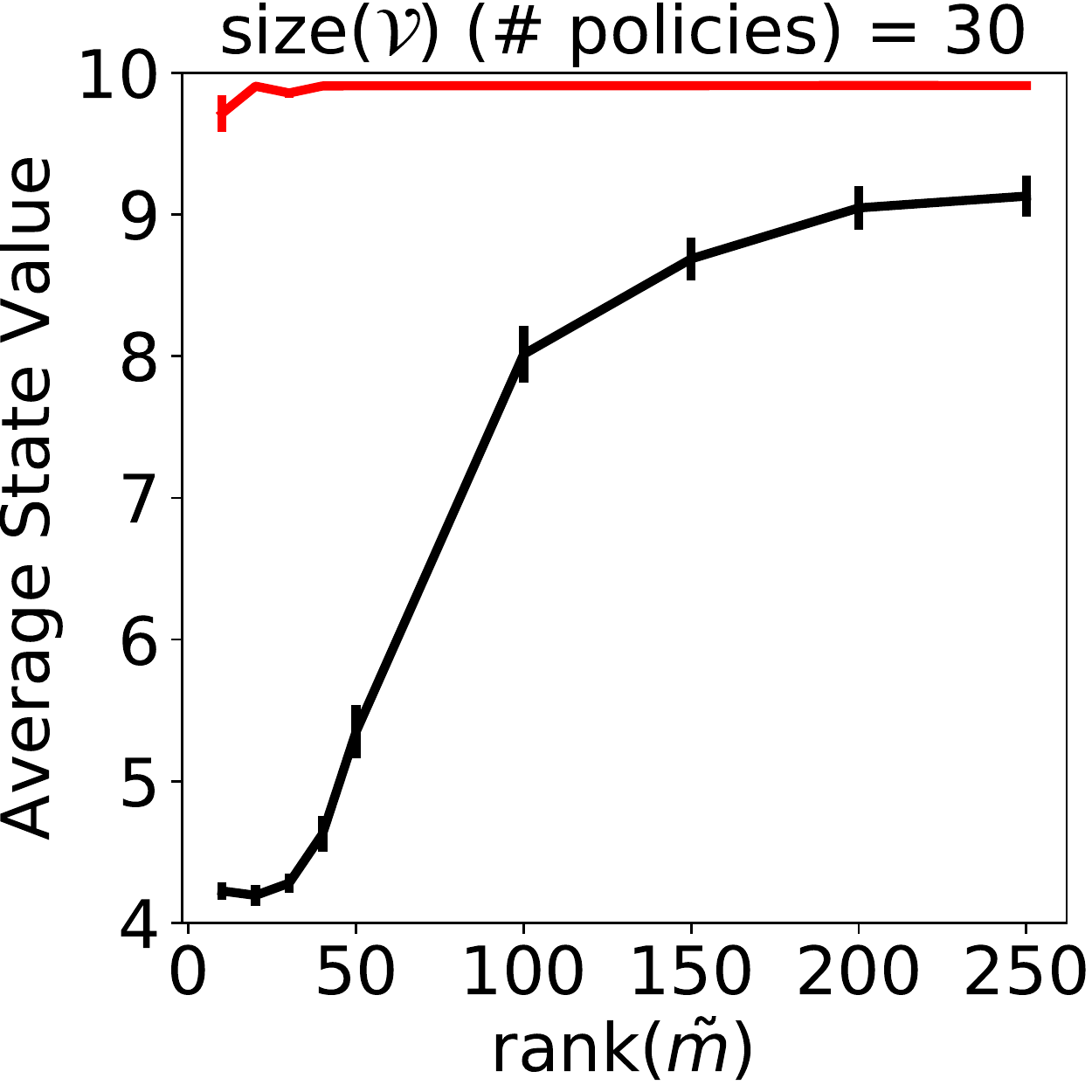}
}
\subfigure[Catch (fixed $\V$)]{
\includegraphics[scale=0.25]{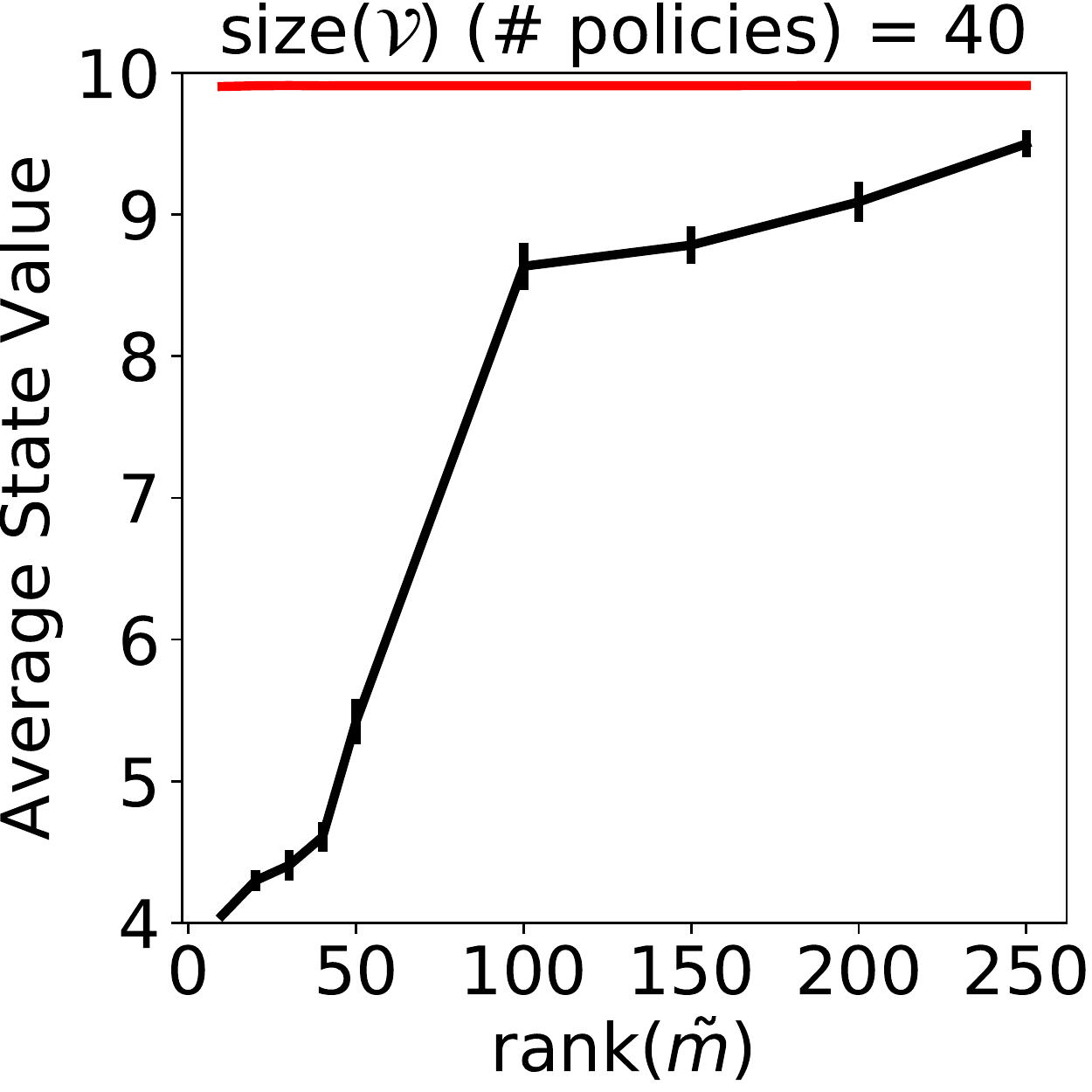}
}
\caption{All Catch results with fixed $\V$ and $\V = \{ v_{\pi_1}, \ldots, v_{\pi_n} \}$.}
\end{figure}

\begin{figure}[H]
\centering
\subfigure[Catch (fixed $\tilde{m}$)]{
\includegraphics[scale=0.25]{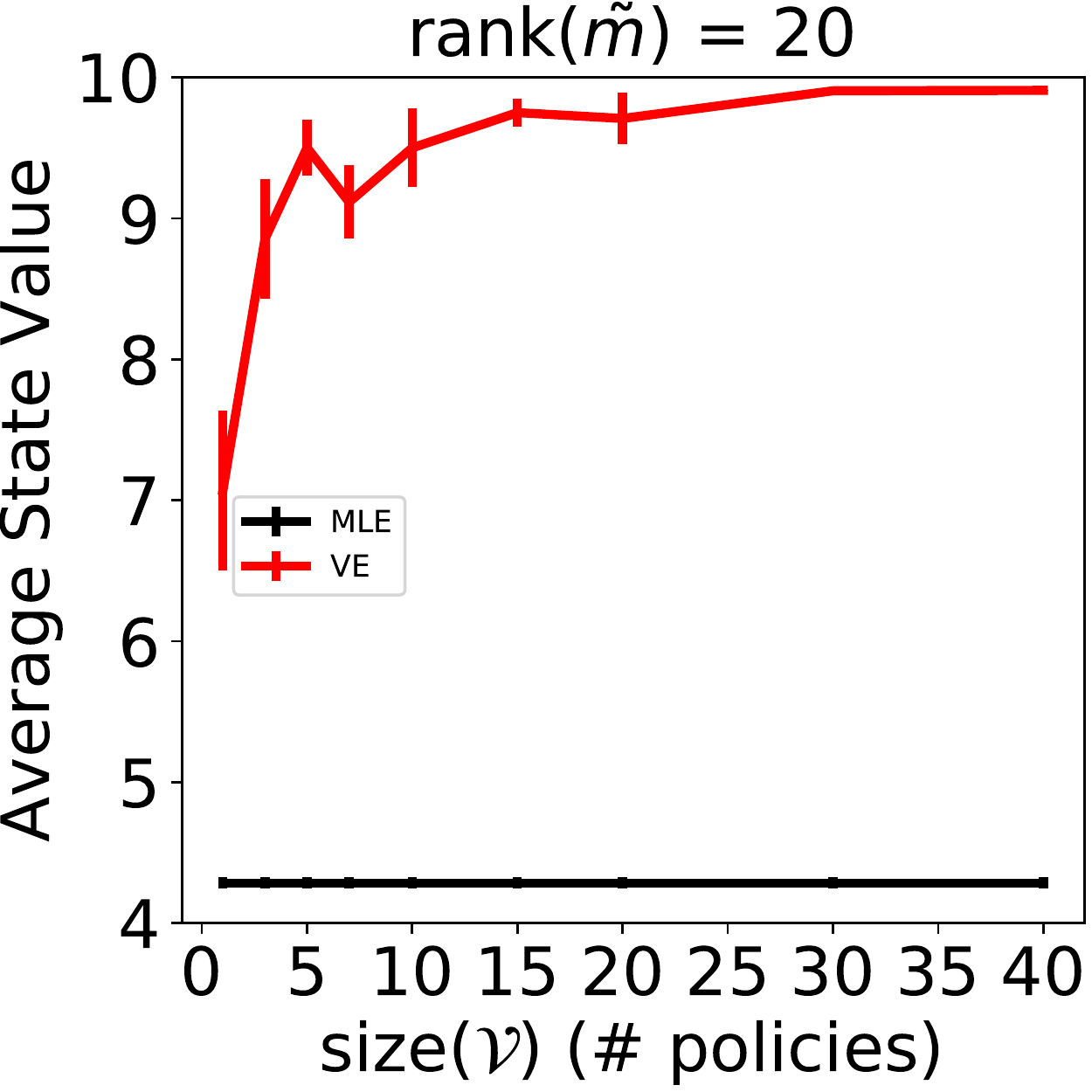}
}
\subfigure[Catch (fixed $\tilde{m}$)]{
\includegraphics[scale=0.25]{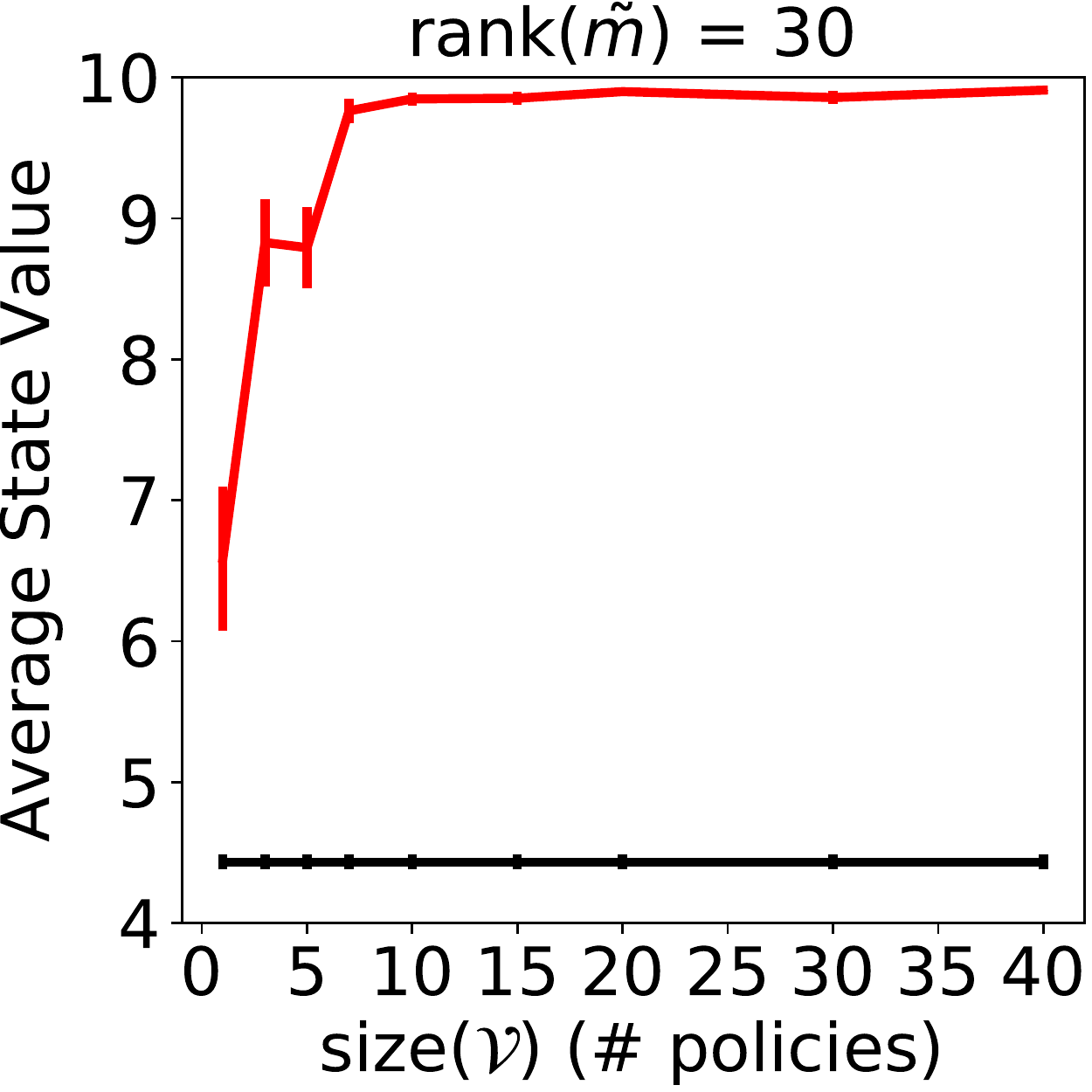}
}
\subfigure[Catch (fixed $\tilde{m}$)]{
\includegraphics[scale=0.25]{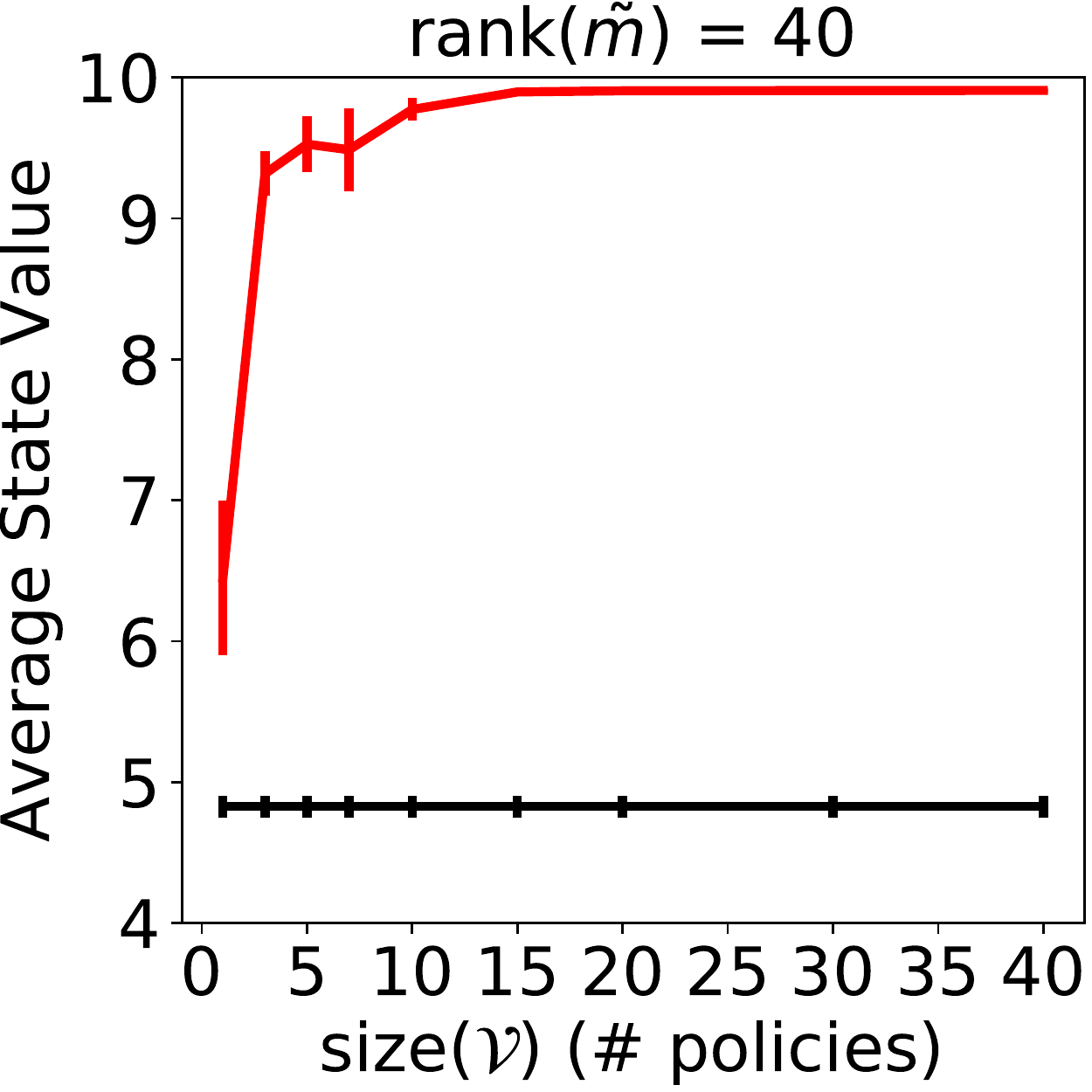}
}
\subfigure[\textbf{Catch (fixed $\boldsymbol{\tilde{m}}$)}]{
\includegraphics[scale=0.25]{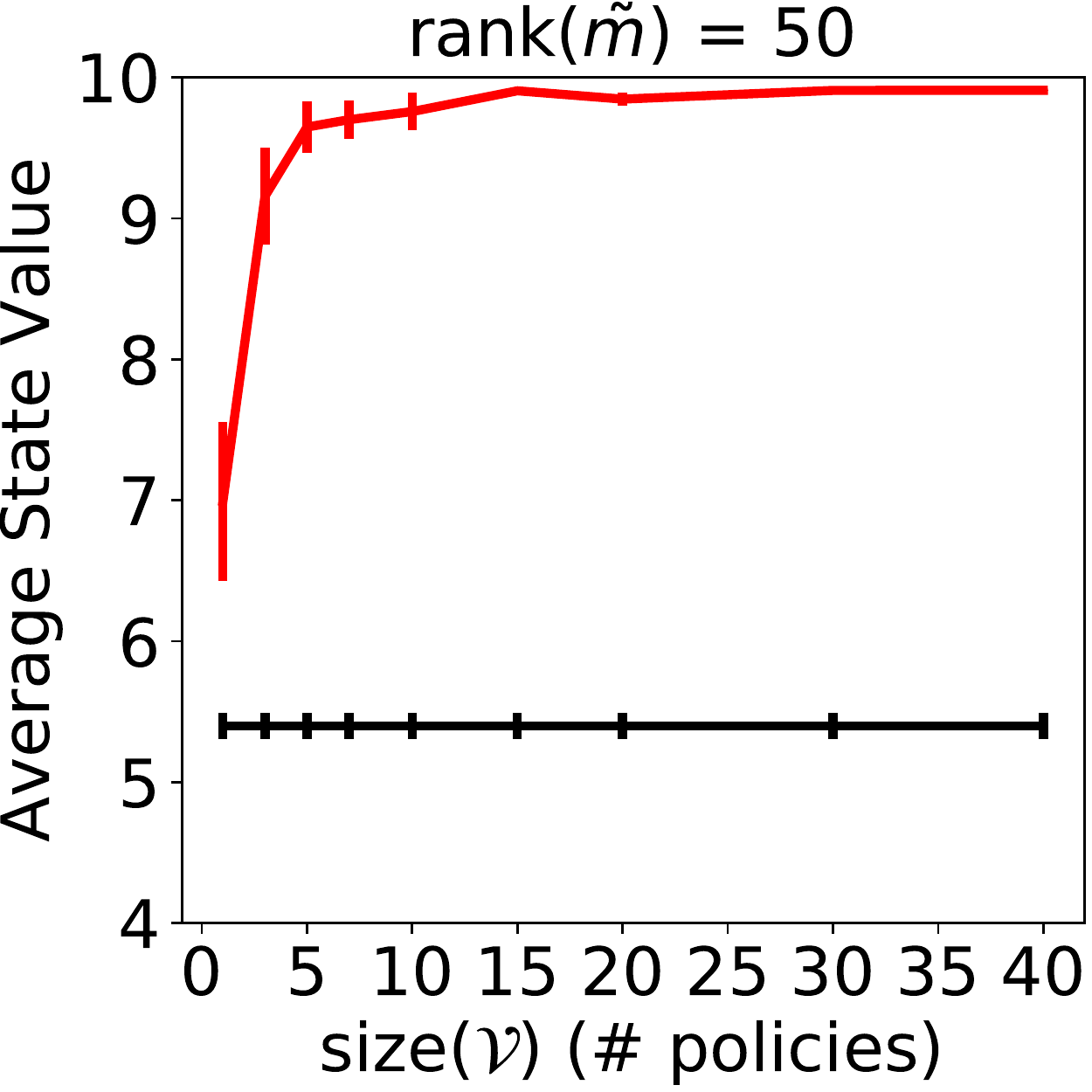}
}
\subfigure[Catch (fixed $\tilde{m}$)]{
\includegraphics[scale=0.25]{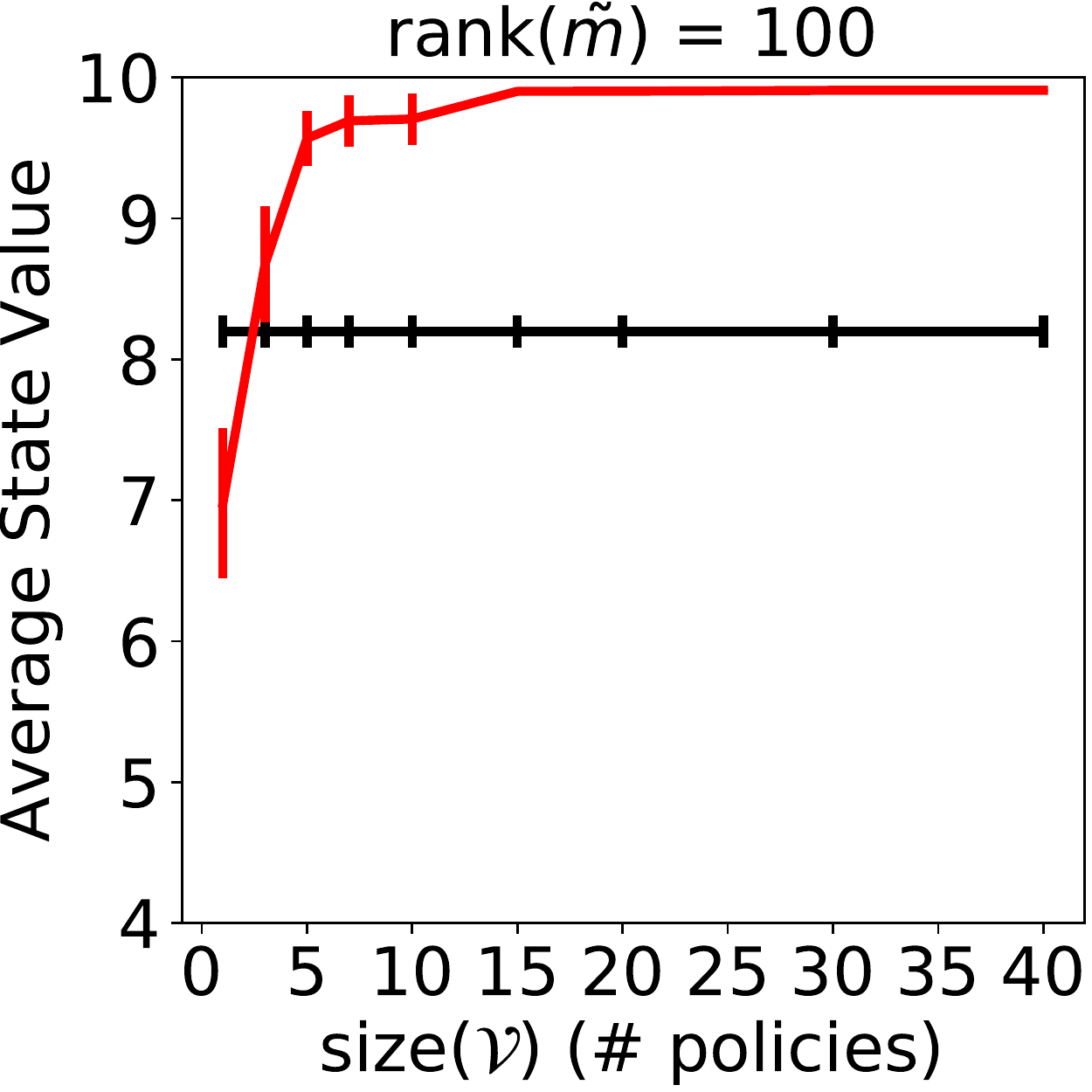}
}
\subfigure[Catch (fixed $\tilde{m}$)]{
\includegraphics[scale=0.25]{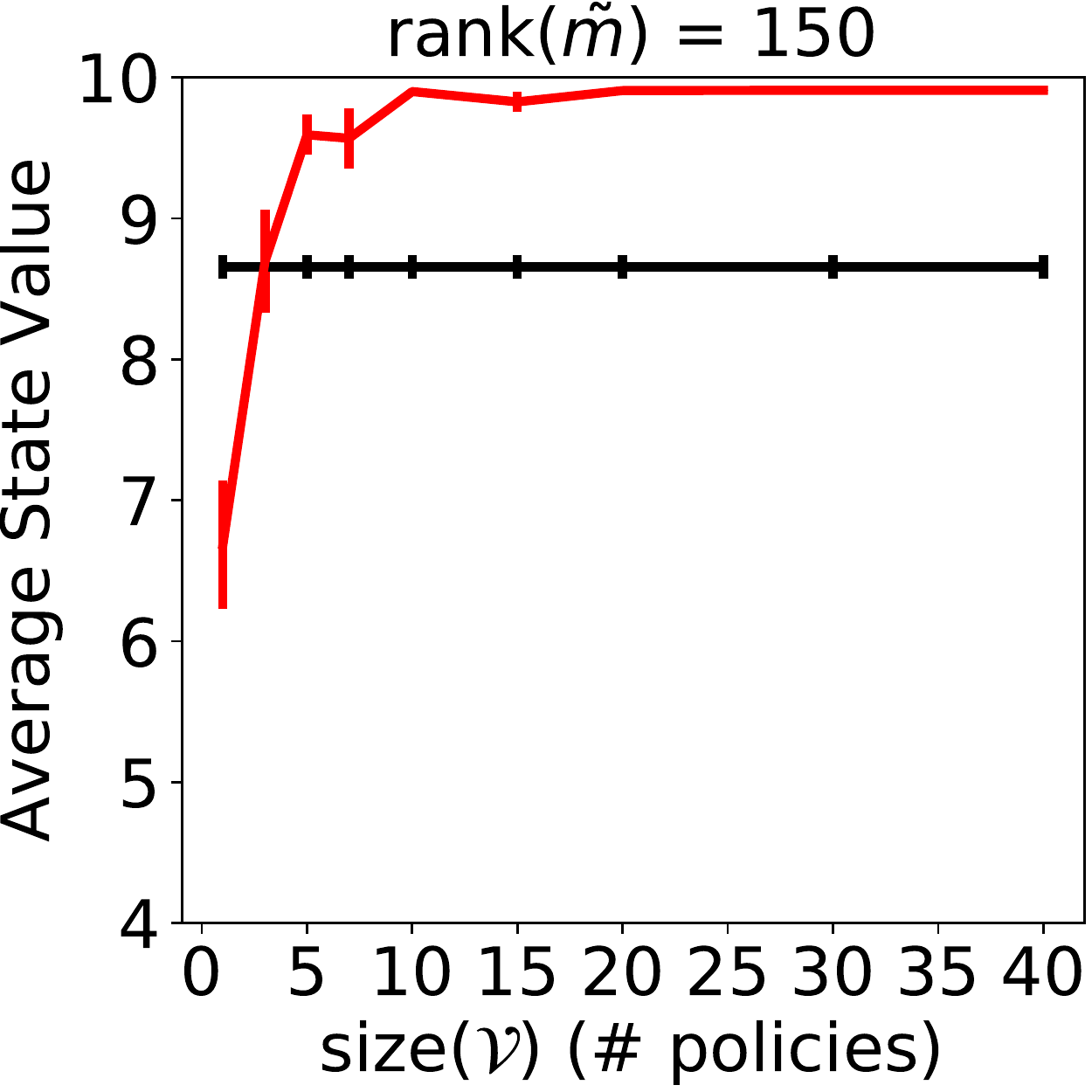}
}
\subfigure[Catch (fixed $\tilde{m}$)]{
\includegraphics[scale=0.25]{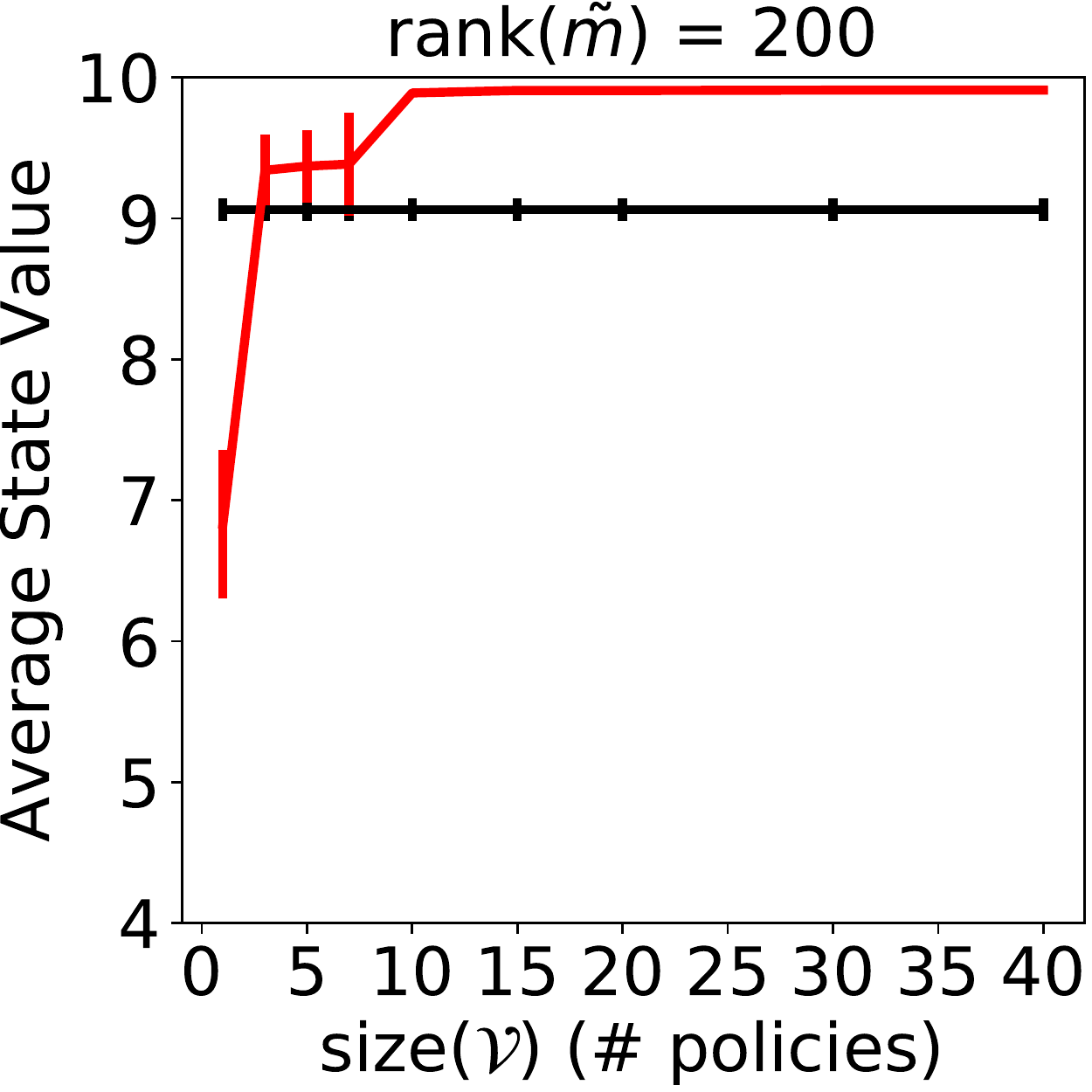}
}
\subfigure[Catch (fixed $\tilde{m}$)]{
\includegraphics[scale=0.25]{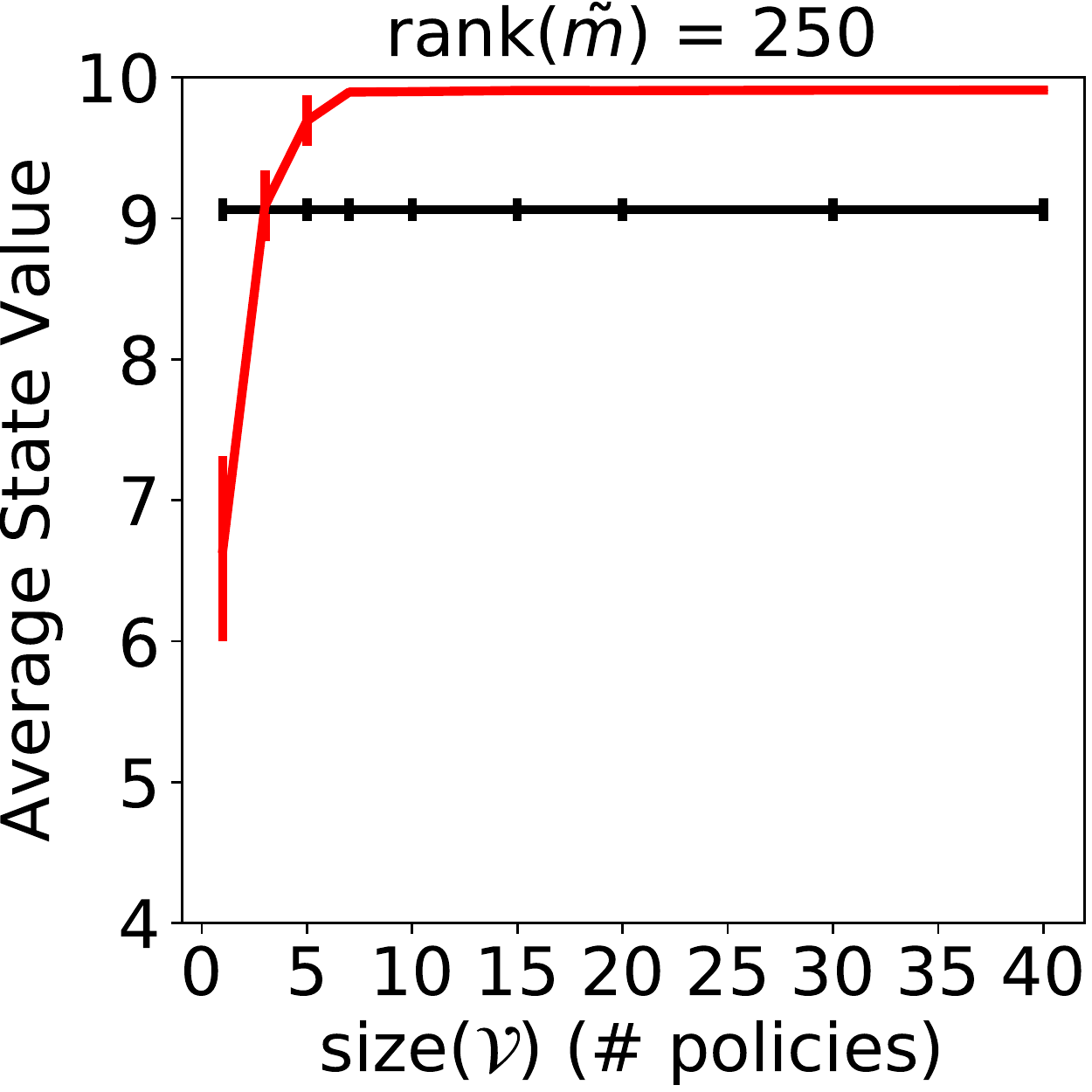}
}
\caption{All Catch results with fixed $\tilde{m}$ and $\V = \{ v_{\pi_1}, \ldots, v_{\pi_n} \}$.}
\end{figure}

\begin{figure}[H]
\centering
\subfigure[Four Rooms (fixed $\V$)]{
\includegraphics[scale=0.25]{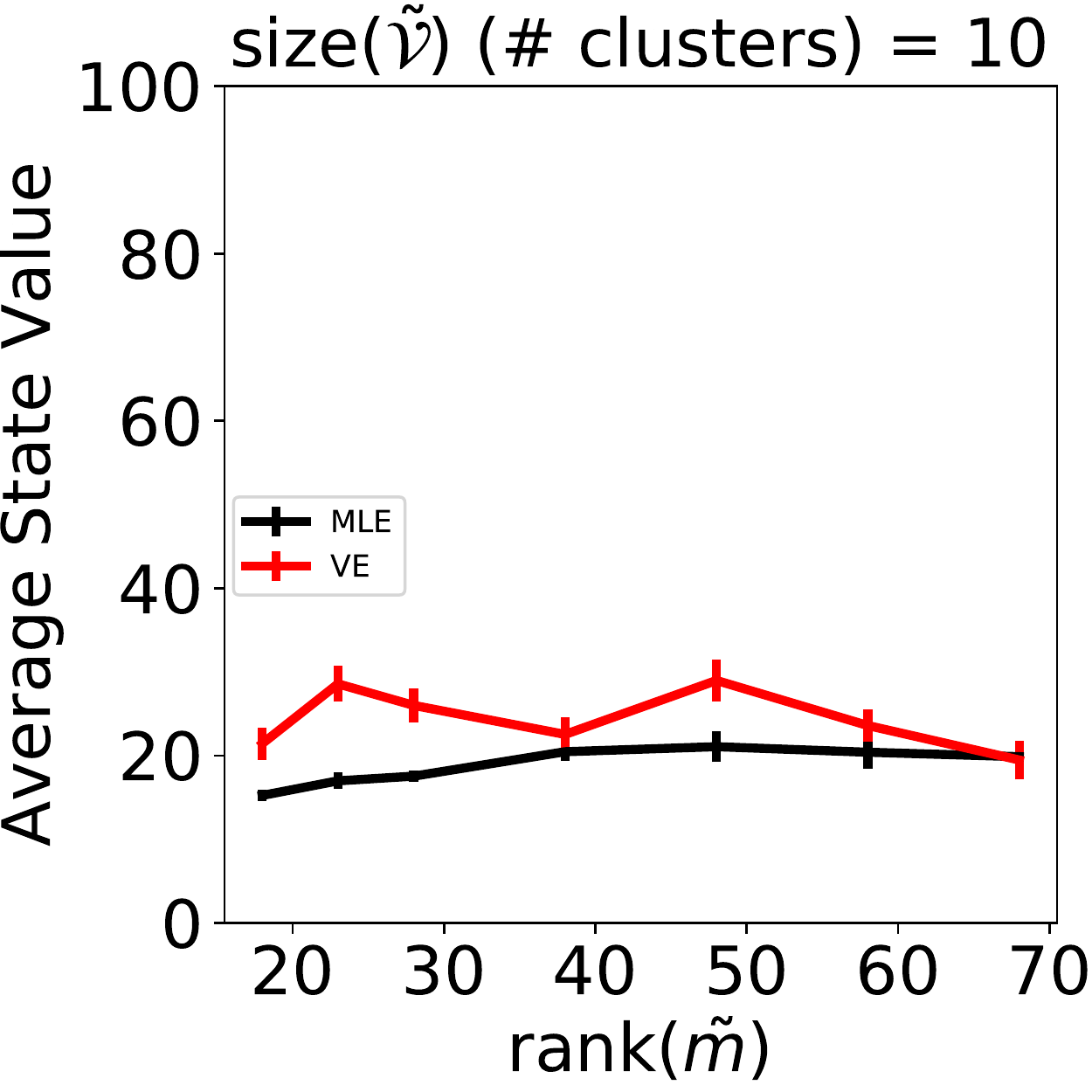}
}
\subfigure[Four Rooms (fixed $\V$)]{
\includegraphics[scale=0.25]{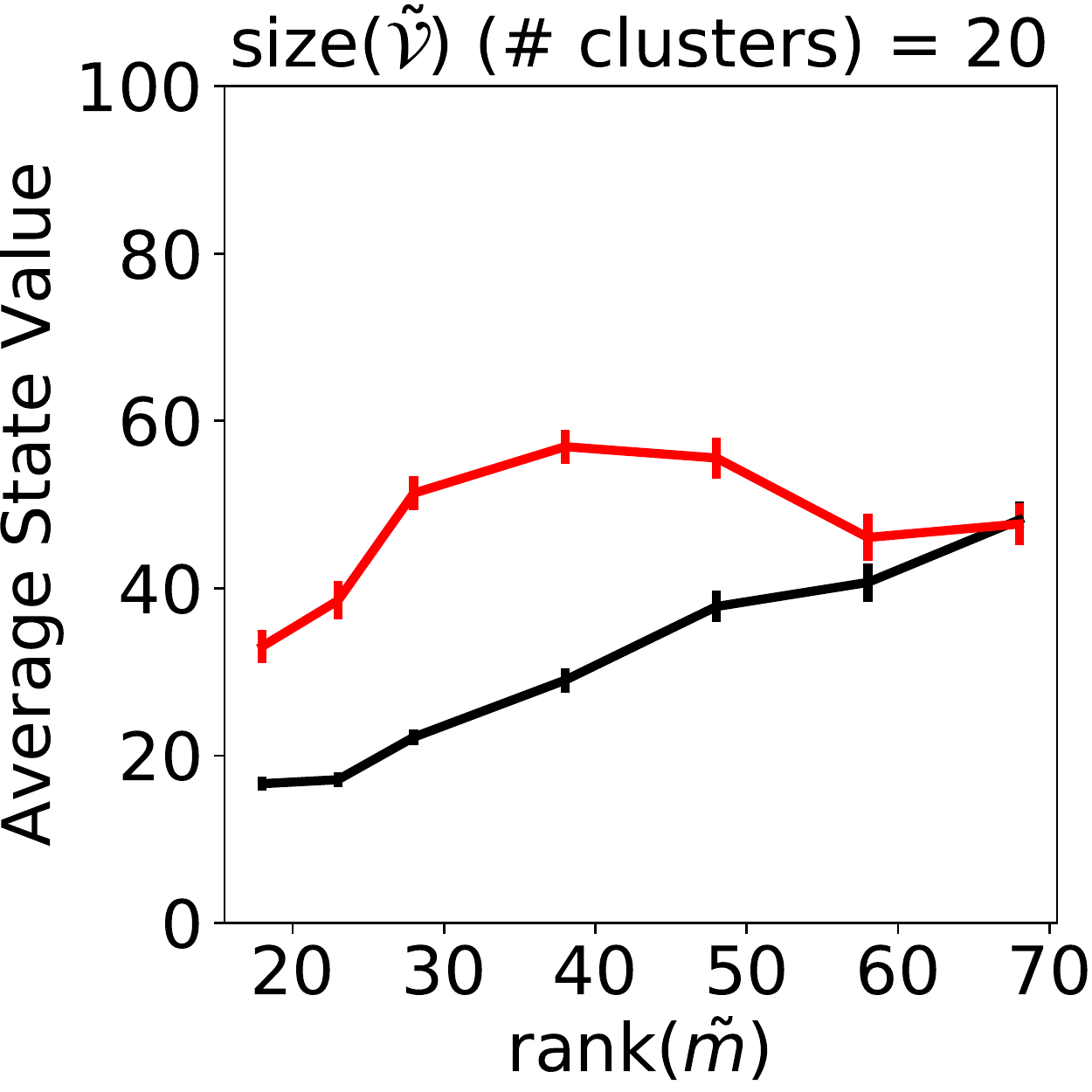}
}
\subfigure[Four Rooms (fixed $\V$)]{
\includegraphics[scale=0.25]{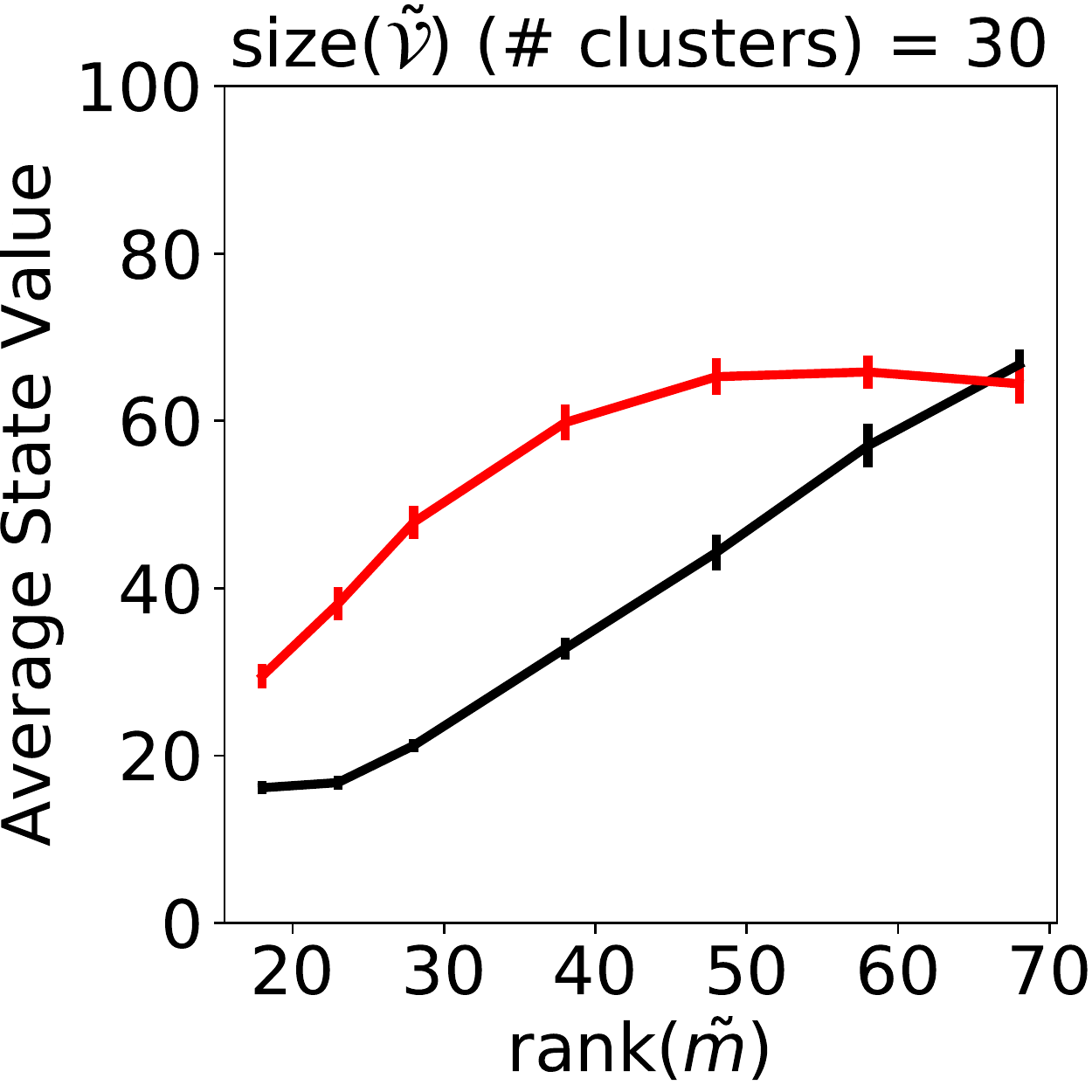}
}

\subfigure[Four Rooms (fixed $\V$)]{
\includegraphics[scale=0.25]{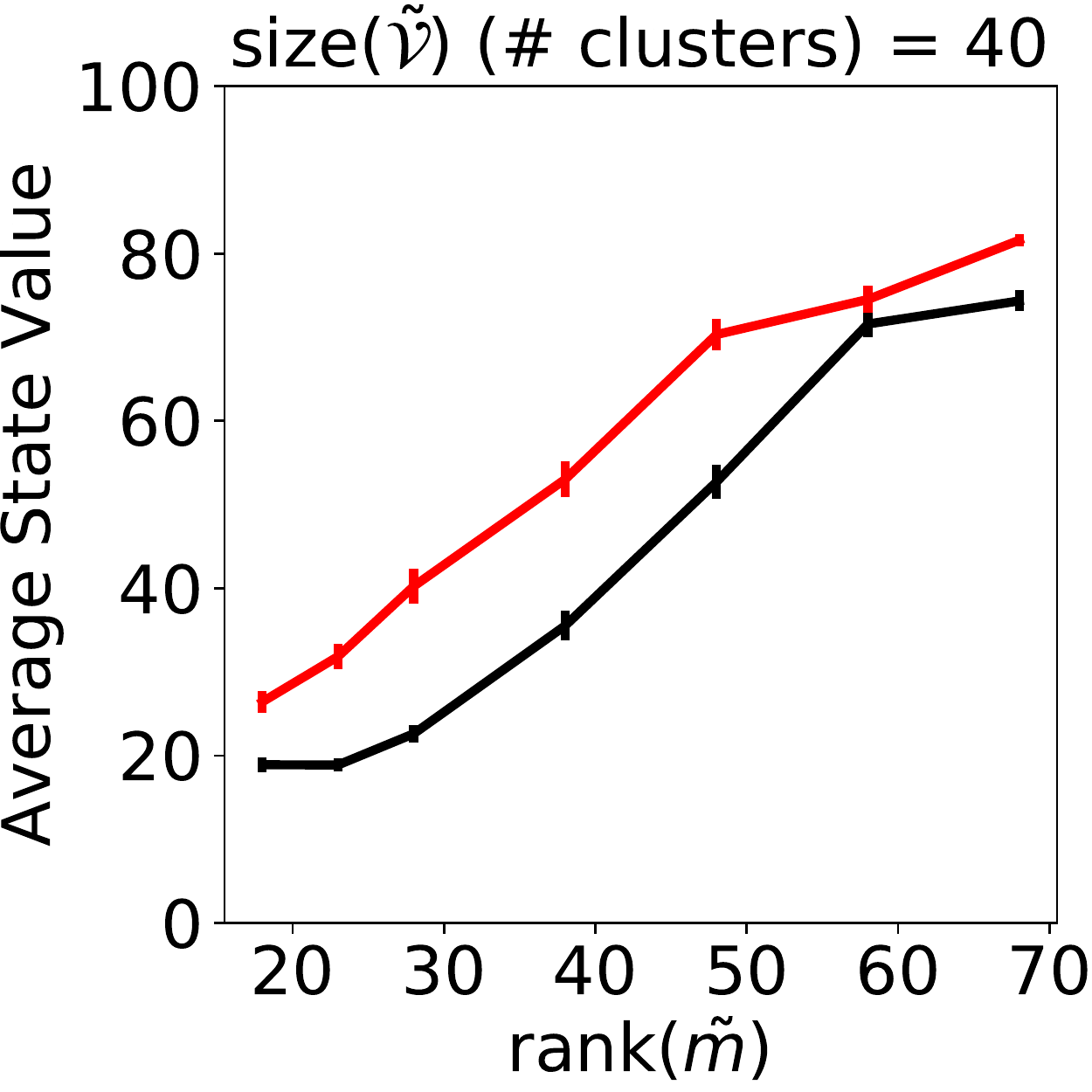}
}
\subfigure[Four Rooms (fixed $\V$)]{
\includegraphics[scale=0.25]{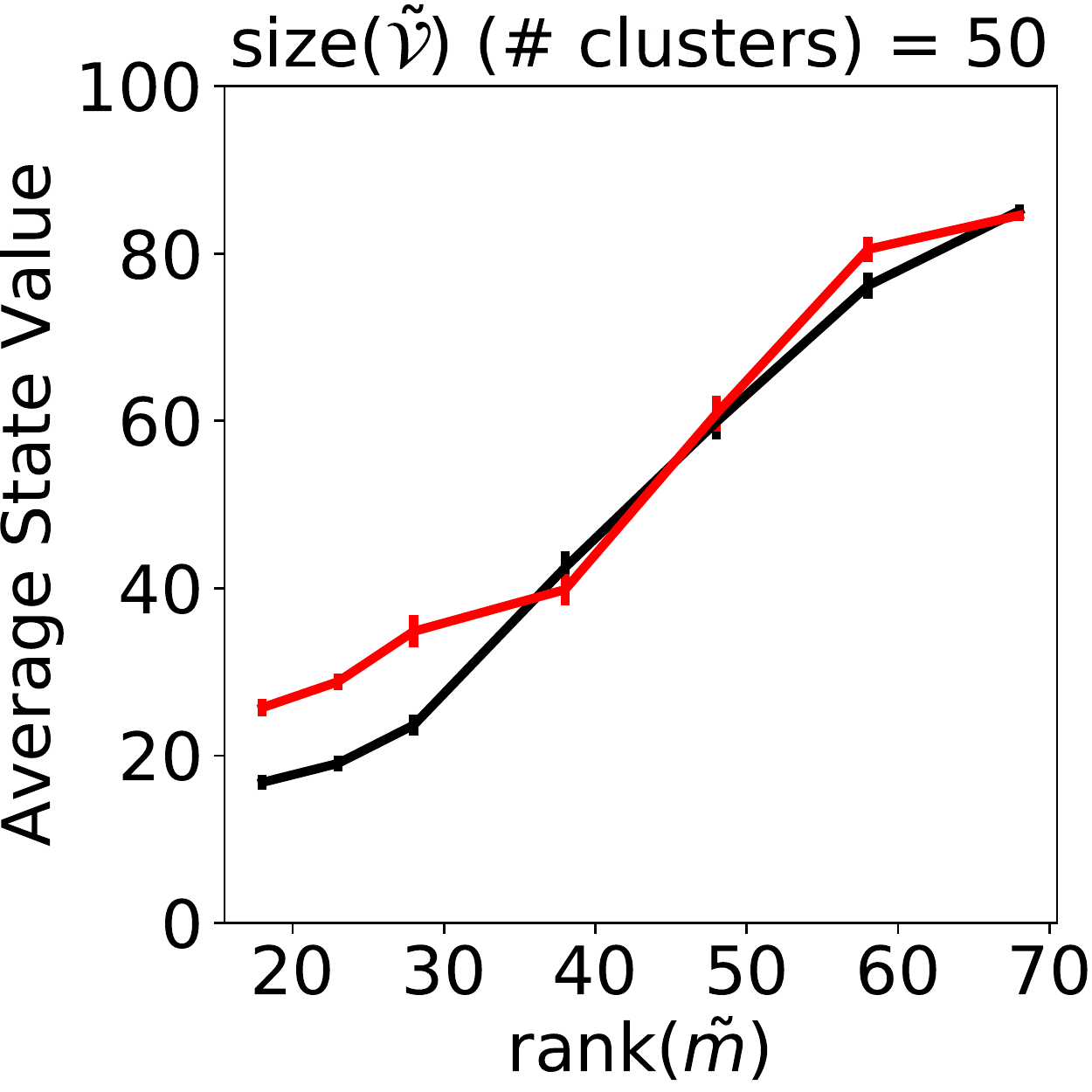}
}
\subfigure[Four Rooms (fixed $\V$)]{
\includegraphics[scale=0.25]{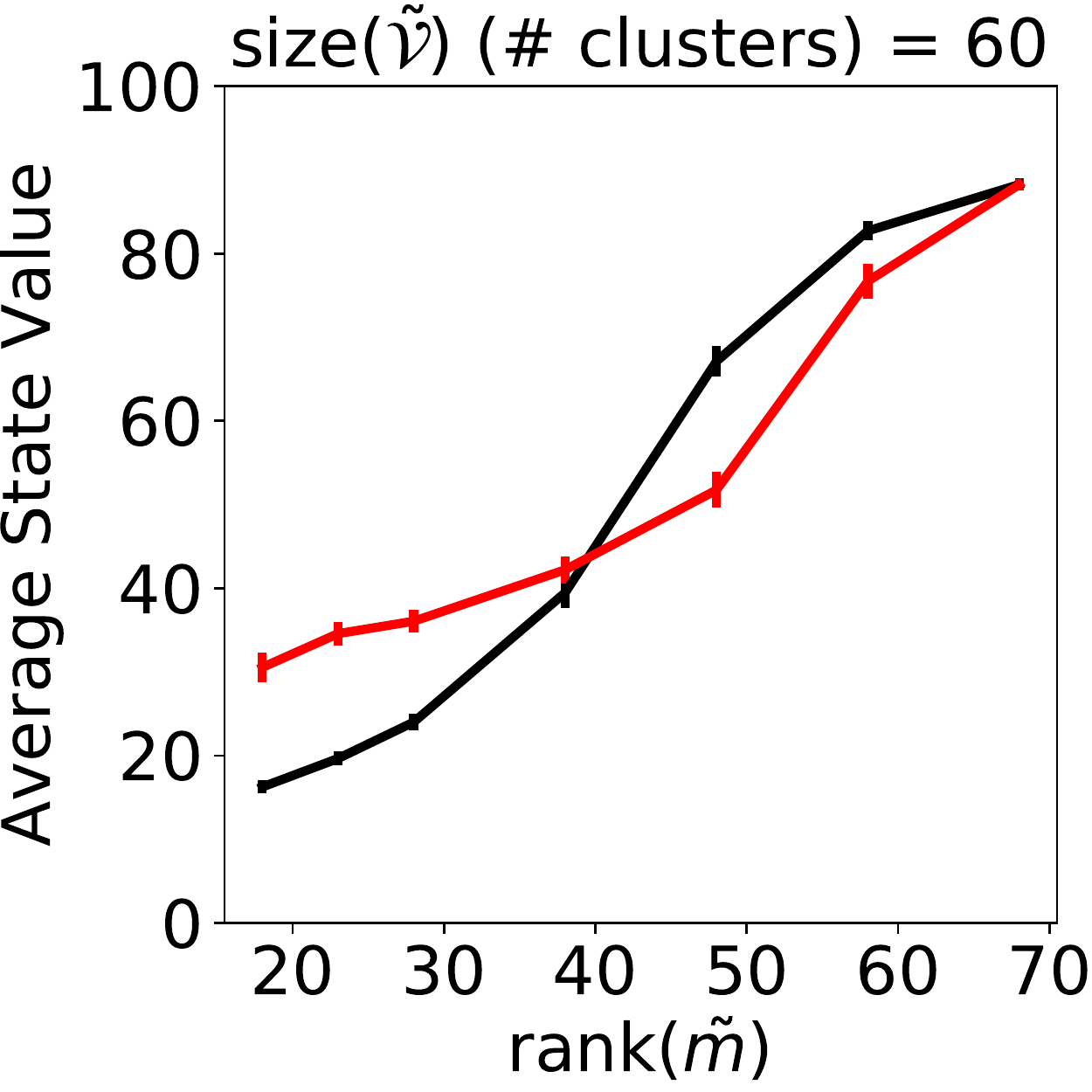}
}
\caption{All Four Rooms results with fixed $\V$ and $\V = \AV$.}
\end{figure}

\begin{figure}[H]
\centering
\subfigure[Four Rooms (fixed $\mt$)]{
\includegraphics[scale=0.25]{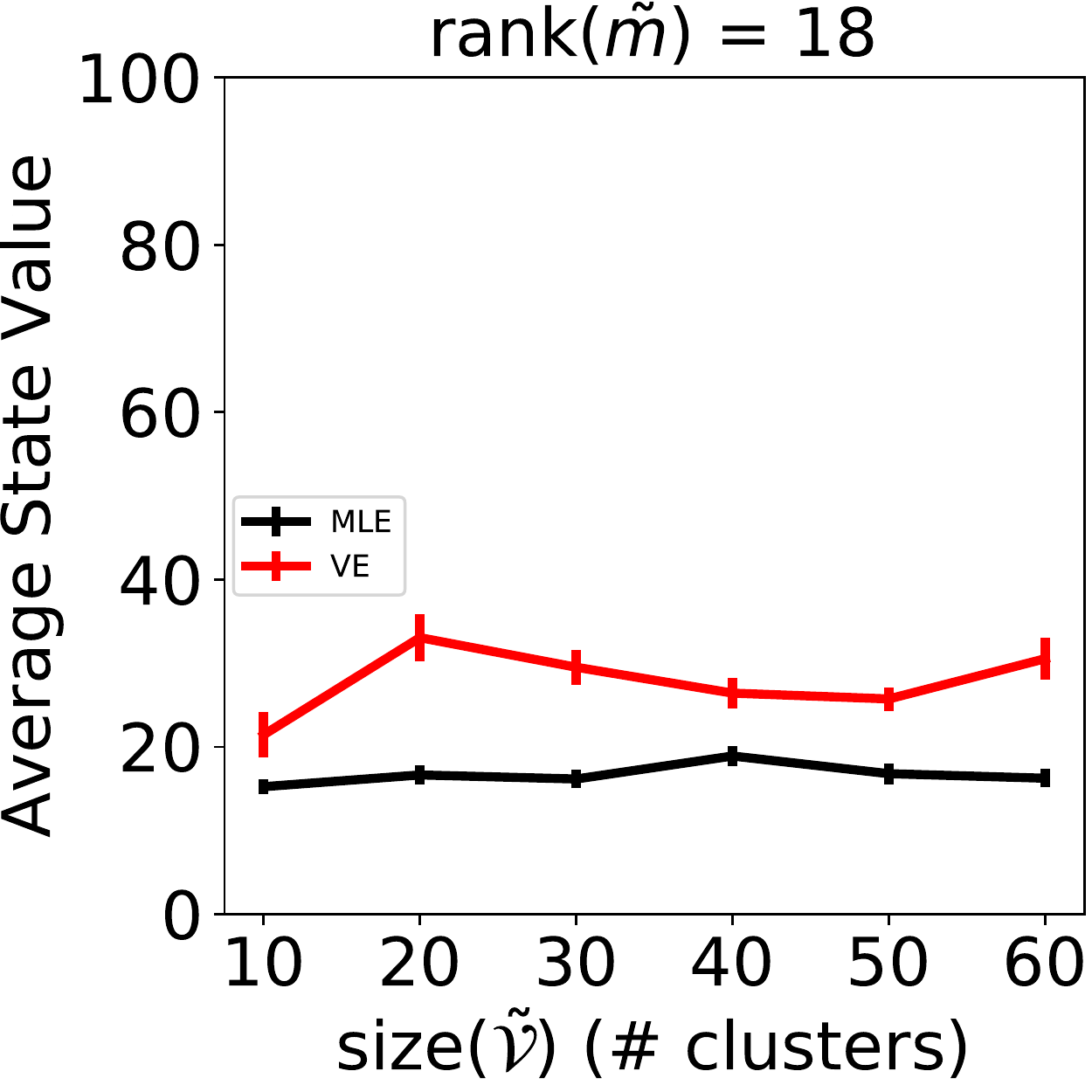}
}
\subfigure[Four Rooms (fixed $\mt$)]{
\includegraphics[scale=0.25]{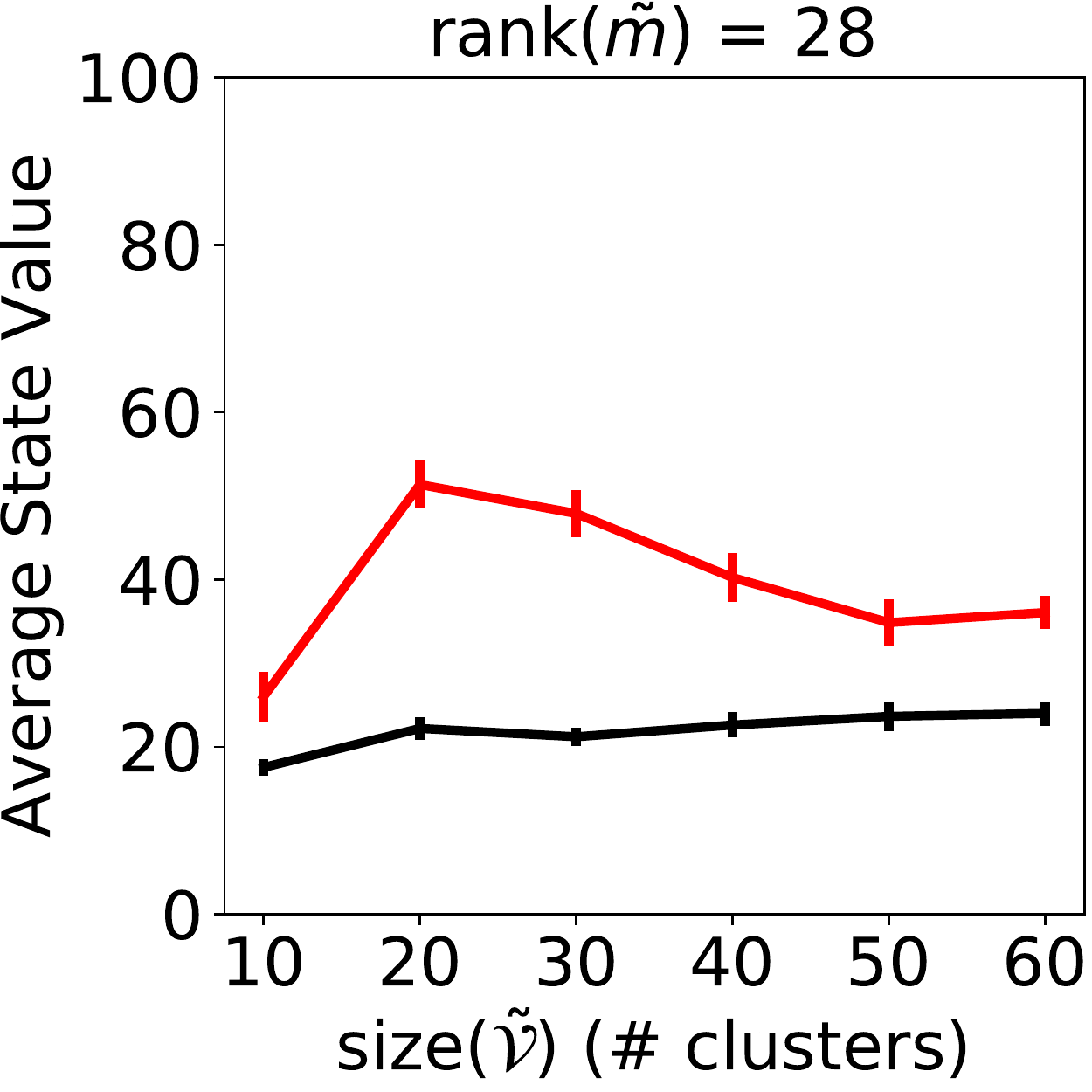}
}
\subfigure[Four Rooms (fixed $\mt$)]{
\includegraphics[scale=0.25]{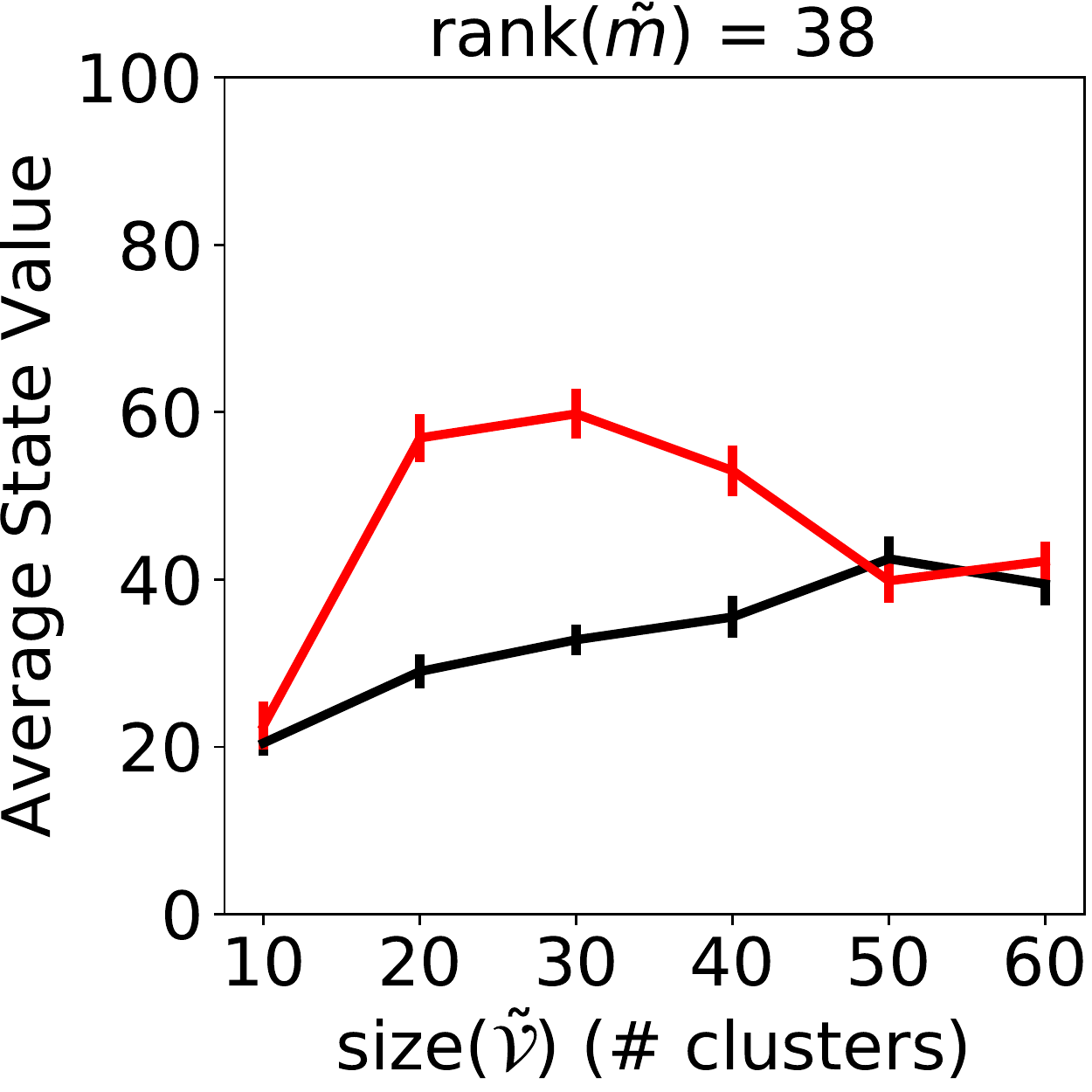}
}

\subfigure[Four Rooms (fixed $\mt$)]{
\includegraphics[scale=0.25]{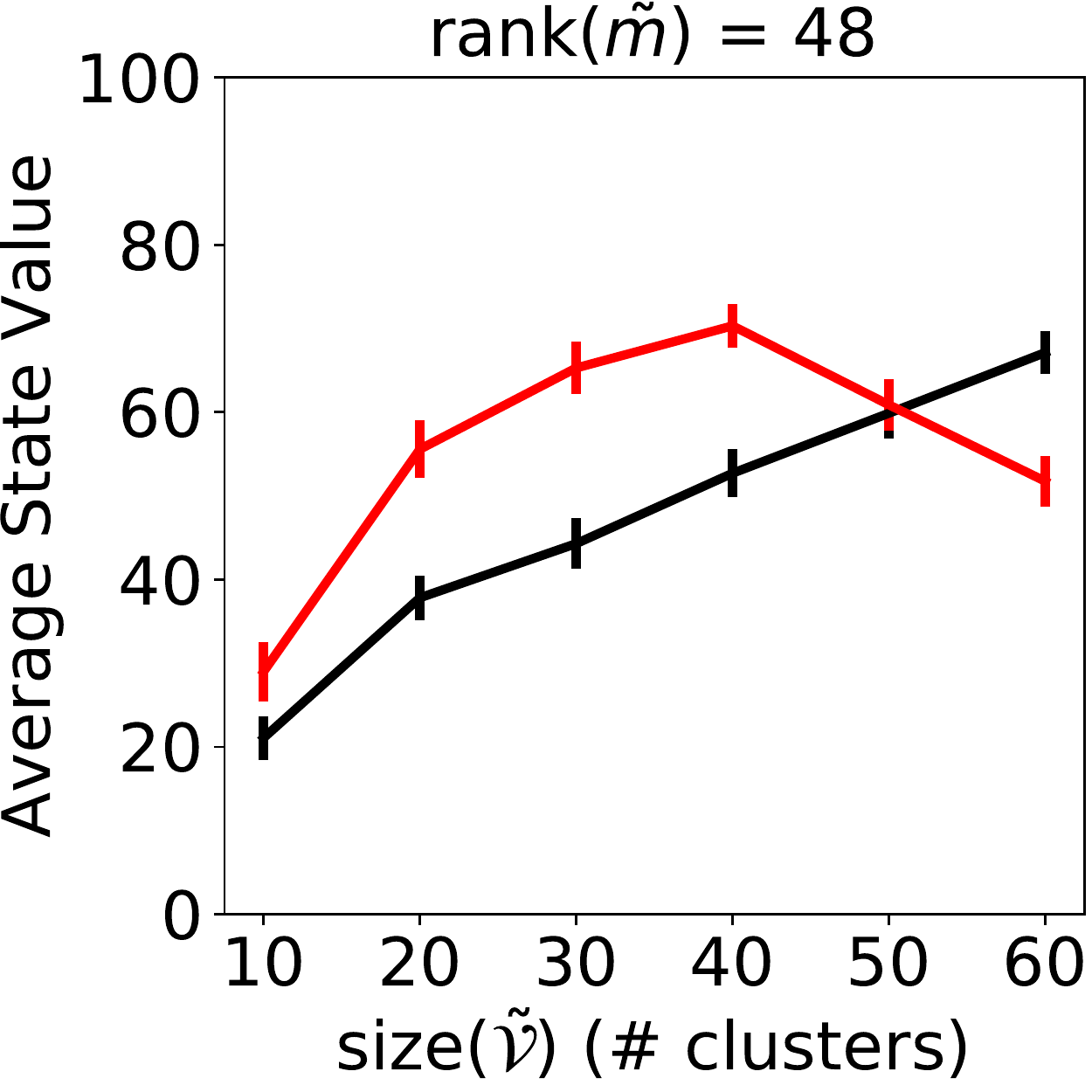}
}
\subfigure[Four Rooms (fixed $\mt$)]{
\includegraphics[scale=0.25]{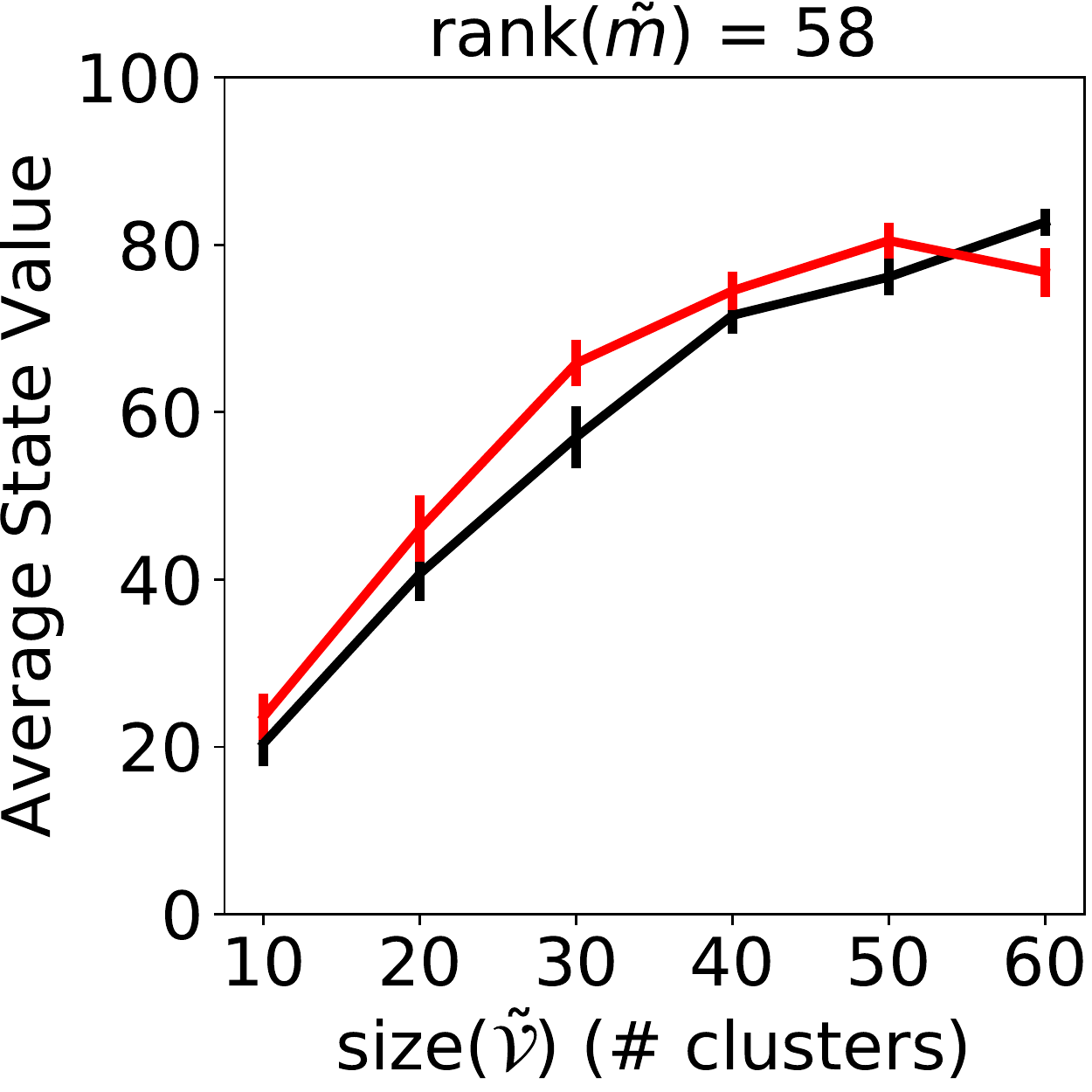}
}
\subfigure[Four Rooms (fixed $\mt$)]{
\includegraphics[scale=0.25]{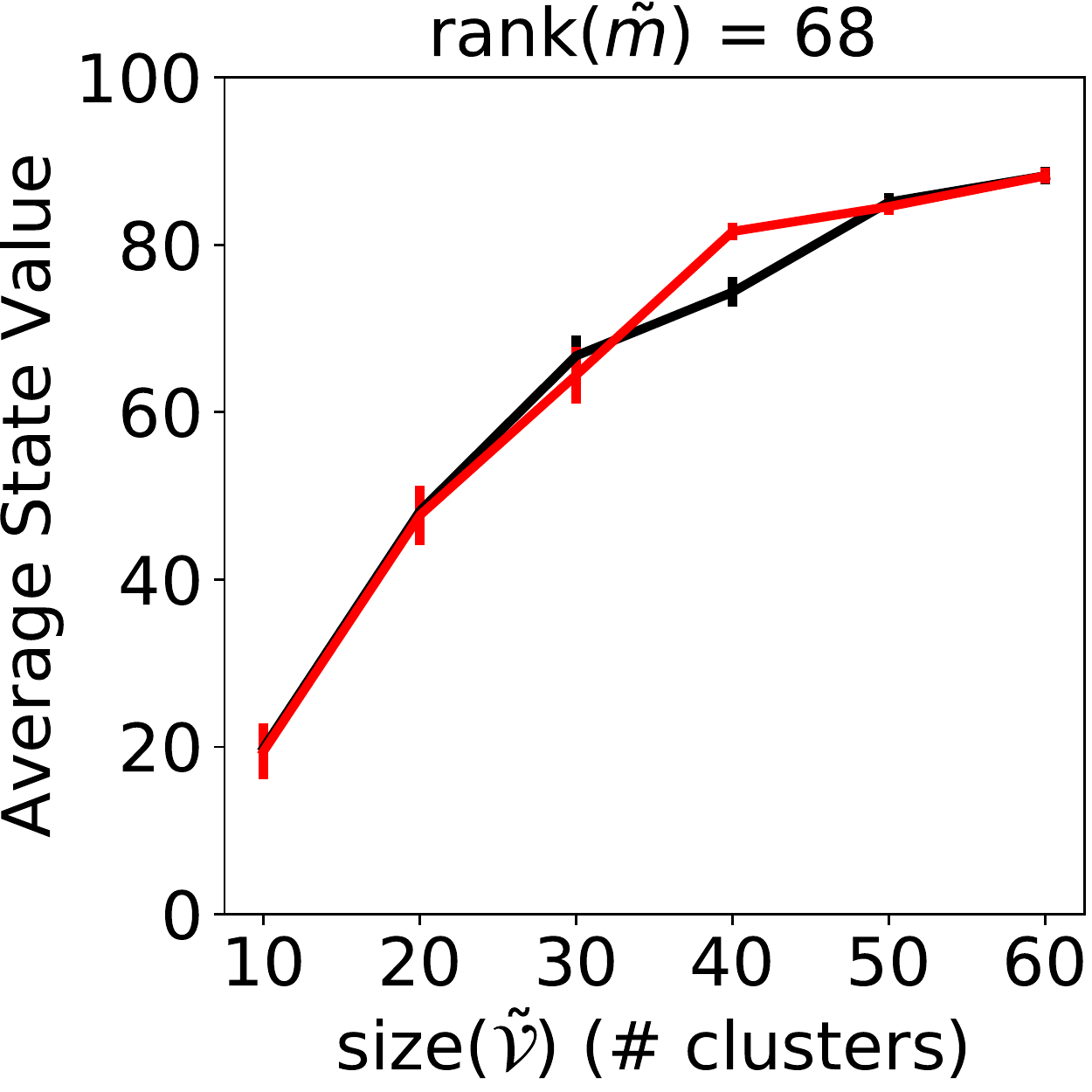}
}
\caption{All Four Rooms results with fixed $\mt$ and $\V = \AV$.}
\end{figure}

\begin{figure}[H]
\centering
\subfigure[Four Rooms (fixed $\V$)]{
\includegraphics[scale=0.25]{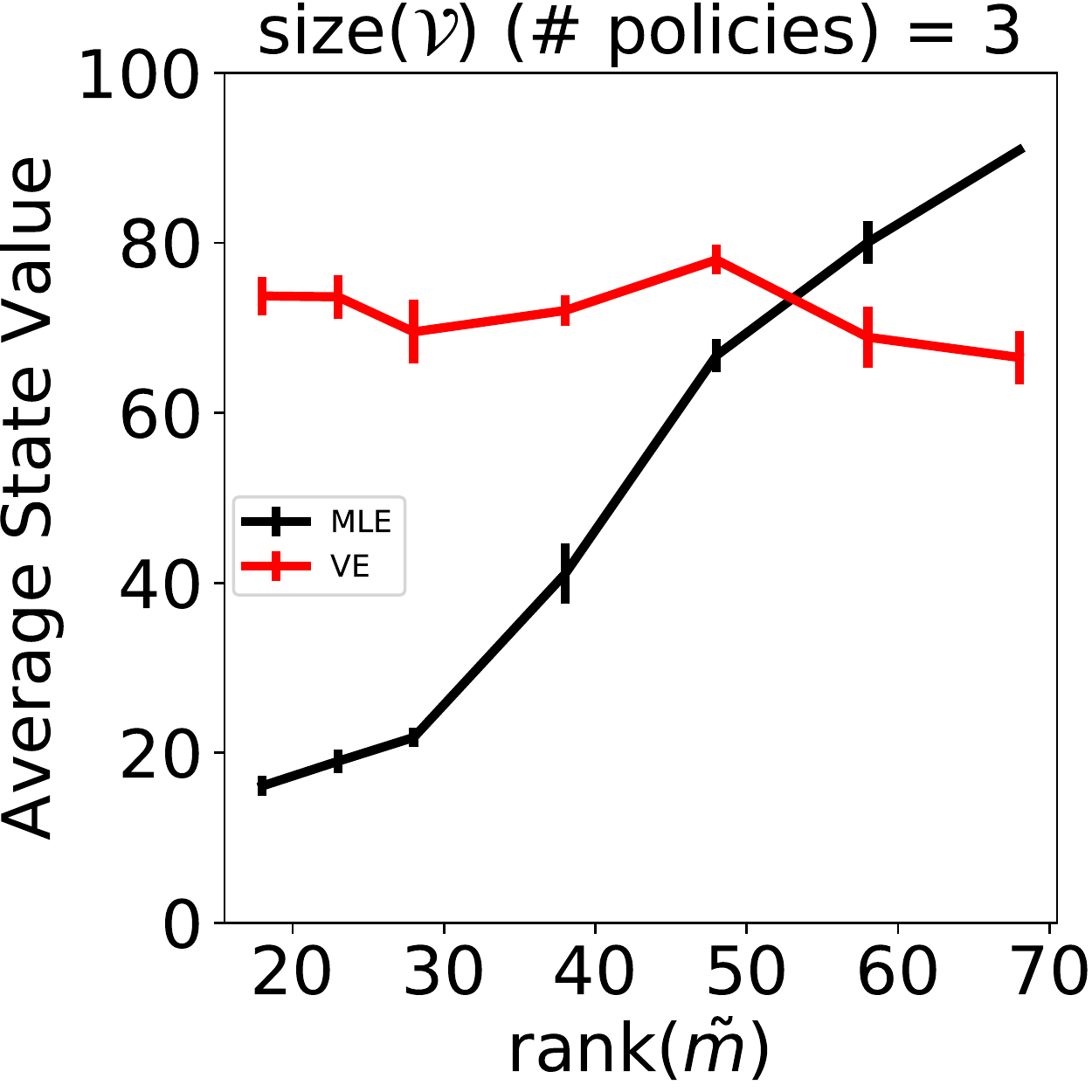}
}
\subfigure[Four Rooms (fixed $\V$)]{
\includegraphics[scale=0.25]{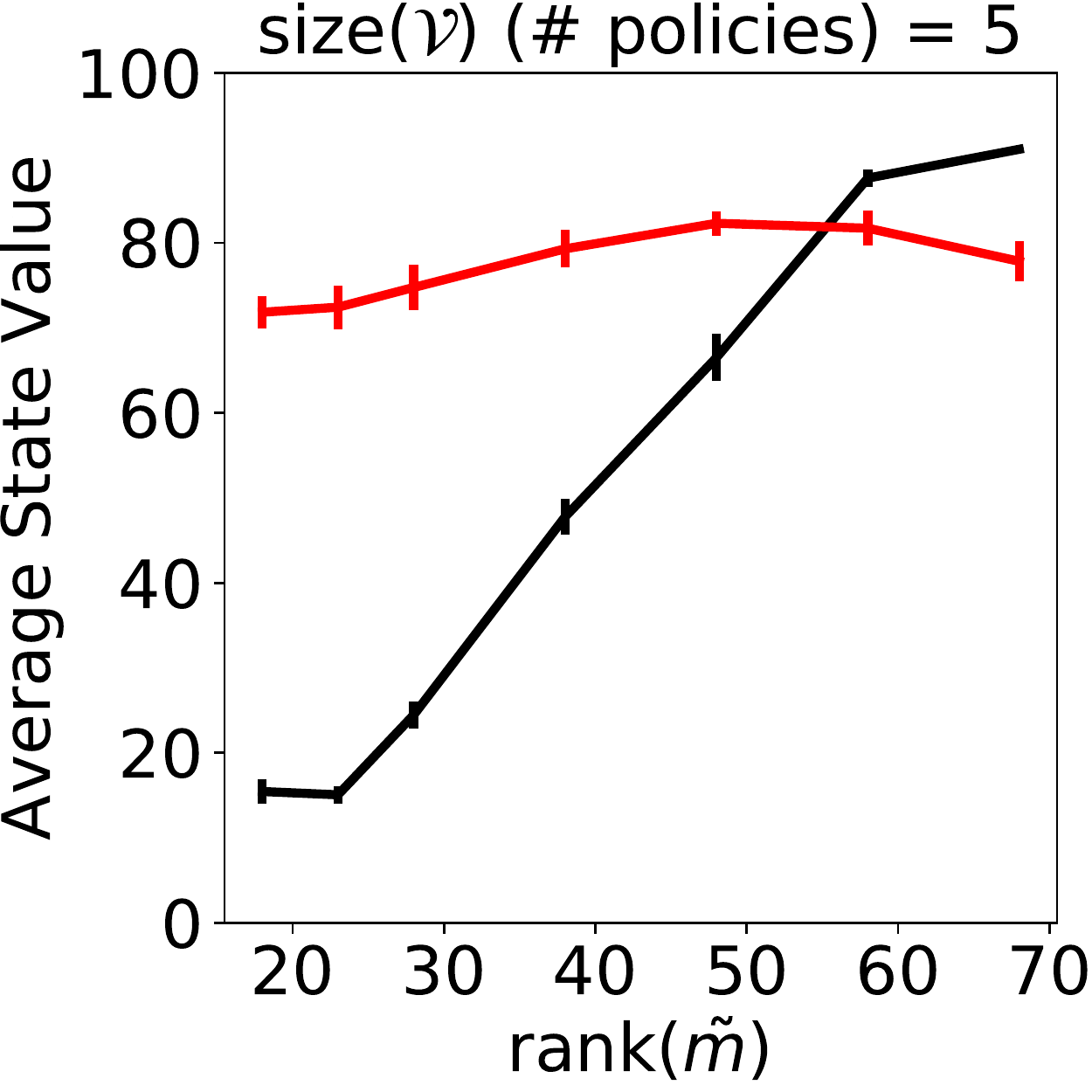}
}
\subfigure[Four Rooms (fixed $\V$)]{
\includegraphics[scale=0.25]{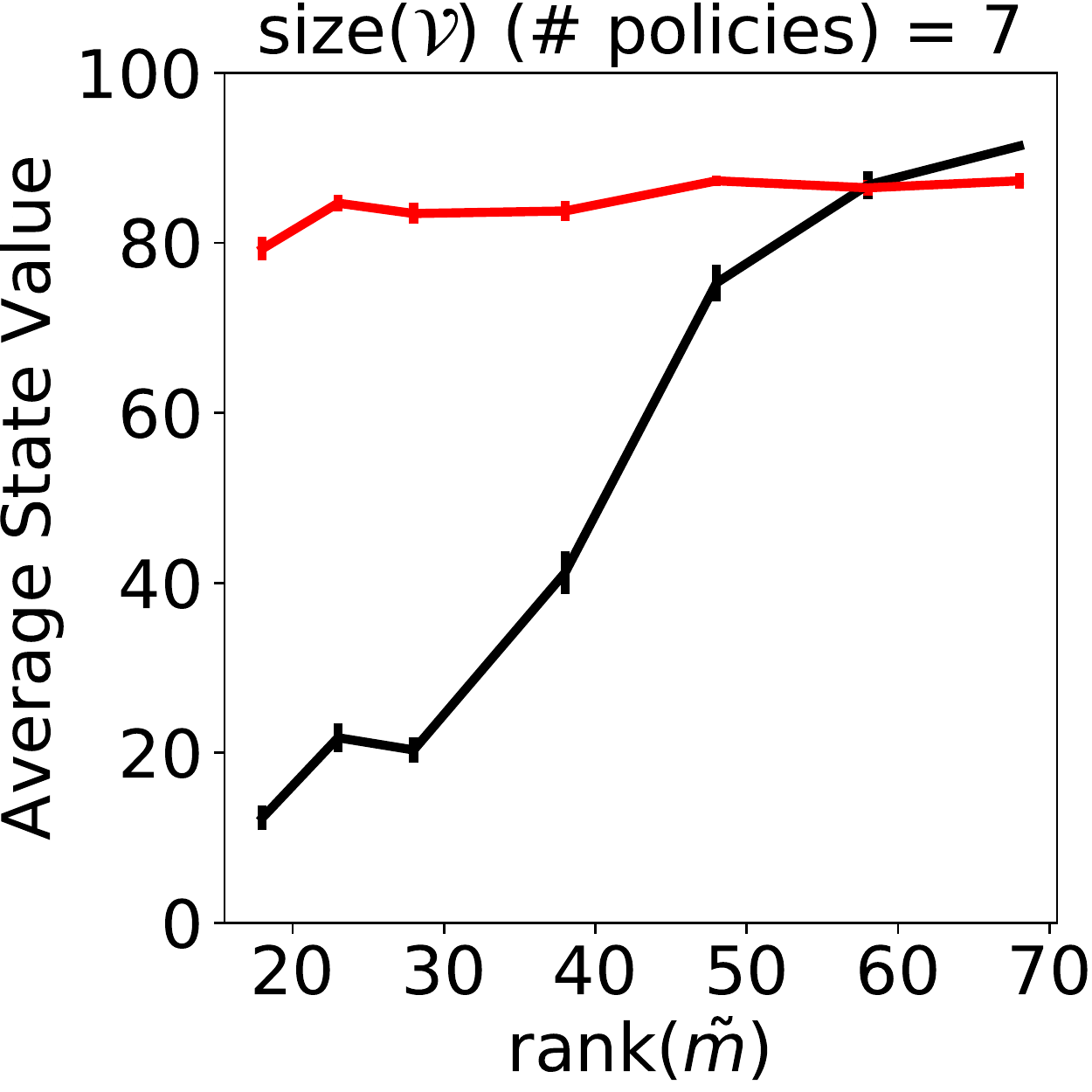}
}
\subfigure[\textbf{Four Rooms (fixed $\boldsymbol{\V}$)}]{
\includegraphics[scale=0.25]{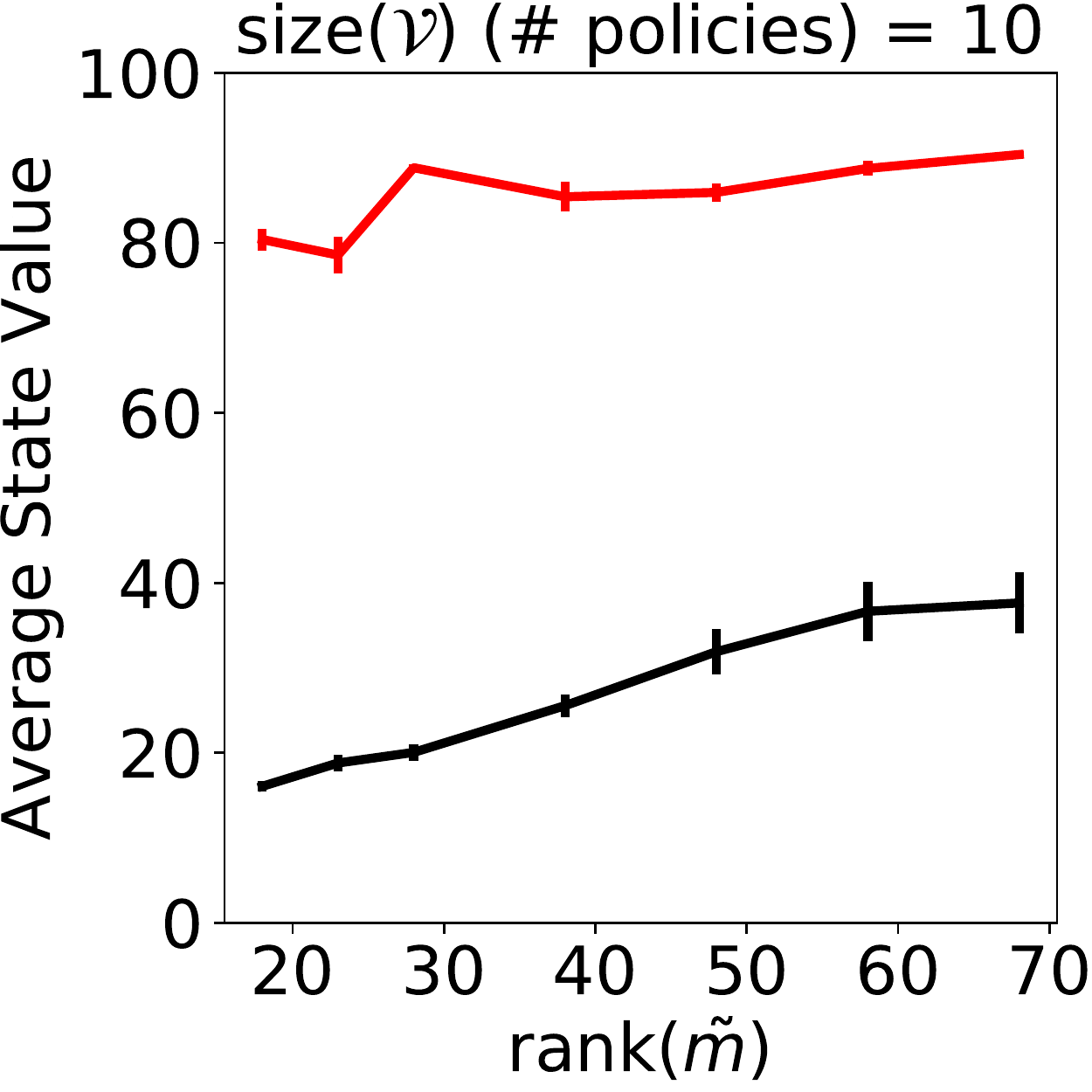}
}
\subfigure[Four Rooms (fixed $\V$)]{
\includegraphics[scale=0.25]{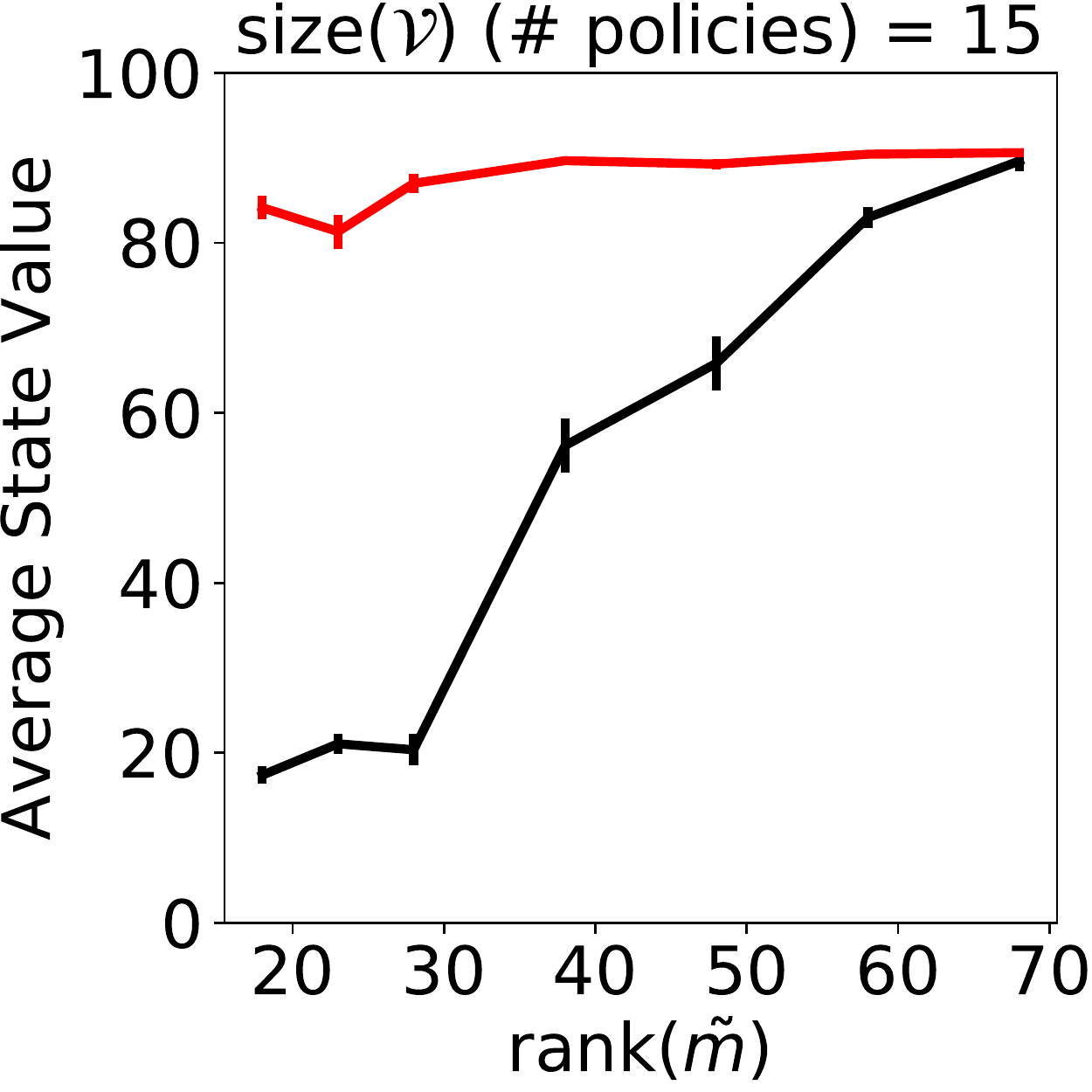}
}
\subfigure[Four Rooms (fixed $\V$)]{
\includegraphics[scale=0.25]{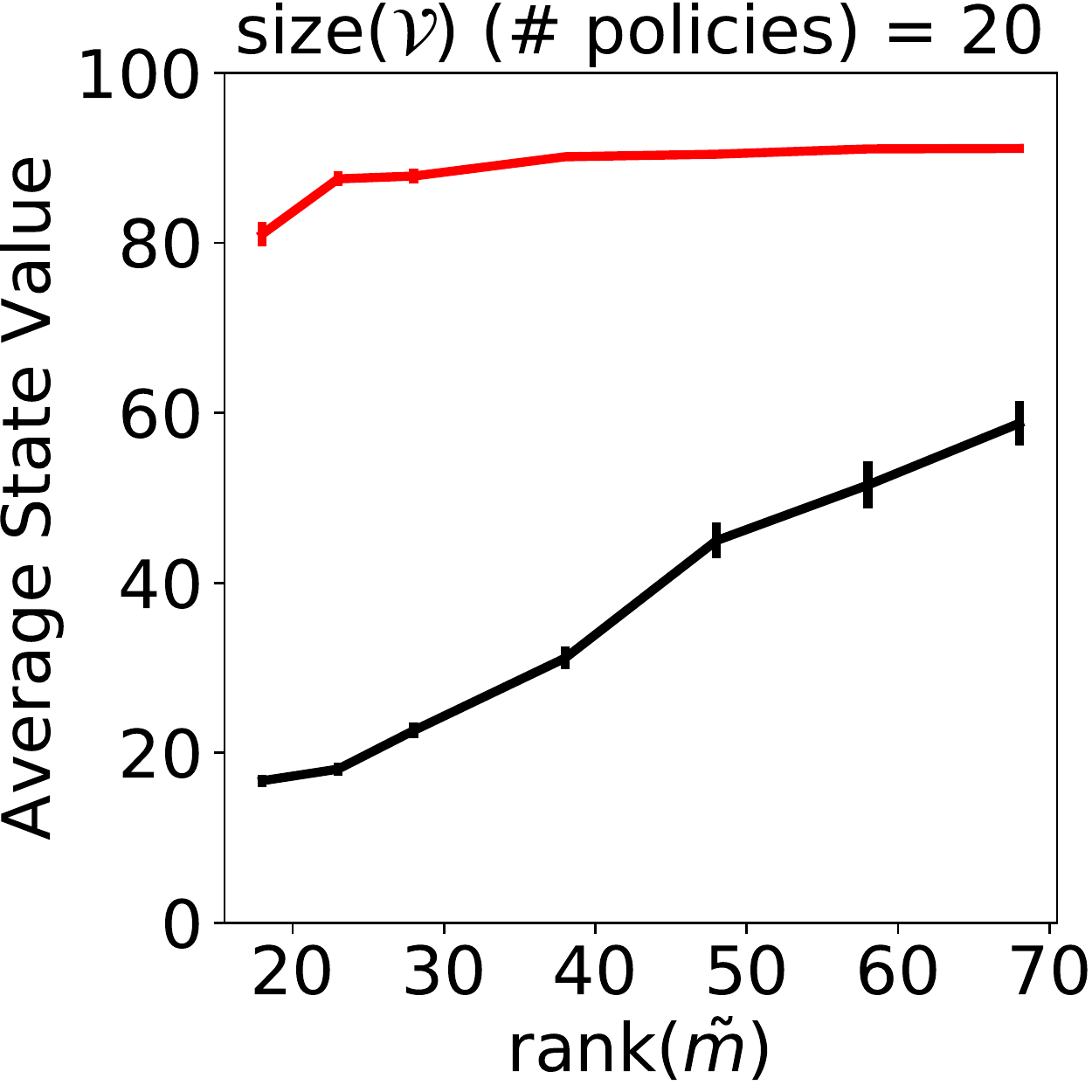}
}
\subfigure[Four Rooms (fixed $\V$)]{
\includegraphics[scale=0.25]{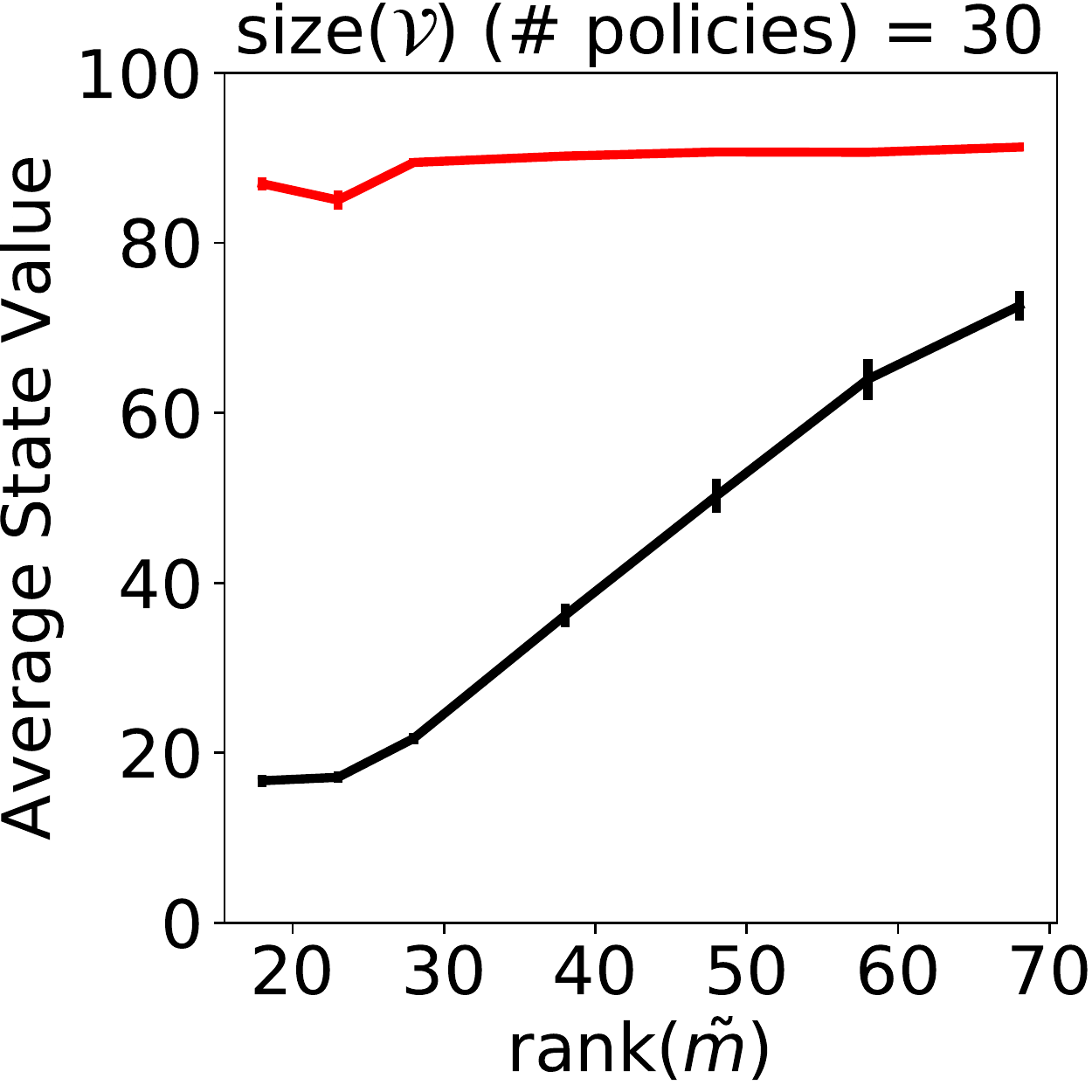}
}
\subfigure[Four Rooms (fixed $\V$)]{
\includegraphics[scale=0.25]{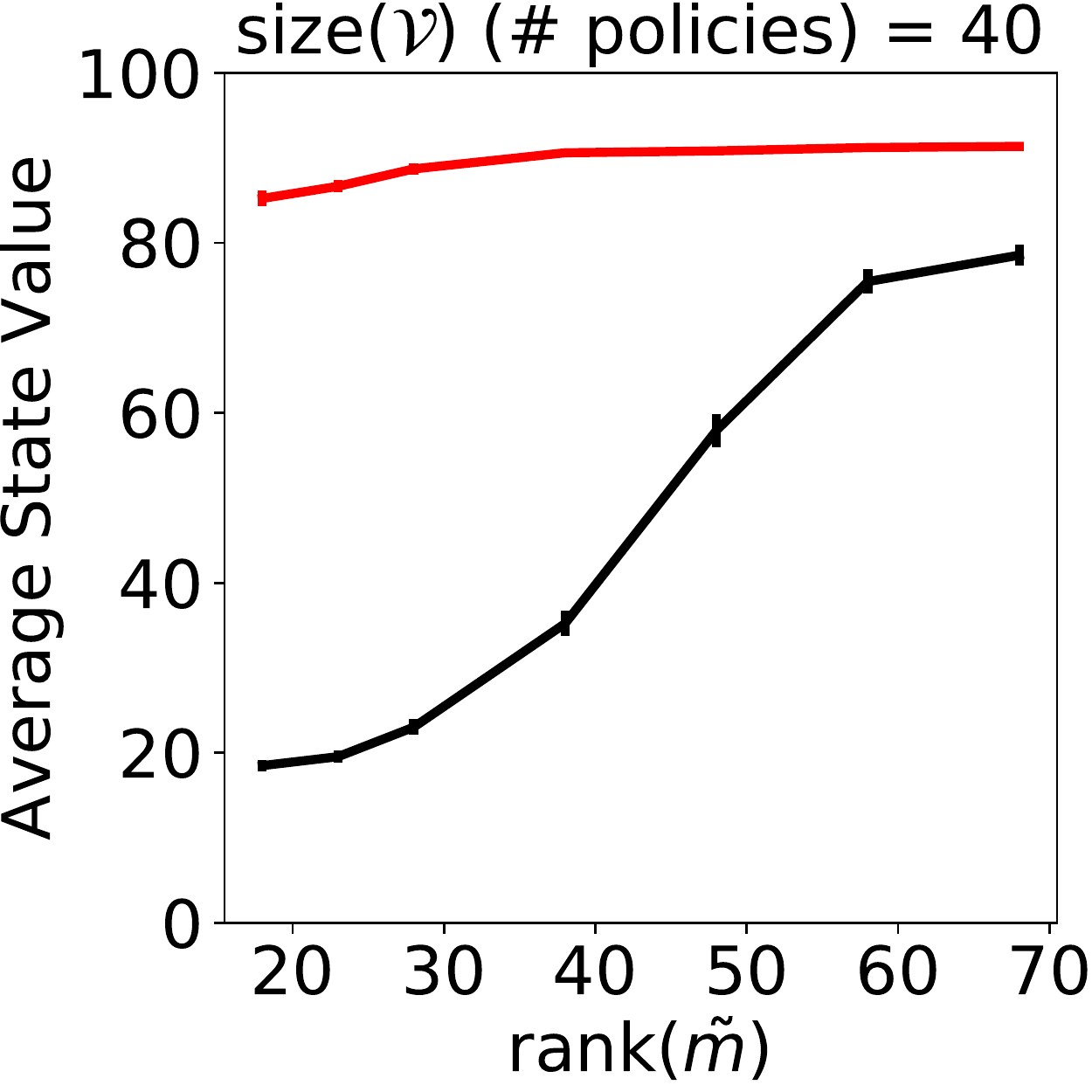}
}
\caption{All Four Rooms results with fixed $\V$ and $\V = \{v_{\pi_1}, \ldots, v_{\pi_n}\}$.}
\end{figure}

\begin{figure}[H]
\centering
\subfigure[Four Rooms (fixed $\mt$)]{
\includegraphics[scale=0.25]{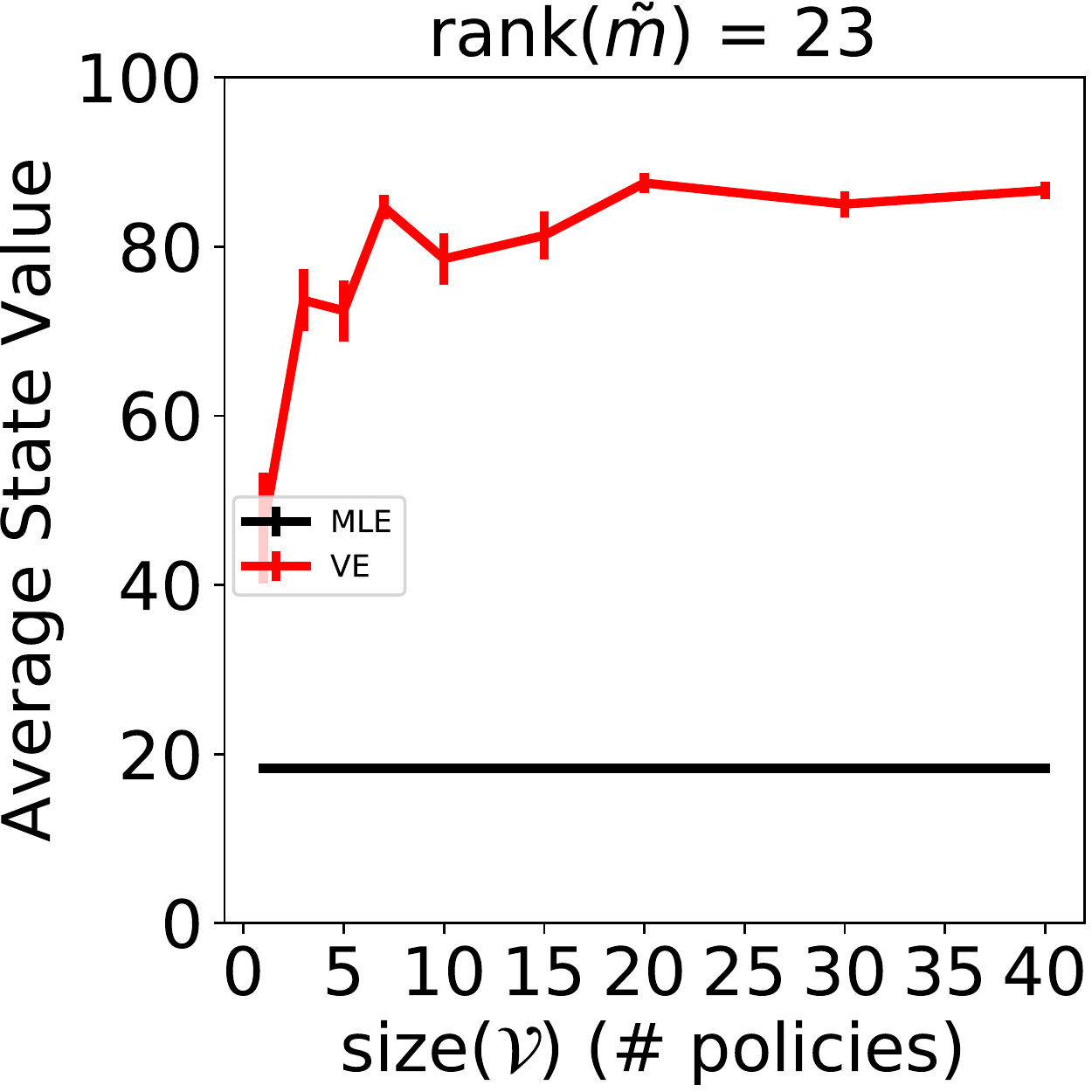}
}
\subfigure[Four Rooms (fixed $\mt$)]{
\includegraphics[scale=0.25]{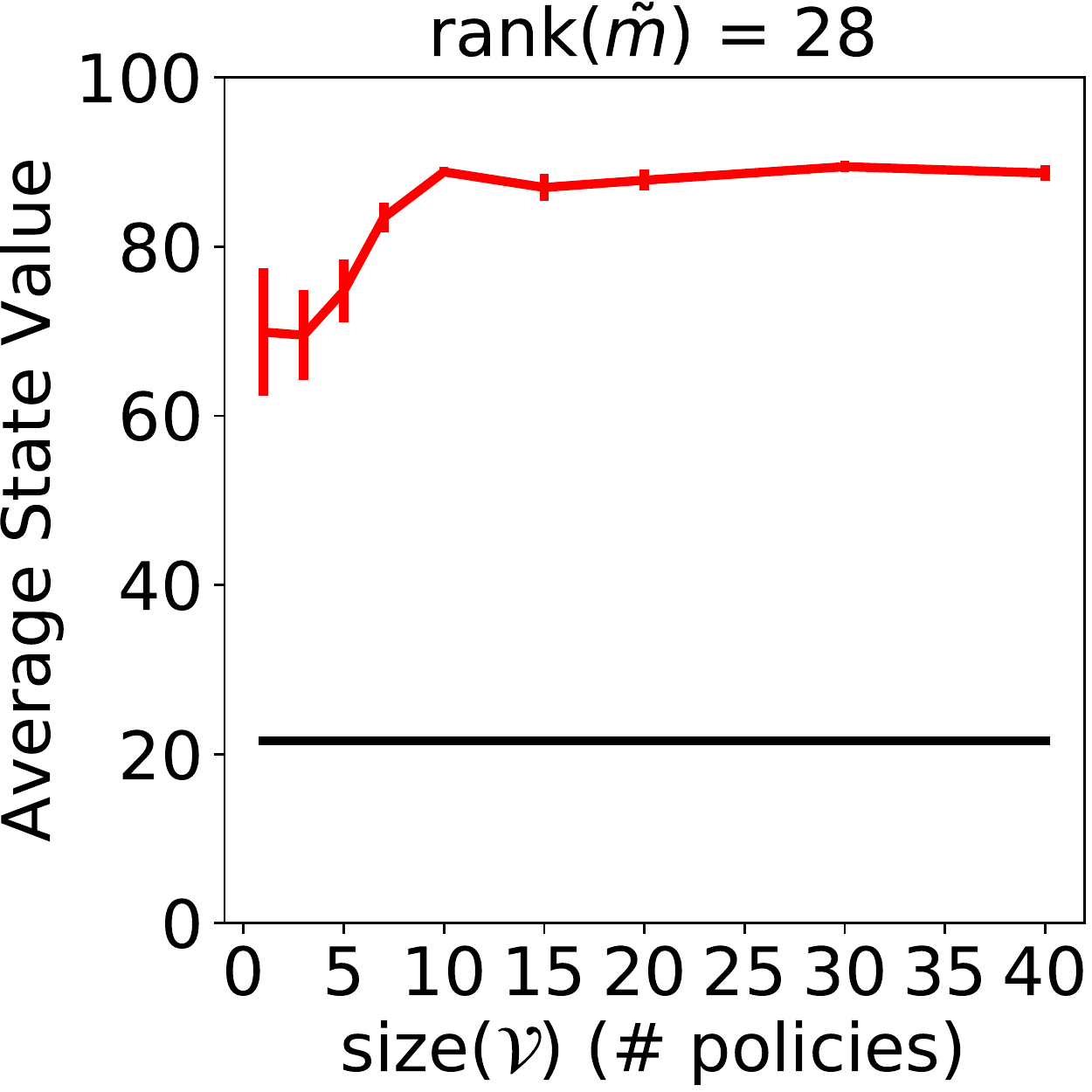}
}
\subfigure[\textbf{Four Rooms (fixed $\boldsymbol{\mt}$)}]{
\includegraphics[scale=0.25]{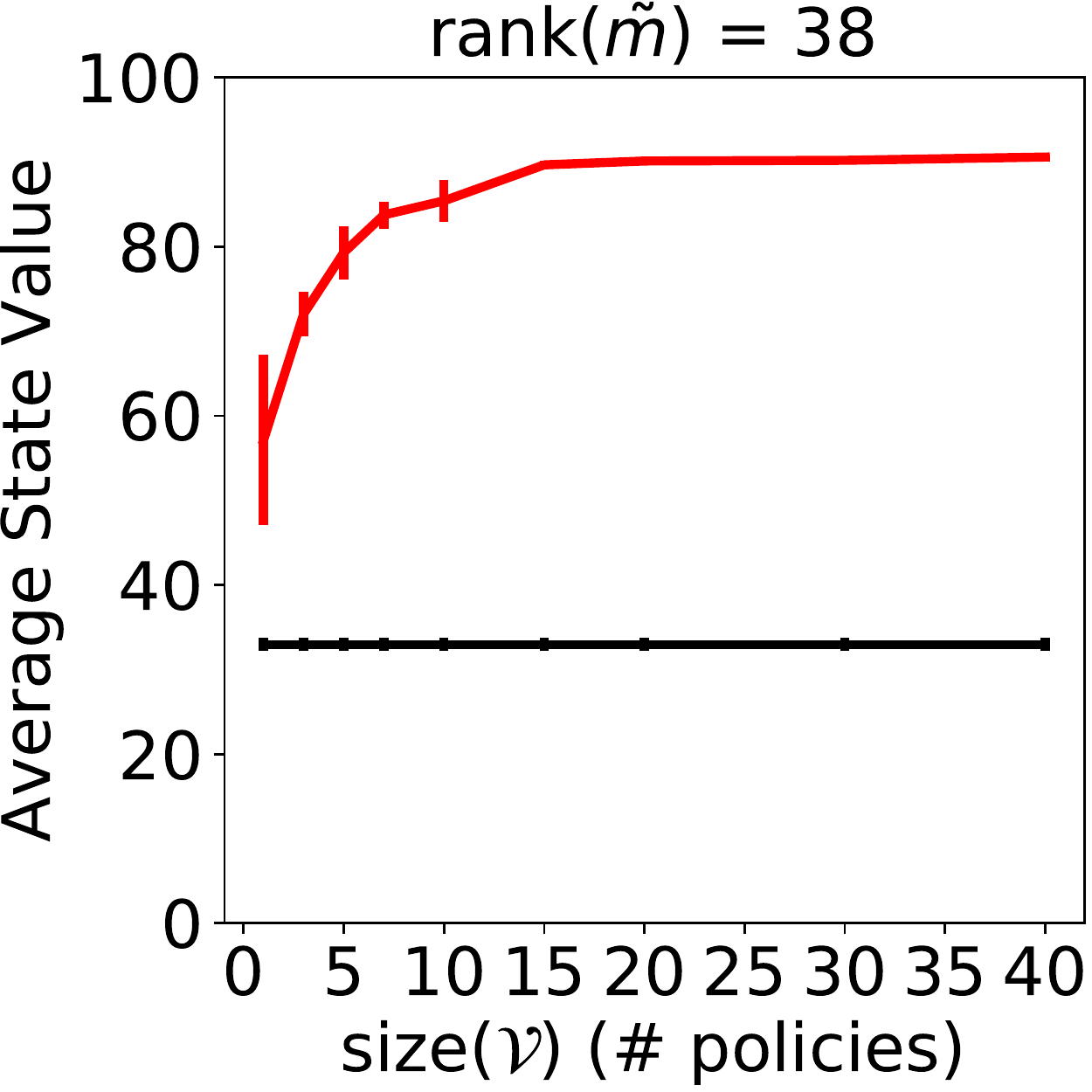}
}

\subfigure[Four Rooms (fixed $\mt$)]{
\includegraphics[scale=0.25]{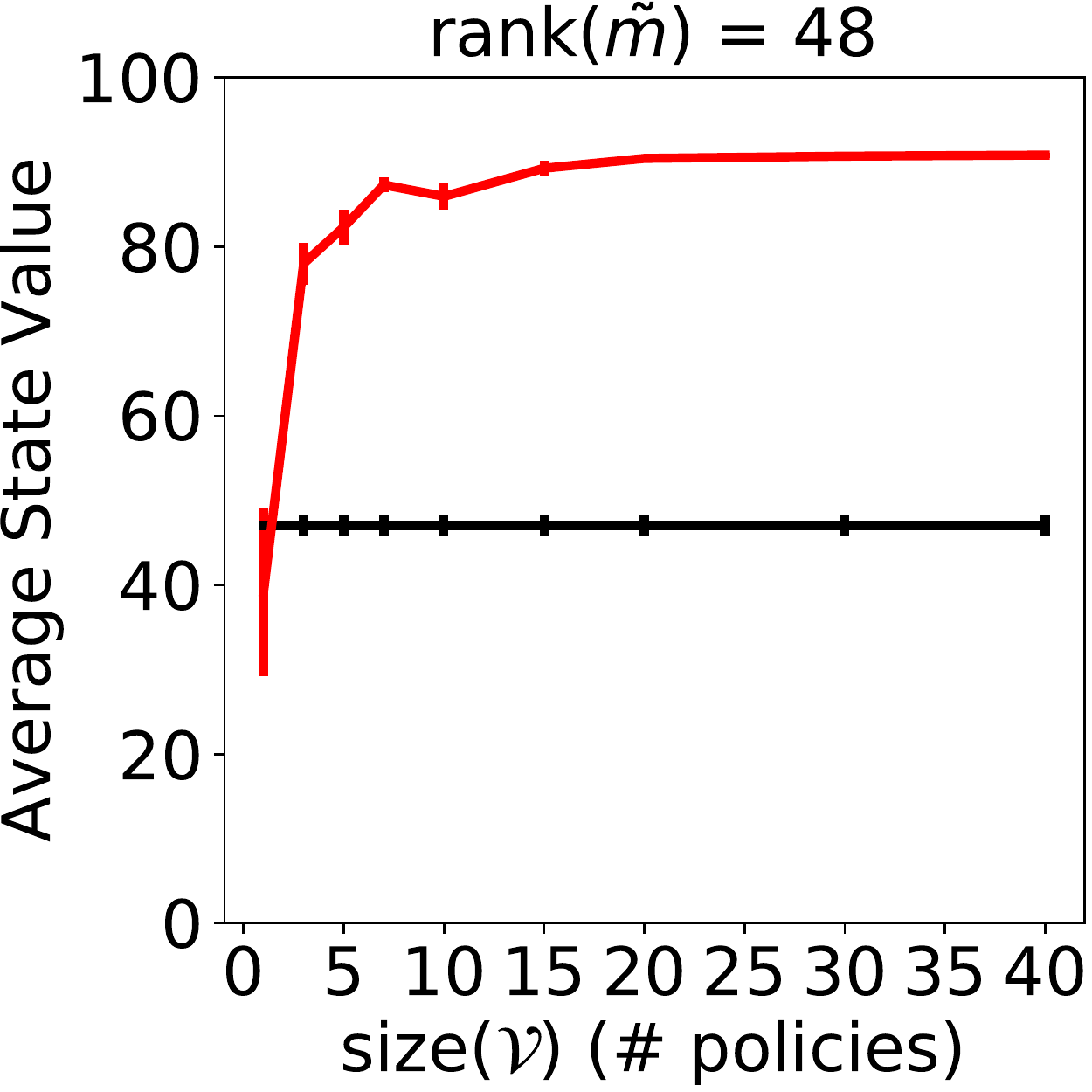}
}
\subfigure[Four Rooms (fixed $\mt$)]{
\includegraphics[scale=0.25]{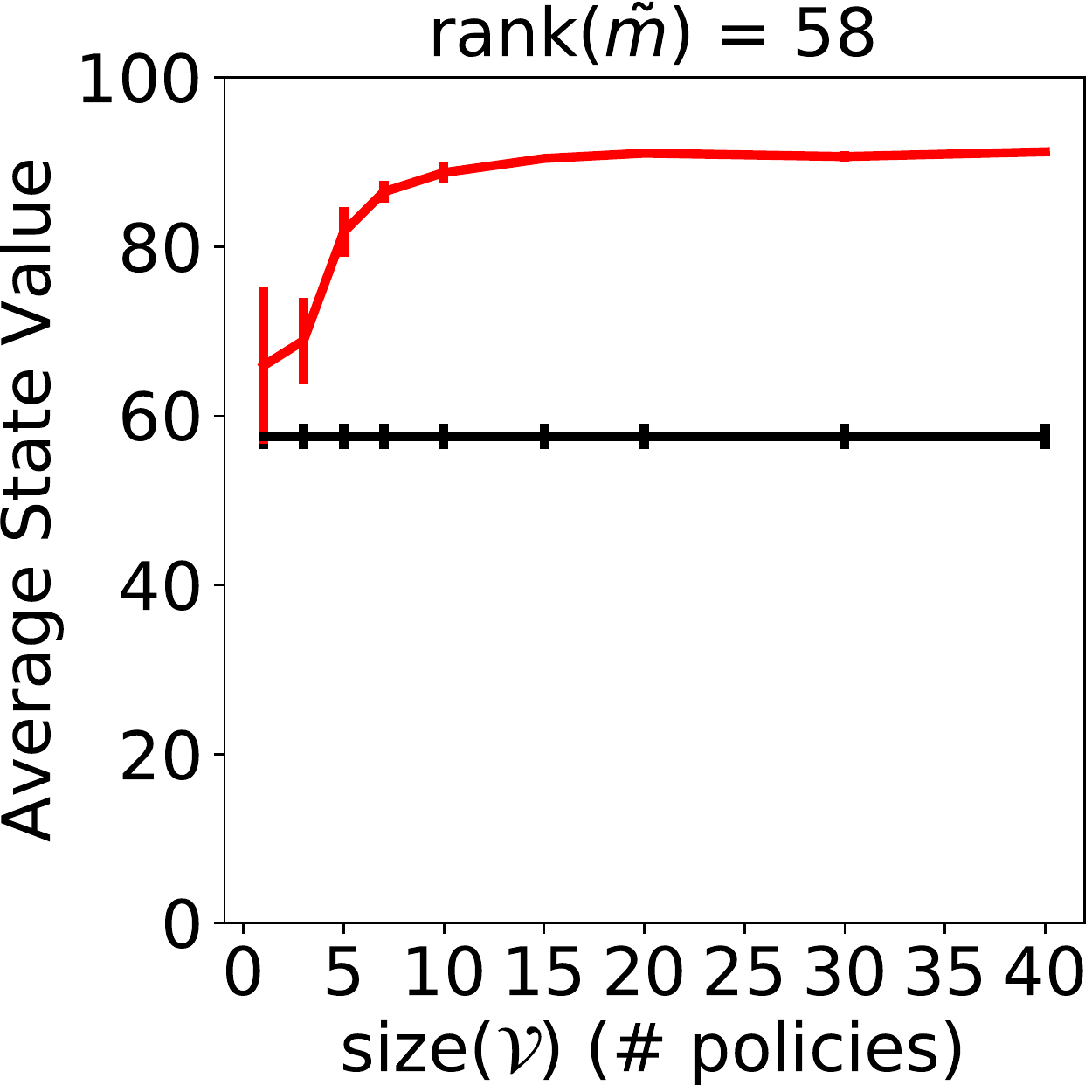}
}
\subfigure[Four Rooms (fixed $\mt$)]{
\includegraphics[scale=0.25]{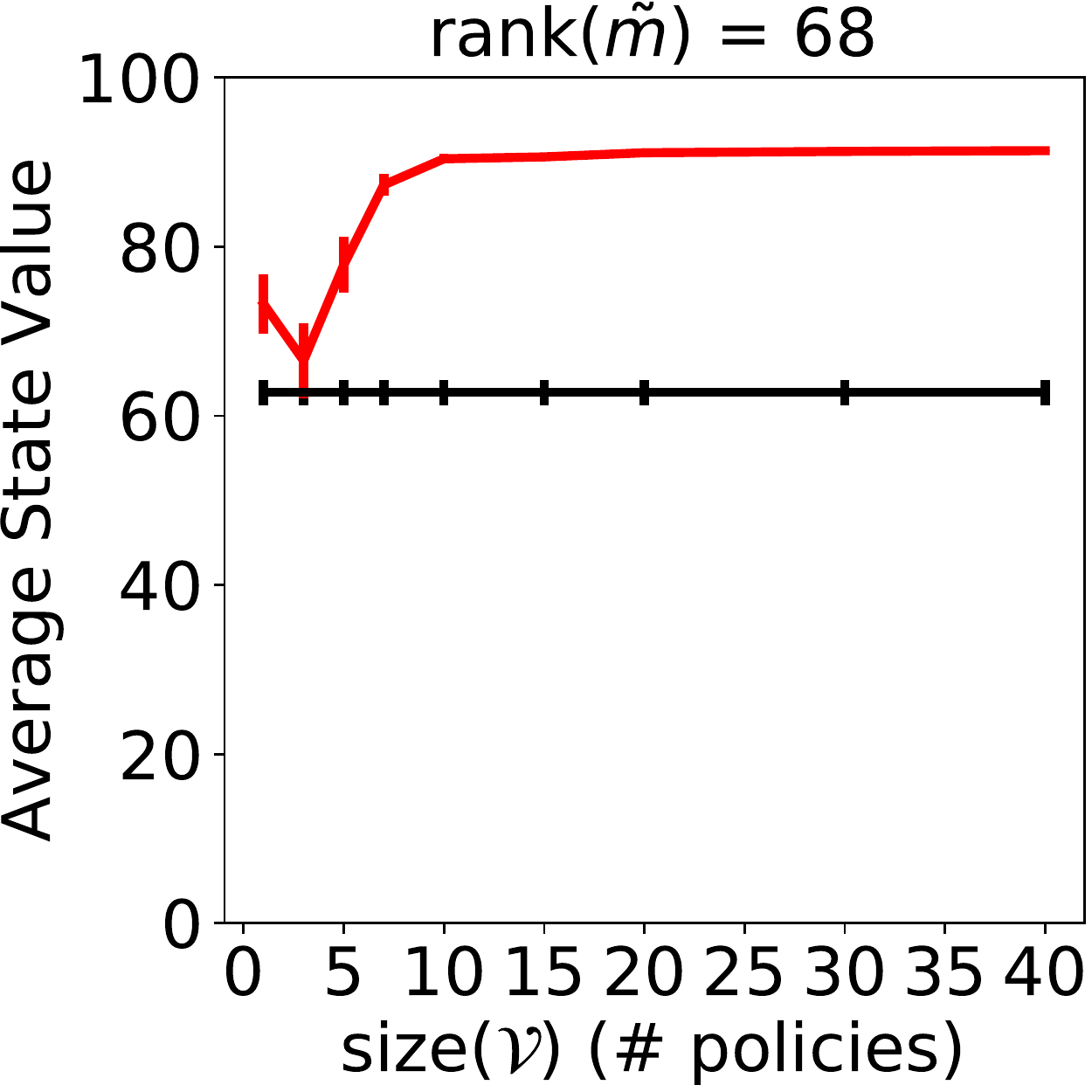}
}
\caption{All Four Rooms results with fixed $\mt$ and $\V = \{v_{\pi_1}, \ldots, v_{\pi_n}\}$.}
\end{figure}

\begin{figure}[H]
\centering
\subfigure[Cart-pole (fixed $\V$)]{
\includegraphics[scale=0.25]{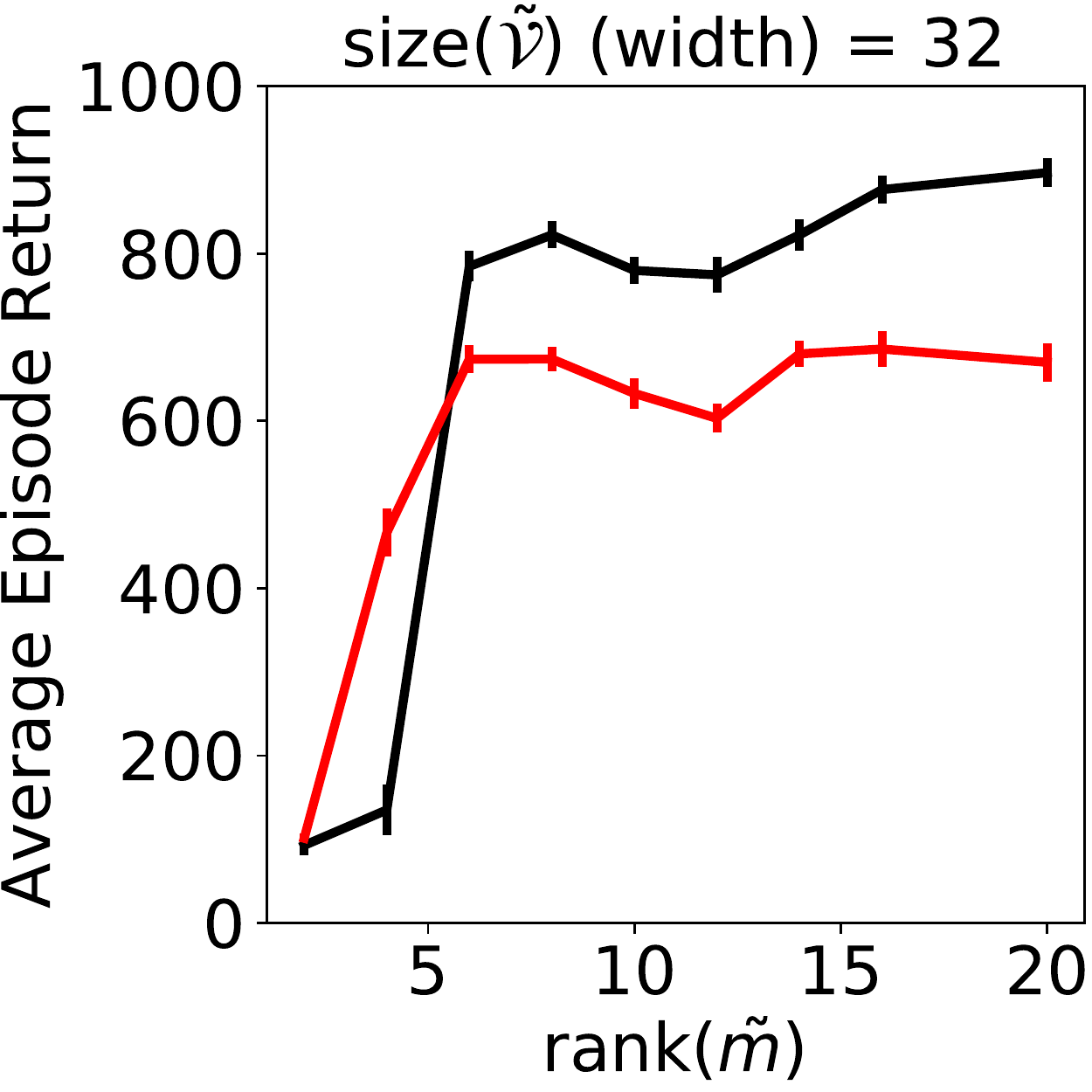}
}
\subfigure[Cart-pole (fixed $\V$)]{
\includegraphics[scale=0.25]{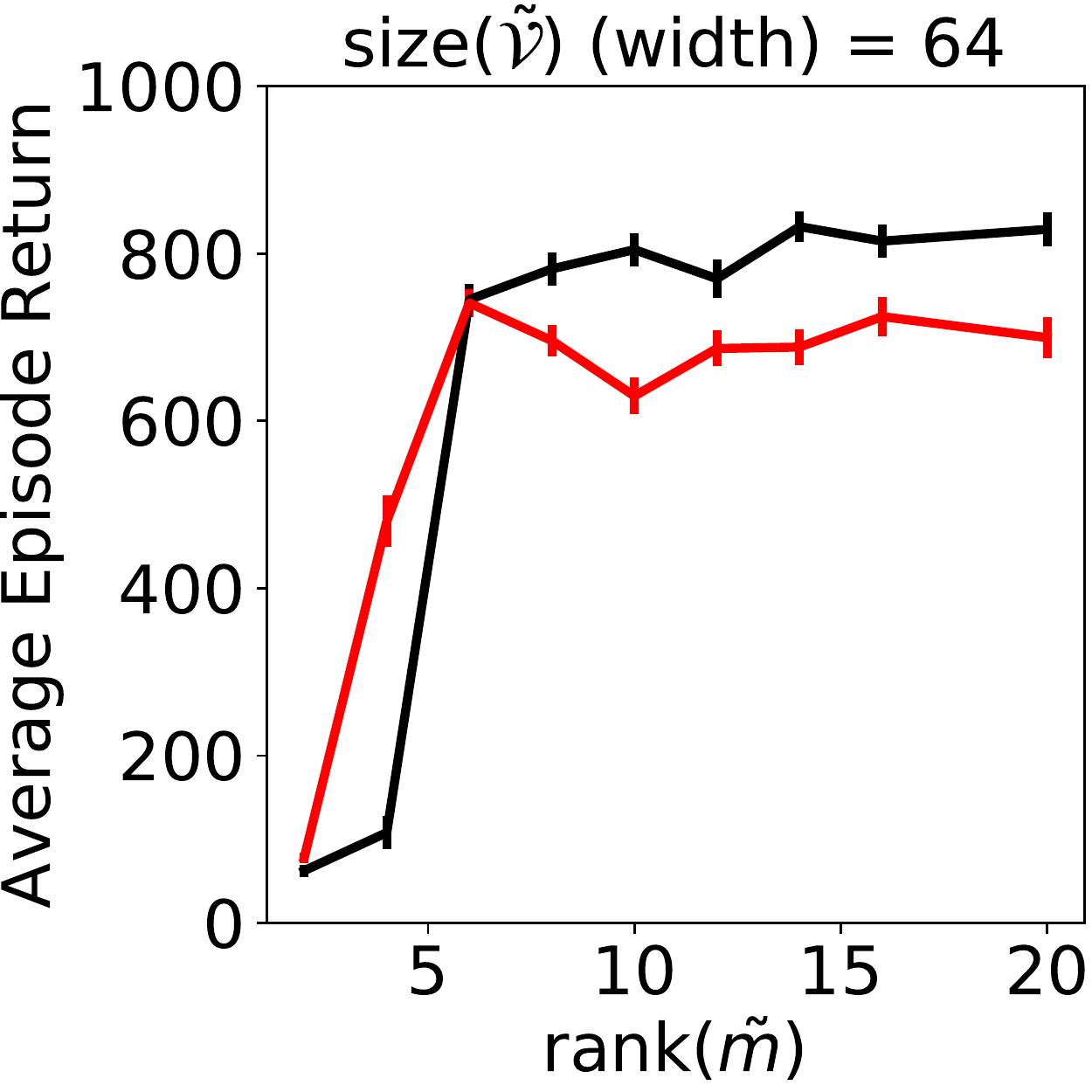}
}
\subfigure[\textbf{Cart-pole (fixed $\boldsymbol{\V}$)}]{
\includegraphics[scale=0.25]{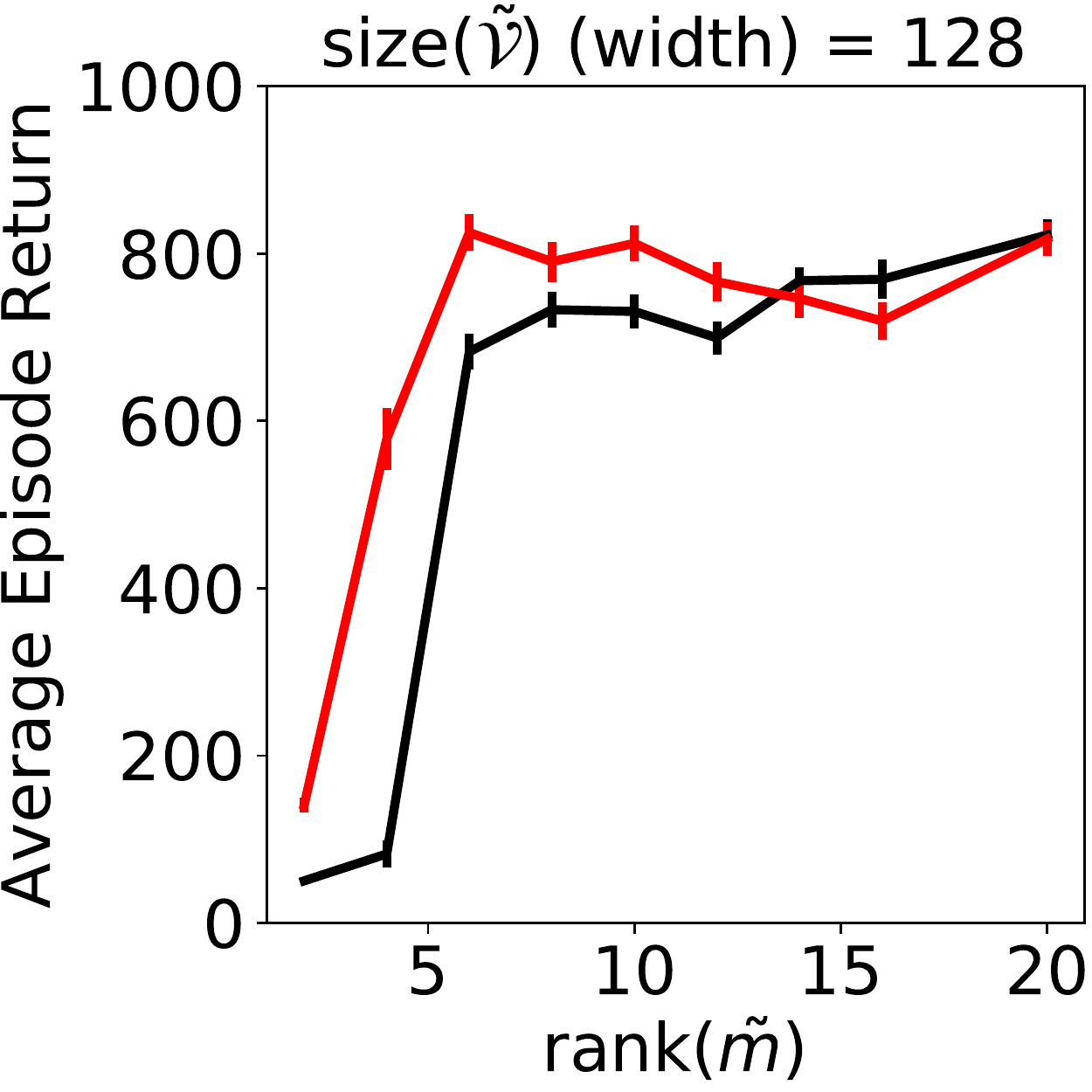}
}

\subfigure[Cart-pole (fixed $\V$)]{
\includegraphics[scale=0.25]{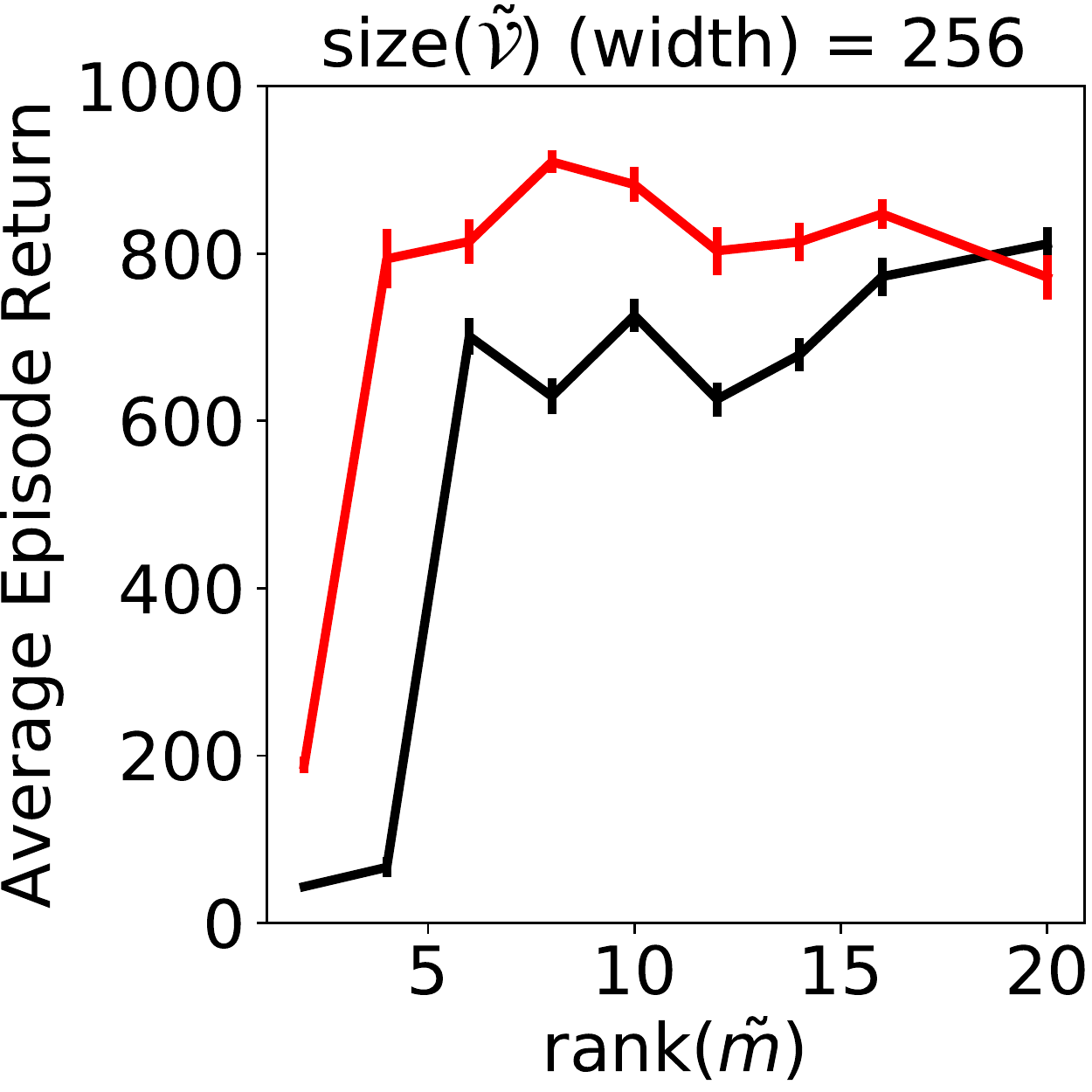}
}
\subfigure[Cart-pole (fixed $\V$)]{
\includegraphics[scale=0.25]{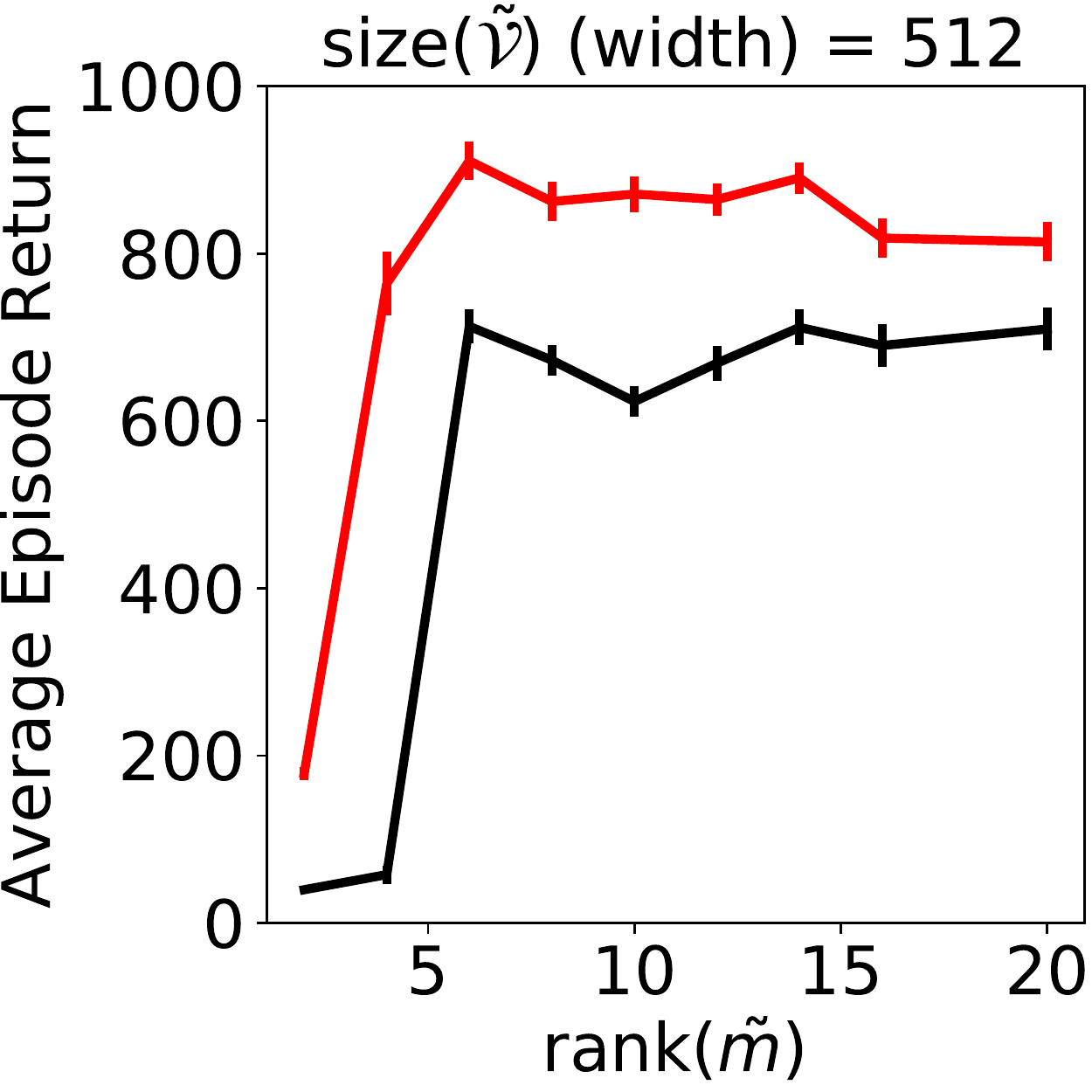}
}
\subfigure[Cart-pole (fixed $\V$)]{
\includegraphics[scale=0.25]{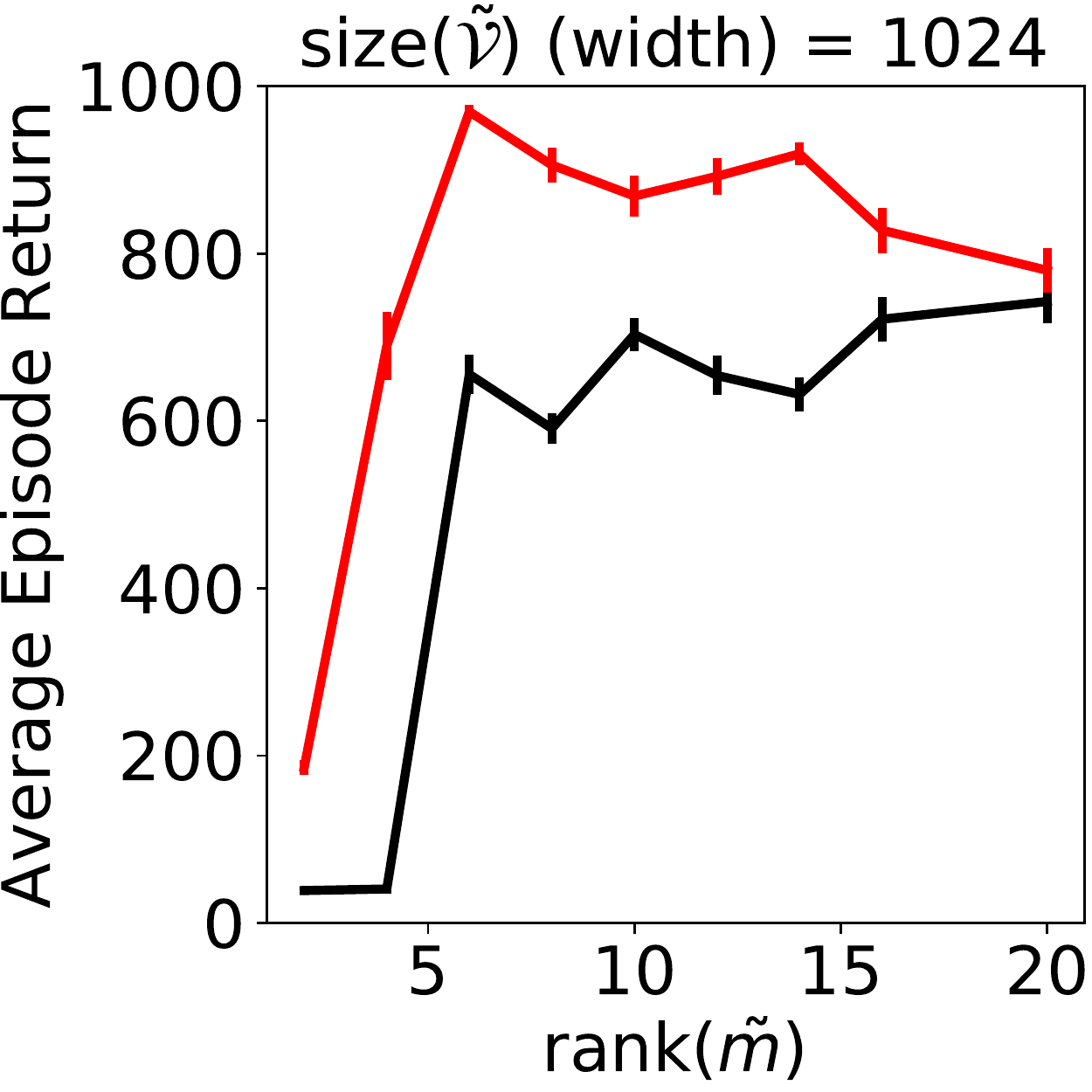}
}
\caption{All Cart-pole results results with fixed $\V$ and $\span(\V) \approx \AV$.}
\end{figure}

\begin{figure}[htp]
\centering
\subfigure[Cart-pole (fixed $\mt$)]{
\includegraphics[scale=0.25]{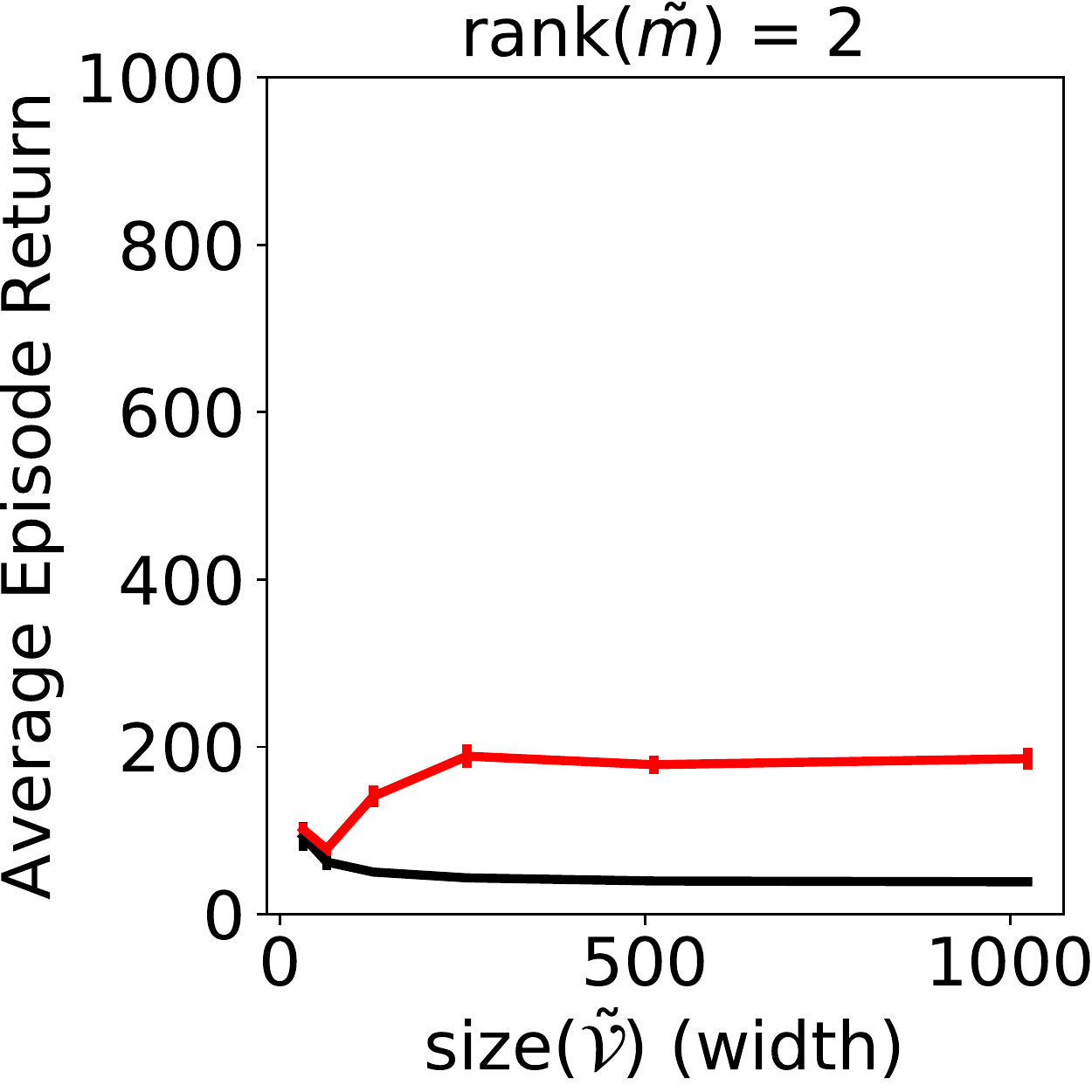}
}
\subfigure[Cart-pole (fixed $\mt$)]{
\includegraphics[scale=0.25]{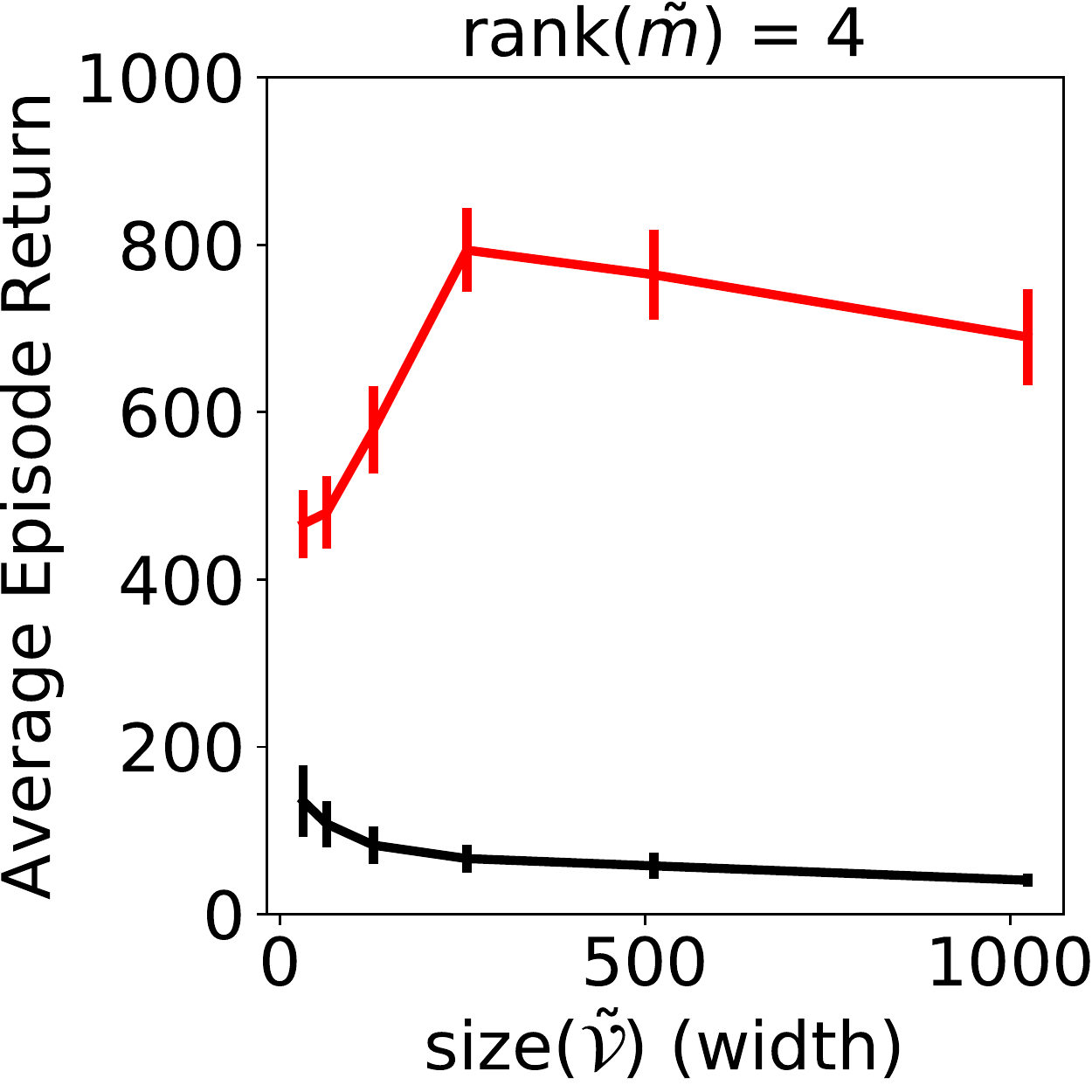}
}
\subfigure[Cart-pole (fixed $\mt$)]{
\includegraphics[scale=0.25]{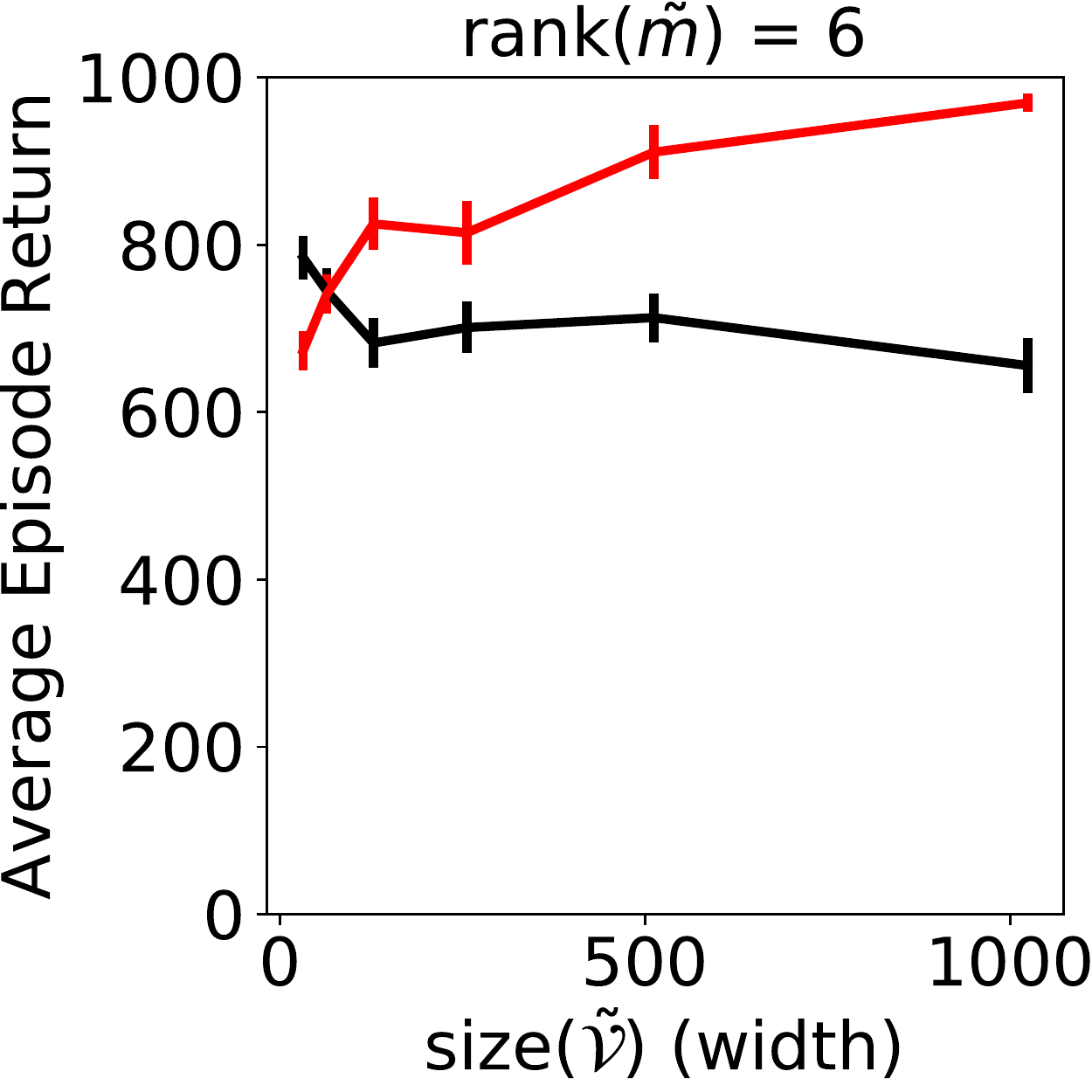}
}
\subfigure[Cart-pole (fixed $\mt$)]{
\includegraphics[scale=0.25]{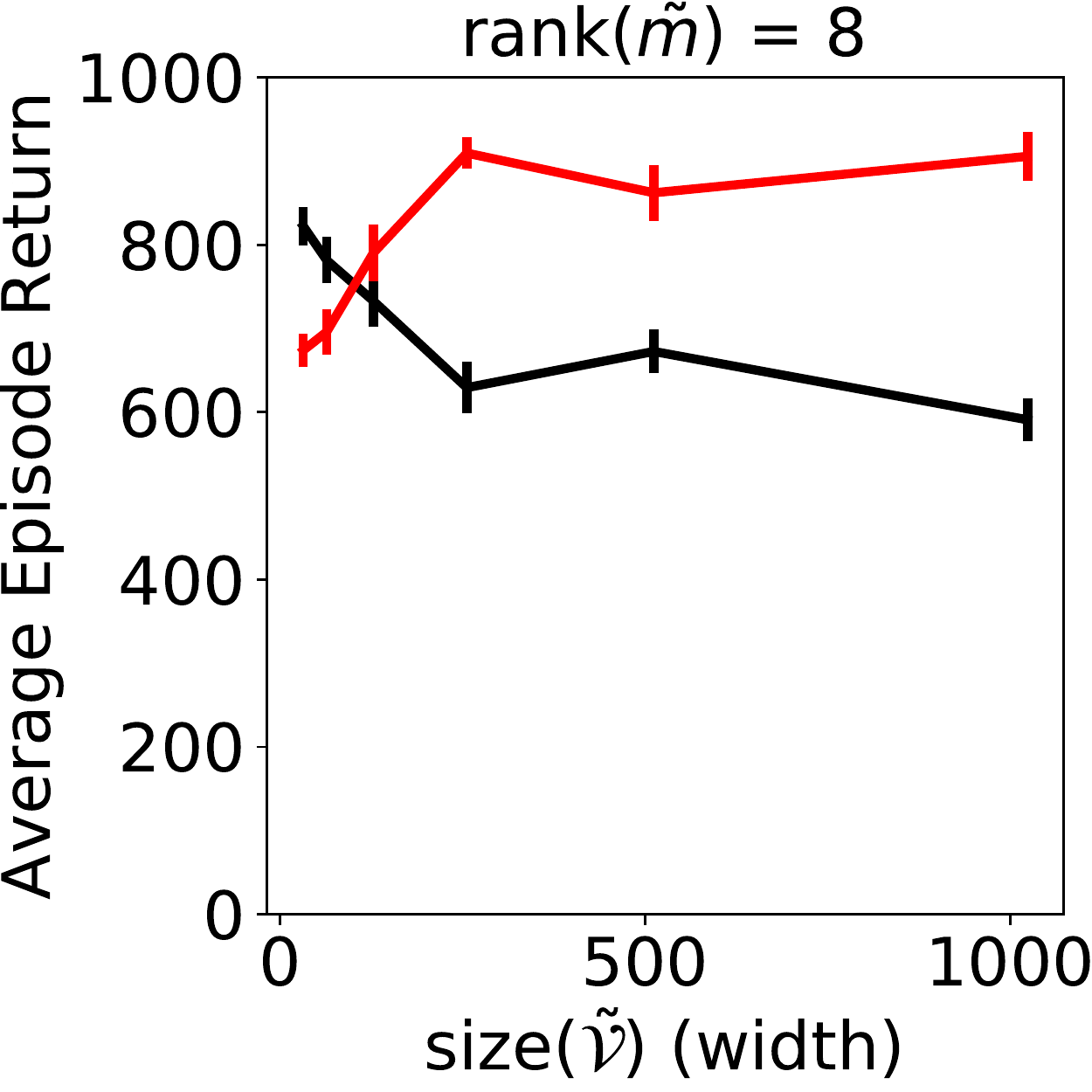}
}

\subfigure[\textbf{Cart-pole (fixed $\boldsymbol{\mt}$)}]{
\includegraphics[scale=0.25]{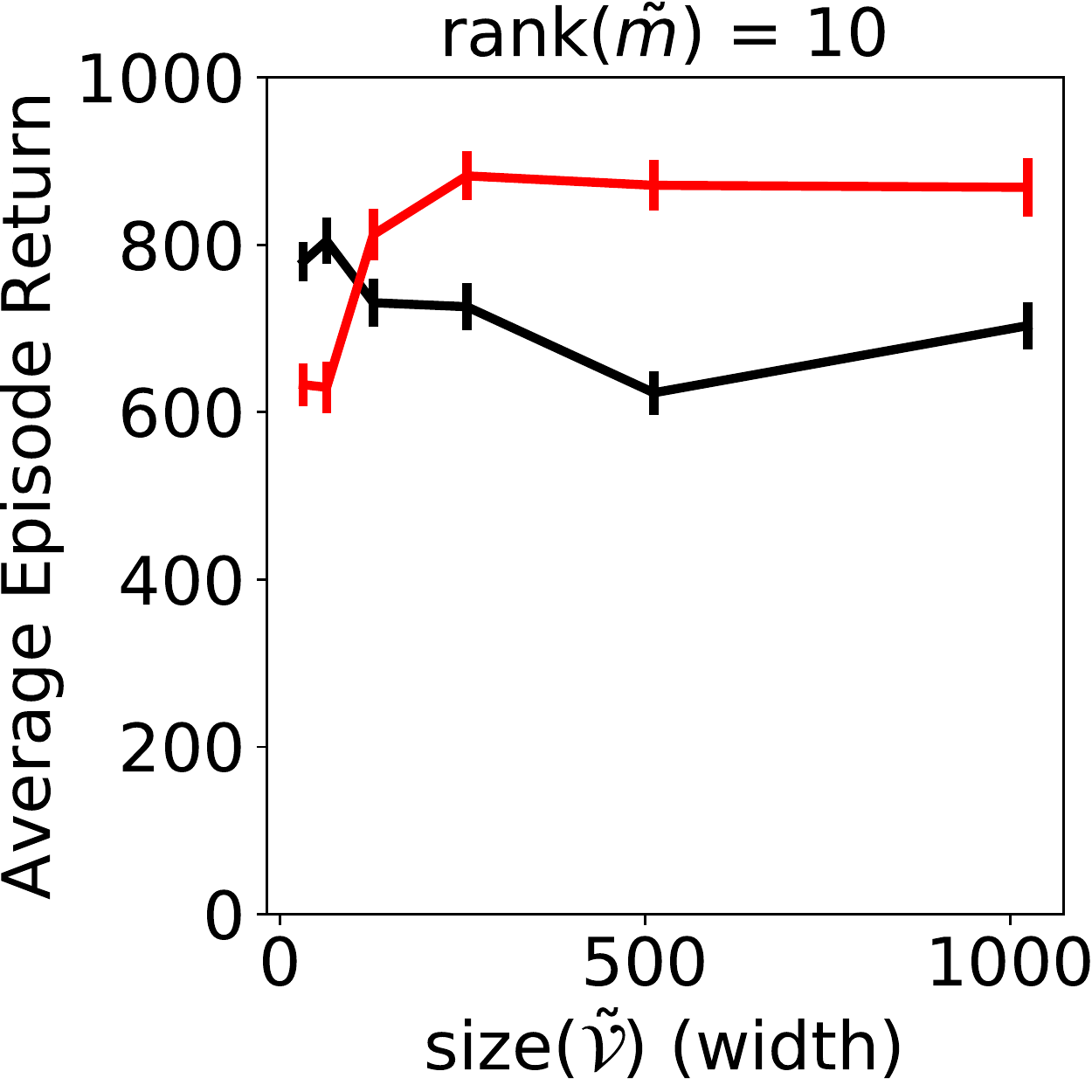}
}
\subfigure[Cart-pole (fixed $\mt$)]{
\includegraphics[scale=0.25]{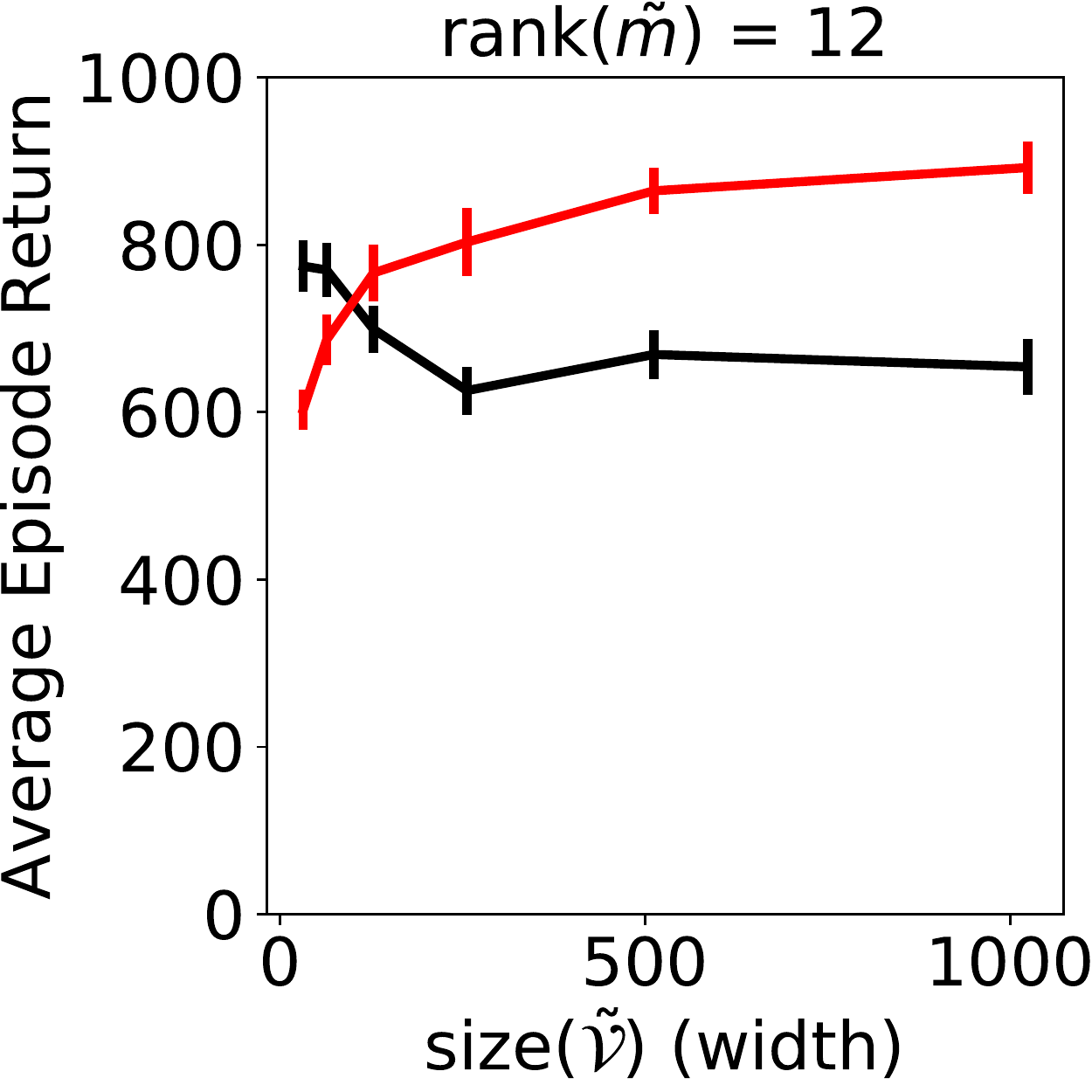}
}
\subfigure[Cart-pole (fixed $\mt$)]{
\includegraphics[scale=0.25]{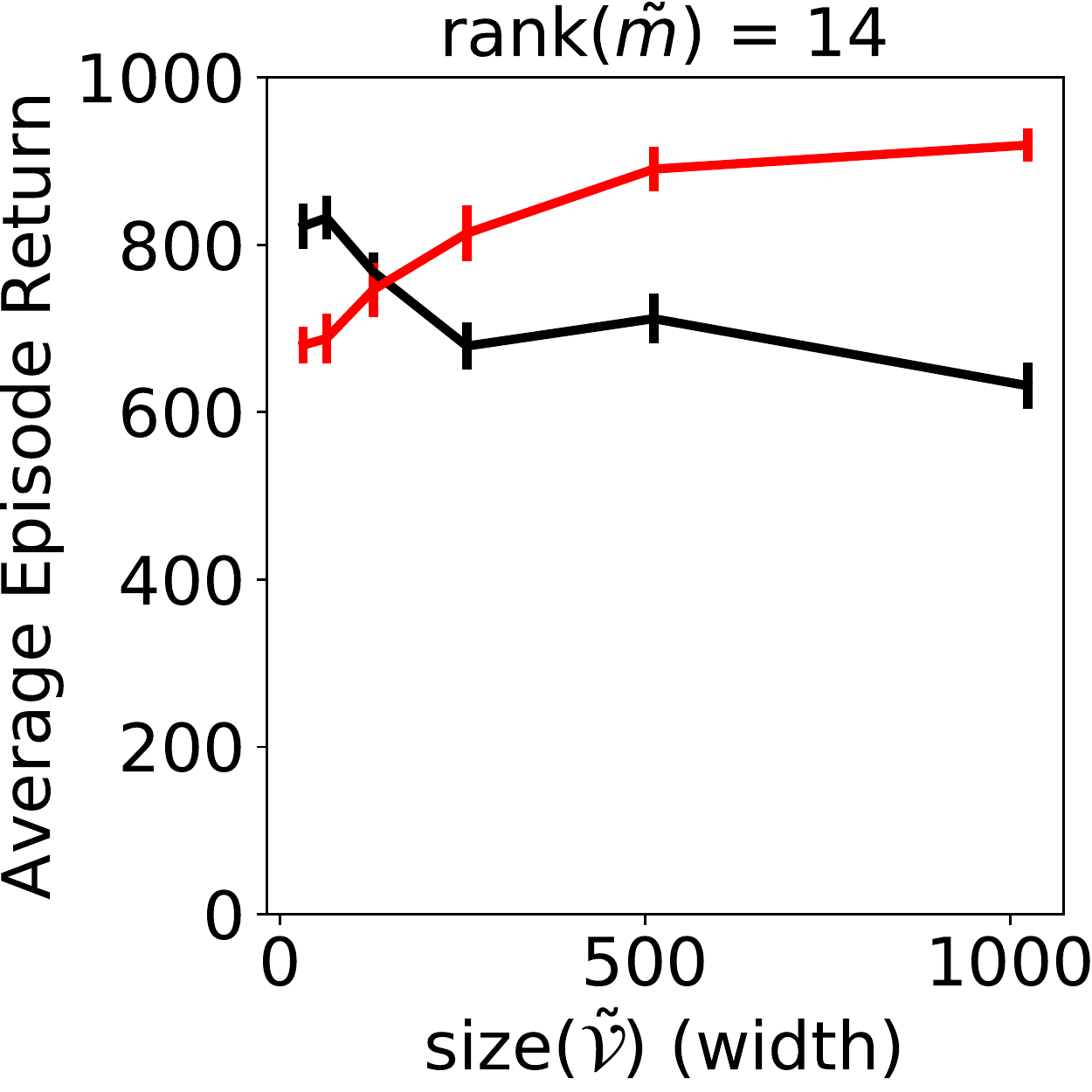}
}
\subfigure[Cart-pole (fixed $\mt$)]{
\includegraphics[scale=0.25]{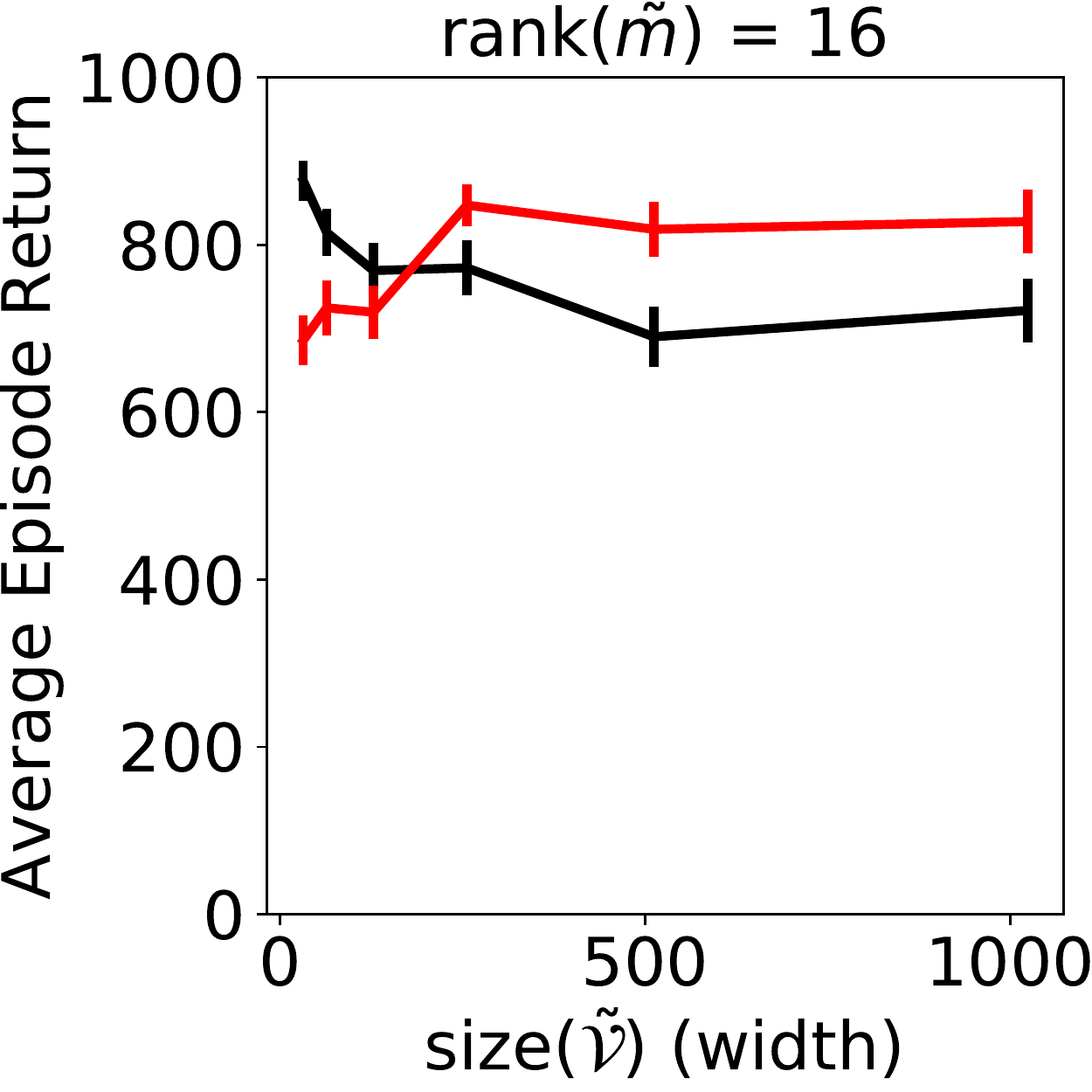}
}
\caption{All Cart-pole results results with fixed $\mt$ and $\span(\V) \approx \AV$.}
\end{figure}

\subsubsection{Hyperparameters}
\label{sec:hyperparams}
Table~\ref{tab:hyper} provides a list detailing the different hyperparameters used throughout our pipeline. 

\begin{table}
\begin{center}
 \begin{tabularx}{\textwidth}{
    >{\hsize=0.6\hsize}X  
    >{\hsize=0.2\hsize}X
    >{\hsize=1.2\hsize}X}
 \textbf{Hyperparameter} & \textbf{Value} & \textbf{Description} \\
 \hline
 minibatch size & 32 & Number of samples passed at a time during a training step of any learning method. \\
 model learning rate & 5e-5 & Learning rate used to train all models. \\
 \# model samples & 1,000,000 & Number of transitions sampled by a random policy in the Data Collection phase. \\
 model depth & 2 & Number of hidden layers in the model architecture. \\
 model width & 256 & Number of units per hidden layer. \\ 
 model activation & tanh & Activation function following a hidden layer. \\
 model learning max steps & 1,000,000 & Maximum number of training iterations. \\
 $\gamma$ & 0.99 & Discount factor used across environments. \\
 LSTD samples / policy & 10,000 & Number of samples collected for each phase of policy evaluation using LSTD. \\
 \# policy iteration steps & 40 & Number of steps of policy iteration in the policy construction phase, when applicable. \\
 DQN learning rate & 5e-4 & Learning rate for DQN. \\
 DQN \# environment steps & 2,500,000 & Number of environment steps that DQN learns over. \\
 DQN learning frequency & 4 & A learning update is applied after this many environment steps. \\ 
 DQN depth & 1 & Number of hidden layers in the DQN. \\
 DQN activation & tanh & Activation function following a hidden layer. \\ 
 DQN target update & 2000 & Number of environment steps before the target network in the DQN is updated. \\
 Tabular \# eval episodes & 20 & Number of episodes to average performance over to assess a policy in the tabular setting. \\
 DQN \# eval episodes & 100 & Number of episodes to average DQN policy performance over at the end of training. \\
 DQN $\epsilon$ & 0.05 & Chance of picking a random action during training. \\
 Optimizer & Adam & Optimizer used for all learning operations. Default Adam parameters were used. \\ 
 \hline
\end{tabularx}
\end{center}
\caption{List of hyperparameters used in the experiments. \label{tab:hyper}}
\end{table}

%% file: ve.bbl
\begin{thebibliography}{45}
\providecommand{\natexlab}[1]{#1}
\providecommand{\url}[1]{\texttt{#1}}
\expandafter\ifx\csname urlstyle\endcsname\relax
  \providecommand{\doi}[1]{doi: #1}\else
  \providecommand{\doi}{doi: \begingroup \urlstyle{rm}\Url}\fi

\bibitem[Abachi et~al.(2020)Abachi, Ghavamzadeh, and massoud
  Farahmand]{abachi2020policy}
Romina Abachi, Mohammad Ghavamzadeh, and Amir massoud Farahmand.
\newblock {Policy-Aware Model Learning for Policy Gradient Methods}.
\newblock \emph{arXiv preprint cs.AI:2003.00030}, 2020.

\bibitem[Asadi et~al.(2018)Asadi, Cater, Misra, and
  Littman]{asadi2018equivalence}
Kavosh Asadi, Evan Cater, Dipendra Misra, and Michael~L. Littman.
\newblock {Equivalence Between Wasserstein and Value-Aware Model-Based
  Reinforcement Learning}.
\newblock In \emph{{FAIM} Workshop on Prediction and Generative Modeling in
  Reinforcement Learning}, 2018.

\bibitem[Ayoub et~al.(2020)Ayoub, Jia, Szepesv\'ari, Wang, and
  Yang]{ayoub2020model}
Alex Ayoub, Zeyu Jia, Csaba Szepesv\'ari, Mengdi Wang, and Lin Yang.
\newblock {Model-Based Reinforcement Learning with Value-Targeted Regression}.
\newblock In \emph{Proceedings of the International Conference on Machine
  Learning ({ICML})}, 2020.

\bibitem[Barto et~al.(1983)Barto, Sutton, and Anderson]{barto1983neuronlike}
Andrew~G. Barto, Richard~S. Sutton, and Charles~W Anderson.
\newblock {N}euronlike {A}daptive {E}lements {T}hat {C}an {S}olve {D}ifficult
  {L}earning {C}ontrol {P}roblems.
\newblock \emph{{IEEE Transactions on Systems, Man, and Cybernetics}}, pages
  834--846, 1983.

\bibitem[Bellemare et~al.(2019)Bellemare, Dabney, Dadashi, Taiga, Castro,
  Le~Roux, Schuurmans, Lattimore, and Lyle]{bellemare2019geometric}
Marc Bellemare, Will Dabney, Robert Dadashi, Adrien~Ali Taiga, Pablo~Samuel
  Castro, Nicolas Le~Roux, Dale Schuurmans, Tor Lattimore, and Clare Lyle.
\newblock {A Geometric Perspective on Optimal Representations for Reinforcement
  Learning}.
\newblock In \emph{Advances in Neural Information Processing Systems}, pages
  4360--4371, 2019.

\bibitem[Bertsekas and Tsitsiklis(1996)]{bertsekas96neuro-dynamic}
Dimitri~P. Bertsekas and John~N. Tsitsiklis.
\newblock \emph{{N}euro-{D}ynamic {P}rogramming}.
\newblock Athena Scientific, 1996.

\bibitem[Biza et~al.(2020)Biza, Platt, van~de Meent, and
  Wong]{biza2020learning}
Ondrej Biza, Robert Platt, Jan-Willem van~de Meent, and Lawson~LS Wong.
\newblock {Learning Discrete State Abstractions with Deep Variational
  Inference}.
\newblock \emph{arXiv preprint arXiv:2003.04300}, 2020.

\bibitem[Castro(2020)]{castro2020scalable}
Pablo~Samuel Castro.
\newblock {Scalable Methods for Computing State Similarity in Deterministic
  Markov Decision Processes}.
\newblock In \emph{Proceedings of the {AAAI} Conference on Artificial
  Intelligence}, volume~34, pages 10069--10076, 2020.

\bibitem[Corneil et~al.(2018)Corneil, Gerstner, and Brea]{corneil18a}
Dane Corneil, Wulfram Gerstner, and Johanni Brea.
\newblock {Efficient Model-Based Deep Reinforcement Learning with Variational
  State Tabulation}.
\newblock \emph{arXiv preprint arXiv:1802.04325}, 2018.

\bibitem[Dabney et~al.(2020)Dabney, Barreto, Rowland, Dadashi, Quan, Bellemare,
  and Silver]{dabney2020valueimprovement}
Will Dabney, Andr\'{e} Barreto, Mark Rowland, Robert Dadashi, John Quan,
  Marc~G. Bellemare, and David Silver.
\newblock {The Value-Improvement Path: Towards Better Representations for
  Reinforcement Learning}, 2020.

\bibitem[Dadashi et~al.(2019)Dadashi, Taiga, Roux, Schuurmans, and
  Bellemare]{dadashi2019value}
Robert Dadashi, Adrien~Ali Taiga, Nicolas~Le Roux, Dale Schuurmans, and Marc~G.
  Bellemare.
\newblock {The Value Function Polytope in Reinforcement Learning}.
\newblock In \emph{Proceedings of the International Conference on Machine
  Learning ({ICML})}, volume~97, pages 1486--1495, 2019.

\bibitem[Dean and Givan(1997)]{dean1997model}
Thomas Dean and Robert Givan.
\newblock {Model Minimization in Markov Decision Processes}.
\newblock In \emph{AAAI/IAAI}, pages 106--111, 1997.

\bibitem[Farahmand(2018)]{farahmand2018iterative}
Amir{-}massoud Farahmand.
\newblock {Iterative Value-Aware Model Learning}.
\newblock In \emph{Advances in Neural Information Processing Systems
  ({NeurIPS})}, pages 9090--9101, 2018.

\bibitem[Farahmand et~al.(2013)Farahmand, Barreto, and
  Nikovski]{farahmand2013value}
Amir-Massoud Farahmand, Andr\'{e} Barreto, and Daniel Nikovski.
\newblock {Value-Aware Loss Function for Model Learning in Reinforcement
  Learning}.
\newblock In \emph{Proceedings of the European Workshop on Reinforcement
  Learning ({EWRL})}, 2013.

\bibitem[Farahmand et~al.(2017)Farahmand, Barreto, and
  Nikovski]{farahmand2017value}
Amir-Massoud Farahmand, Andr\'{e} Barreto, and Daniel Nikovski.
\newblock {Value-Aware Loss Function for Model-Based Reinforcement Learning}.
\newblock In \emph{Proceedings of the International Conference on Artificial
  Intelligence and Statistics ({AISTATS})}, volume~54, pages 1486--1494, 2017.

\bibitem[Farquhar et~al.(2018)Farquhar, Rockt{\"a}schel, Igl, and
  Whiteson]{farquhar2018treeqn}
G~Farquhar, T~Rockt{\"a}schel, M~Igl, and S~Whiteson.
\newblock {T}ree{Q}{N} and {A}{T}ree{C}: {D}ifferentiable {T}ree-{S}tructured
  {M}odels for {D}eep {R}einforcement {L}earning.
\newblock In \emph{6th International Conference on Learning Representations,
  ICLR 2018-Conference Track Proceedings}, volume~6. ICLR, 2018.

\bibitem[Ferns et~al.(2004)Ferns, Panangaden, and Precup]{ferns2004metrics}
Norm Ferns, Prakash Panangaden, and Doina Precup.
\newblock {Metrics for Finite Markov Decision Processes.}
\newblock In \emph{UAI}, volume~4, pages 162--169, 2004.

\bibitem[Fran{\c{c}}ois-Lavet et~al.(2019)Fran{\c{c}}ois-Lavet, Bengio, Precup,
  and Pineau]{franccois2019combined}
Vincent Fran{\c{c}}ois-Lavet, Yoshua Bengio, Doina Precup, and Joelle Pineau.
\newblock {Combined Reinforcement Learning via Abstract Representations}.
\newblock In \emph{Proceedings of the {AAAI} Conference on Artificial
  Intelligence}, volume~33, pages 3582--3589, 2019.

\bibitem[Gelada et~al.(2019)Gelada, Kumar, Buckman, Nachum, and
  Bellemare]{gelada19a}
Carles Gelada, Saurabh Kumar, Jacob Buckman, Ofir Nachum, and Marc~G Bellemare.
\newblock {DeepMDP: Learning Continuous Latent Space Models for Representation
  Learning}.
\newblock In \emph{International Conference on Machine Learning}, pages
  2170--2179, 2019.

\bibitem[Givan et~al.(2003)Givan, Dean, and Greig]{givan2003equivalence}
Robert Givan, Thomas Dean, and Matthew Greig.
\newblock {Equivalence Notions and Model Minimization in {M}arkov Decision
  Processes}.
\newblock \emph{Artificial Intelligence}, 147\penalty0 (1-2):\penalty0
  163--223, 2003.

\bibitem[Igl et~al.(2018)Igl, Zintgraf, Le, Wood, and Whiteson]{igl18a}
Maximillian Igl, Luisa Zintgraf, Tuan~Anh Le, Frank Wood, and Shimon Whiteson.
\newblock {Deep Variational Reinforcement Learning for POMDPs}.
\newblock In \emph{ICML 2018: Proceedings of the Thirty-Fifth International
  Conference on Machine Learning}, July 2018.
\newblock URL
  \url{http://www.cs.ox.ac.uk/people/shimon.whiteson/pubs/iglicml18.pdf}.

\bibitem[Joseph et~al.(2013)Joseph, Geramifard, Roberts, How, and
  Roy]{joseph2013reinforcement}
Joshua Joseph, Alborz Geramifard, John~W Roberts, Jonathan~P How, and Nicholas
  Roy.
\newblock {R}einforcement {L}earning with {M}isspecified {M}odel {C}lasses.
\newblock In \emph{2013 IEEE International Conference on Robotics and
  Automation}, pages 939--946. IEEE, 2013.

\bibitem[Li et~al.(2006)Li, Walsh, and Littman]{li2006towards}
Lihong Li, Thomas~J. Walsh, and Michael~L. Littman.
\newblock {Towards a Unified Theory of State Abstraction for {MDPs}}.
\newblock In \emph{Proceedings of the International Symposium on Artificial
  Intelligence and Mathematics}, pages 531--539, 2006.

\bibitem[Millar(2011)]{milar2003maximum}
Russell~B. Millar.
\newblock \emph{Maximum Likelihood Estimation and Inference}.
\newblock Hoboken: Wiley, 2011.

\bibitem[Mnih et~al.(2014)Mnih, Heess, Graves, and
  Kavukcuoglu]{mnih2014recurrent}
Volodymyr Mnih, Nicolas Heess, Alex Graves, and Koray Kavukcuoglu.
\newblock {Recurrent Models of Visual Attention}.
\newblock In \emph{Advances in Neural Information Processing Systems ({NIPS})},
  pages 2204--2212, 2014.

\bibitem[Oh et~al.(2017)Oh, Singh, and Lee]{oh2017value}
Junhyuk Oh, Satinder Singh, and Honglak Lee.
\newblock {V}alue {P}rediction {N}etworks.
\newblock In \emph{Advances in Neural Information Processing Systems}, pages
  6118--6128, 2017.

\bibitem[Parr et~al.(2008)Parr, Li, Taylor, Painter-Wakefield, and
  Littman]{parr2008analysis}
Ronald Parr, Lihong Li, Gavin Taylor, Christopher Painter-Wakefield, and
  Michael~L. Littman.
\newblock {An Analysis of Linear Models, Linear Value-Function Approximation,
  and Feature Selection for Reinforcement Learning}.
\newblock In \emph{Proceedings of the International Conference on Machine
  Learning ({ICML})}, pages 752--759, 2008.

\bibitem[Poupart and Boutilier(2002)]{poupart2002value}
Pascal Poupart and Craig Boutilier.
\newblock {Value-Directed Compression of POMDPs}.
\newblock In \emph{Advances in Neural Information Processing Systems ({NIPS})},
  pages 1547--1554. {MIT} Press, 2002.

\bibitem[Poupart and Boutilier(2013)]{poupart2013value}
Pascal Poupart and Craig Boutilier.
\newblock {Value-Directed Belief State Approximation for {POMDPs}}.
\newblock \emph{CoRR}, abs/1301.3887, 2013.
\newblock URL \url{http://arxiv.org/abs/1301.3887}.

\bibitem[Puterman(1994)]{puterman94markov}
Martin~L. Puterman.
\newblock \emph{{M}arkov {D}ecision {P}rocesses---{D}iscrete {S}tochastic
  {D}ynamic {P}rogramming}.
\newblock John Wiley \& Sons, Inc., 1994.

\bibitem[Ravindran and Barto(2004)]{ravindran2004approximate}
Balaraman Ravindran and Andrew~G Barto.
\newblock {Approximate Homomorphisms: A Framework for Non-Exact Minimization in
  Markov Decision Processes}.
\newblock 2004.

\bibitem[Russell and Norvig(2003)]{russel2003artificial}
Stuart~J. Russell and Peter Norvig.
\newblock \emph{Artificial Intelligence: A Modern Approach}.
\newblock Pearson Education, 3 edition, 2003.

\bibitem[Schrittwieser et~al.(2019)Schrittwieser, Antonoglou, Hubert, Simonyan,
  Sifre, Schmitt, Guez, Lockhart, Hassabis, Graepel,
  et~al.]{schrittwieser2019mastering}
Julian Schrittwieser, Ioannis Antonoglou, Thomas Hubert, Karen Simonyan,
  Laurent Sifre, Simon Schmitt, Arthur Guez, Edward Lockhart, Demis Hassabis,
  Thore Graepel, et~al.
\newblock {M}astering {A}tari, {G}o, {C}hess and {S}hogi by {P}lanning with a
  {L}earned {M}odel.
\newblock \emph{arXiv preprint arXiv:1911.08265}, 2019.

\bibitem[Silver et~al.(2017)Silver, van Hasselt, Hessel, Schaul, Guez, Harley,
  Dulac-Arnold, Reichert, Rabinowitz, Barreto, et~al.]{silver2017predictron}
David Silver, Hado van Hasselt, Matteo Hessel, Tom Schaul, Arthur Guez, Tim
  Harley, Gabriel Dulac-Arnold, David Reichert, Neil Rabinowitz, Andre Barreto,
  et~al.
\newblock {T}he {P}redictron: {E}nd-to-{E}nd {L}earning and {P}lanning.
\newblock In \emph{Proceedings of the 34th International Conference on Machine
  Learning-Volume 70}, pages 3191--3199. JMLR. org, 2017.

\bibitem[Sutton(1988)]{sutton88learning}
Richard~S. Sutton.
\newblock Learning to {P}redict by the {M}ethods of {T}emporal {D}ifferences.
\newblock \emph{Machine Learning}, 3:\penalty0 9--44, 1988.

\bibitem[Sutton and Barto(2018)]{sutton2018reinforcement}
Richard~S. Sutton and Andrew~G. Barto.
\newblock \emph{{R}einforcement {L}earning: {A}n {I}ntroduction}.
\newblock MIT Press, 2018.
\newblock URL
  \url{https://mitpress.mit.edu/books/reinforcement-learning-second-edition}.
\newblock 2nd edition.

\bibitem[Sutton et~al.(1999)Sutton, Precup, and Singh]{sutton1999between}
Richard~S Sutton, Doina Precup, and Satinder Singh.
\newblock {B}etween {MDP}s and {S}emi-{MDP}s: {A} {F}ramework for {T}emporal
  {A}bstraction in {R}einforcement {L}earning.
\newblock \emph{Artificial intelligence}, 112\penalty0 (1-2):\penalty0
  181--211, 1999.

\bibitem[Sutton et~al.(2008)Sutton, Szepesv\'{a}ri, Geramifard, and
  Bowling]{sutton2008dyna}
Richard~S. Sutton, Csaba Szepesv\'{a}ri, Alborz Geramifard, and Michael
  Bowling.
\newblock {Dyna-Style Planning with Linear Function Approximation and
  Prioritized Sweeping}.
\newblock In \emph{Proceedings of the Conference on Uncertainty in Artificial
  Intelligence ({UAI})}, page 528–536, 2008.

\bibitem[Szepesv{\'a}ri(2010)]{szepesvari2010algorithms}
Csaba Szepesv{\'a}ri.
\newblock \emph{{A}lgorithms for {R}einforcement Learning}.
\newblock Synthesis Lectures on Artificial Intelligence and Machine Learning.
  Morgan {\&} Claypool Publishers, 2010.

\bibitem[Tamar et~al.(2016)Tamar, Wu, Thomas, Levine, and
  Abbeel]{tamar2016value}
Aviv Tamar, Yi~Wu, Garrett Thomas, Sergey Levine, and Pieter Abbeel.
\newblock {Value Iteration Networks}.
\newblock In \emph{Advances in Neural Information Processing Systems}, pages
  2154--2162, 2016.

\bibitem[Taylor et~al.(2009)Taylor, Precup, and Panagaden]{taylor2009bounding}
Jonathan Taylor, Doina Precup, and Prakash Panagaden.
\newblock {Bounding Performance Loss in Approximate {MDP} Homomorphisms}.
\newblock In \emph{Advances in Neural Information Processing Systems ({NIPS})},
  pages 1649--1656, 2009.

\bibitem[van~der Pol et~al.(2020)van~der Pol, Kipf, Oliehoek, and
  Welling]{van2020plannable}
Elise van~der Pol, Thomas Kipf, Frans~A Oliehoek, and Max Welling.
\newblock {Plannable Approximations to MDP Homomorphisms: Equivariance under
  Actions}.
\newblock In \emph{{International Conference on Autonomous Agents and
  MultiAgent Systems}}, 2020.

\bibitem[Watter et~al.(2015)Watter, Springenberg, Boedecker, and
  Riedmiller]{watter2015embed}
Manuel Watter, Jost Springenberg, Joschka Boedecker, and Martin Riedmiller.
\newblock {Embed to Control: A Locally Linear Latent Dynamics Model for Control
  from Raw Images}.
\newblock In \emph{Advances in Neural Information Processing Systems ({NIPS})},
  pages 2746--2754, 2015.

\bibitem[Zhang et~al.(2019{\natexlab{a}})Zhang, Lipton, Pineda,
  Azizzadenesheli, Anandkumar, Itti, Pineau, and Furlanello]{zhang2019learning}
Amy Zhang, Zachary~C Lipton, Luis Pineda, Kamyar Azizzadenesheli, Anima
  Anandkumar, Laurent Itti, Joelle Pineau, and Tommaso Furlanello.
\newblock {Learning Causal State Representations of Partially Observable
  Environments}.
\newblock \emph{arXiv preprint arXiv:1906.10437}, 2019{\natexlab{a}}.

\bibitem[Zhang et~al.(2019{\natexlab{b}})Zhang, Vikram, Smith, Abbeel, Johnson,
  and Levine]{zhang2019solar}
Marvin Zhang, Sharad Vikram, Laura Smith, Pieter Abbeel, Matthew Johnson, and
  Sergey Levine.
\newblock {SOLAR: Deep Structured Representations for Model-Based Reinforcement
  Learning}.
\newblock In \emph{International Conference on Machine Learning ({ICML})},
  pages 7444--7453. PMLR, 2019{\natexlab{b}}.

\end{thebibliography}
